\documentclass[10pt]{article}

\usepackage[margin=0.8in]{geometry}
\usepackage{amsmath,amssymb,amsthm,mathtools,bm}
\usepackage{mathrsfs}
\usepackage{booktabs}
\usepackage{graphicx}
\usepackage{microtype}
\usepackage{enumitem}
\usepackage{natbib}
\usepackage{xcolor}
\usepackage{subcaption}
\usepackage[colorlinks=true,allcolors=blue!55!black]{hyperref}
\usepackage[nameinlink,capitalize,noabbrev]{cleveref}

\newtheorem{theorem}{Theorem}[section]
\newtheorem{proposition}[theorem]{Proposition}
\newtheorem{lemma}[theorem]{Lemma}
\newtheorem{corollary}[theorem]{Corollary}
\theoremstyle{definition}
\newtheorem{definition}[theorem]{Definition}
\newtheorem{remark}[theorem]{Remark}
\newtheorem{assumption}[theorem]{Assumption}

\newcommand{\widesim}[2][1.5]{
  \mathrel{\overset{#2}{\scalebox{#1}[1]{$\sim$}}}
}
\newcommand{\iid}{\widesim[2.5]{i.i.d.}}
\DeclareMathOperator*{\argmin}{\arg\!\min}

\DeclareMathOperator{\Tr}{Tr}

\title{An Active-Bottleneck Mechanism for Weak-to-Strong Generalization}

\author{Mohammad Zeinalpour \and Amir Najafi\thanks{Corresponding author, Email: \texttt{amir.najafi@sharif.edu}}}
\date{
\vspace*{0mm}
\parbox{\linewidth}{
\centering
\small
  Department of Computer Engineering,\\
  Sharif University of Technology
  \endgraf\medskip
  }
}

\begin{document}

\maketitle

\begin{abstract}
Weak-to-strong generalization (W2SG) occurs when a student trained on a teacher's predictions outperforms that teacher. We study when this happens under fully converged, ridgeless two-stage learning, with no early stopping, no explicit regularization, and no assumption that the student is more expressive than the teacher. In two-stage linear regression, a teacher is fit from $n$ labeled examples and a student is trained solely on the teacher's predictions on $m$ fresh, unlabeled inputs. Although both stages share the same hypothesis class and the same training rule, we show that the student outperforms the teacher exactly when $m$ lies in an explicit intermediate range: too few pseudo-labels leave the student without enough signal, too many let it inherit the teacher's noise. Under power-law covariance, we derive this range in closed form as a function of the spectral decay and noise level, including regimes where the improving region splits into two disjoint intervals of $m$. We then study a random-feature model in which the student has strictly more features than the teacher, and identify two regimes, again given by explicit thresholds: one where improvement occurs only for $m$ in a bounded interval, and one where it occurs only once the student width $N_S$ exceeds an explicit threshold. Both regimes are governed by a single \emph{active-bottleneck} principle: whichever of $m$ or $N_S$ is scarcer controls how much teacher error is filtered out, while increasing the other resource only reduces estimation noise. Together, these results show that finite data and finite width can themselves regularize a two-stage learner, with no explicit mechanism doing so.
\end{abstract}

\tableofcontents

%%%%%%%%%%%%%%%%%%%%%%%%%%%%%%%%%%%%%%%%%%%%%%%%%%%%%%%%%%%%%%%%%%%%%%%%%%%%%%%%%%%%%%%%%%%%%%%%%%%%%%%%%%%%%%%%%%%%%%%%%%%%%%%%%%%%%%%%%%%%%%%%%%%%%%%%%%%%%%%%%%%%%%%%%%%%%%%%%%%%%%%%%%%%%%%%%%%%%%%%%%%%%%%%%%%%%%%%%%%%%%%%%%%%%%%%%%%%%%%%%%%%%%%%%%%%%%%%%%%%%%%%%%%%%%%%%%%%%%%%%%%%%%%%%%%%%%%%%%%%%%%%%%%%%%%%%%%%%%%%%%%%%%%%%%%%%%%%%%%%%%%%%%%%%%%%%%%%%%%%%%%%%%%%%%%%%%%%%%%%%%%%%%%%%%%%%%%%%%%%%%%%%%%%%%%%%%%%%%%%%%%%%%%%%%%%%%%%%%%%%%%%%%%%%%%%%%%%%%%%%%%%%%%%%%%%%%%%%%%%%%%%%%%%%%%%%%%%%%%%%%%%%%%%%%%%%%%%%%%%%%%%%%%%%%%%%%%%%%%%%%%%%%%%%%%%%%%%%%%%%%%%%%%%%%%%%%%%%%%%%%%%%%%%%%%%%%%%%%%%%%%%%%%%%%%%%%%%%%%%%%%%%%%%%%%%%%%%%%%%%%%%%%%%%%%%%%%%%%%%%%%%%%%%%%%%%%%%%%%%%%%%%%%%%%%%%

\section{Introduction}
\label{sec:intro}

As AI systems approach and eventually exceed human ability on complex
tasks, humans will need to supervise models more capable than themselves.
\cite{burns2023weak} propose a way to study this problem today: fine-tune a
strong pretrained model using only the predictions of a weaker,
already fine-tuned one, and ask whether the strong model can outperform
its supervisor. They call this \emph{weak-to-strong generalization} (W2SG),
and show that it occurs empirically --- a GPT-4-level model supervised only
by a GPT-2-level model can recover a substantial fraction of the
performance gap between them. The finding is counterintuitive: the strong
model is trained explicitly to fit the weak model's labels, yet it does
not merely imitate them.

Existing theoretical explanations for W2SG typically rely on some form of
\emph{asymmetry} between the two stages. Either the student is assumed to
have strictly greater representational capacity than the teacher
\citep{xue2025representations,
dong2025discrepancies}; or the student is
trained with heavier or differently structured regularization than the
teacher \citep{moniri2026mechanisms}; or the student is stopped early,
before it has had the chance to fully absorb the teacher's errors
\citep{medvedev2025weak,
geng2026weak}. Each of these
mechanisms gives the student some structural advantage that the teacher
does not share. This raises a natural question: is such an asymmetry
\emph{necessary}? Can a student with exactly the teacher's capacity, trained
by exactly the teacher's rule, run to exact convergence with no ridge
penalty at all, still outperform a teacher it was trained only to imitate?

We show that the answer is yes, and that a single, simple resource is
enough to cause it: the number of pseudo-labeled examples, $m$, that the
student sees in its own training stage. In our first setting, both the
teacher and the student solve ridgeless linear regression over the
\emph{same} hypothesis class, and the student is trained to convergence on
$m$ fresh inputs labeled by the teacher. Because $m$ is finite, the
student's minimum-norm solution is constrained to an $m$-dimensional
subspace of parameter space; it does not have room to reproduce the
teacher exactly. This subspace can retain more of the teacher's useful
signal than its noise, or the reverse, depending on $m$. Under power-law
covariance, we make this trade-off completely explicit: we solve for the
critical values of $m$ --- as a fixed threshold, or as a limiting ratio
$m/p$ when $m$ scales with the ambient dimension $p$ --- that separate the
regime where the student is worse than the teacher from the regime where
it is better, and we show that this improving region can itself be a
single interval, all of $(0,p)$, or two disjoint intervals of $m$,
depending on the noise level and the covariance decay rate.

This single-resource result suggests something more general is at work: a
shared resource between the two training stages must be constrained to an
intermediate range --- tight enough to keep out the teacher's error, loose
enough to let the target signal through --- for W2SG to occur. We call such
a resource, whichever it turns out to be, an \emph{active bottleneck}, and
in our second setting we test this idea against a genuine capacity
asymmetry. We give the student strictly more random features than the
teacher, so that the student is a truly stronger model in the usual sense,
and ask whether pseudo-sample size still matters once this asymmetry is
present. It does, but its role now shares the stage with a second resource,
student width. We identify two regimes, again with explicit thresholds:
in one, improvement holds only when $m$ falls inside a bounded interval,
which we solve for exactly; in the other, improvement requires the student
width $N_S$ to exceed an explicit threshold, and now \emph{more} pseudo-labels
are beneficial rather than harmful. The two regimes are governed by a single
principle: whichever resource, $m$ or $N_S$, is scarcer at a given point
in parameter space is the one that determines how much of the teacher's
error is filtered out; the other resource, once it is not the bottleneck,
only reduces estimation noise and can be increased freely. This is the
sense in which our two settings tell one story: Case~I is the special case
of this principle in which only one resource, $m$, is available to bind,
and student and teacher share the same capacity; Case~II makes the
principle explicit and shows it survives, in modified form, once a genuine
capacity asymmetry is introduced.

Our emphasis on an intermediate, rather than merely large, pseudo-sample
size is not only a theoretical curiosity: large-scale empirical studies of
distillation scaling laws have observed the same non-monotonicity directly
in trained language models, where increasing the number of teacher-labeled
examples eventually collapses W2SG even between models with a large
capacity gap \citep{busbridge2025distillation}. We discuss the relationship
between this empirical finding and our theoretical account in
Section~\ref{sec:related-work}.

\paragraph{Contributions.}
\begin{itemize}
    \item \textbf{W2SG at ridgeless convergence with no capacity asymmetry.}
    In two-stage linear regression with identical teacher and student
    hypothesis classes, we show that a student trained to convergence on the
    teacher's own predictions can strictly outperform that teacher, with no
    early stopping and no explicit regularization at either stage.

    \item \textbf{An explicit sample-size phase diagram.} Under power-law
    covariance, we solve for the critical pseudo-sample size(s) separating
    improvement from non-improvement, as fixed thresholds and as limiting
    ratios $m/p$, and characterize when the improving region is a single
    interval, extends to the full overparameterized range, or splits into
    two disjoint intervals of $m$ (Theorem~\ref{thm:case-I-phase-diagram}).

    \item \textbf{A random-feature model with a genuine student advantage,
    and an active-bottleneck principle.} We extend the analysis to random
    features with $N_S>N_T$, and give explicit critical thresholds in $m$
    and in $N_S$ separating the sample-bottleneck and feature-bottleneck
    regimes (Theorem~\ref{thm:caseII}, Corollary~\ref{cor:caseII}). Both
    regimes instantiate the same active-bottleneck principle: the scarcer
    of the two Stage-II resources controls teacher-error filtering, while
    the other resource only reduces estimation noise.
\end{itemize}

We discuss the relationship of these results to the broader W2SG literature
--- in particular to alternative mechanisms based on early stopping,
asymmetric regularization, representation geometry, and overparameterization,
and to empirical scaling-law evidence at LLM scale --- in
Section~\ref{sec:related-work}.

%%%%%%%%%%%%%%%%%%%%%%%%%%%%%%%%%%%%%%%%%%%%%%%%%%%%%%%%%%%%%%%%%%%%%%%%%%%%%%%%%%%%%%%%%%%%%%%%%%%%%%%%%%%%%%%%%%%%%%%%%%%%%%%%%%%%%%%%%%%%%%%%%%%%%%%%%%%%%%%%%%%%%%%%%%%%%%%%%%%%%%%%%%%%%%%%%%%%%%%%%%%%%%%%%%%%%%%%%%%%%%%%%%%%%%%%%%%%%%%%%%%%%%%%%%%%%%%%%%%%%%%%%%%%%%%%%%%%%%%%%%%%%%%%%%%%%%%%%%%%%%%%%%%%%%%%%%%%%%%%%%%%%%%%%%%%%%%%%%%%%%%%%%%%%%%%%%%%%%%%%%%%%%%%%%%%%%%%%%%%%%%%%%%%%%%%%%%%%%%%%%%%%%%%%%%%%%%%%%%%%%%%%%%%%%%%%%%%%%%%%%%%%%%%%%%%%%%%%%%%%%%%%%%%%%%%%%%%%%%%%%%%%%%%%%%%%%%%%%%%%%%%%%%%%%%%%%%%%%%%%%%%%%%%%%%%%%%%%%%%%%%%%%%%%%%%%%%%%%%%%%%%%%%%%%%%%%%%%%%%%%%%%%%%%%%%%%%%%%%%%%%%%%%%%%%%%%%%%%%%%%%%%%%%%%%%%%%%%%%%%%%%%%%%%%%%%%%%%%%%%%%%%%%%%%%%%%%%%%%%%%%%%%%%%%%%%%%%%%%%%%%%%%%%%%%%%%%%%%%%%%%%%%%%%%%%%%%%%%%%%%%%%%%%%%%%%%%%%%%%%%%%%%%%%%%%%%%%%%%%%%%%%%%%%%%%%%%%%%%%%%%%%%%%%%%%%%%%%%%%%%%%%%%%%%%%%%%%%%%%%%%%%%%%%%%%%%%%%%%%%%%%%%%%%%%%%%%%%%%%%%%%%%%%%%%%%%%%%%%%%%%%%%%%%%%%%%%%%%%%%%%%%%%%%%%%%%%%%%%%%%%%%%%%%%%%%%%%%%%%%%%%%%%%%%%%%%%%%%%%%%%%%%%%%%%%%%%%%%%%%%%%%%%%%%%%%%%%%%%%

\section{Related Work}
\label{sec:related-work}

A growing theoretical literature explains W2SG through different lenses.
We organize it along three axes bearing directly on our contribution:
\emph{what mechanism} drives improvement, \emph{whether the student has a
capacity advantage} over the teacher, and \emph{which resource}, if any, is
identified as the source of the effect.

\paragraph{Mechanisms proposed for W2SG.}
One line of work gives structural conditions guaranteeing improvement.
\cite{charikar2024quantifying} relate improvement to the student's
\emph{misfit} to the teacher's labels under realizability and convexity;
\cite{mulgund2025relating} extend this to Bregman divergences, and
\cite{xu2025emergence} give a bias--variance decomposition connecting
improvement to how well the student approximates a posterior-mean teacher.
\cite{lang2024theoretical} use expansion-type conditions under which the
student corrects and extends the teacher's pseudo-label coverage, and
\cite{shin2025weak} study how easy/hard data overlap supports this.
\cite{somerstep2025transfer} cast W2SG as latent-concept transfer. These
frameworks isolate \emph{properties of a fixed predictor pair} --- misfit,
coverage, alignment --- sufficient for improvement. We instead fix a
standard training procedure at both stages and ask how the
\emph{resources} supplied to it determine improvement, without assuming
any property of the resulting predictors.

A second line explains W2SG through training dynamics.
\cite{medvedev2025weak} show that early-stopped population gradient flow in
random-feature regression yields improvement, because finite training time
suppresses teacher error before the student fully absorbs it;
\cite{geng2026weak} show a related effect for finite-time logistic
training. Both rely on stopping \emph{before} convergence. We study the
complementary regime: our student always converges, and the constraint
that produces improvement is on the data or features it receives, not on
training time.

\paragraph{Symmetric versus asymmetric capacity.}
Several analyses assume the student has an advantage the teacher lacks.
\cite{moniri2026mechanisms} show a student can compensate for a teacher
that is under-regularized relative to it, so asymmetric regularization is
itself the source of improvement. \cite{xue2025representations}
characterize teacher-error propagation through principal-kernel geometry,
allowing the student's representation to differ favorably from the
teacher's. \cite{dong2025discrepancies} derive an inherited-variance
decomposition governed by teacher--student representation overlap, again
asymmetric. \cite{oh2026linear} let the student be a nonlinear network
trained beyond a linear convolutional teacher. In each case, some
asymmetry --- regularization, representation quality, or nonlinearity ---
does part of the work. Our Case~I removes this axis entirely: teacher and
student share one hypothesis class, one ridgeless rule, and one
convergence guarantee. Case~II reintroduces capacity asymmetry
deliberately, via $N_S>N_T$, to test whether a resource-based mechanism
survives alongside it.

\paragraph{Resource-based explanations and scaling.}
Closest to our framing are analyses treating sample size or feature count
as explicit, finite parameters. \cite{ildiz2025high} characterize
surrogate predictors that improve a downstream ridgeless learner relative
to training directly on the target; their benchmark is improvement over
direct target training, whereas ours is improvement over the teacher
Stage~I actually produces. \cite{wu2026improved} derive deterministic
equivalents for random-feature \emph{ridge} regression and find improved
scaling exponents via a joint choice of sample size, width, and ridge;
larger sample sizes and widths are generally beneficial once the ridge is
tuned. Our Case~II student is ridgeless throughout, and on the
sample-bottleneck branch W2SG holds only inside a bounded interval of
$m$ --- more pseudo-labels eventually destroy it. \cite{wu2025provable}
study a structurally similar two-stage ridgeless pipeline, but their
source of improvement is a built-in asymmetry: the student's covariates
are strictly stronger than the teacher's throughout. Our Case~I needs no
such asymmetry at all, and our Case~II treats a representational
advantage as a separate, explicit parameter analyzed jointly with $m$.
\cite{yao2026blessing} show a shared pretrained initialization can induce
W2SG in a single-index model --- a different, complementary source of
advantage rather than a competing explanation. None of these works
identifies a single scarce resource governing a sharp transition between
improvement and its absence.

\paragraph{Overparameterization: helpful, harmful, or beside the point?}
A further group discusses overparameterization directly without
explaining why it helps or hurts. \cite{wu2026improved} show that, with
ridge regularization, the interplay between the ridge and the degree of
overparameterization can improve W2SG scaling laws. \cite{dong2025discrepancies}
find, in kernel regression with an assumed teacher--student discrepancy,
that W2SG occurs both overparameterized-ridgeless and
underparameterized-with-ridge. Both leave open \emph{why}
overparameterization itself helps. \cite{moniri2026mechanisms} argue the
opposite: without regularization, an overparameterized student loses
Stage-I information that a regularized student would retain, which they
treat as harmful.

Our Case~I agrees that information is lost --- this is exactly
Proposition~\ref{prop:case-I-sample-bottleneck} --- but shows the loss is
not intrinsically harmful. Whether it helps depends on which directions of
parameter space it comes from, which depends on the covariance structure.
Under power-law covariance, Theorem~\ref{thm:case-I-phase-diagram} shows
this same loss produces $\mathsf{W2SG}>0$ throughout the overparameterized
range for $\alpha\le2$, and on an explicit sub-range for $\alpha>2$.
Overparameterization is neither uniformly good nor bad here: it is the
mechanism, and its effect depends on where the discarded information lives
spectrally.

\paragraph{Evidence from LLM-scale distillation.}
\cite{busbridge2025distillation} report an extensive empirical study of
distillation scaling laws and find that the number of teacher-labeled
examples $m$ has a non-monotone effect: beyond a point, more pseudo-labels
cause overfitting to the teacher and drive W2SG toward zero, even with a
large capacity gap remaining. This matches the upper half of our own phase
diagrams and is, to our knowledge, the first evidence that this shape
holds at LLM scale rather than only in linear models. We do not see this
as being in tension with our interval characterization: they report the
shape empirically, with no mechanism or explicit threshold, which is
exactly what Theorems~\ref{thm:case-I-phase-diagram}
and~\ref{thm:caseII} supply; and an interval with a finite upper endpoint
already implies "large $m$ is eventually bad," so the two claims agree
rather than compete.

\paragraph{What is new here.}
No prior analysis combines all three features of our setting: (i) teacher
and student trained by the \emph{same procedure} to \emph{exact
convergence}, with no early stopping or ridge; (ii) in Case~I, no capacity
or regularization asymmetry; and (iii) an explicit, closed-form
characterization of the resource level --- $m$ in Case~I, jointly $m$ and
$N_S$ in Case~II --- separating improvement from its absence, rather than
a sufficient condition or scaling exponent. A resource constraint shared
between the two stages can itself regularize the student, and this
constraint must be \emph{active} --- neither too loose nor too tight ---
for improvement to occur; the non-monotonicity of
\citet{busbridge2025distillation} suggests this picture is not an
artifact of the linear models we analyze.

\section{Preliminaries}
\label{sec:preliminaries}

\paragraph{Notation.}
For a positive integer $d$, let $[d]:=\{1,\ldots,d\}$. Vectors are written
using lowercase letters, such as $x$, $\beta$, and $w$, whereas matrices are
written using uppercase letters, such as $X$, $U$, and $\Sigma$. All vectors
are column vectors. We use $A^\dagger$, $\operatorname{range}(A)$, and
$\operatorname{row}(A)$ to denote the Moore--Penrose pseudoinverse, column
space, and row space of $A$, respectively. The notation $I_d$ denotes the
$d\times d$ identity matrix, and $\operatorname{diag}(\cdot)$ denotes a
diagonal or block-diagonal matrix. We use $\operatorname{Tr}(\cdot)$ for the
non-normalized trace of a matrix and $\operatorname{tr}(\cdot)$ for the
normalized trace. For a symmetric matrix $A$, we write
$\lambda_1(A)\ge\lambda_2(A)\ge\cdots$ for its eigenvalues in nonincreasing
order. We use $\|\cdot\|_{\mathrm{op}}$ for the operator norm and
$\|\cdot\|_\Sigma$ for the covariance-weighted norm
$\|v\|_\Sigma^2:=v^\top\Sigma v$. Subscripts $T$ and $S$ refer to teacher and
student quantities, respectively. Unless stated otherwise, asymptotic
relations are taken as $p\to\infty$. For positive sequences $a_p$ and $b_p$,
we write $a_p\sim b_p$ if $a_p/b_p\to1$, $a_p=O(b_p)$ if $a_p/b_p$ remains
bounded, $a_p=o(b_p)$ if $a_p/b_p\to0$, and $a_p=\Theta(b_p)$ if both
$a_p=O(b_p)$ and $b_p=O(a_p)$.

\paragraph{Two-stage learning.}
We study a two-stage prediction pipeline that is common to both settings
analyzed in this paper; each setting specializes it in a different way, and
we defer every specialization --- the choice of hypothesis class, learning
rule, and noise model --- to Sections~\ref{sec:case-I-ridgeless}
and~\ref{sec:case-II-random-features}. A target coefficient
$\beta^\star\in\mathbb R^p$ is fixed throughout a given problem instance; it
may itself be drawn from some distribution independently of everything else
described below, in which case all expectations we define are understood to
also average over $\beta^\star$. An input $x\in\mathbb R^p$ is drawn from a
distribution $P_x$, and observations $(x,y)$ satisfy
\begin{equation}
    \psi^\star(x):=x^\top\beta^\star,
    \qquad
    y=\psi^\star(x)+\varepsilon,
    \label{eq:general-data-generation}
\end{equation}
for some label noise $\varepsilon$, independent of $x$, whose distribution
is specified separately in each case we study.

\emph{Stage I} produces a teacher predictor $\widehat h_T\in\mathcal H_T$
from $n$ labeled observations $\mathcal D_n:=\{(x_i,y_i)\}_{i=1}^n$ drawn
from \eqref{eq:general-data-generation}, via some learning rule
$\mathscr A_T$, so that $\widehat h_T=\mathscr A_T(\mathcal D_n)$. The
hypothesis class $\mathcal H_T$ may itself depend on auxiliary randomness
independent of $\mathcal D_n$ --- for instance, a random feature embedding
--- in which case $\widehat h_T$ is understood to be conditional on that
randomness as well.

\emph{Stage II} produces a student predictor $\widehat h_S\in\mathcal H_S$
using \emph{no labeled data of its own}. Instead, $m$ fresh, unlabeled
inputs $\widetilde{\mathcal X}_m:=\{\widetilde x_i\}_{i=1}^m$ are drawn
i.i.d.\ from $P_x$, independently of $\mathcal D_n$ and of any randomness
underlying $\mathcal H_T$, and are labeled by the teacher itself:
$\widetilde y_i:=\widehat h_T(\widetilde x_i)$ for $i=1,\ldots,m$. Writing
$\widetilde{\mathcal D}_m:=\{(\widetilde x_i,\widetilde y_i)\}_{i=1}^m$ for
the resulting pseudo-labeled dataset, the student is
$\widehat h_S:=\mathscr A_S(\widetilde{\mathcal D}_m)\in\mathcal H_S$ for
some learning rule $\mathscr A_S$; as with $\mathcal H_T$, the class
$\mathcal H_S$ may depend on auxiliary randomness independent of
$\widetilde{\mathcal D}_m$. Throughout, the student sees only the teacher's
predictions, never $\beta^\star$, $\varepsilon$, or $\mathcal D_n$ directly.

\paragraph{Population risk and the W2SG gain.}
For a predictor $h$, its population risk against the target coefficient is
\begin{equation}
    \mathcal R(h;\beta^\star)
    :=
    \mathbb E_{x\sim P_x}
    \left[
        \bigl(h(x)-x^\top\beta^\star\bigr)^2
    \right],
    \label{eq:general-population-risk}
\end{equation}
where the test point $x$ is independent of all training data and of any
auxiliary randomness underlying $h$. We write the corresponding
\emph{averaged} teacher and student risks as
\begin{equation}
    \overline{\mathcal R}_T
    :=
    \mathbb E\!\left[
        \mathcal R(\widehat h_T;\beta^\star)
    \right],
    \qquad
    \overline{\mathcal R}_S
    :=
    \mathbb E\!\left[
        \mathcal R(\widehat h_S;\beta^\star)
    \right],
    \label{eq:expected-risks}
\end{equation}
where each expectation is taken over every source of randomness entering
the corresponding predictor: $\beta^\star$ (if random), the Stage-I data
$\mathcal D_n$ and label noise, any randomness underlying $\mathcal H_T$,
and, for $\overline{\mathcal R}_S$, additionally the Stage-II inputs
$\widetilde{\mathcal X}_m$ and any randomness underlying $\mathcal H_S$.
The \emph{weak-to-strong generalization gain} is
\begin{equation}
    \mathsf{W2SG}
    :=
    \overline{\mathcal R}_T-\overline{\mathcal R}_S,
    \label{eq:w2s-gain}
\end{equation}
and we say weak-to-strong generalization occurs when $\mathsf{W2SG}>0$: the
student, trained only on the teacher's own predictions and with no access
to labeled data, achieves strictly lower expected risk than the teacher it
was trained to imitate. Sections~\ref{sec:case-I-ridgeless} and~\ref{sec:case-II-random-features}
instantiate this template in two ways that differ in every specialized
choice --- hypothesis class, learning rule, and noise model --- but share
the risk and gain defined here.

\section{Case I: Ridgeless Regression}
\label{sec:case-I-ridgeless}

We first instantiate the two-stage template of Section~\ref{sec:preliminaries} in the setting where the teacher and student share the same linear hypothesis class, the same feature space, and the same ridgeless training rule. Ridgeless linear regression is not a toy warm-up here: even in this classical setting, minimum-norm interpolation is known to exhibit genuinely surprising behavior --- double descent, benign overfitting, and sharp phase transitions in the proportional asymptotic regime --- and it has been the object of substantial recent theoretical attention in its own right \citep{hastie2022surprises,cheng2024dimension,han2026distribution,wu2020optimal,richards2021asymptotics,loureiro2021learning}. We use it here because it lets us isolate a single mechanism for W2SG in its purest form: with teacher and student identical in every respect except the data each one sees, any improvement of the student over the teacher must come entirely from the \emph{pseudo-sample size} $m$.

Prior analyses of W2SG in ridgeless regression have focused on representation mismatch, scaling behavior, and optimization dynamics between teacher and student \citep{dong2025discrepancies,wu2025provable,
wu2026improved}. We instead remove every such source of asymmetry: teacher and student share one feature space and one hypothesis class, and both are trained to exact convergence with zero ridge. We show that W2SG nonetheless occurs, governed entirely by $m$.

%%%%%%%%%%%%%%%%%%%%%%%%%%%%%%%%%%%%%%%%%%%%%%%%%%%%%%%%%%%%%%%%%%%%%%%%%%%%%%%%%%%%%%%%%%%%%%%%%%%%%%%%%%%%%%%%%%%%%%%%%%%%%%%%%%%%%%%%%%%%%%%%%%%%%%%%%%%%%%%%%%%%%%%%%%%%%%%%%%%%%%%%%%%%%%%%%%%%%%%%%%%%%%%%%%%%%%%%%%%%%%%%%%%%%%%%%%%%%%%%%%%%%%%%%%%%%%%%%%%%%%%%%%%%%%%%%%%%%%%%%%%%%%%%%%%%%%%%%%%%%%%%%%%%%%%%%%%%%%%%%%%%%%%%%%%%%%%%%%%%%%%%%%%%%%%%%%%%%%%%%%%%%%%%%%%%%%%%%%%%%%%%%%%%%%%%%%%%%%%%%%%%%%%%%%%%%%%%%%%%%%%%%%%%%%%%%%%%%%%%%%%%%%%%%%%%%%%%%%%%%%%%%%%%%%%%%%%%%%%%%%%%%%%%%%%%%%%%%%%%%%%%%%%%%%%%%%%%%%%%%%%%%%%%%%%%%%%%%%%%%%%%%%%%%%%%%%%%%%%%%%%%%%%%%%%%%%%%%%%%%%%%%%%%%%%%%%%%%%%%%%%%%%%%%%%%%%%%

\subsection{Model and Training Procedure}
\label{subsec:model-setting_1}

We specialize the data-generating process of \eqref{eq:general-data-generation} to Gaussian covariates, $x\sim\mathcal N(0,\Sigma)$ with $\Sigma\in\mathbb R^{p\times p}$ positive definite, and to Gaussian label noise $\varepsilon\sim\mathcal N(0,\sigma^2)$ in Stage~I, independent of $x$. The target coefficient $\beta^\star$ is random, independent of all training and test randomness, with
\begin{equation}
    \mathbb E[\beta^\star]=0,
    \qquad
    \mathbb E[\beta^\star\beta^{\star\top}]=I_p.
\end{equation}
We do not need to assume a specific distribution for $\beta^\star$ beyond these two moments: every quantity we study --- the teacher and student risks, and the W2SG gain --- is a linear or quadratic functional of $\beta^\star$ through $\mathbb E[\beta^\star\beta^{\star\top}]$ alone, so any isotropic distribution with these moments (Gaussian, uniform on a sphere of the matching radius, or otherwise) yields identical results.

Both stages share the linear hypothesis class $\mathcal H:=\{h_\beta:x\mapsto x^\top\beta,\ \beta\in\mathbb R^p\}$. Given $X=[x_1|\cdots|x_n]^\top\in\mathbb R^{n\times p}$ and $y=(y_1,\ldots,y_n)^\top$, the teacher is the minimum-norm interpolator
\begin{equation}
    \widehat\beta_T=X^\dagger y,
    \qquad
    \widehat h_T(x)=x^\top\widehat\beta_T.
    \label{eq:case-I-teacher}
\end{equation}
Independently of Stage~I, draw $\widetilde x_i\iid\mathcal N(0,\Sigma)$ for $i=1,\ldots,m$, form the noiseless pseudo-labels $\widetilde y_i=\widetilde x_i^\top\widehat\beta_T$, and write $\widetilde X=[\widetilde x_1|\cdots|\widetilde x_m]^\top$, $\widetilde y=(\widetilde y_1,\ldots,\widetilde y_m)^\top$. The student is again the minimum-norm interpolator of its own (pseudo-labeled) data,
\begin{equation}
    \widehat\beta_S
    =\widetilde X^\dagger\widetilde y
    =\widetilde X^\dagger\widetilde X\widehat\beta_T,
    \qquad
    \widehat h_S(x)=x^\top\widehat\beta_S.
    \label{eq:case-I-student}
\end{equation}

For Gaussian covariates, the population risk \eqref{eq:general-population-risk} of any $h_\beta\in\mathcal H$ takes the standard closed form
\begin{equation}
    \mathcal R(h_\beta;\beta^\star)
    =\|\Sigma^{1/2}(\beta-\beta^\star)\|_2^2,
    \label{eq:case-I-risk}
\end{equation}
used throughout the ridgeless-regression literature cited above. The averaged risks \eqref{eq:expected-risks} therefore specialize to
\begin{equation}
    \overline{\mathcal R}_T
    =\mathbb E_{\beta^\star,X,\varepsilon}\!
      \left[\mathcal R(\widehat h_T;\beta^\star)\right],
    \qquad
    \overline{\mathcal R}_S
    =\mathbb E_{\beta^\star,X,\varepsilon,\widetilde X}\!
      \left[\mathcal R(\widehat h_S;\beta^\star)\right].
    \label{eq:case-I-expected-risks}
\end{equation}

%%%%%%%%%%%%%%%%%%%%%%%%%%%%%%%%%%%%%%%%%%%%%%%%%%%%%%%%%%%%%%%%%%%%%%%%%%%%%%%%%%%%%%%%%%%%%%%%%%%%%%%%%%%%%%%%%%%%%%%%%%%%%%%%%%%%%%%%%%%%%%%%%%%%%%%%%%%%%%%%%%%%%%%%%%%%%%%%%%%%%%%%%%%%%%%%%%%%%%%%%%%%%%%%%%%%%%%%%%%%%%%%%%%%%%%%%%%%%%%%%%%%%%%%%%%%%%%%%%%%%%%%%%%%%%%%%%%%%%%%%%%%%%%%%%%%%%%%%%%%%%%%%%%%%%%%%%%%%%%%%%%%%%%%%%%%%%%%%%%%%%%%%%%%%%%%%%%%%%%%%%%%%%%%%%%%%%%%%%%%%%%%%%%%%%%%%%%%%%%%%%%%%%%%%%%%%%%%%%%%%%%%%%%%%%%%%%%%%%%%%%%%%%%%%%%%%%%%%%%%%%%%%%%%%%%%%%%%%%%%

\subsection{Stage-II Projection Geometry}
\label{subsec:stage-II-projection-geometry}

Because $\widetilde X$ is drawn independently of everything determining $\widehat\beta_T$ --- in particular, independently of $X$, $\varepsilon$, and $\beta^\star$ --- we can study the statistical behavior of the Stage-II design and the statistical behavior of the teacher separately, and only combine them at the level of the final risk. This separation is what makes Lemma~\ref{lem:stage-II-expected-spectral-retention} below useful: it is a statement purely about $\widetilde X$, valid for \emph{any} fixed $\widehat\beta_T$, and can therefore be applied after averaging over Stage~I.

The minimum-norm student solution satisfies
\begin{equation}
    \widehat\beta_S
    =P_{\widetilde X}\widehat\beta_T,
    \qquad
    P_{\widetilde X}:=\widetilde X^\dagger\widetilde X,
\end{equation}
where $P_{\widetilde X}$ is the orthogonal projector onto the row space of $\widetilde X$. Under our Gaussian design, once $m\ge p$ this projector is the identity almost surely, so the student exactly reproduces the teacher and the two-stage procedure is statistically inert. The only regime in which Stage~II can do anything other than copy the teacher is therefore the overparameterized one, $m<p$, and we restrict attention to it from here on.

\begin{proposition}[Necessity of a finite-sample bottleneck]
\label{prop:case-I-sample-bottleneck}
If $m\ge p$, then $\widehat\beta_S=\widehat\beta_T$ almost surely and
$\mathsf{W2SG}=0$. Hence W2SG in Case~I requires $m<p$.
\end{proposition}
\begin{proof}
For $m\ge p$, the Gaussian design $\widetilde X$ has full column rank
almost surely because $\Sigma\succ0$. Therefore
$\widetilde X^\dagger\widetilde X=I_p$, so $\widehat\beta_S=\widehat\beta_T$
and $\mathsf{W2SG}=0$.
\end{proof}

When $m<p$, $P_{\widetilde X}$ has rank exactly $m$: Stage~II retains the component of $\widehat\beta_T$ in an $m$-dimensional random subspace --- the row space of $\widetilde X$ --- and discards its component in the $(p-m)$-dimensional orthogonal complement. The sample size $m$ controls the \emph{dimension} of this retained subspace; the realized covariates $\widetilde X$ control its \emph{orientation}. This is the source of the trade-off at the heart of Case~I: increasing $m$ lets more of the teacher's useful signal reach the student, but by the same mechanism it also lets more of the teacher's error reach the student.

Covariance geometry biases which directions this random subspace preferentially retains. The next lemma makes this precise \emph{in expectation}: retention is not a hard cutoff at the top $m$ eigenvalues of $\Sigma$, but a soft preference for higher-variance directions, with a total retention budget fixed by $m$.

\begin{lemma}[Expected spectral retention]
\label{lem:stage-II-expected-spectral-retention}
Let $1\le m<p$ and $\Sigma=V\operatorname{diag}(\lambda_1,\ldots,\lambda_p)V^\top$,
with $V=[v_1,\ldots,v_p]$ orthogonal and $\lambda_1\ge\cdots\ge\lambda_p>0$.
Define $\ell_i:=\mathbb E_{\widetilde X}\|P_{\widetilde X}v_i\|_2^2$. Then
\[
    \mathbb E_{\widetilde X}[P_{\widetilde X}]
    =V\operatorname{diag}(\ell_1,\ldots,\ell_p)V^\top,
    \qquad
    1>\ell_1\ge\cdots\ge\ell_p>0,
    \qquad
    \sum_{i=1}^{p}\ell_i=m.
\]
\end{lemma}
Proof is outlined in Appendix~\ref{app:A}(see Lemma~\ref{lem:stage-II-rank-spectral-filter}). Thus $m$ sets a total expected retention budget across all $p$ directions, while $\Sigma$ determines how that budget is allocated: higher-variance directions are retained with no smaller expected weight than lower-variance ones, but \emph{every} direction, including the top eigendirection, retains a strictly positive expected fraction below one. A finite-sample bottleneck therefore attenuates both signal and teacher error, never one exclusively; whether the net effect helps or hurts the student is exactly what the risk analysis below resolves.

\subsection{Power-Law Covariance and Two Regimes for \texorpdfstring{$m$}{m}}
\label{subsec:case-I-power-law-asymptotics}

% We now specialize $\Sigma$ to a power-law spectrum, the standard testbed
% for learning curves and scaling laws in high-dimensional regression
% \citep{caponnetto2007optimal,steinwart2009optimal,spigler2020asymptotic,
% cui2021generalization,bordelon2024dynamical,defilippis2026scaling,
% wu2026improved}; we adopt it because it is already
% standard in this literature, not because our results depend on features
% specific to it.

% \textcolor{red}{Mohammad: this statement is kinda bold, power law has some aspect that make our result possible, the first one is being non-isotropic, in isotropic covariance, there is no way to get w2sg without regularization.
% Second,the proper difference in magnitude of variance in each coordinate, would result in that specific coordinate to be noise dominated, which result to our characterization. Conversely, I can say that our result has been valid to some other non-isotropic covariance beyond power law.
% besides all of this, i don't think using power law is that bad, comparing to \cite{wu2025provable}'s stylized spiked covariance power law is more smoother and non-trivial.}

We now specialize $\Sigma$ to a power-law spectrum. Non-isotropy is essential to the mechanism: under an isotropic covariance, the projection of Lemma~\ref{lem:stage-II-expected-spectral-retention} retains every direction with the same expected weight $m/p$, attenuating signal and teacher error identically, so W2SG cannot occur without an explicit ridge. A power-law spectrum gives the graded, coordinate-dependent separation between signal and noise-dominated directions that a finite-sample projection needs to filter one from the other. We use it both for this reason and because it is the standard testbed for learning curves and scaling laws in high-dimensional regression \citep{caponnetto2007optimal,steinwart2009optimal,spigler2020asymptotic,cui2021generalization,bordelon2024dynamical,defilippis2026scaling,wu2026improved}, smoother than the spiked-covariance model used for a related purpose by \cite{wu2025provable}. We conjecture the same qualitative behavior holds for other sufficiently graded non-isotropic spectra, though not for every non-isotropic $\Sigma$.

\begin{assumption}[Power-law spectrum and asymptotic scaling]
\label{ass:case-I-asymptotics}
The population covariance has eigenvalues $\lambda_i(\Sigma)=i^{-\alpha}$
for $i=1,\ldots,p$, with $1<\alpha<1+\sqrt2$. As $p\to\infty$, the Stage-I
sample size satisfies $p/n\to\phi_n\in(1,\infty)$, and the Stage-II sample
size $m$ follows one of two regimes:
\begin{itemize}
    \item \emph{Proportional regime}: $m\to\infty$ with $p/m\to\phi_m\in(1,\infty)$, so $m=\Theta(p)$;
    \item \emph{Fixed-sample regime}: $m$ is held at a fixed finite value as $p\to\infty$, so $m/p\to0$.
\end{itemize}
\end{assumption}

These two regimes are not two arbitrary sub-cases: they probe genuinely
different scales at which a critical pseudo-sample size can appear. A
critical threshold that is $O(1)$ in absolute terms is invisible in the
proportional regime, where $m$ is already growing linearly with $p$; a
critical threshold that is $\Theta(p)$ is invisible in the fixed-sample
regime, where $m$ never leaves a bounded range. We need both to see the
full picture. The restriction $\alpha<1+\sqrt2$ is a technical requirement
for the phase-diagram analysis below (not for
Lemma~\ref{lem:stage-II-expected-spectral-retention}); we discuss
$\alpha>1+\sqrt2$, where the picture is only partially resolved, in
Appendix~\ref{app:A}.

\subsection{Asymptotic W2SG Phase Diagram}
\label{subsec:case-I-phase-diagram}

Under Assumption~\ref{ass:case-I-asymptotics}, the teacher and student
risks admit deterministic equivalents: asymptotically exact
high-dimensional approximations, in the sense of
\citet{atanasov2026scaling}, that replace the random resolvents governing
$\overline{\mathcal R}_T$ and $\overline{\mathcal R}_S$ by explicit
functions of $\Sigma$ and the two sample ratios (also, see \cite{atanasov2025two} for two point DE). We write
$\mathsf{W2SG}_p^{\mathrm{DE}}$ for the resulting deterministic equivalent
of the finite-$p$ gain $\mathsf{W2SG}_p$; the full derivation, including the
standard Stage-I risk formula in terms of a self-consistent ridge parameter
$\tau_T$ and its associated coefficient $\gamma_T\in(0,1)$ (already known
from the ridgeless-regression literature cited above), is given in
Appendix~\ref{app:A}. Here we state only what is needed to read the
theorem: $\gamma_T$ is a scalar, determined by $\alpha$ and $\phi_n$
alone, that measures how much of the Stage-I estimator's error is
variance (driven by $\sigma^2$) rather than bias.

% The two regimes of Assumption~\ref{ass:case-I-asymptotics} produce
% different phase pictures, which we state informally here and prove in full
% in Appendix~\ref{app:A} (Theorems~\ref{thm:threecase-global-positive},
% \ref{thm:fp-unique-transition}, \ref{thm:three-regime-phase-diagram} for
% the proportional regime; Theorems~\ref{thm:fixed-m-DE-gain},
% \ref{fmphase:thm:subcritical-unique}, \ref{fmphase:thm:supercritical-integer}
% for the fixed-sample regime).
% \textcolor{red}{Mohammad: 
% I think after this point, there are some contradiction that may perplexed the reader.
% to recap of what we derived, first we find out the formula of $\mathsf{W2SG}_p^{\mathrm{DE}}$ in finite $p$ regime, but directly analyze it  was not possible due to complex formula. so we took $p\to\infty$ which has been stated in Assumption~\ref{ass:case-I-asymptotics}. Now in the following theorem(part (a)) i think $\mathsf{W2SG}_p^{\mathrm{DE}}$ is suggesting that we have finite $p$ which is not true, I used $\mathsf{W2SG}^{\mathrm{DE}}$ in my derivations (equation~\ref{eq:powerlaw-final-limit-Ay})}

The two regimes of Assumption~\ref{ass:case-I-asymptotics} produce
different phase pictures, which we state informally here and prove in full
in Appendix~\ref{app:A} (Theorems~\ref{thm:threecase-global-positive},
\ref{thm:fp-unique-transition}, \ref{thm:three-regime-phase-diagram} for
the proportional regime; Theorems~\ref{thm:fixed-m-DE-gain},
\ref{fmphase:thm:subcritical-unique}, \ref{fmphase:thm:supercritical-integer}
for the fixed-sample regime). In the proportional regime, $m=\Theta(p)$
and we state results directly for the $p\to\infty$ limit
$\mathsf{W2SG}^{\mathrm{DE}}:=\lim_{p\to\infty}\mathsf{W2SG}_p^{\mathrm{DE}}$;
in the fixed-sample regime, $m$ stays finite as $p\to\infty$ and we state
results for the corresponding limit $G_m:=\lim_{p\to\infty}\mathsf{W2SG}_p^{\mathrm{DE}}$,
which still depends on $m$.
\begin{theorem}[W2SG phase diagram, informal]
\label{thm:case-I-phase-diagram}
Suppose Assumption~\ref{ass:case-I-asymptotics} holds and $\sigma^2>0$.

\smallskip\noindent\textnormal{\textbf{(a) Proportional regime} ($m=\Theta(p)$).}
% \begin{enumerate}[label=(\roman*)]
\begin{itemize}
    \item[\textnormal{(i)}] If $1<\alpha\le2$, then $\mathsf{W2SG}^{\mathrm{DE}}>0$ for
    \textbf{every} $\phi_m>1$ --- that is, for every pseudo-sample size
    $m<p$ growing proportionally with $p$, with no further condition on
    $\sigma^2$ or on $\phi_m$.
    \item[\textnormal{(ii)}] If $2<\alpha\le1+\sqrt2$, there is a unique critical ratio
    $\phi_{m,\mathrm{crit}}$ (solution of a fixed-point equation, see Appendix \ref{app:A}), depending on $\alpha$, $\phi_n$, and $\sigma^2$
    but requiring no restriction on $\sigma^2$, such that
    $\mathsf{W2SG}^{\mathrm{DE}}>0$ if and only if $\phi_m<\phi_{m,\mathrm{crit}}$,
    equivalently $m>m_{\mathrm{crit}}^{\mathrm{FP}}(p)\sim p/\phi_{m,\mathrm{crit}}=\Theta(p)$.
% \end{enumerate}
\end{itemize}

\smallskip\noindent\textnormal{\textbf{(b) Fixed-sample regime} ($m=O(1)$).}
Here $\mathsf{W2SG}_p^{\mathrm{DE}}$ converges, as $p\to\infty$ with $m$
fixed, to an explicit limit $G_m$ depending only on $m$, $\alpha$, and
$\sigma^2\gamma_T/(1-\gamma_T)$.
% \begin{enumerate}[label=(\roman*)]
\begin{itemize}
    \item[\textnormal{(i)}] If $1<\alpha<2$: for $\sigma^2$ below an explicit threshold
    $\sigma_c^2(\alpha,\phi_n)$, $G_m<0$ at $m=1$; and $G_m\to(2-\alpha)\sigma^2\gamma_T/(1-\gamma_T)>0$
    as $m\to\infty$, so at least one crossing $m_*$ exists. For $\sigma^2$
    below a second, smaller threshold $\sigma_0^2(\alpha,\phi_n)<\sigma_c^2(\alpha,\phi_n)$,
    this crossing is unique, and as $\sigma^2\downarrow0$ it is located at
    \begin{equation}
        m_{\mathrm{crit}}^{\mathrm{HP}}
        \sim
        \left[
            \frac{c_\alpha(1-\gamma_T)}{(2-\alpha)\sigma^2\gamma_T}
        \right]^{\frac1{\alpha-1}},
        \qquad
        c_\alpha:=\left[\frac{\pi}{\alpha}\csc\frac{\pi}{\alpha}\right]^\alpha,
        \label{eq:mcrit-HP}
    \end{equation}
    and $\mathsf{W2SG}_p^{\mathrm{DE}}\to G_m>0$ for every $m>m_{\mathrm{crit}}^{\mathrm{HP}}$.
    \item[\textnormal{(ii)}] If $2<\alpha\le5/2$: there is an explicit threshold
    $\sigma_{\mathrm{safe}}^2(\alpha,\phi_n)$ such that, for
    $\sigma^2\le\sigma_{\mathrm{safe}}^2(\alpha,\phi_n)$, $G_m\le0$ for
    \emph{every} fixed $m\ge1$: in this low-noise regime, no fixed
    pseudo-sample size produces W2SG at all.
% \end{enumerate}
\end{itemize}
Explicit formulas for $\sigma_{\mathrm{safe}},\sigma_{0},\sigma_{c}$ can be found in appendix~\ref{app:A}
\end{theorem}
Two qualitative points are worth drawing out explicitly. First, part~(a)
shows that when the pseudo-sample size scales with the ambient dimension,
a single, unconditional statement is possible for $\alpha\le2$: W2SG is
guaranteed for \emph{every} $m<p$, with no bottleneck fine-tuning needed
at all; the bottleneck only becomes a genuine constraint, with a sharp
critical threshold, once $\alpha>2$. Second, part~(b) shows that the fixed-
and proportional-sample regimes are not merely two proof techniques for
the same fact --- they can disagree qualitatively. For $2<\alpha\le5/2$
and sufficiently small noise, no \emph{fixed} $m$ ever produces W2SG, yet
part~(a)(ii) guarantees a $\Theta(p)$-scale pseudo-sample size that does.
The bottleneck effect, in this regime, is a genuinely proportional
phenomenon: it requires $m$ to grow with $p$, not merely to be large in
absolute terms.

\begin{remark}[Beyond the small-noise threshold, and disconnected regions]
\label{rem:case-I-disconnected} The explicit formula \eqref{eq:mcrit-HP} for $m_{\mathrm{crit}}^{\mathrm{HP}}$ is an asymptotic statement as $\sigma^2\downarrow0$, though in practice it tracks the true critical $m$ closely even at moderate noise levels. For noise above the thresholds in Theorem~\ref{thm:case-I-phase-diagram}, teacher error can be concentrated enough near specific high-variance directions --- rather than spread smoothly across the spectrum --- that the Stage-II projection improves on the teacher even at very small $m$. Combined with the proportional-scale threshold of part~(a), the set of pseudo-sample sizes producing W2SG can then split into a window near $m=O(1)$ and a separate window near $m=\Theta(p)$. This is not merely an artifact of two disconnected proof techniques: the intermediate scaling $m\to\infty$ with $m/p\to0$, analyzed in Appendix~\ref{app:A} (Section~\ref{subsec:sparse-growing-extension}), contains no critical $m$ at all, so our fixed-sample and proportional results together already give a complete picture across every scaling of $m$ with $p$.
\end{remark}
%%%%%%%%%%%%%%%%%%%%%%%%%%%%%%%%%%%%%%%%%%%%%%%%%%%%%%%%%%%%%%%%%%%%%%%%%%%%%%%%%%%%%%%%%%%%%%%%%%%%%%%%%%%%%%%%%%%%%%%%%%%%%%%%%%%%%%%%%%%%%%%%%%%%%%%%%%%%%%%%%%%%%%%%%%%%%%%%%%%%%%%%%%%%%%%%%%%%%%%%%%%%%%%%%%%%%%%%%%%%%%%%%%%%%%%%%%%%%%%%%%%%%%%%%%%%%%%%%%%%%%%%%%%%%%%%%%%%%%%%%%%%%%%%%%%%%%%%%%%%%%%%%%%%%%%%%%%%%%%%%%%%%%%%%%%%%%%%%%%%%%%%%%%%%%%%%%%%%%%%%%%%%%%%%%%%%%%%%%%%%%%%%%%%%%%%%%%%%%%%%%%%%%%%%%%%%%%%%%%%%%%%%%%%%%%%%%%%%%%%%%%%%%%%%%%%%%%%%%%%%%%%%%%%%%%%%%%%%%%%%%%%%%%%%%%%%%%%%%%%%%%%%%%%%%%%%%%%%%%%%%%%%%%%%%%%%%%%%%%%%%%%%%%%%%%%%%%%%%%%%%%%%%%%%%%%%%%%%%%%%%%%%%%%%%%%%%%%%%%%%%%%%%%%%%%%%%%%%%%%%%%%%%%%%%%%%%%%%%%%%%%%%%%%%%%%%%%%%%%%%%%%%%%%%%%%%%%%%%%%%%%%%%%%%%%%%%%%%%%%%%%%%%%%%%%%%%%%%%%%%%%%%%%%%%%%%%%%%%%%%%%%%%%%%%%%%%%%%%%%%%%%%%%%%%%%%%%%%%%%%%%%%%%%%%%%%%%%%%%%%%%%%%%%%%%%%%%%%%%%%%%%%%%%%%%%%%%%%%%%%%%%%%%%%%%%%%%%%%%%%%%%%%%%%%%%%%%%%%%%%%%%%%%%%%%%%%%%%%%%%%%%%%%%%%%%%%%%%%%%%%%%%%%%%%%%%%%%%%%%%%%%%%%%%%%%%%%%%%%%%%%%%%%%%%%%%%%%%%%%%%%%%%%%%%%%%%%%%%%%%%%%%%%%%%%%%%%%%%%%%%%%%%%%%%%%%%%%%%%%%%%%%%%%%%%%%%%%%%%%%%%%%%%%%%%%%%%%%%%%%%%%%%%%%%%%%%%%%%%%%%%%%%%%%%%%%%%%%%%%%%%%%%%%%%%%%%%%%%%%%%%%%%%%%%%%%%%%%%%%%%%%%%%%%%%%%%%%%%%%%%%%%%%%%%%%%%%%%%%%%%%%%%%%%%%%%%%%%%%%%%%%%%%%%%%%%%%%%%%%%%%%%%%%%%%%%%%%%%%%%%%%%%%%%%%%%%%%%%%%%%%%%%%%%%%%%%%%%%%%%%%%%%%%%%%%%%%%%%%%%%%%%%%%%%%%%%%%%%%%%%%%%%%%%%%%%%%%%%%%%%%%%%%%%%%%
\section{Case II: The Random-Feature Model}
\label{sec:case-II-random-features}

Case~I shows that a representational advantage is not necessary for W2SG:
with teacher and student sharing one hypothesis class, the pseudo-sample
size $m$ alone was enough to act as the bottleneck. We now ask a different
question. Suppose the student genuinely \emph{is} a stronger model, with
access to a richer representation than the teacher. Does the bottleneck
mechanism survive this asymmetry, and if so, is $m$ still the resource
that plays the bottleneck role, or can that role pass to a different
resource entirely? This section shows that the mechanism survives, and
that the bottleneck can indeed be either resource: the pseudo-sample size
$m$, as in Case~I, or the student's feature width $N_S$, which has no
counterpart in Case~I at all. Which resource is active, and where the
handoff between them occurs, is the content of
Theorem~\ref{thm:caseII} below.

\subsection{Model and Two-Stage Training}
\label{subsec:case-II-model}

We adopt the linear random-feature model of
\citet{medvedev2025weak}, which gives
teacher and student each their own independently sampled random-feature
map --- a standard device for endowing a linear model with a controllable,
finite amount of representational capacity, and the natural setting in
which to let the student have genuinely more capacity than the teacher.
We keep their model exactly, but ask a different question of it: rather
than showing that early stopping in Stage~II produces W2SG, as
\citet{medvedev2025weak} do, we train
the student to full convergence and ask whether the \emph{finiteness} of
its two resources --- pseudo-sample size and feature width --- is by
itself enough. This makes our result complementary to theirs: two
different mechanisms, early stopping versus a resource bottleneck, both
sufficient for W2SG in the same underlying model.

\paragraph{Target and covariance structure.}
Let $x\sim\mathcal N(0,\Sigma)$, where
$\Sigma:=\operatorname{diag}(I_k,\delta_pI_{p-k})$: the first $k$
coordinates have unit variance and the remaining $p-k$ have variance
$\delta_p\in(0,1)$, a sequence that we allow to depend on $p$ and whose
asymptotic rate (relative to $p-k$) will matter once we take $p\to\infty$
in Section~\ref{sec:pf-DE-only}. We refer to the first block as the
\emph{strong} coordinates and the second as the \emph{weak} coordinates.
The target coefficient $\beta^\star\in\mathbb R^p$ is deterministic ---
unlike Case~I, where $\beta^\star$ was random with isotropic
second moment --- and is supported entirely on the strong block, so
$\beta_i^\star=0$ for $i>k$: the signal we want the student to recover
lives entirely in the high-variance directions, while the weak block is
pure background. We normalize $\|\beta^\star\|_\Sigma^2=1$. There is no
label noise in this section ($\varepsilon\equiv0$ in
\eqref{eq:general-data-generation}); the only source of teacher error here
is a finite random-feature representation, not observation noise, and
isolating this source is the point of the model.

\paragraph{Stage I: teacher.}
Let $U_T\in\mathbb R^{N_T\times p}$ have i.i.d.\ $\mathcal N(0,1)$ entries;
$U_T$ is drawn once and then held fixed throughout Stage~I --- it plays
the role of a fixed, randomly initialized feature map, not a parameter
that Stage~I's learning rule optimizes over. Conditional on $U_T$, the
teacher's feature extractor is $f_T(x)=U_Tx\in\mathbb R^{N_T}$, and its
hypothesis class is
$\mathcal H_T(U_T):=\{h_{T,w}:x\mapsto x^\top U_T^\top w:w\in\mathbb R^{N_T}\}$.
Rather than fitting $n$ noisy labeled samples as in Case~I, the teacher
here is idealized as the \emph{population} risk minimizer within
$\mathcal H_T(U_T)$ --- the $n\to\infty$ limit of Stage~I training, which
lets us isolate representation bias as the sole source of teacher error,
uncontaminated by any Stage-I estimation variance:
\begin{equation}
    w_T^\star
    :=
    \argmin_{w\in\mathbb R^{N_T}}~\|w\|_2
    \quad\text{subject to}\quad
    U_T\Sigma U_T^\top w=U_T\Sigma\beta^\star.
    \label{eq:case-II-teacher-solution}
\end{equation}
The induced coefficient in the original input space is
$\beta_T^\star:=U_T^\top w_T^\star$, and the teacher predictor is
$h_T^\star(x):=x^\top\beta_T^\star$. Because $w_T^\star$ solves the exact
population problem, $h_T^\star$'s only error, conditional on $U_T$, is
approximation bias from the finiteness of $N_T$; the only remaining
randomness, across different draws of $U_T$, is randomness in \emph{which}
$N_T$-dimensional representation the teacher happened to receive, not
sample noise.

\paragraph{Stage II: student.}
Independently of $U_T$, draw fresh covariates
$\widetilde x_1,\ldots,\widetilde x_m\iid\mathcal N(0,\Sigma)$ and assign
pseudo-labels $\widetilde y_i:=h_T^\star(\widetilde x_i)=\widetilde x_i^\top\beta_T^\star$,
giving the Stage-II dataset $\widetilde{\mathcal D}_m$ of
Section~\ref{sec:preliminaries}, with design matrix $\widetilde X$ and
label vector $\widetilde y$. The student is given its \emph{own}
independent random-feature map, $U_S\in\mathbb R^{N_S\times p}$ with
i.i.d.\ $\mathcal N(0,1)$ entries, independent of $U_T$ and of
$\widetilde{\mathcal D}_m$, with feature extractor $f_S(x)=U_Sx$ and
hypothesis class $\mathcal H_S(U_S)$ defined analogously to
$\mathcal H_T(U_T)$. This is the crucial departure from Case~I: the
student's representation is no longer tied to the teacher's, and the
student's width $N_S$ is a resource entirely under our control, with no
counterpart on the teacher side. Once $N_S$ is allowed to differ from
$N_T$ --- in particular, once $N_S>N_T$ --- the student is a genuinely
stronger model in the ordinary sense of having more capacity, and it is
exactly this second resource, alongside $m$, that gives Case~II its two
possible bottlenecks.

Because $\widetilde y$ is generated by $h_T^\star\notin\mathcal H_S(U_S)$
in general, it need not be exactly interpolable by the student; Stage~II
therefore performs minimum-norm \emph{least-squares} regression, rather
than minimum-norm interpolation as in Case~I:
\begin{equation}
    \widehat w_S
    =
    \argmin_{w\in\mathbb R^{N_S}}~\|w\|_2
    \qquad\text{subject to}\qquad
    w\in\argmin_{v\in\mathbb R^{N_S}}~
    \left\|\widetilde XU_S^\top v-\widetilde y\right\|_2^2,
\end{equation}
with induced coefficient $\widehat\beta_S:=U_S^\top\widehat w_S$ and
student predictor $\widehat h_S(x):=x^\top\widehat\beta_S$.

\paragraph{Risks.}
The population risk \eqref{eq:general-population-risk} again takes the
closed form $\mathcal R(h_\beta;\beta^\star)=(\beta-\beta^\star)^\top\Sigma(\beta-\beta^\star)$,
so the averaged risks of \eqref{eq:expected-risks} specialize to
\begin{equation}
    \overline{\mathcal R}_T
    =\mathbb E_{U_T}\!\left[\mathcal R(h_T^\star;\beta^\star)\right],
    \qquad
    \overline{\mathcal R}_S
    =\mathbb E_{U_T,U_S,\widetilde{\mathcal X}_m}\!
      \left[\mathcal R(\widehat h_S;\beta^\star)\right],
    \label{eq:case-II-expected-risks}
\end{equation}
where $\overline{\mathcal R}_S$ additionally averages over Stage-II sample
randomness and over the student's own representation $U_S$.

\subsection{Two Stage-II Resources, Three Regimes}
\label{subsec:case-II-bottleneck-regimes}

As in Case~I, whether Stage~II can do anything nontrivial at all is
governed by a rank condition --- but here there are \emph{two} independent
Stage-II quantities, $m$ and $N_S$, that can each be the scarcer one
relative to $p$.

\begin{definition}[Stage-II bottleneck regimes]
\label{def:stage-II-bottlenecks}
The Stage-II regime is determined by which of $m$, $N_S$, and $p$ is
smallest:
\begin{itemize}
    \item the \emph{sample-bottleneck regime}: $m<\min\{N_S,p\}$;
    \item the \emph{feature-bottleneck regime}: $N_S<\min\{m,p\}$;
    \item the \emph{no-bottleneck regime}: $p<\min\{m,N_S\}$.
\end{itemize}
\end{definition}

\begin{proposition}[No-bottleneck regime]
\label{prop:no-stage-II-bottleneck}
If $p<\min\{m,N_S\}$, then $\widehat h_S=h_T^\star$ almost surely, so
$\overline{\mathcal R}_S=\overline{\mathcal R}_T$ and $\mathsf{W2SG}=0$.
\end{proposition}
\begin{proof}
Since $m>p$ and $N_S>p$, both $\widetilde X\in\mathbb R^{m\times p}$ and
$U_S\in\mathbb R^{N_S\times p}$ have full column rank almost surely; hence
$\operatorname{range}(U_S^\top)=\mathbb R^p$, and because
$\widetilde y=\widetilde X\beta_T^\star$, the unique minimizer of the
Stage-II loss is $\widehat\beta_S=\beta_T^\star$, giving
$\widehat h_S=h_T^\star$, $\overline{\mathcal R}_S=\overline{\mathcal R}_T$,
and $\mathsf{W2SG}=0$.
\end{proof}

Exactly as in Case~I, then, W2SG is only possible once \emph{at least
one} Stage-II resource is scarce relative to $p$ --- both being abundant
just reproduces the teacher exactly. What Case~II adds is that this
scarcity can come from either resource, and the two subsections below show
that it produces genuinely different behavior depending on which one it
is.

\subsection{Asymptotic W2SG Phase Diagram}
\label{sec:pf-DE-only}

We work under the same style of deterministic-equivalent analysis
introduced in Section~\ref{subsec:case-I-phase-diagram}, now applied
jointly to the two Stage-II resources, and let all five relevant
quantities grow proportionally with $p$.

\begin{assumption}[Proportional asymptotic scaling]
\label{ass:positive-fraction}
Let $d_p:=p-k$, and assume
\begin{equation}
    \frac kp\to\vartheta,
    \qquad
    d_p\delta_p\to\eta,
    \qquad
    \frac{N_T}p\to\rho_T,
    \qquad
    \frac mp\to\rho_m,
    \qquad
    \frac{N_S}p\to\rho_S,
\end{equation}
where $\vartheta,\rho_T\in(0,1)$, $\eta\in(0,\infty)$, and
$\rho_m,\rho_S\in(0,\infty)$. We consider the \emph{regular} fixed-ratio
regime, in which $\rho_T\neq\vartheta$, $\rho_m\neq\rho_S$,
$\rho_m,\rho_S\neq\vartheta$, and $\rho_m,\rho_S\neq1$; the nontrivial excluded boundary cases are treated in
Section~\ref{sec:case-II-critical-boundaries} of
Appendix~\ref{app:B}.
\end{assumption}

Here $\vartheta$ is the asymptotic fraction of coordinates carrying signal,
and $\rho_T,\rho_m,\rho_S$ are the three Stage-II-relevant resource
ratios of Definition~\ref{def:stage-II-bottlenecks}, now expressed in the
proportional limit. As in Case~I, we use \emph{deterministic equivalents}:
$\mathsf{W2SG}_p^{\mathrm{DE}}$ denotes the high-dimensional asymptotic
approximation to the finite-$p$ gain $\mathsf{W2SG}_p$, in the same sense
introduced in Section~\ref{subsec:case-I-phase-diagram}; the full
derivation is in Appendix~\ref{app:B}.

\begin{theorem}[W2SG phase diagram]
\label{thm:caseII}
Suppose Assumption~\ref{ass:positive-fraction} holds. Define
\begin{equation}
    \Lambda_T
    :=
    \begin{cases}
        \vartheta-\rho_T, & \rho_T<\vartheta,\\[1mm]
        \dfrac{(1-\vartheta)(\rho_T-\vartheta)}{1-\rho_T}, & \rho_T>\vartheta.
    \end{cases}
\end{equation}

\smallskip\noindent\textnormal{(i) \textbf{Necessary condition.}}
If $\min\{\rho_m,\rho_S\}<\vartheta$, then $\mathsf{W2SG}_p^{\mathrm{DE}}<0$:
if either Stage-II resource falls below the signal fraction, the student
cannot even recover the signal, let alone improve on the teacher.

\smallskip\noindent\textnormal{(ii) \textbf{Sample-bottleneck regime}}
($\vartheta<\rho_m<\min\{1,\rho_S\}$, so $m$, not $N_S$, is the scarcer
Stage-II resource). Let
$\Delta_{\mathrm{DE}}:=(\rho_S+3\vartheta-\Lambda_T)^2-16\vartheta\rho_S$.
If $\Delta_{\mathrm{DE}}>0$, then
\begin{equation}
\label{eq:main-positive-fraction-SB-rho-m}
    \mathsf{W2SG}_p^{\mathrm{DE}}>0
    \iff
    \rho_m\in
    \left(
        \frac{\rho_S+3\vartheta-\Lambda_T-\sqrt{\Delta_{\mathrm{DE}}}}{4},
        \frac{\rho_S+3\vartheta-\Lambda_T+\sqrt{\Delta_{\mathrm{DE}}}}{4}
    \right)
    \cap\bigl(\vartheta,\min\{1,\rho_S\}\bigr).
\end{equation}
If $\Delta_{\mathrm{DE}}\le0$, there is no positive region on this branch.
This is the direct analogue of the Case~I phase diagram: $m$ must lie in a
\emph{bounded window} for the student to improve on the teacher.

\smallskip\noindent\textnormal{(iii) \textbf{Feature-bottleneck regime}}
($\vartheta<\rho_S<\min\{1,\rho_m\}$, so $N_S$, not $m$, is now the
scarcer Stage-II resource). Here
\begin{equation}
\label{eq:main-positive-fraction-FB-rho-m}
    \mathsf{W2SG}_p^{\mathrm{DE}}>0
    \iff
    \rho_S-\vartheta>\Lambda_T
    \quad\text{and}\quad
    \rho_m>\frac{2\rho_S(\rho_S-\vartheta)}{\rho_S-\vartheta-\Lambda_T}.
\end{equation}
Unlike (ii), the condition on $\rho_m$ is a \emph{half-line}: once $N_S$
is the active bottleneck and satisfies $\rho_S-\vartheta>\Lambda_T$, more
pseudo-labels are unambiguously good, and W2SG occurs for every $m$ large
enough.
\end{theorem}

The handoff between the two regimes in the theorem is governed entirely by
\emph{which of $m,N_S$ is smaller relative to $p$}: regime (ii) is exactly
$\rho_m<\rho_S$ (with both above $\vartheta$), and regime (iii) is exactly
$\rho_S<\rho_m$. There is no third, intermediate mechanism --- whichever
resource is scarcer is the one whose value must be tuned into a window (or
past a threshold) for W2SG to occur, and the theorem's two branches are
this dichotomy made precise.

\begin{remark}[Reading the phase diagram]
\label{rem:case-II-reading}
Three consequences of Theorem~\ref{thm:caseII} are worth stating
explicitly.

First, combining (ii) and (iii), any positive-W2SG region forces
$\rho_S>\vartheta+\Lambda_T>\rho_T$: the student must have asymptotically
\emph{more} random features than the teacher. This is not an assumption we
imposed --- it falls out of the theorem. It confirms that a genuine
capacity advantage is \emph{necessary} for W2SG once representations
differ, but the theorem shows it is not \emph{sufficient}: satisfying
$\rho_S>\vartheta+\Lambda_T$ only opens the door to a bottleneck condition
on $m$ still needing to hold.

Second, the qualitative character of the two regimes differs, not just
their formulas. In the sample-bottleneck regime (ii), the pseudo-sample
size itself must be tuned to a bounded window, exactly as in Case~I. In
the feature-bottleneck regime (iii), by contrast, $m$ only needs to clear
a threshold, after which it is purely beneficial: the sample-side
trade-off of Case~I disappears once $N_S$ takes over as the active
constraint.

Third, this asymmetry reflects a general division of labor between the two
resources: whichever resource is the active bottleneck governs how much of
the teacher's approximation error the student filters out, while the
\emph{other} resource, once it is not the bottleneck, only reduces
estimation noise around whatever the bottleneck resource already lets
through --- which is why it can be safely increased without limit.
\end{remark}

The next corollary rearranges Theorem~\ref{thm:caseII} to hold $\rho_m$
fixed and vary $\rho_S$ instead, exhibiting the same two regimes from the
feature-width side.

\begin{corollary}[Phase diagram in student width]
\label{cor:caseII}
Suppose the conditions of Theorem~\ref{thm:caseII} hold and fix
$\rho_m>\vartheta$.

\smallskip\noindent\textnormal{(i) \textbf{Feature-bottleneck regime}}
($\vartheta<\rho_S<\min\{1,\rho_m\}$). Let
$D_S:=(\rho_m+2\vartheta)^2-8\rho_m(\vartheta+\Lambda_T)$. If $D_S>0$,
\begin{equation}
\label{eq:varying-width-feature-interval}
    \mathsf{W2SG}_p^{\mathrm{DE}}>0
    \iff
    \rho_S\in
    \left(
        \frac{\rho_m+2\vartheta-\sqrt{D_S}}4,
        \frac{\rho_m+2\vartheta+\sqrt{D_S}}4
    \right)
    \cap\bigl(\vartheta,\min\{1,\rho_m\}\bigr);
\end{equation}
if $D_S\le0$, there is no positive region on this branch. The
feature-bottleneck positive region is thus a \emph{bounded interval} in
$\rho_S$ --- too little width starves the student, too much lets it
overfit the teacher's bias.

\smallskip\noindent\textnormal{(ii) \textbf{Sample-bottleneck regime}}
($\vartheta<\rho_m<\min\{1,\rho_S\}$). If $\rho_m\le2\vartheta$, there is
no positive region on this branch. If $\rho_m>2\vartheta$,
\begin{equation}
\label{eq:varying-width-sample-half-line}
    \mathsf{W2SG}_p^{\mathrm{DE}}>0
    \iff
    \rho_S>\frac{\rho_m(2\rho_m-3\vartheta+\Lambda_T)}{\rho_m-2\vartheta}.
\end{equation}
The sample-bottleneck positive region is thus a \emph{half-line} in
$\rho_S$: once $m$ is already the active bottleneck, additional student
width is purely beneficial.
\end{corollary}

Corollary~\ref{cor:caseII} is the mirror image of
Theorem~\ref{thm:caseII}: interval versus half-line swaps sides depending
on which variable is held fixed and which is scanned, but the underlying
division is the same one identified in
Remark~\ref{rem:case-II-reading} --- whichever resource is currently the
bottleneck needs a window; the other only needs to clear a threshold.
Together, Theorem~\ref{thm:caseII} and Corollary~\ref{cor:caseII} are what
we mean by the \emph{active-bottleneck principle}: the scarcer of the two
Stage-II resources determines how much teacher error the student filters
out and must sit at an intermediate value to do so well, while the other
resource, once identified as non-bottleneck, can be increased freely and
only serves to reduce estimation noise. This is the sense in which Case~I
is the single-resource special case of the same phenomenon: with only $m$
available to bind, there is no "other resource" to play the free role, and
the sample-bottleneck branch of Theorem~\ref{thm:caseII} is what survives.

\section{Experimental Results}
\label{sec:experiments}

We validate the theory against Monte Carlo simulations by comparing the
finite-$p$ deterministic equivalents with empirically measured W2SG.
The experiments are organized around the two theoretical settings studied
in the paper. For Case~I, we examine the ridgeless-regression phase diagram
on the two sides of the critical exponent $\alpha=2$ and then study the
finite-size behavior associated with the fixed-$m$ and proportional
scalings. For Case~II, we validate the random-feature phase diagram,
separate the sample- and feature-bottleneck mechanisms, and include a
negative-control experiment in which the theory predicts no admissible
positive-W2SG region.

% =============================================================================
% CASE I
% =============================================================================

\begin{figure}[t]
    \centering
    \begin{subfigure}{0.48\textwidth}
        \centering
        \includegraphics[width=\linewidth]{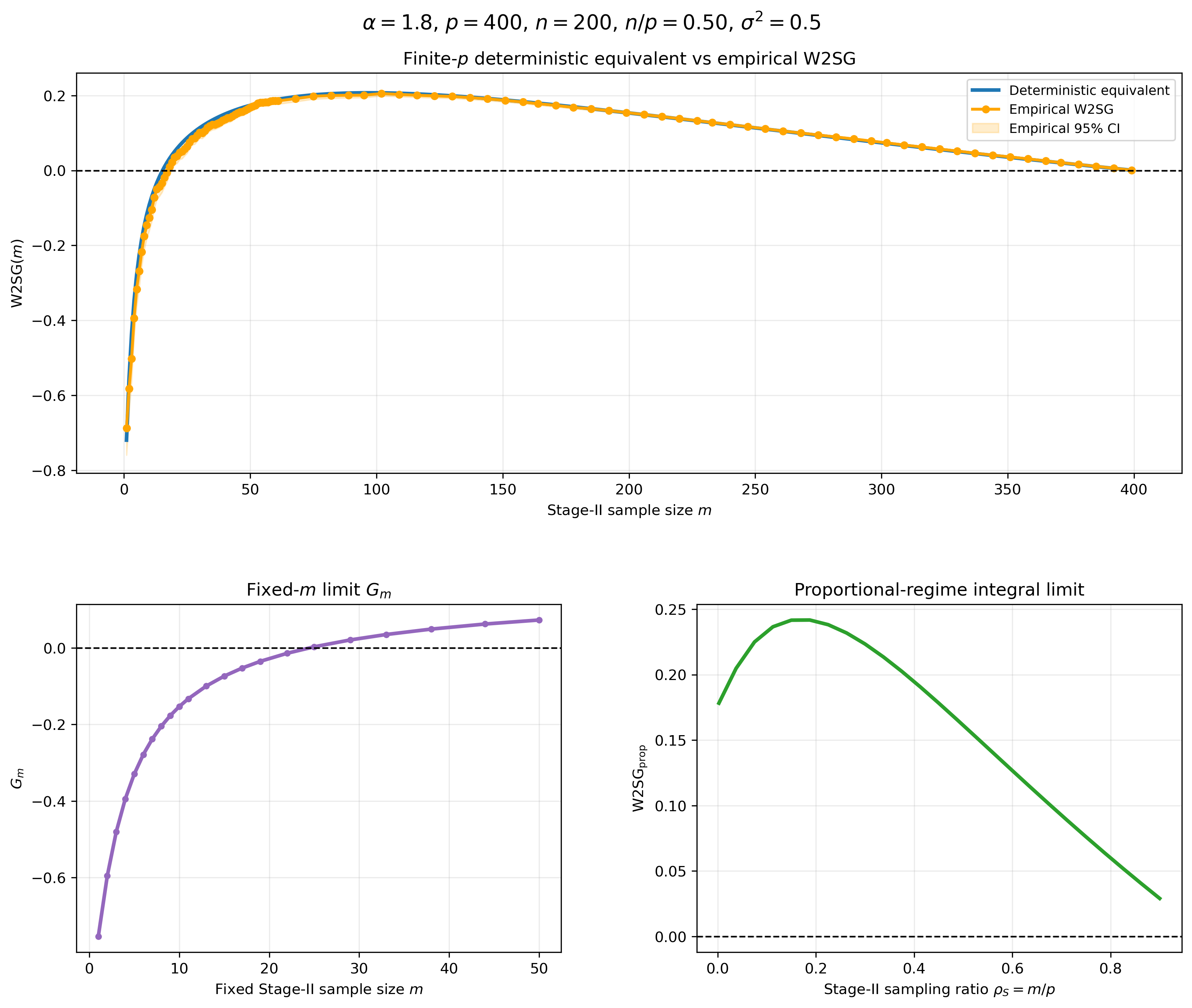}
        \caption{$\sigma^2=0.5$.}
        \label{fig:a}
    \end{subfigure}
    \hfill
    \begin{subfigure}{0.48\textwidth}
        \centering
        \includegraphics[width=\linewidth]{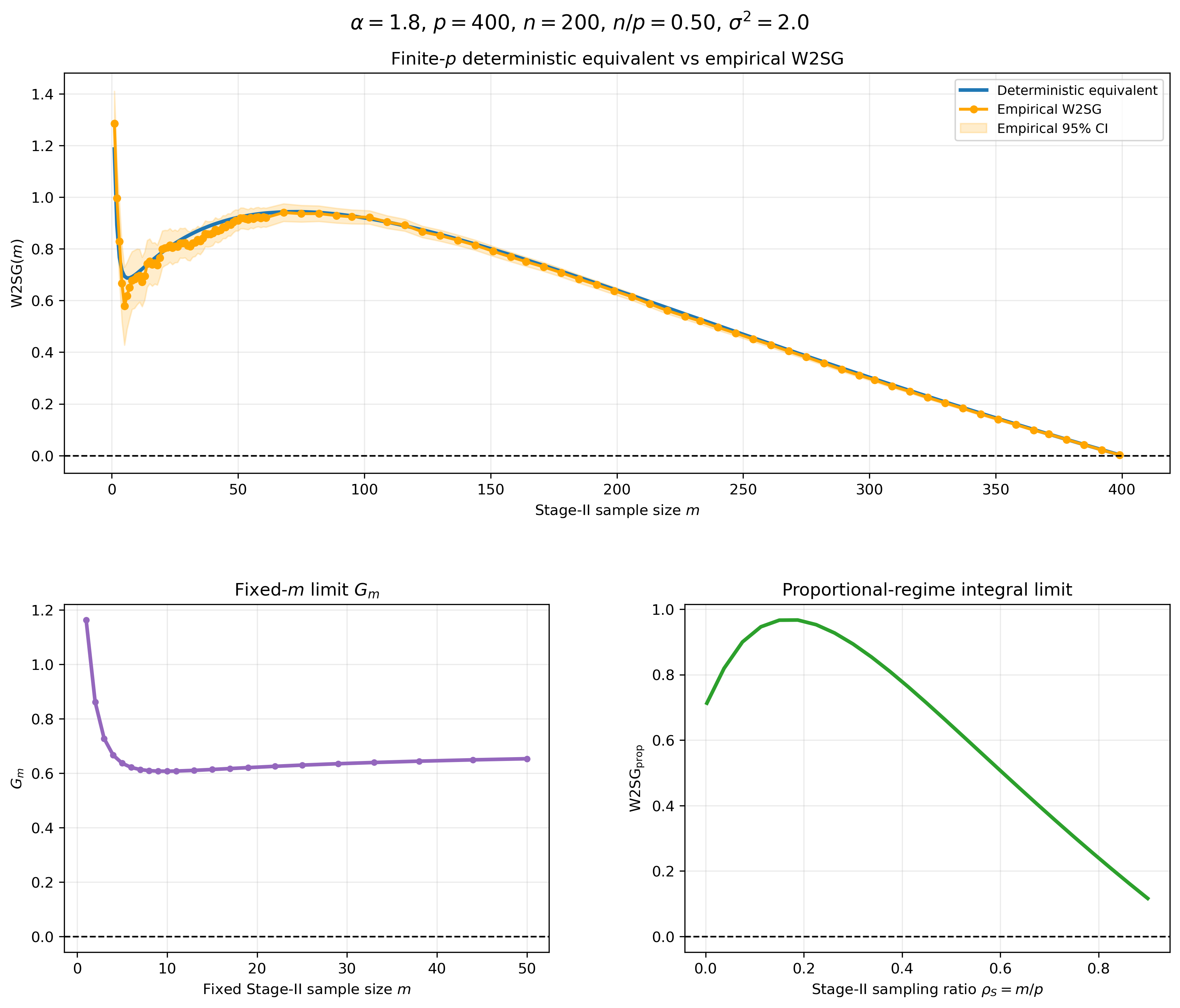}
        \caption{$\sigma^2=2$.}
        \label{fig:b}
    \end{subfigure}
    \caption{\textbf{Case I with $1<\alpha<2$.}
    Numerical validation at $\alpha=1.8$, $p=400$, and $n=200$ using
    $1000$ Monte Carlo trials.  In each subfigure, the top panel compares
    the finite-$p$ deterministic equivalent with empirical W2SG and its
    $95\%$ confidence interval; the bottom-left panel shows the fixed-$m$
    limit $G_m$; and the bottom-right panel shows the proportional-regime
    limit as a function of the Stage-II sampling ratio.}
    \label{fig:combined}
\end{figure}

\subsection{Case I: Ridgeless Regression}
\label{subsec:caseI}

We first consider the ridgeless-regression setting of
Theorem~\ref{thm:case-I-phase-diagram}. We examine the W2SG phase behavior
on the two sides of the critical exponent $\alpha=2$, comparing the
finite-$p$ deterministic equivalent with Monte Carlo simulations and with
the fixed-$m$ and proportional asymptotic limits. We then study finite-size
scaling to illustrate how these two asymptotic regimes emerge under
different scalings of the Stage-II sample size $m$.
% -----------------------------------------------------------------------------
% Case I phase diagram
% -----------------------------------------------------------------------------

\subsubsection{Phase Diagram}
\label{subsubsec:caseI-phase}

We first validate the Case~I deterministic equivalent and phase diagram at
$p=400$ and $n/p=0.5$.  We consider $\alpha=1.8$ and $\alpha=2.2$,
representing the two sides of the critical exponent $\alpha=2$, and use
$\sigma^2\in\{0.5,2\}$.  For each setting, we compare the finite-$p$
deterministic equivalent with empirical Monte Carlo W2SG and also display
the fixed-$m$ and proportional asymptotic limits.

\begin{figure}[t]
    \centering
    \begin{subfigure}{0.48\textwidth}
        \centering
        \includegraphics[width=\linewidth]{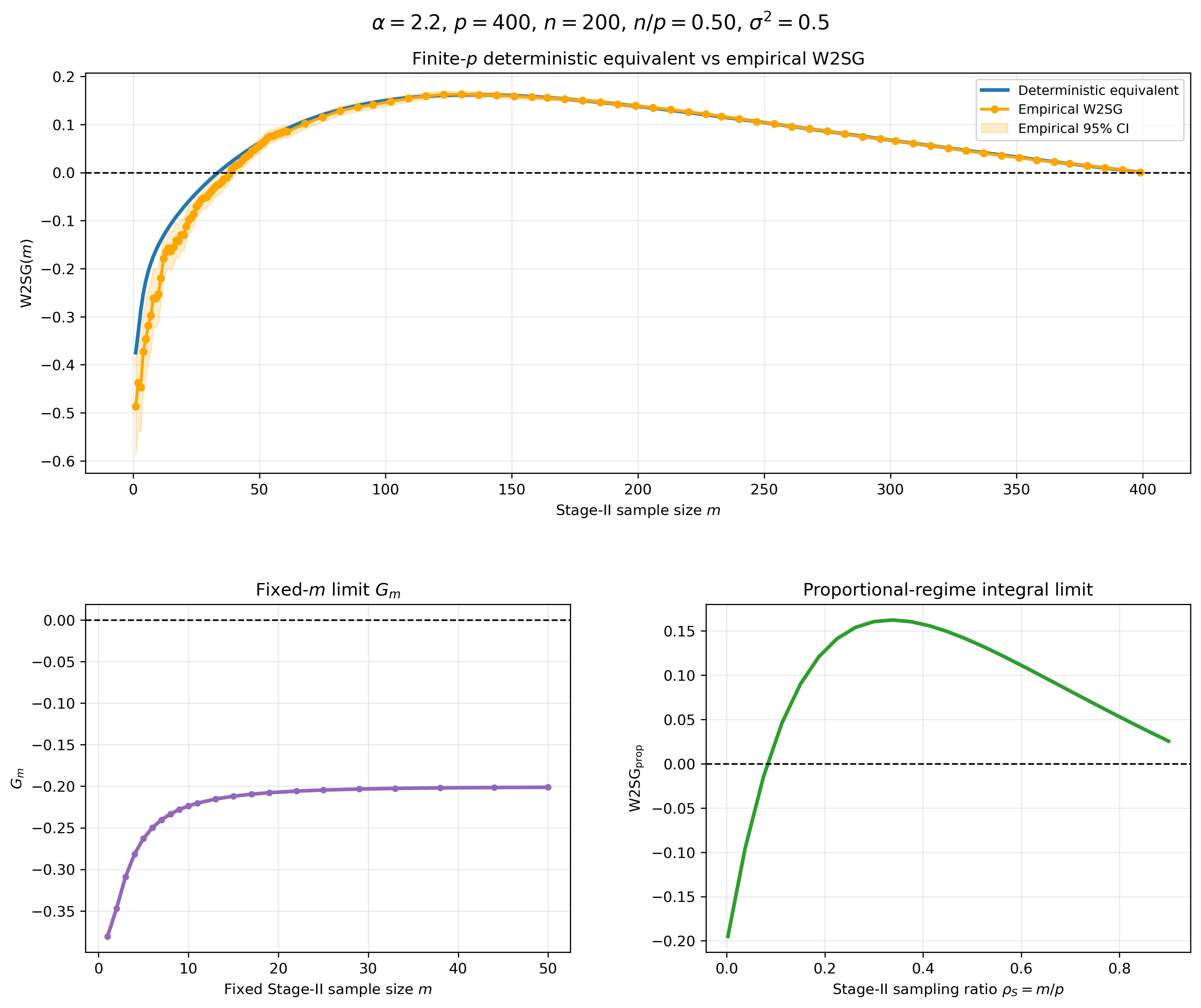}
        \caption{$\sigma^2=0.5$.}
        \label{fig:c}
    \end{subfigure}
    \hfill
    \begin{subfigure}{0.48\textwidth}
        \centering
        \includegraphics[width=\linewidth]{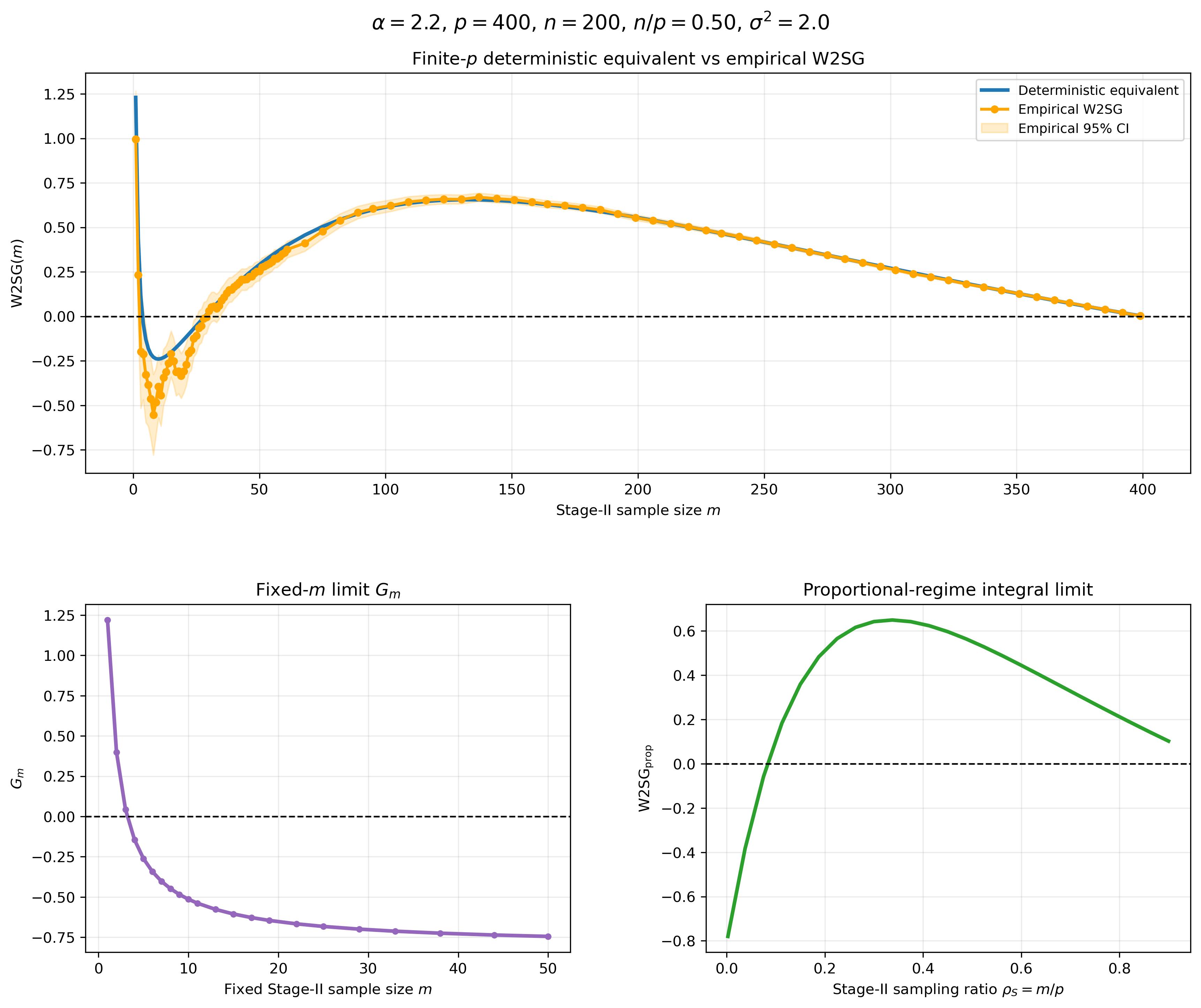}
        \caption{$\sigma^2=2$.}
        \label{fig:d}
    \end{subfigure}
    \caption{\textbf{Case I with $\alpha>2$.}
    Numerical validation at $\alpha=2.2$, $p=400$, and $n=200$ using
    $1000$ Monte Carlo trials.  The panel layout is the same as in
    Figure~\ref{fig:combined}: finite-$p$ deterministic equivalent versus
    empirical W2SG on top, fixed-$m$ limit on the bottom left, and
    proportional-regime limit on the bottom right.}
    \label{fig:combined_alpha_gt_2}
\end{figure}

\paragraph{Regime $1<\alpha<2$.}
Figure~\ref{fig:combined} considers $\alpha=1.8$ and contrasts two noise
levels.  At $\sigma^2=0.5$ (Figure~\ref{fig:a}), the finite-$p$
deterministic equivalent is negative for the smallest pseudo-sample sizes
and crosses zero at approximately $m=15.25$; the empirical curve crosses at
approximately $m=17.39$.  The fixed-sample limit exhibits the same eventual
sign change, but at the larger value $m\approx24.47$.  By contrast, the
proportional-regime limit is positive throughout the displayed interval of
fixed sampling ratios $m/p>0$.  These observations illustrate the distinction
between the two asymptotic scalings in
Theorem~\ref{thm:case-I-phase-diagram}: small fixed $m$ can remain unfavorable
even though every fixed positive proportional ratio is favorable when
$\alpha\le2$. At the higher noise level $\sigma^2=2$
(Figure~\ref{fig:b}), the gain is positive throughout all three displayed
regimes.  The finite-$p$ curve is nevertheless strongly non-monotone: it
decreases sharply at very small $m$, rises over an intermediate range, and
then decreases toward zero as $m$ approaches $p$.  The fixed-$m$ and
proportional limits display the same broad spectral trade-off at their
respective scales.  In both noise settings, the finite-$p$ deterministic
equivalent follows the Monte Carlo mean closely, including through the
small-$m$ transition and the non-monotone region.

\begin{figure*}[t]
    \centering

    \begin{subfigure}{0.49\textwidth}
        \centering
        \includegraphics[width=\linewidth]{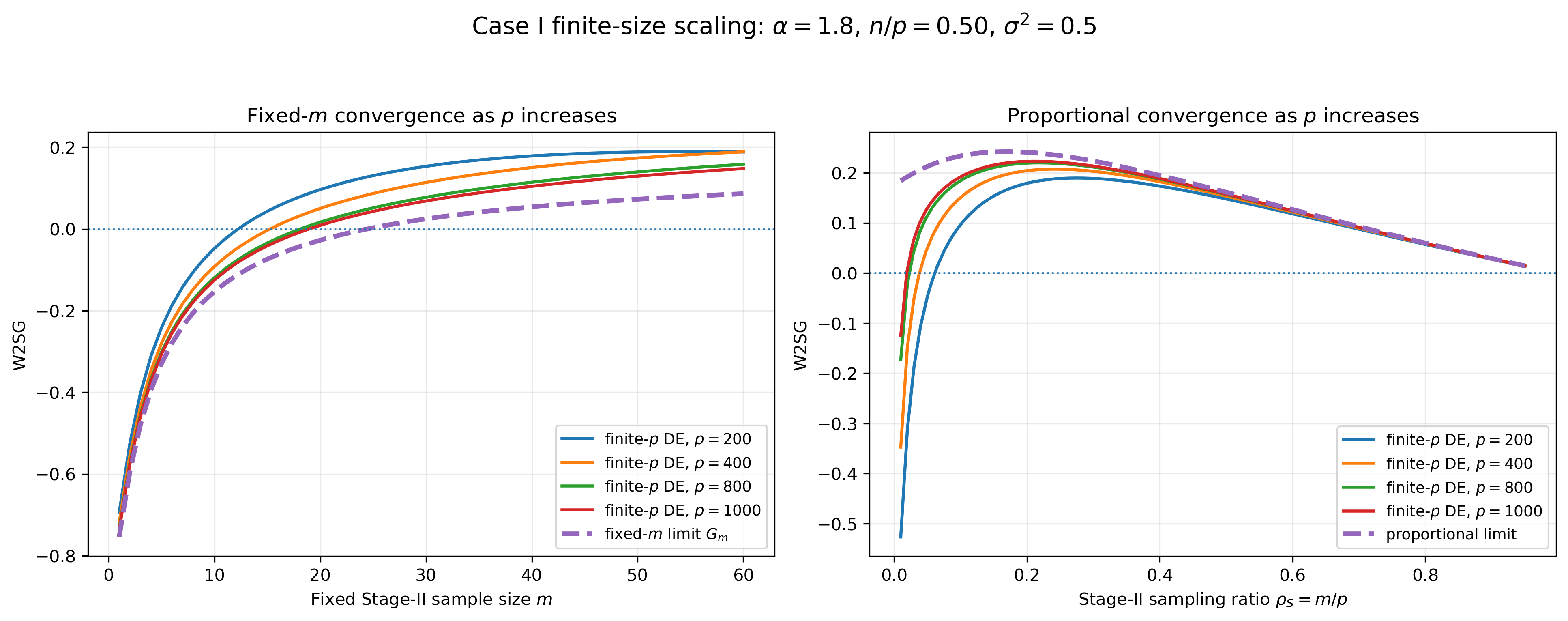}
        \caption{$\alpha=1.8$, $\sigma^2=0.5$.}
        \label{fig:caseI-p-scaling-a18-s05}
    \end{subfigure}
    \hfill
    \begin{subfigure}{0.49\textwidth}
        \centering
        \includegraphics[width=\linewidth]{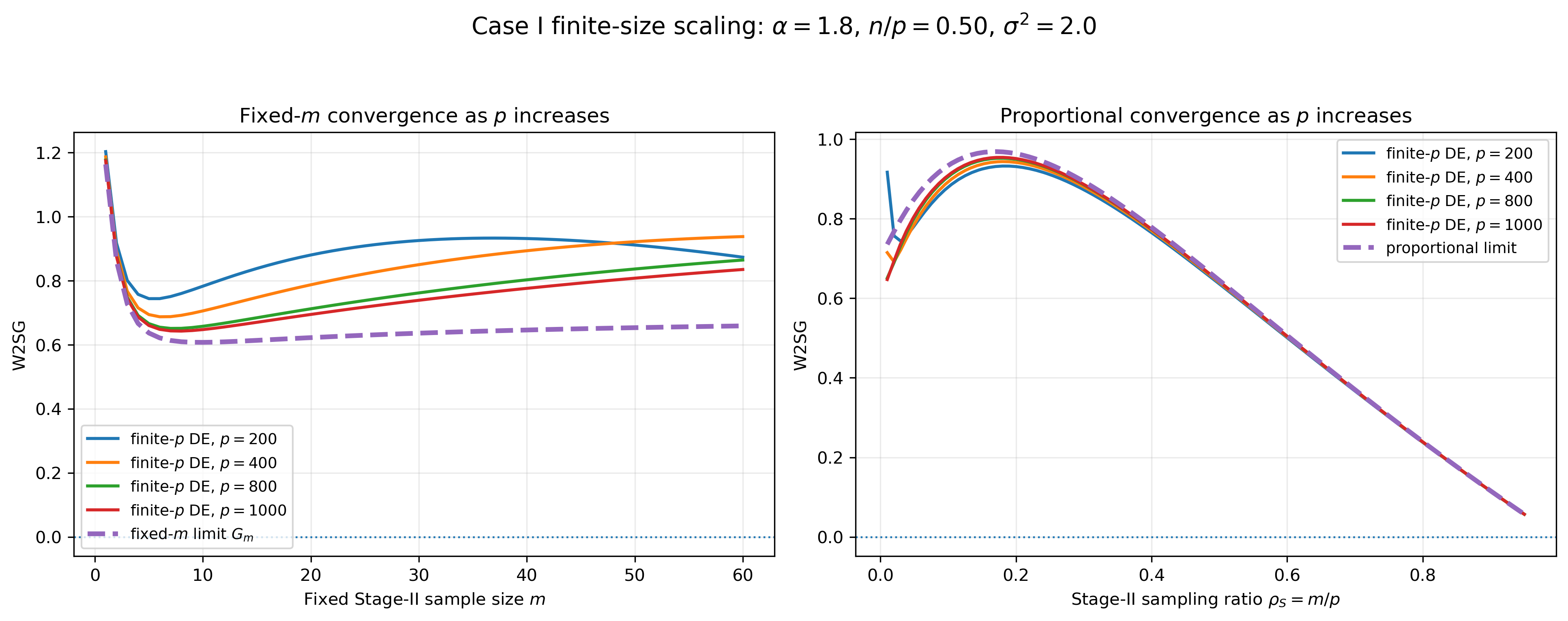}
        \caption{$\alpha=1.8$, $\sigma^2=2$.}
        \label{fig:caseI-p-scaling-a18-s2}
    \end{subfigure}

    \vspace{0.6em}

    \begin{subfigure}{0.49\textwidth}
        \centering
        \includegraphics[width=\linewidth]{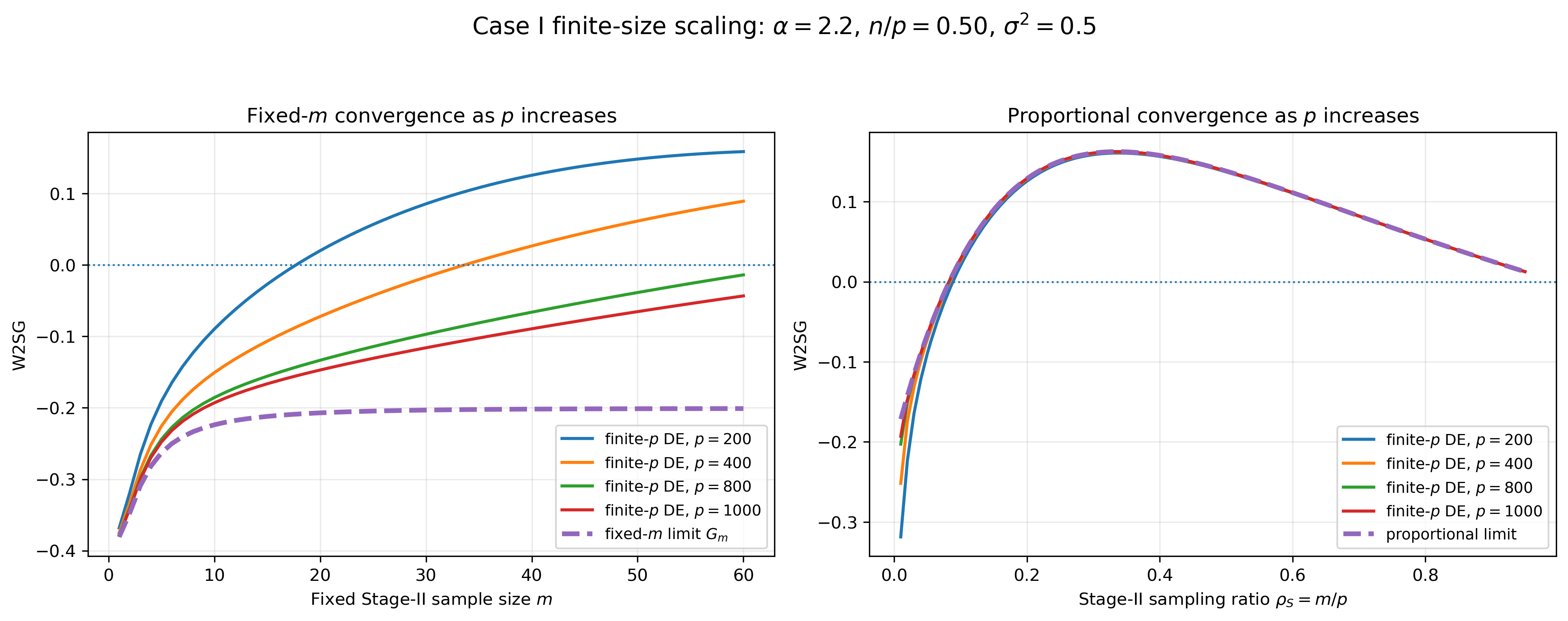}
        \caption{$\alpha=2.2$, $\sigma^2=0.5$.}
        \label{fig:caseI-p-scaling-a22-s05}
    \end{subfigure}
    \hfill
    \begin{subfigure}{0.49\textwidth}
        \centering
        \includegraphics[width=\linewidth]{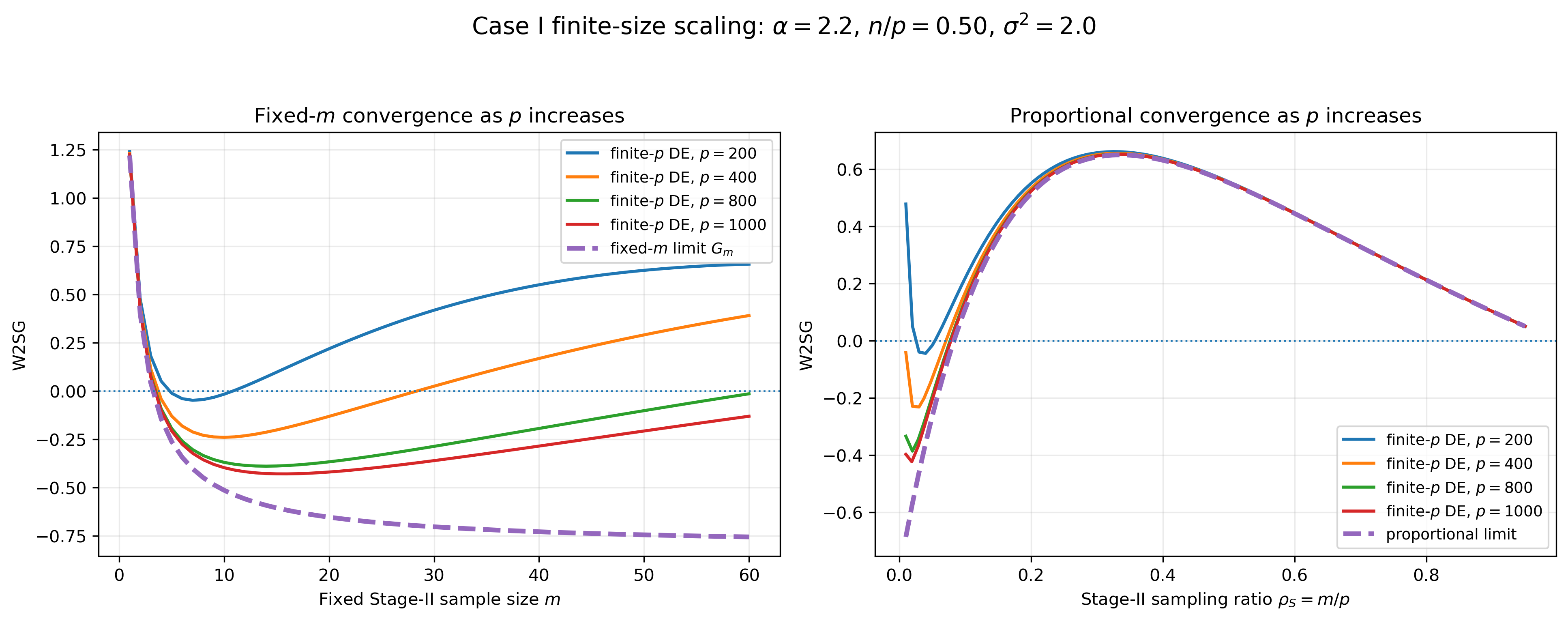}
        \caption{$\alpha=2.2$, $\sigma^2=2$.}
        \label{fig:caseI-p-scaling-a22-s2}
    \end{subfigure}

    \caption{\textbf{Finite-size convergence in Case I.}
    Finite-$p$ deterministic equivalents for
    $p\in\{200,400,800,1000\}$ with $n/p=0.5$, compared with the two
    asymptotic limits in Theorem~\ref{thm:case-I-phase-diagram}.  In each
    subfigure, the left panel keeps the absolute Stage-II sample size $m$
    fixed and compares $\mathsf{W2SG}^{\mathrm{DE}}_p(m)$ with the fixed-$m$
    limit $G_m$ (purple dashed curve).  The right panel keeps the sampling
    ratio $\rho_m=m/p$ fixed and compares the finite-$p$ curves with the
    proportional-regime limit (purple dashed curve).  The four subfigures
    use the same $(\alpha,\sigma^2)$ pairs as the main Case~I experiments.}
    \label{fig:caseI-p-scaling}
\end{figure*}

\paragraph{Regime $\alpha>2$.}
Figure~\ref{fig:combined_alpha_gt_2} repeats the comparison at $\alpha=2.2$.
At $\sigma^2=0.5$ (Figure~\ref{fig:c}), the fixed-$m$ limit $G_m$ remains
negative throughout the displayed range, whereas the proportional limit is
negative for small sampling ratios and becomes positive only after a critical
ratio $\rho_{m,\mathrm{crit}}\approx0.0840$.  At $p=400$, this predicts a
critical pseudo-sample size $p\rho_{m,\mathrm{crit}}\approx33.6$, essentially
coinciding with the finite-$p$ deterministic-equivalent crossing
$m\approx33.64$.  The empirical crossing, $m\approx38.36$, is slightly shifted
at this finite dimension but exhibits the same transition.  This is the
clearest numerical example of the qualitative disagreement between the
fixed-sample and proportional regimes for $\alpha>2$: a fixed number of
pseudo-samples remains unfavorable, while a sufficiently large
$\Theta(p)$-scale pseudo-sample size yields positive W2SG. At $\sigma^2=2$
(Figure~\ref{fig:d}), the phase structure is richer.  The finite-$p$
deterministic equivalent is positive at the very smallest $m$, becomes
negative over an intermediate small-$m$ interval, and becomes positive again,
with zero crossings at approximately $m=3.73$ and $m=28.38$.  The empirical
curve shows the same disconnected structure, with crossings near $m=2.54$
and $m=29.14$.  The fixed-$m$ limit likewise changes sign near $m=3.23$ and
is negative thereafter over the displayed range, whereas the proportional
limit again crosses from negative to positive near $\rho_m\approx0.084$.
This is the finite-dimensional manifestation of the disconnected
positive-W2SG regions discussed in
Remark~\ref{rem:case-I-disconnected}: improvement can occur at very small
$m$ and again at proportional scale, while an intermediate range remains
unfavorable.

% -----------------------------------------------------------------------------
% Case I finite-size scaling
% -----------------------------------------------------------------------------

\subsubsection{Finite-Size Scaling}
\label{subsubsec:caseI-scaling}

\begin{figure}[t]
    \centering
    \includegraphics[width=\linewidth]{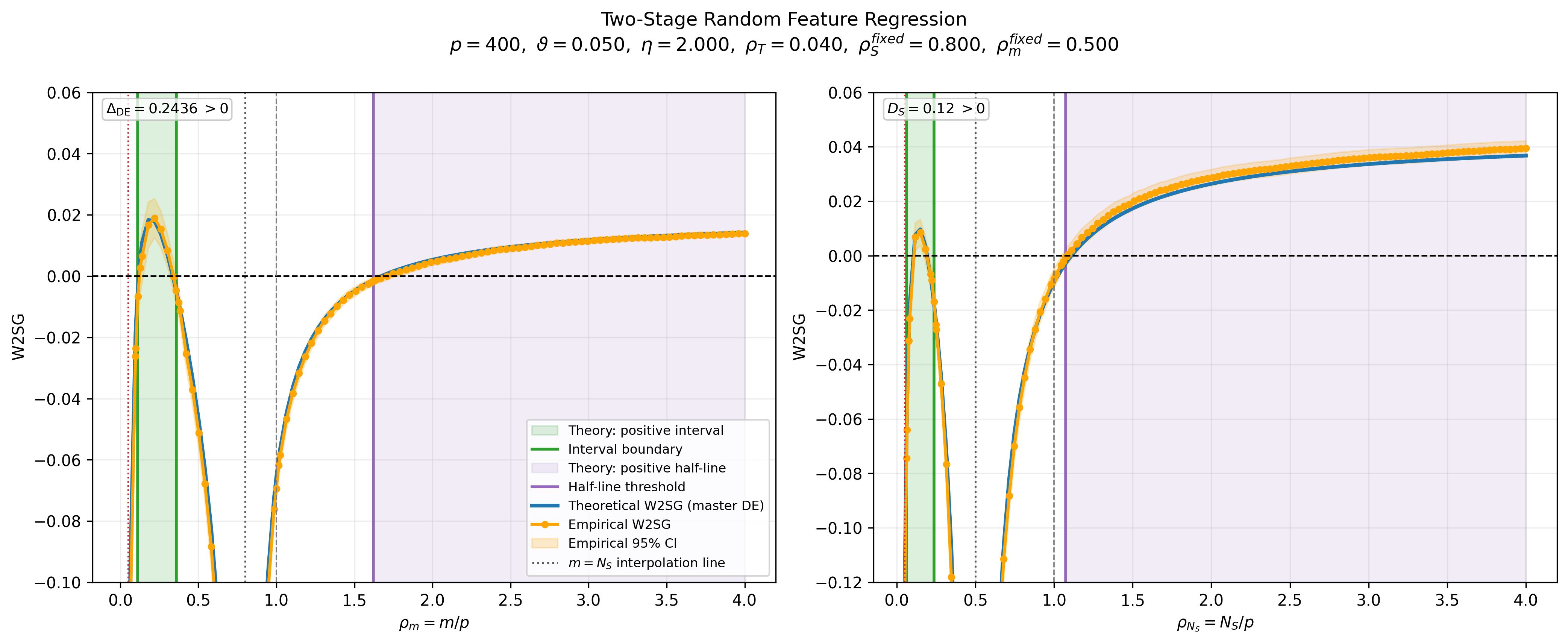}
    \caption{\textbf{Two-stage random-feature regression.} Empirical and
    deterministic-equivalent W2SG for $p=400$, $\vartheta=0.05$, $\eta=2$,
    $\rho_T=0.04$, over $200$ Monte Carlo trials. Left: W2SG versus
    $\rho_m=m/p$ at fixed $\rho_S=0.8$. Right: W2SG versus $\rho_S=N_S/p$
    at fixed $\rho_m=0.5$. Blue: finite-$p$ master deterministic
    equivalent. Orange: empirical mean with $95\%$ confidence band. Green
    lines: predicted bounded positive-W2SG interval. Purple line:
    threshold of the predicted positive half-line. Gray dotted line: the
    interpolation boundary $m=N_S$. Red dotted line: the signal threshold
    $\rho=\vartheta$. Gray dashed line: the ambient-rank boundary $\rho=1$.}
    \label{fig:caseII-rf-phase-diagram}
\end{figure}

Figure~\ref{fig:caseI-p-scaling} checks directly that the two asymptotic
formulas arise from two different large-$p$ scalings of the same finite-$p$
deterministic equivalent.  In the left panels, $m$ is held fixed as $p$
increases.  The finite-$p$ curves move toward the fixed-sample limit $G_m$;
for $\alpha=1.8$ and $\sigma^2=0.5$, this is visible in the rightward motion
of the finite-$p$ zero crossing toward the crossing of $G_m$.  For
$\alpha=2.2$ and $\sigma^2=0.5$, increasing $p$ instead drives the fixed-$m$
curve downward toward the negative limiting branch, reinforcing that fixed
$m$ does not access the positive proportional regime. In the right panels,
$m/p$ is held fixed.  Away from the small-$\rho_m$ boundary, the finite-$p$
curves collapse rapidly onto the proportional-limit prediction as $p$
increases.  For $\alpha=1.8$, the limiting proportional curve remains
positive for both noise levels, in agreement with
Theorem~\ref{thm:case-I-phase-diagram}(a)(i).  For $\alpha=2.2$, the curves
converge to the proportional phase transition predicted by
Theorem~\ref{thm:case-I-phase-diagram}(a)(ii).  The $\sigma^2=2$ setting has
the largest finite-size corrections near $\rho_m=0$, precisely where the
proportional scaling approaches the fixed-$m$/sparse boundary. Taken
together, the four scaling experiments suggest that the asymptotic behavior
of the deterministic equivalent depends critically on how the Stage-II
sample size $m$ scales with the ambient dimension $p$.  In particular,
fixed-$m$ and proportional-$m$ limits capture different high-dimensional
regimes, highlighting the need to analyze distinct scalings of $m$ when
characterizing W2SG asymptotically.

\subsection{Case II: Random-Feature Regression}
\label{subsec:caseII}

We next examine the random-feature setting of
Theorem~\ref{thm:caseII} and Corollary~\ref{cor:caseII}, where the
Stage-II learner is controlled by both the pseudo-sample size $m$ and the
student feature width $N_S$.  We first validate the predicted phase
diagram, then isolate the two bottleneck mechanisms through a
branch-separation experiment, and finally consider a negative-control
setting in which the theory predicts no admissible positive-W2SG region.

\begin{figure*}[t]
    \centering
    \includegraphics[width=\textwidth]
    {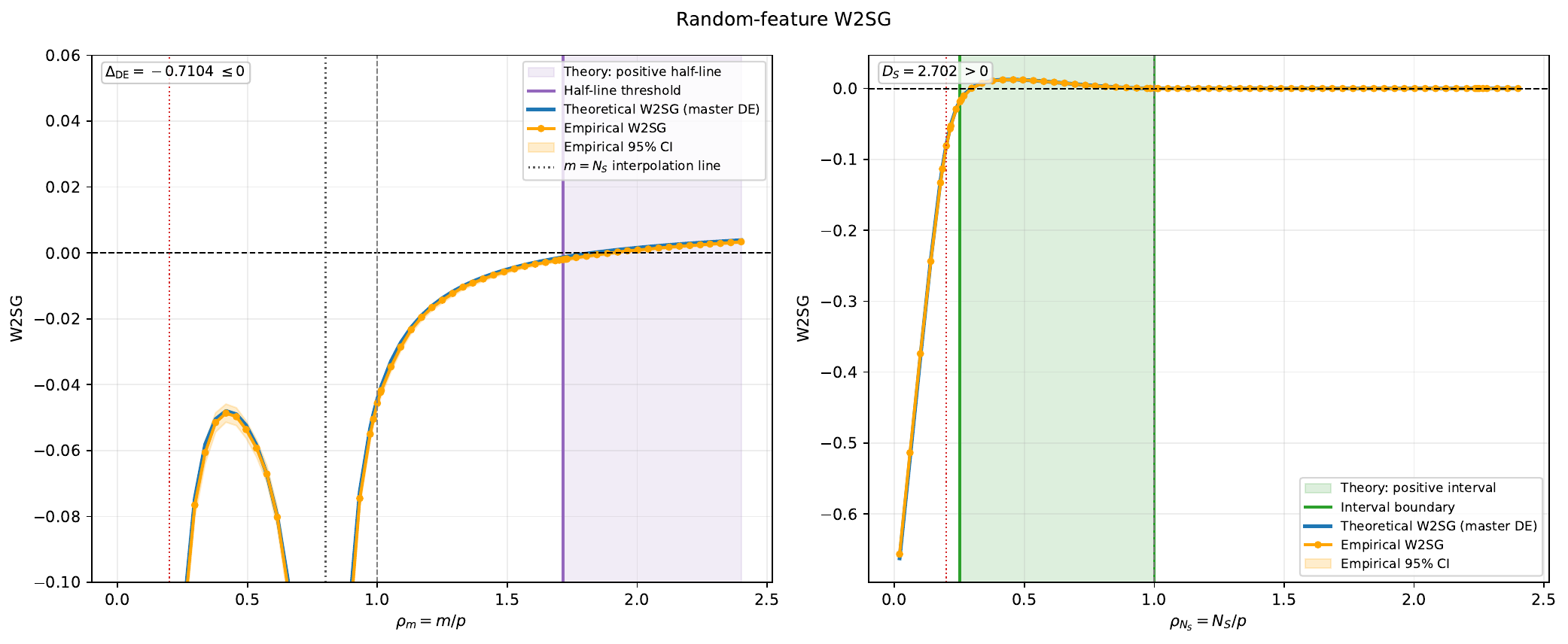}
    \caption{
    \textbf{Case II branch separation.}
    A configuration in which the two scans realize different components
    of the phase diagram. Left: the bounded positive interval in
    $\rho_m$ is absent, while the positive half-line remains. Right: a
    bounded positive interval in $\rho_S$ survives. The plotting
    conventions are the same as in
    Figure~\ref{fig:caseII-rf-phase-diagram}. Results use $500$ Monte
    Carlo trials.
    }
    \label{fig:caseII-branch-separation}
\end{figure*}

\subsubsection{Phase Diagram}
\label{subsubsec:caseII-phase}

Figure~\ref{fig:caseII-rf-phase-diagram} tests both branches of
Theorem~\ref{thm:caseII} and Corollary~\ref{cor:caseII} at $p=400$,
$\vartheta=0.05$, $\eta=2$, $\rho_T=0.04$. Sweeping $\rho_m$ at fixed
$\rho_S=0.8$ (left panel) recovers the sample-bottleneck branch: a bounded
positive-W2SG interval, exactly as in
Theorem~\ref{thm:caseII}(ii). Sweeping $\rho_S$ at fixed $\rho_m=0.5$
(right panel) recovers the feature-bottleneck branch instead: a bounded
interval in $\rho_S$ per Corollary~\ref{cor:caseII}(i), with positivity
resuming as a half-line once $\rho_S$ grows large enough to make $m$ the
active bottleneck again. In both panels, the empirical curve tracks the
finite-$p$ master deterministic equivalent closely away from the internal
interpolation boundary $m=N_S$, where the risk is singular and
finite-dimensional fluctuations are strongly amplified, exactly as
predicted by the exact interpolation classification of
Appendix~\ref{app:B}.

\subsubsection{Separation of the Two Phase Mechanisms}
\label{subsubsec:caseII-branch-separation}

Figure~\ref{fig:caseII-branch-separation} separates the two phase
mechanisms by choosing a configuration in which one bounded interval
disappears while another branch survives.  The headline setting displays
several positive branches simultaneously. To verify that these branches
are genuinely distinct rather than different representations of the same
phenomenon, we take $p=400$, $\vartheta=0.2$, $\eta=2$, and
$\rho_T=0.16$, with $\rho_S=0.8$ in the $\rho_m$-scan and
$\rho_m=2.25$ in the $\rho_S$-scan.

\begin{figure*}[t]
    \centering
    \includegraphics[width=\textwidth]
    {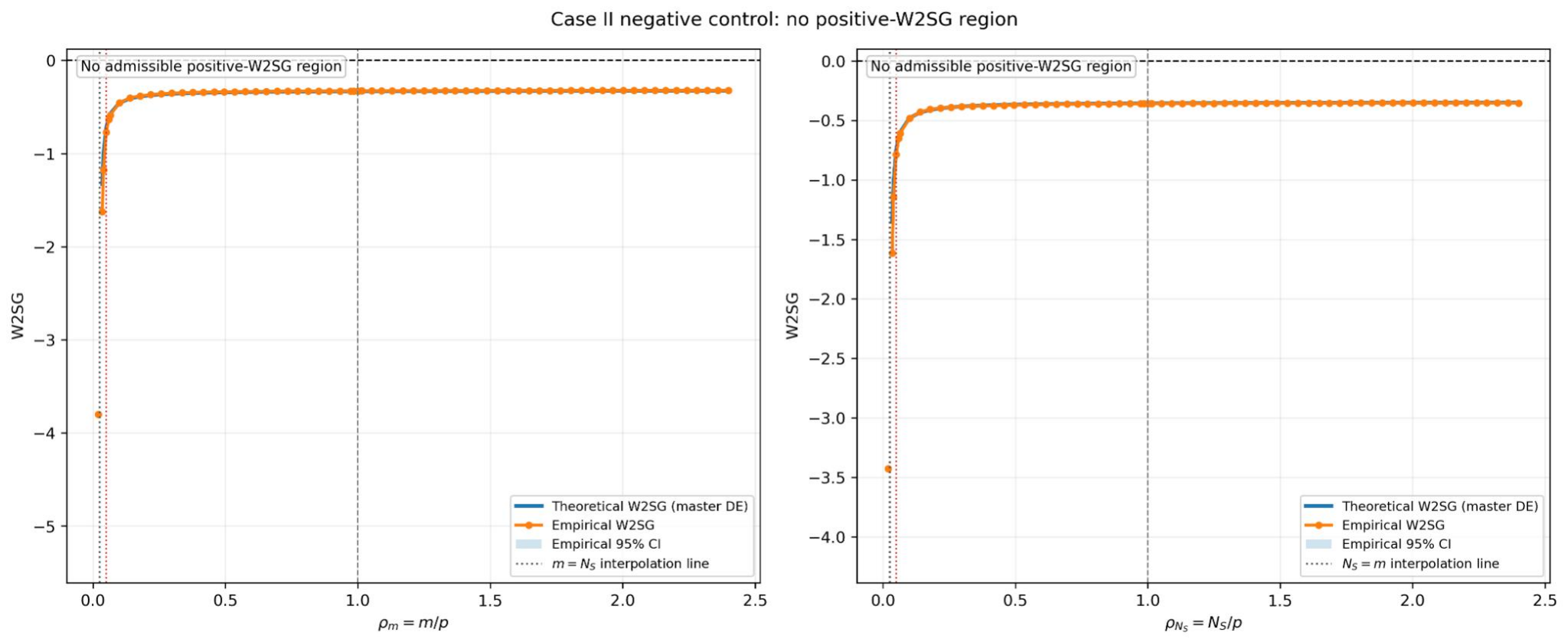}
    \caption{
    \textbf{Case II negative control.}
    A configuration in which the fixed Stage-II resource lies below the
    signal threshold, so neither scan admits an admissible positive-W2SG
    region. Left: W2SG as a function of $\rho_m=m/p$ at fixed $\rho_S$.
    Right: W2SG as a function of $\rho_S=N_S/p$ at fixed $\rho_m$.
    In both scans the finite-$p$ master deterministic equivalent remains
    strictly negative throughout the displayed range and closely follows
    the empirical Monte Carlo mean. The gray dotted line denotes the
    interpolation boundary $m=N_S$ and the red dotted line denotes the
    signal threshold $\rho=\vartheta$. Shaded bands are empirical $95\%$
    confidence intervals.
    }
    \label{fig:caseII-negative-control}
\end{figure*}

For the left scan in Figure~\ref{fig:caseII-branch-separation}, the
interval discriminant satisfies $\Delta_{\mathrm{DE}}<0$, so no bounded
positive-$\rho_m$ interval exists. Nevertheless, the half-line branch
remains admissible, with the asymptotic threshold at
$\rho_m\simeq1.714$. The finite-$p$ deterministic-equivalent and empirical
first zero crossings occur at approximately
$\rho_m^{\mathrm{DE}}\simeq1.827$ and
$\rho_m^{\mathrm{emp}}\simeq1.902$, respectively. In contrast, the right
scan has $D_S>0$ and retains a bounded interval
$\rho_S\in(0.252,1)$. The corresponding finite-$p$ and empirical
transitions, approximately $0.294$ and $0.296$, are nearly coincident.
This experiment therefore isolates the two geometric components of the
phase diagram: the bounded interval can disappear without eliminating the
positive half-line, while the complementary scan can simultaneously retain
its bounded positive branch.

\subsubsection{Negative Control}
\label{subsubsec:caseII-negative-control}

Figure~\ref{fig:caseII-negative-control} provides a negative control in
which the branch constraints exclude all positive-W2SG regions. We retain
$p=400$, $\vartheta=0.05$, $\eta=2$, and $\rho_T=0.04$, but fix both
Stage-II resources below the signal threshold,
$\rho_S^{\mathrm{fixed}}=\rho_m^{\mathrm{fixed}}=0.025<\vartheta$.

Although the raw quadratic discriminants are positive in the configuration
of Figure~\ref{fig:caseII-negative-control}, their formal roots lie outside
the admissible bottleneck branches once the signal-threshold and
active-resource constraints are enforced. Consequently, neither scan
contains a positive-W2SG interval or half-line. Consistent with this
prediction, both the deterministic equivalent and the empirical Monte
Carlo curves remain strictly below zero throughout the scanned domains.
This negative control is useful because it shows that the phase theory
does not merely fit observed sign changes: it also predicts settings in
which improvement is absent altogether.
\section{Conclusion}
\label{sec:conclusion}

We showed that weak-to-strong generalization can arise from a single, simple mechanism --- a finite resource shared between two training stages, constrained to an intermediate range --- without early stopping, without explicit regularization, and, in Case~I, without any capacity advantage for the student at all. In ridgeless linear regression with identical teacher and student hypothesis classes, the pseudo-sample size $m$ alone determines whether the student improves on its teacher, and we gave the explicit critical values of $m$, in both the fixed-sample and proportional regimes, that separate improvement from its absence. In the random-feature setting, where the student genuinely has more capacity than the teacher, the same trade-off persists but can be governed by either of two resources, pseudo-sample size or student width, and we identified an active-bottleneck principle unifying both regimes: whichever resource is scarcer controls how much teacher error is filtered out, while the other resource, once identified as non-binding, can be increased freely.

Several directions remain open. Our power-law and two-block covariance models isolate the mechanism cleanly, but we only conjecture, rather than prove, that a comparably graded non-isotropic spectrum produces the same qualitative phase transitions beyond these specific families. The boundary cases excluded from Assumption~\ref{ass:positive-fraction} in Case~II, and the regime $\alpha>1+\sqrt2$ in Case~I, are only partially resolved. More broadly, we view the active-bottleneck principle as plausibly extending beyond the two linear models studied here --- to nonlinear students, to other resources such as training time, and, ultimately, to the large-scale settings where W2SG was first observed and where recent empirical work already points to the same non-monotonic dependence on pseudo-sample size that our theory predicts.

%%%%%%%%%%%%%%%%%%%%%%%%%%%%%%%%%%%%%%%%%%%%%%%%%%%%%%%%%%%%%%%%%%%%%%%%%%%%%%%%%%%%%%%%%%%%%%%%%%%%%%%%%%%%%%%%%%%%%%%%%%%%%%%%%%%%%%%%%%%%%%%%%%%%%%%%%%%%%%%%%%%%%%%%%%%%%%%%%%%%%%%%%%%%%%%%%%%%%%%%%%%%%%%%%%%%%%%%%%%%%%%%%%%%%%%%%%%%%%%%%%%%%%%%%%%%%%%%%%%%%%%%%%%%%%%%%%%%%%%%%%%%%%%%%%%%%%%%%%%%%%%%%%%%%%%%%%%%%%%%%%%%%%%%%%%%%

\section*{AI Use Statement}
Generative AI tools were used to assist with literature discovery,
organization and language editing of the abstract, introduction, and
related-work discussion, and to provide feedback on manuscript structure
and experimental design. All citations, mathematical claims, proofs,
code, and reported results were independently verified by the authors.
The authors take responsibility for the final content.

\bibliographystyle{plainnat}
\bibliography{references}

%%%%%%%%%%%%%%%%%%%%%%%%%%%%%%%%%%%%%%%%%%%%%%%%%%%%%%%%%%%%%%%%%%%%%%%%%%%%%%%%%%%%%%%%%%%%%%%%%%%%%%%%%%%%%%%%%%%%%%%%%%%%%%%%%%%%%%%%%%%%%%%%%%%%%%%%%%%%%%%%%%%%%%%%%%%%%%%%%%%%%%%%%%%%%%%%%%%%%%%%%%%%%%%%%%%%%%%%%%%%%%%%%%%%%%%%%%%%%%%%%%%%%%%%%%%%%%%%%%%%%%%%%%%%%%%%%%%%%%%%%%%%%%%%%%%%%%%%%%%%%%%%%%%%%%%%%%%%%%%%%%%%%%%%%%%%%

\appendix
\section{Two-Stage Minimum-Norm Interpolation Theory}
\label{app:A}
This section develops the theoretical analysis of the two-stage
minimum-norm interpolation model used in
Section~\ref{sec:case-I-ridgeless}.  In this setting, the teacher and
student operate in the same ambient linear hypothesis space, but are
trained on different samples: Stage~I fits a ridgeless estimator to
noisy labeled data, while Stage~II fits a second ridgeless estimator to
fresh covariates pseudo-labeled by the Stage-I teacher.  The central
question is whether this second interpolation step can improve
population prediction risk despite receiving no information beyond the
teacher-generated labels.

We first develop the deterministic-equivalent description of the
two-stage procedure in the proportional high-dimensional regime.  After
recording the resolvent deterministic equivalence used throughout the
analysis, we derive the Stage-I prediction-risk formula and then analyze
Stage~II conditionally on the realized teacher.  This separates the
student risk into the Stage-I error carried into the second stage, the
additional error generated by Stage-II interpolation, and the cross
term produced by the interaction between the two stages.  A weighted
ridgeless pseudoinverse deterministic equivalent then allows the
remaining Stage-I randomness to be evaluated.  Under the additional
isotropic random-teacher assumption, these calculations yield a fully
deterministic expression for the Stage-I-to-Stage-II risk improvement
in terms of the population spectrum and the effective Stage-I and
Stage-II resolution parameters.

We then specialize the deterministic-equivalent formula to a
power-law covariance spectrum and study its asymptotic sign.  The
fixed-point equations for the effective resolutions are reduced to
limiting integral equations, and the trace terms in the W2SG gain are
analyzed separately before being combined into a scalar sign
criterion.  This leads to the proportional W2SG phase diagram and
reveals a qualitative change at the spectral-decay exponent
\(\alpha=2\): the regime \(1<\alpha\le2\) exhibits positive
deterministic-equivalent gain throughout the nontrivial proportional
Stage-II range, whereas for \(\alpha>2\) the gain can change sign as
the Stage-II sampling ratio varies.  The subsequent analysis
characterizes the existence and, where established, uniqueness of
these proportional transitions.

Finally, we move beyond fixed proportional Stage-II sampling.  We
first consider a sparse-growing Stage-II sample size and derive the
corresponding deterministic-equivalent limit.  We then study the
ultra-sparse regime in which the Stage-II sample count is fixed while
the ambient dimension diverges.  Because this regime lies outside the
standard proportional deterministic-equivalence theorem, its analysis
is stated under an explicit continuation assumption.  Within that
framework, we derive the fixed-\(m\) gain, identify noise- and
spectral-decay-dependent sign transitions, obtain critical and
large-\(m\) asymptotics, and give explicit conditions for W2SG and
non-W2SG behavior.  Together, these results describe how the benefit
of the second interpolation stage depends on spectral decay, label
noise, and the amount of pseudo-labeled Stage-II data across
proportional, sparse-growing, and fixed-sample regimes.

\subsection{Deterministic Equivalence}

We begin by stating the deterministic-equivalence results that form
the basis of our analysis
\cite{atanasov2026scaling,atanasov2025two}.
These results apply in proportional high-dimensional asymptotic
regimes, where the sample size and feature dimension grow at
comparable rates. Any additional assumptions required in the
subsequent derivations are stated explicitly.
\begin{lemma}[Resolvent deterministic equivalence]
\label{lem:strong-deterministic-equivalence}

We adopt the notation and definition of \emph{strong deterministic
equivalence} used by \cite[Sec.~II.E.3]{atanasov2026scaling}.
In particular, the symbol $\simeq$ below has the same test-matrix
meaning as in that reference.

Let $\{A_p\}_{p\ge1}$ be a sequence of random $p\times p$ matrices and
let $\{B_p\}_{p\ge1}$ be a sequence of deterministic $p\times p$
matrices. Following \cite[Sec.~II.E.3]{atanasov2026scaling}, we write
\[
A_p \simeq B_p
\]
if, for every deterministic test matrix $M_p$ satisfying
\[
\sup_p \|M_p\|_{\mathrm{op}}<\infty,
\]
one has
\[
\frac{\operatorname{tr}(A_pM_p)}
     {\operatorname{tr}(B_pM_p)}
\xrightarrow[p\to\infty]{\mathbb P}1.
\]
The same notion is used in
\cite[Eqs.~(18)--(20)]{atanasov2026scaling} to obtain strong
deterministic equivalents for empirical-covariance resolvents.
For the sample covariance matrix, this framework yields the following deterministic equivalents in the proportional high-dimensional regime:
\[
(\widehat{\Sigma}_p+\lambda I)^{-1}
\simeq
\frac{\tau}{\lambda}
(\Sigma_p+\tau I)^{-1}.
\tag{DE1}
\label{eq:resolvent-DE}
\]

\[
\widehat{\Sigma}_p(\widehat{\Sigma}_p+\lambda I)^{-1}
\simeq
\Sigma_p(\Sigma_p+\tau I)^{-1}.
\tag{DE2}
\label{eq:shrinkage-DE}
\]
where
\[
x_i \sim \mathcal{N}(0,\Sigma_p),
\qquad
\widehat{\Sigma}_p
=
\frac{1}{n_p}\sum_{i=1}^{n_p} x_i x_i^\top,
\qquad
p,n_p \to \infty,
\qquad
\frac{p}{n_p}\to\phi\in(0,\infty).
\]
Here $\lambda>0$ is the ridge regularization parameter and
$\tau=\tau(\lambda)$ denotes the corresponding renormalized ridge
parameter in the proportional regime. Its precise characterization
will be introduced later in the derivation. The ridgeless regime is
obtained, when appropriate, by subsequently taking the limit
$\lambda\downarrow0$. In the overparameterized regime $\phi>1$,
$\tau(\lambda)$ may remain nonzero in this limit, whereas in the
underparameterized regime $\phi<1$, one has $\tau(\lambda)\to0$.
\end{lemma}
We now introduce the two-stage ridgeless linear-regression model with
Gaussian covariates that will be used throughout the subsequent analysis.
\subsection{Two-Stage Linear Regression with Gaussian Covariates}

Let
\begin{itemize}
  \item $n$ be the number of data points in Stage~I,
  \item $m$ be the number of data points in Stage~II,
  \item $p$ be the dimension of the feature space.
\end{itemize}

We consider the linear model
\begin{equation}
y = X\beta^\star + \varepsilon,
\end{equation}
where $X\in\mathbb{R}^{n\times p}$ is the design matrix,
$\beta^\star\in\mathbb{R}^p$ is the true parameter vector, and
$\varepsilon\in\mathbb{R}^n$ is the label-noise vector.

For the time being, $\beta^\star$ is treated as deterministic; later,
when a random-teacher model is required, the corresponding distributional
assumptions will be stated explicitly.

The rows $x_i^\top$ of $X$ satisfy
$
x_i
\stackrel{\mathrm{i.i.d.}}{\sim}
\mathcal N(0,\Sigma),
$
where $\Sigma\in\mathbb R^{p\times p}$ is the population covariance
matrix, assumed positive definite, $\Sigma\succ0$. The noise components
are independent of the design and satisfy
$
\varepsilon_i
\stackrel{\mathrm{i.i.d.}}{\sim}
\mathcal N(0,\sigma^2).
$
For an independent test covariate $x\sim\mathcal N(0,\Sigma)$ and any
$\widehat\beta\in\mathbb R^p$, we define the population prediction risk
by
\begin{align}
\mathcal R(h_{\widehat\beta})
:=&
\mathbb E_x
\left[
\bigl(x^\top\widehat\beta-x^\top\beta^\star\bigr)^2
\right]
=
(\widehat\beta-\beta^\star)^\top
\Sigma
(\widehat\beta-\beta^\star)
=
\left\|
\Sigma^{1/2}
(\widehat\beta-\beta^\star)
\right\|_2^2 .
\label{eq:population-risk-stage1}
\end{align}
\subsection{Stage I: Ridgeless Regression}
\paragraph{Stage-I ridgeless regression.}
In Stage~I, we fit the ridgeless least-squares estimator
\begin{equation}
\widehat{\beta}_T
:=
X^\dagger y,
\label{eq:stage1-ridgeless}
\end{equation}
where $X^\dagger$ denotes the Moore--Penrose pseudoinverse of $X$.

For Gaussian covariates, the design matrix has maximal rank almost
surely. Therefore, when $n\ge p$,
\begin{equation}
\widehat{\beta}_T
=
(X^\top X)^{-1}X^\top y,
\end{equation}
which is the ordinary least-squares estimator. When $p>n$,
\begin{equation}
\widehat{\beta}_T
=
X^\top(XX^\top)^{-1}y,
\end{equation}
which is the unique minimum-Euclidean-norm solution of
$
X\beta=y.$
Thus, in the overparameterized regime $p>n$, the ridgeless
least-squares estimator coincides with the minimum-norm interpolator.

Define the empirical covariance matrix
$
\widehat{\Sigma}
:=
\frac{1}{n}X^\top X
\in\mathbb{R}^{p\times p}.
$
For fixed $p$, the law of large numbers gives
$\widehat{\Sigma}\to\Sigma$ as $n\to\infty$ under the usual moment
assumptions. In the proportional high-dimensional regime considered
below, however, $\widehat{\Sigma}$ does not in general converge to
$\Sigma$ in operator norm, and its effect must instead be described
through deterministic equivalents.

An equivalent covariance-form representation of
\eqref{eq:stage1-ridgeless} is
\begin{equation}
\widehat{\beta}_T
=
\widehat{\Sigma}^{\dagger}
\left(\frac{1}{n}X^\top y\right),
\label{eq:stage1-ridgeless-closed-form}
\end{equation}
where $\widehat{\Sigma}^{\dagger}$ denotes the Moore--Penrose
pseudoinverse.
A substantial body of recent work has developed both asymptotic and non-asymptotic risk characterizations for the minimum $\ell_2$-norm interpolator, corresponding to ridgeless regression in the overparameterized regime $p \geq n$ \cite{hastie2022surprises,cheng2024dimension,han2026distribution,wu2020optimal,richards2021asymptotics,loureiro2021learning}.
To analyze the Stage-I population risk in the proportional high-dimensional
regime, we use the strong deterministic-equivalence framework introduced in
Lemma~\ref{lem:strong-deterministic-equivalence}, following \cite{atanasov2026scaling}. In particular, the random
resolvents involving $\widehat{\Sigma}$ that appear in the risk calculation
are replaced by their deterministic equivalents in the bounded-test-matrix
sense. Since $\mathcal R(h_{\widehat\beta_T})$ already averages over an
independent test covariate, we further average over the randomness of
the training sample and define
\[
\overline{\mathcal R}_T
:=
\mathbb E_{X,\varepsilon}
\!\left[
\mathcal R(h_{\widehat\beta_T})
\right].
\]
This quantity corresponds to the average generalization error considered in
\cite[Sec.~III.A--B]{atanasov2026scaling}. Applying the deterministic-equivalent
resolvent identities of Lemma~\ref{lem:strong-deterministic-equivalence} yields
\[
\overline{\mathcal R}_T
\overset{\mathrm{DE}}{\simeq}
\overline{\mathcal R}_T^{\mathrm{DE}},
\]
where $\overset{\mathrm{DE}}{\simeq}$ at the scalar level denotes asymptotic equality in the
proportional high-dimensional limit, following the notation used in
\cite[Eq.~(24)]{atanasov2026scaling}. The explicit expression for
$\overline{\mathcal R}_T^{\mathrm{DE}}$ derived below is the ridgeless specialization of
the ridge-regression generalization-error formula in
\cite[Sec.~III.B, Eqs.~(24)--(25)]{atanasov2026scaling}.

\begin{theorem}[Stage-I risk deterministic equivalent]
\label{thm:stage1-risk-DE}
The averaged Stage-I population risk admits the deterministic equivalent
\begin{align}
\overline{\mathcal R}_T^{\mathrm{DE}}
=&\;
\frac{1}{1-\gamma_T}
\left[
\tau_T^{\,2}\,
\beta^{\star\top}
\Sigma(\Sigma+\tau_T I_p)^{-2}
\beta^\star
+
\sigma^2\,\gamma_T
\right],
\label{eq:stage1-risk-DE-final}
\end{align}
where
\[
\gamma_T
:=
\frac{p}{n}\,
\frac1p\sum_{i=1}^p
\frac{\lambda_i^2}{(\lambda_i+\tau_T)^2}
=
\frac1n\sum_{i=1}^p
\frac{\lambda_i^2}{(\lambda_i+\tau_T)^2},
\]
and, in the ridgeless overparameterized regime, $\tau_T>0$ is determined
by
\[
\sum_{i=1}^p
\frac{\lambda_i}{\lambda_i+\tau_T}
=
n.
\]
\end{theorem}
\begin{proof}
This follows directly from the ridgeless specialization of the
generalization-error deterministic equivalent in
\cite[Sec.~III.B, Eqs.~(24)--(25)]{atanasov2026scaling}.
\end{proof}
\subsection{Stage II Risk Derivation (Distilled Stage-II Sample)}

\paragraph{Stage-II data model.}
Draw a fresh independent dataset
\[
\{\widetilde x_i\}_{i=1}^m
\stackrel{\mathrm{i.i.d.}}{\sim}
\mathcal N(0,\Sigma),
\qquad
\widetilde X\in\mathbb R^{m\times p},
\]
and define the \emph{distilled} labels, with no additional label noise, by
$
\widetilde y := \widetilde X\widehat\beta_T.
$
\paragraph{Stage-II ridgeless estimator.}
In Stage~II, we fit the ridgeless least-squares estimator
\[
\widehat\beta_S
=
\begin{cases}
\displaystyle
\argmin_{\beta\in\mathbb R^p}
\frac1m\|\widetilde X\beta-\widetilde y\|_2^2,
& m\ge p, \\[1.2em]
\displaystyle
\argmin_{\beta\in\mathbb R^p}
\left\{
\|\beta\|_2:
\widetilde X\beta=\widetilde y
\right\},
& m<p.
\end{cases}
\]
Using
$
\widetilde y=\widetilde X\widehat\beta_T,
$
the student estimator can be written as
\[
\widehat\beta_S
=
\begin{cases}
\displaystyle
\argmin_{\beta\in\mathbb R^p}
\frac1m
\left\|
\widetilde X\beta-\widetilde X\widehat\beta_T
\right\|_2^2
=
(\widetilde X^\top\widetilde X)^{-1}
\widetilde X^\top\widetilde X\,\widehat\beta_T
=
\widehat\beta_T,
& m\ge p, \\[1.4em]
\displaystyle
\argmin_{\beta\in\mathbb R^p}
\left\{
\|\beta\|_2:
\widetilde X\beta
=
\widetilde X\widehat\beta_T
\right\}
=
\widetilde X^\top
(\widetilde X\widetilde X^\top)^{-1}
\widetilde X\,\widehat\beta_T,
& m<p.
\end{cases}
\]
When $\widetilde\Sigma$ is singular (e.g., $m<p$), $\widehat\beta_S$ is
understood as the minimum-norm interpolating solution.
Since $\widehat\beta_S$ depends on both the Stage-I training sample
$(X,\varepsilon)$ and the fresh Stage-II design
$\widetilde{X}$, we define the averaged Stage-II population risk by
\[
\overline{\mathcal R}_S
:=
\mathbb E_{X,\varepsilon,\widetilde{X}}
\!\left[
\mathcal R(h_{\widehat\beta_S})
\right].
\]

To analyze the Stage-II population risk in the proportional
high-dimensional regime, we again use the strong deterministic-equivalence
framework introduced in Lemma~\ref{lem:strong-deterministic-equivalence}, following
\cite{atanasov2026scaling}. In particular, conditional on
$\widehat\beta_T$, the random resolvents involving $\widetilde\Sigma$ are
replaced by their deterministic equivalents in the bounded-test-matrix
sense. We therefore write
\[
\overline{\mathcal R}_S
\overset{\mathrm{DE}}{\simeq}
\overline{\mathcal R}_S^{\mathrm{DE}},
\]
where $\overline{\mathcal R}_S^{\mathrm{DE}}$ denotes the deterministic equivalent
of the averaged Stage-II population risk.

We first derive the Stage-II risk while leaving the remaining expectation over the Stage-I estimator $\widehat\beta_T$ unevaluated. The resulting expression depends on moments of $\widehat\beta_T$. In the subsequent derivation, these moments are rewritten as weighted matrix functionals that are compatible with the deterministic-equivalence framework.
\begin{theorem}[Stage-II risk deterministic equivalent]
\label{thm:stage2-risk-DE}

The Stage-II deterministic-equivalent risk is
\begin{align}
\overline{\mathcal R}_S^{\mathrm{DE}}
=&
\mathbb E_{\widehat\beta_T}\Bigg[
\underbrace{
\|\widehat\beta_T-\beta^\star\|_\Sigma^2
}_{\text{Stage-I error carried over}}
+
\underbrace{
\frac{\tau_S^2}{1-\gamma_S}
\sum_{i=1}^p
\frac{
\lambda_i(\widehat\beta_{T,i})^2
}{
(\lambda_i+\tau_S)^2
}
}_{\text{Stage-II distillation error}}
-
\underbrace{
2\sum_{i=1}^p
\lambda_i
\frac{\tau_S}{\lambda_i+\tau_S}
\widehat\beta_{T,i}
\bigl(
\widehat\beta_{T,i}-\beta_i^\star
\bigr)
}_{\text{cross term}}
\Bigg],
\label{eq:stage2-risk-target}
\end{align}
where $\{\lambda_i\}_{i=1}^p$ are the eigenvalues of $\Sigma$,
$\widehat\beta_{T,i}$ and $\beta_i^\star$ denote the corresponding
coordinates in the eigenbasis of $\Sigma$, and
\[
\gamma_S
:=
\frac{p}{m}\cdot\frac1p
\sum_{i=1}^p
\frac{\lambda_i^2}{(\lambda_i+\tau_S)^2}
=
\frac1m
\sum_{i=1}^p
\frac{\lambda_i^2}{(\lambda_i+\tau_S)^2}.
\]
In the ridgeless overparameterized regime, $\tau_S>0$ is determined by
\[
\sum_{i=1}^p
\frac{\lambda_i}{\lambda_i+\tau_S}
=
m.
\]

Consequently, define the averaged Stage-I-to-Stage-II risk improvement by
\[
\mathsf{W2SG}_p
:=
\overline{\mathcal R}_T-\overline{\mathcal R}_S
\overset{\mathrm{DE}}{\simeq}
\mathsf{W2SG}_p^{\mathrm{DE}},
\]
where
\[
\mathsf{W2SG}_p^{\mathrm{DE}}
:=
\overline{\mathcal R}_T^{\mathrm{DE}}
-
\overline{\mathcal R}_S^{\mathrm{DE}}.
\]
At the deterministic-equivalent level, the Stage-I error carried over
cancels, yielding
\begin{align}
\mathsf{W2SG}_p^{\mathrm{DE}}
=&
\mathbb E_{\widehat\beta_T}\Bigg[
2\sum_{i=1}^p
\lambda_i
\frac{\tau_S}{\lambda_i+\tau_S}
\widehat\beta_{T,i}
\bigl(
\widehat\beta_{T,i}-\beta_i^\star
\bigr)
-
\frac{\tau_S^2}{1-\gamma_S}
\sum_{i=1}^p
\frac{
\lambda_i(\widehat\beta_{T,i})^2
}{
(\lambda_i+\tau_S)^2
}
\Bigg].
\label{eq:w2sg-coordinate-stage2}
\end{align}
\end{theorem}

\begin{proof}
Define the Stage-II empirical covariance matrix by
\[
\widetilde\Sigma
:=
\frac{1}{m}\widetilde X^\top\widetilde X
\in\mathbb R^{p\times p}.
\]

For the deterministic-equivalent calculation, we first introduce the
auxiliary ridge estimator
\[
\widehat\beta_{S,\lambda_S}
=
\argmin_{\beta\in\mathbb R^p}
\left\{
\frac1m\|\widetilde X\beta-\widetilde y\|_2^2
+
\lambda_S\|\beta\|_2^2
\right\}.
\]
Using $\widetilde y=\widetilde X\widehat\beta_T$, this is equivalently
\[
\widehat\beta_{S,\lambda_S}
=
\argmin_{\beta\in\mathbb R^p}
\left\{
\frac1m\|\widetilde X(\beta-\widehat\beta_T)\|_2^2
+
\lambda_S\|\beta\|_2^2
\right\}.
\]
Although our Stage-II estimator is ridgeless, we keep
$\lambda_S>0$ throughout the deterministic-equivalent calculation,
following \cite{atanasov2026scaling}, and take the limit
$\lambda_S\downarrow0$ only at the end. For $\lambda_S>0$,
\begin{align}
\widehat\beta_{S,\lambda_S}
=&
(\widetilde X^\top\widetilde X+m\lambda_S I_p)^{-1}
\widetilde X^\top\widetilde y
=
(\widetilde X^\top\widetilde X+m\lambda_S I_p)^{-1}
\widetilde X^\top\widetilde X\,\widehat\beta_T
=
(\widetilde\Sigma+\lambda_S I_p)^{-1}
\widetilde\Sigma\,\widehat\beta_T.
\end{align}
Equivalently,
\[
\widehat\beta_{S,\lambda_S}
=
\Big(
I_p-\lambda_S(\widetilde\Sigma+\lambda_S I_p)^{-1}
\Big)\widehat\beta_T.
\]
Hence, defining the Stage-II update by
$\Delta_{S,\lambda_S}:=\widehat\beta_{S,\lambda_S}-\widehat\beta_T$,
we obtain
$
\Delta_{S,\lambda_S}
=
-\lambda_S
(\widetilde\Sigma+\lambda_S I_p)^{-1}
\widehat\beta_T.
$

\medskip\noindent\textbf{Two-stage population risk decomposition.}\
Write
\[
\widehat\beta_{S,\lambda_S}-\beta^\star
=
\Delta_{S,\lambda_S}
+
(\widehat\beta_T-\beta^\star).
\]
Expanding the population risk and conditioning first on
$\widehat\beta_T$ gives
\begin{align}
\overline{\mathcal R}_{S,\lambda_S}
=&
\mathbb E_{\widehat\beta_T}\Big[
\|\widehat\beta_T-\beta^\star\|_\Sigma^2
+
\mathbb E_{\widetilde X}
\big[
\|\Delta_{S,\lambda_S}\|_\Sigma^2
\mid\widehat\beta_T
\big]
+
2\,
\mathbb E_{\widetilde X}
\big[
(\widehat\beta_T-\beta^\star)^\top
\Sigma\Delta_{S,\lambda_S}
\mid\widehat\beta_T
\big]
\Big].
\label{eq:RS-decomp}
\end{align}
\medskip\noindent\textbf{Deterministic equivalence for Stage II conditional on
$\widehat\beta_T$.}\
Condition on $\widehat\beta_T$ and view Stage~II as ridge regression
with fixed teacher $\widehat\beta_T$ and zero label noise. In the
proportional limit, applying the deterministic-equivalence framework
of Lemma~\ref{lem:strong-deterministic-equivalence} to the Stage-II resolvent terms, in the same manner as in
the Stage-I risk derivation, yields
\[
\mathbb E_{\widetilde X}
\Big[
(\widehat\beta_{S,\lambda_S}-\widehat\beta_T)^\top
\Sigma
(\widehat\beta_{S,\lambda_S}-\widehat\beta_T)
\,\Big|\,
\widehat\beta_T
\Big]
\overset{\mathrm{DE}}{\simeq}
\frac{\tau_S^2}{1-\gamma_S}
\sum_{i=1}^p
\frac{
\lambda_i(\widehat\beta_{T,i})^2
}{
(\lambda_i+\tau_S)^2
},
\]
where
\[
\gamma_S
=
\frac{p}{m}\cdot\frac1p
\sum_{i=1}^p
\frac{\lambda_i^2}{(\lambda_i+\tau_S)^2}.
\]
The renormalized ridge parameter $\tau_S$ is related to the explicit
ridge parameter $\lambda_S$ through the corresponding fixed-point
equation
\[
\tau_S
\left(
1-
\frac1m
\sum_{i=1}^p
\frac{\lambda_i}{\lambda_i+\tau_S}
\right)
=
\lambda_S.
\]
Thus, in the ridgeless overparameterized limit
$\lambda_S\downarrow0$ with $\tau_S>0$,
\[
\sum_{i=1}^p
\frac{\lambda_i}{\lambda_i+\tau_S}
=
m.
\]

For the cross term, Lemma~\ref{lem:strong-deterministic-equivalence} gives the strong deterministic equivalence
\[
\widetilde\Sigma
(\widetilde\Sigma+\lambda_S I_p)^{-1}
\overset{\mathrm{DE}}{\simeq}
\Sigma
(\Sigma+\tau_S I_p)^{-1}.
\]
Hence, conditional on $\widehat\beta_T$, the Stage-II estimator is a
ridge-regression estimator with fixed teacher $\widehat\beta_T$.
Applying the corresponding deterministic equivalent from
\cite{atanasov2026scaling} gives
\[
\mathbb E_{\widetilde X}
\big[
\widehat\beta_{S,\lambda_S}
\mid\widehat\beta_T
\big]
\overset{\mathrm{DE}}{\simeq}
\Sigma(\Sigma+\tau_S I_p)^{-1}
\widehat\beta_T.
\]
Consequently,
\[
\mathbb E_{\widetilde X}
\big[
\Delta_{S,\lambda_S}
\mid\widehat\beta_T
\big]
\overset{\mathrm{DE}}{\simeq}
-\tau_S
(\Sigma+\tau_S I_p)^{-1}
\widehat\beta_T.
\]
Therefore,
\begin{align}
&\mathbb E_{\widetilde X}
\Big[
(\widehat\beta_T-\beta^\star)^\top
\Sigma\Delta_{S,\lambda_S}
\mid\widehat\beta_T
\Big]
\overset{\mathrm{DE}}{\simeq}
-\tau_S
(\widehat\beta_T-\beta^\star)^\top
\Sigma(\Sigma+\tau_S I_p)^{-1}
\widehat\beta_T
=
-\sum_{i=1}^p
\lambda_i
\frac{\tau_S}{\lambda_i+\tau_S}
\widehat\beta_{T,i}
(\widehat\beta_{T,i}-\beta_i^\star).
\label{eq:cross-DE}
\end{align}

\medskip\noindent\textbf{Final Stage-II risk.}\
Substituting the deterministic equivalents above into
\eqref{eq:RS-decomp} and then taking the ridgeless limit
$\lambda_S\downarrow0$ yields
\begin{align}
\overline{\mathcal R}_S
&\overset{\mathrm{DE}}{\simeq}
\mathbb E_{\widehat\beta_T}\Bigg[
\underbrace{
\|\widehat\beta_T-\beta^\star\|_\Sigma^2
}_{\text{Stage-I error carried over}}
+
\underbrace{
\frac{\tau_S^2}{1-\gamma_S}
\sum_{i=1}^p
\frac{
\lambda_i(\widehat\beta_{T,i})^2
}{
(\lambda_i+\tau_S)^2
}
}_{\text{Stage-II distillation error}}
-
\underbrace{
2\sum_{i=1}^p
\lambda_i
\frac{\tau_S}{\lambda_i+\tau_S}
\widehat\beta_{T,i}
(\widehat\beta_{T,i}-\beta_i^\star)
}_{\text{cross term}}
\Bigg].
\end{align}
This is precisely the deterministic equivalent
$\overline{\mathcal R}_S^{\mathrm{DE}}$ stated in
\eqref{eq:stage2-risk-target}.
\end{proof}

\begin{lemma}[Weighted ridgeless pseudoinverse deterministic equivalent]
\label{lem:weighted-ridgeless-DE}

Let $B=B_p$ be a deterministic matrix with uniformly bounded
operator norm.

Let $\tau_T>0$ be determined by
\begin{equation}
\frac1n
\operatorname{Tr}\!\left[
\Sigma(\Sigma+\tau_T I_p)^{-1}
\right]
=
1,
\label{eq:tauT-ridgeless-fixed-point}
\end{equation}
and define
\begin{equation}
\gamma_T
:=
\frac1n
\operatorname{Tr}\!\left[
\Sigma^2(\Sigma+\tau_T I_p)^{-2}
\right].
\label{eq:gammaT-ridgeless}
\end{equation}

Then
\begin{equation}
\mathbb E_X
\operatorname{Tr}\!\left[
B\widehat\Sigma^\dagger
\right]
\overset{\mathrm{DE}}{\simeq}
\frac{1}{1-\gamma_T}
\operatorname{Tr}\!\left[
B\Sigma(\Sigma+\tau_T I_p)^{-2}
\right].
\label{eq:weighted-pseudoinverse-DE-lemma}
\end{equation}
\end{lemma}

\begin{proof}
We start from Lemma~\ref{lem:strong-deterministic-equivalence}.
For $\lambda>0$, that lemma gives
\begin{equation}
(\widehat\Sigma+\lambda I_p)^{-1}
\overset{\mathrm{DE}}{\simeq}
\frac{\tau(\lambda)}{\lambda}
(\Sigma+\tau(\lambda)I_p)^{-1},
\label{eq:base-resolvent-DE}
\end{equation}
where the renormalized ridge $\tau(\lambda)$ is determined by
\begin{equation}
\tau
\left\{
1-
\frac1n
\operatorname{Tr}\!\left[
\Sigma(\Sigma+\tau I_p)^{-1}
\right]
\right\}
=
\lambda.
\label{eq:tau-fp}
\end{equation}
\medskip
\noindent\textbf{Weighted pseudoinverse deterministic equivalent.}
Starting from the finite-ridge relation
\eqref{eq:shrinkage-DE}, differentiate with respect to
$\lambda$. Since
\[
-\frac{d}{d\lambda}
\left[
\widehat\Sigma(\widehat\Sigma+\lambda I_p)^{-1}
\right]
=
\widehat\Sigma(\widehat\Sigma+\lambda I_p)^{-2},
\]
and, by the chain rule,
\[
-\frac{d}{d\lambda}
\left[
\Sigma(\Sigma+\tau(\lambda)I_p)^{-1}
\right]
=
\frac{d\tau}{d\lambda}
\Sigma(\Sigma+\tau(\lambda)I_p)^{-2},
\]
testing against $B$ gives
\begin{align}
\mathbb E_{X}
\operatorname{Tr}\!\left[
B\widehat\Sigma
(\widehat\Sigma+\lambda I_p)^{-2}
\right]
\overset{\mathrm{DE}}{\simeq}
\frac{d\tau}{d\lambda}
\operatorname{Tr}\!\left[
B\Sigma
(\Sigma+\tau(\lambda)I_p)^{-2}
\right].
\label{eq:finite-ridge-derivative-DE}
\end{align}

It remains to evaluate $d\tau/d\lambda$. From
\eqref{eq:tau-fp},
\[
\lambda
=
\tau
-
\frac{\tau}{n}
\operatorname{Tr}\!\left[
\Sigma(\Sigma+\tau I_p)^{-1}
\right].
\]
Differentiating with respect to $\tau$ gives
\[
\frac{d\lambda}{d\tau}
=
1-
\frac1n
\operatorname{Tr}\!\left[
\Sigma(\Sigma+\tau I_p)^{-1}
\right]
+
\frac{\tau}{n}
\operatorname{Tr}\!\left[
\Sigma(\Sigma+\tau I_p)^{-2}
\right].
\]
Using
\[
\Sigma(\Sigma+\tau I_p)^{-1}
-
\tau\Sigma(\Sigma+\tau I_p)^{-2}
=
\Sigma^2(\Sigma+\tau I_p)^{-2},
\]
this simplifies to
\begin{equation}
\frac{d\lambda}{d\tau}
=
1-
\frac1n
\operatorname{Tr}\!\left[
\Sigma^2(\Sigma+\tau I_p)^{-2}
\right].
\label{eq:dlambda-dtau}
\end{equation}
Hence, as $\lambda\downarrow0$ and
$\tau(\lambda)\to\tau_T$,
\[
\left.
\frac{d\tau}{d\lambda}
\right|_{\lambda\downarrow0}
=
\frac{1}{1-\gamma_T},
\]
where
\[
\gamma_T
=
\frac1n
\operatorname{Tr}\!\left[
\Sigma^2(\Sigma+\tau_T I_p)^{-2}
\right].
\]

Finally, by spectral calculus,
\[
\lim_{\lambda\downarrow0}
\widehat\Sigma
(\widehat\Sigma+\lambda I_p)^{-2}
=
\widehat\Sigma^\dagger.
\]
Taking $\lambda\downarrow0$ in
\eqref{eq:finite-ridge-derivative-DE} therefore yields
\[
\mathbb E_{X}
\operatorname{Tr}\!\left[
B\widehat\Sigma^\dagger
\right]
\overset{\mathrm{DE}}{\simeq}
\frac{1}{1-\gamma_T}
\operatorname{Tr}\!\left[
B\Sigma(\Sigma+\tau_T I_p)^{-2}
\right],
\]
which proves the lemma.
\end{proof}

We next derive a fully deterministic expression for the Stage-I-to-Stage-II risk improvement. A related two-stage ridgeless-regression setting is studied in \cite[Definition~3]{ildiz2025high}, where the Stage-I surrogate sample and Stage-II target sample may have different population covariances and noise levels. Their analysis therefore accommodates a more general distribution-shift setting and also provides a generalization bound under additional boundedness assumptions. Some of those assumptions, including boundedness conditions involving the inverse covariance, are not imposed in the present analysis. Accordingly, we work instead with deterministic-equivalent formulas under the spectral assumptions stated above and do not claim a corresponding non-asymptotic generalization bound.

\begin{theorem}[Fully deterministic Stage-I-to-Stage-II risk improvement]
\label{thm:fully-deterministic-risk-improvement}

Assume, in addition, that the teacher is random and independent of the design and
label noise and is isotropic:
\[
\mathbb E[\beta^\star]=0,
\qquad
\mathbb E[\beta^\star\beta^{\star\top}]=I_p.
\]
Then the teacher-averaged Stage-I-to-Stage-II risk improvement
admits the deterministic equivalent
\[
\mathbb E_{\beta^\star}
\!\left[
\mathsf{W2SG}_p
\right]
\overset{\mathrm{DE}}{\simeq}
\mathsf{W2SG}_p^{\mathrm{DE}},
\]
where
\begin{align}
\mathsf{W2SG}_p^{\mathrm{DE}}
={}&
-
\frac{\tau_S^2}{1-\gamma_S}
\Tr
\left[
\Sigma^2
(\Sigma+\tau_SI_p)^{-2}
(\Sigma+\tau_TI_p)^{-1}
\right]
\notag\\
+&
\frac{2\sigma^2\tau_S}{n(1-\gamma_T)}
\Tr
\left[
\Sigma^2
(\Sigma+\tau_SI_p)^{-1}
(\Sigma+\tau_TI_p)^{-2}
\right]
\notag\\
-&
\frac{\sigma^2\tau_S^2}
{n(1-\gamma_T)(1-\gamma_S)}
\Tr
\left[
\Sigma^2
(\Sigma+\tau_SI_p)^{-2}
(\Sigma+\tau_TI_p)^{-2}
\right].
\label{eq:w2sg-fully-deterministic}
\end{align}

Since all matrices in
\eqref{eq:w2sg-fully-deterministic}
are functions of $\Sigma$, they are simultaneously diagonalizable.
Thus, the same deterministic equivalent can be written in the eigenbasis
of $\Sigma$ as
\begin{align}
\mathsf{W2SG}_p^{\mathrm{DE}}
={}&
-
\frac{\tau_S^2}{1-\gamma_S}
\sum_{i=1}^p
\frac{\lambda_i^2}
{(\lambda_i+\tau_S)^2(\lambda_i+\tau_T)}
\notag\\
+&
\frac{2\sigma^2\tau_S}{n(1-\gamma_T)}
\sum_{i=1}^p
\frac{\lambda_i^2}
{(\lambda_i+\tau_S)(\lambda_i+\tau_T)^2}
\notag\\
-&
\frac{\sigma^2\tau_S^2}
{n(1-\gamma_T)(1-\gamma_S)}
\sum_{i=1}^p
\frac{\lambda_i^2}
{(\lambda_i+\tau_S)^2(\lambda_i+\tau_T)^2}.
\label{eq:w2sg-eigenvalue-final}
\end{align}
\end{theorem}
\begin{proof}
To complete the derivation of the Stage-I-to-Stage-II risk improvement,
we must average the Stage-II deterministic-equivalent expression over
the remaining Stage-I randomness. Direct coordinate-wise application
of deterministic equivalence to the random coordinates
$\widehat\beta_{T,i}$ is not justified by Lemma~\ref{lem:strong-deterministic-equivalence}, which is formulated
in terms of deterministic test matrices and trace functionals.
We therefore rewrite the coordinate expression in terms of weighted
quadratic forms that are compatible with the deterministic-equivalence
framework of Lemma~\ref{lem:strong-deterministic-equivalence}.

Since all matrices appearing in the Stage-II deterministic equivalent
are functions of the population covariance $\Sigma$, define
\begin{equation}
A_S
:=
\tau_S
\Sigma
(\Sigma+\tau_S I_p)^{-1},
\label{eq:AS-definition}
\end{equation}
and
\begin{equation}
D_S
:=
2\tau_S
\Sigma
(\Sigma+\tau_S I_p)^{-1}
-
\frac{\tau_S^2}{1-\gamma_S}
\Sigma
(\Sigma+\tau_S I_p)^{-2}.
\label{eq:DS-definition}
\end{equation}
Indeed, in the eigenbasis of $\Sigma$,
\[
[A_S]_{ii}
=
\lambda_i
\frac{\tau_S}{\lambda_i+\tau_S},
\]
whereas
\[
[D_S]_{ii}
=
2\lambda_i
\frac{\tau_S}{\lambda_i+\tau_S}
-
\frac{\tau_S^2}{1-\gamma_S}
\frac{\lambda_i}{(\lambda_i+\tau_S)^2}.
\]
Therefore, the Stage-II expression
\eqref{eq:w2sg-coordinate-stage2} can be written as
\begin{align}
\mathbb E_{\beta^\star}
\!\left[
\mathsf{W2SG}_p
\right]
\overset{\mathrm{DE}}{\simeq}
{}&
\mathbb E
\left[
\widehat\beta_T^\top
D_S
\widehat\beta_T
\right]
-
2
\mathbb E
\left[
\beta^{\star\top}
A_S
\widehat\beta_T
\right].
\label{eq:w2sg-stage1-weighted-forms}
\end{align}

Recall that
$
\widehat\Sigma
=
\frac{1}{n}X^\top X,
P_T=\widehat\Sigma^\dagger\widehat\Sigma
$
and the Stage-I labels satisfy
\begin{equation}
y
=
X\beta^\star+\varepsilon,
\qquad
\mathbb E[\varepsilon]=0,
\qquad
\mathbb E[\varepsilon\varepsilon^\top]
=
\sigma^2I_n.
\label{eq:stage1-model}
\end{equation}
The Stage-I ridgeless estimator can be written as
\begin{align}
\widehat\beta_T
=
\widehat\Sigma^\dagger
\frac{X^\top y}{n}
=
\widehat\Sigma^\dagger
\widehat\Sigma
\beta^\star
+
\widehat\Sigma^\dagger
\frac{X^\top\varepsilon}{n}.
\label{eq:stage1-estimator-decomposition}
\end{align}
Since $\widehat\Sigma$ is symmetric positive semidefinite, $P_T$ is an
orthogonal projector and hence
$P_T^\top=P_T,
P_T^2=P_T$. Thus,
\begin{equation}
\widehat\beta_T
=
P_T\beta^\star
+
\widehat\Sigma^\dagger
\frac{X^\top\varepsilon}{n}.
\label{eq:stage1-estimator-PT}
\end{equation}

We first evaluate
$
\mathbb E
\left[
\widehat\beta_T^\top D_S\widehat\beta_T
\right].
$
Conditioning on $X$ and using
\eqref{eq:stage1-estimator-PT}, we have
\begin{align}
\widehat\beta_T\widehat\beta_T^\top
=
P_T\beta^\star\beta^{\star\top}P_T
+
P_T\beta^\star
\varepsilon^\top
\frac{X}{n}
\widehat\Sigma^\dagger
+
\widehat\Sigma^\dagger
\frac{X^\top\varepsilon}{n}
\beta^{\star\top}P_T
+
\widehat\Sigma^\dagger
\frac{X^\top
\varepsilon\varepsilon^\top
X}{n^2}
\widehat\Sigma^\dagger .
\label{eq:stage1-second-moment-expansion}
\end{align}
Because $\beta^\star$ is independent of the design and noise,
$\mathbb E[\beta^\star\beta^{\star\top}]=I_p$, and
$\mathbb E[\varepsilon]=0$, the two mixed terms vanish. Therefore,
\begin{align}
\mathbb E_{\beta^\star,\varepsilon}
\left[
\widehat\beta_T\widehat\beta_T^\top
\mid X
\right]
=
P_T
+
\frac{\sigma^2}{n^2}
\widehat\Sigma^\dagger
X^\top X
\widehat\Sigma^\dagger
=
P_T
+
\frac{\sigma^2}{n}
\widehat\Sigma^\dagger
\widehat\Sigma
\widehat\Sigma^\dagger
=
P_T
+
\frac{\sigma^2}{n}
\widehat\Sigma^\dagger,
\label{eq:stage1-second-moment-final}
\end{align}
where the last equality uses the Moore--Penrose identity
$
\widehat\Sigma^\dagger
\widehat\Sigma
\widehat\Sigma^\dagger
=
\widehat\Sigma^\dagger.
$

Consequently,
\begin{align}
\mathbb E
\left[
\widehat\beta_T^\top
D_S
\widehat\beta_T
\right]
=&
\mathbb E_{X}
\Tr
\left[
D_S
\mathbb E_{\beta^\star,\varepsilon}
\left[
\widehat\beta_T\widehat\beta_T^\top
\mid X
\right]
\right]
=
\mathbb E_{X}\Tr(D_SP_T)
+
\frac{\sigma^2}{n}
\mathbb E_{X}
\Tr
\left[
D_S\widehat\Sigma^\dagger
\right].
\label{eq:stage1-weighted-quadratic}
\end{align}

From \eqref{eq:stage1-estimator-PT},
\begin{align}
\mathbb E
\left[
\beta^{\star\top}
A_S
\widehat\beta_T
\right]
={}&
\mathbb E
\left[
\beta^{\star\top}
A_SP_T
\beta^\star
\right]
+
\mathbb E
\left[
\beta^{\star\top}
A_S
\widehat\Sigma^\dagger
\frac{X^\top\varepsilon}{n}
\right].
\end{align}
The second term vanishes because the label noise is independent and
zero mean. Hence,
\begin{align}
\mathbb E
\left[
\beta^{\star\top}
A_S
\widehat\beta_T
\right]
=&
\mathbb E_{X}
\mathbb E_{\beta^\star}
\left[
\beta^{\star\top}
A_SP_T
\beta^\star
\mid X
\right]
=
\mathbb E_{X}
\Tr
\left[
A_SP_T
\mathbb E[
\beta^\star\beta^{\star\top}]
\right]
=
\mathbb E_{X}
\Tr(A_SP_T).
\label{eq:stage1-weighted-cross}
\end{align}

Substituting
\eqref{eq:stage1-weighted-quadratic}
and
\eqref{eq:stage1-weighted-cross}
into
\eqref{eq:w2sg-stage1-weighted-forms} gives
\begin{align}
\mathbb E_{\beta^\star}
\!\left[
\mathsf{W2SG}_p
\right]
\overset{\mathrm{DE}}{\simeq}
{}&
\mathbb E_{X}
\Tr(D_SP_T)
+
\frac{\sigma^2}{n}
\mathbb E_{X}
\Tr
\left[
D_S\widehat\Sigma^\dagger
\right]
-
2
\mathbb E_{X}
\Tr(A_SP_T).
\label{eq:adapted-stage1-weighted-reduction}
\end{align}

We now evaluate the three remaining Stage-I expectations in
\eqref{eq:adapted-stage1-weighted-reduction}:
$
\mathbb E_{X}\Tr(D_SP_T),
\mathbb E_{X}
\Tr(D_S\widehat\Sigma^\dagger),
\mathbb E_{X}\Tr(A_SP_T).
$
At this stage, $A_S$ and $D_S$ are deterministic matrices because they
depend only on $\Sigma$ and on the Stage-II deterministic-equivalent
parameters $\tau_S$ and $\gamma_S$. Thus, all remaining randomness is
contained in the Stage-I design $X$ through
$
\widehat\Sigma
=
\frac1nX^\top X,
P_T
=
\widehat\Sigma^\dagger\widehat\Sigma.
$

Recall that $\tau_T>0$ is determined in the ridgeless
overparameterized regime by
\begin{equation}
\frac1n
\Tr
\left[
\Sigma(\Sigma+\tau_TI_p)^{-1}
\right]
=
1,
\label{eq:tauT-fixed-point}
\end{equation}
or equivalently
$
\sum_{i=1}^p
\frac{\lambda_i}{\lambda_i+\tau_T}
=
n.
$
Also recall
\begin{equation}
\gamma_T
:=
\frac1n
\Tr
\left[
\Sigma^2(\Sigma+\tau_TI_p)^{-2}
\right]
=
\frac1n
\sum_{i=1}^p
\left(
\frac{\lambda_i}{\lambda_i+\tau_T}
\right)^2.
\label{eq:gammaT-definition}
\end{equation}
For a fixed proportional Stage-II aspect ratio $\phi_S>1$,
$1-\gamma_S$ remains bounded away from zero. Moreover,
$
0
\preceq
\tau_S\Sigma(\Sigma+\tau_S I_p)^{-1}
\preceq
\Sigma,
$
and
$
0
\preceq
\tau_S^2\Sigma(\Sigma+\tau_S I_p)^{-2}
\preceq
\Sigma.
$
Hence, under the assumed uniform bound on
$\|\Sigma\|_{\mathrm{op}}$, both $A_S$ and $D_S$ have uniformly
bounded operator norm.

Next, observe that
\[
P_T
=
\widehat\Sigma^\dagger\widehat\Sigma
=
\lim_{\lambda\downarrow0}
\widehat\Sigma
(\widehat\Sigma+\lambda I_p)^{-1}.
\]
Taking the ridgeless limit in the shrinkage deterministic equivalent
\eqref{eq:shrinkage-DE} therefore gives, for every deterministic matrix
$B$ with uniformly bounded operator norm,
\begin{equation}
\mathbb E_X
\Tr(BP_T)
\overset{\mathrm{DE}}{\simeq}
\Tr\!\left[
B\Sigma(\Sigma+\tau_T I_p)^{-1}
\right].
\label{eq:weighted-projector-DE}
\end{equation}

Applying \eqref{eq:weighted-projector-DE} with $B=D_S$ gives
\begin{equation}
\mathbb E_X
\Tr(D_SP_T)
\overset{\mathrm{DE}}{\simeq}
\Tr\!\left[
D_S
\Sigma
(\Sigma+\tau_T I_p)^{-1}
\right].
\label{eq:DS-PT-DE}
\end{equation}
Similarly, taking $B=A_S$ gives
\begin{equation}
\mathbb E_X
\Tr(A_SP_T)
\overset{\mathrm{DE}}{\simeq}
\Tr\!\left[
A_S
\Sigma
(\Sigma+\tau_T I_p)^{-1}
\right].
\label{eq:AS-PT-DE}
\end{equation}
Finally, applying Lemma~\ref{lem:weighted-ridgeless-DE}
with $B=D_S$ gives
\begin{equation}
\mathbb E_{X}
\Tr
\left[
D_S\widehat\Sigma^\dagger
\right]
\overset{\mathrm{DE}}{\simeq}
\frac{1}{1-\gamma_T}
\Tr
\left[
D_S
\Sigma
(\Sigma+\tau_TI_p)^{-2}
\right].
\label{eq:DS-pseudoinverse-DE}
\end{equation}
Therefore,
\begin{equation}
\frac{\sigma^2}{n}
\mathbb E_{X}
\Tr
\left[
D_S\widehat\Sigma^\dagger
\right]
\overset{\mathrm{DE}}{\simeq}
\frac{\sigma^2}{n(1-\gamma_T)}
\Tr
\left[
D_S
\Sigma
(\Sigma+\tau_TI_p)^{-2}
\right].
\label{eq:DS-noise-DE}
\end{equation}

Substituting
\eqref{eq:DS-PT-DE},
\eqref{eq:DS-noise-DE},
and
\eqref{eq:AS-PT-DE}
into
\eqref{eq:adapted-stage1-weighted-reduction} yields
\begin{align}
\mathbb E_{\beta^\star}
\!\left[
\mathsf{W2SG}_p
\right]
\overset{\mathrm{DE}}{\simeq}
{}&
\Tr
\left[
D_S
\Sigma
(\Sigma+\tau_TI_p)^{-1}
\right]
+
\frac{\sigma^2}{n(1-\gamma_T)}
\Tr
\left[
D_S
\Sigma
(\Sigma+\tau_TI_p)^{-2}
\right]
-
2
\Tr
\left[
A_S
\Sigma
(\Sigma+\tau_TI_p)^{-1}
\right].
\label{eq:w2sg-after-stage1-DE}
\end{align}

Recall that
\[
A_S
=
\tau_S
\Sigma
(\Sigma+\tau_SI_p)^{-1},
\]
and
\[
D_S
=
2\tau_S
\Sigma
(\Sigma+\tau_SI_p)^{-1}
-
\frac{\tau_S^2}{1-\gamma_S}
\Sigma
(\Sigma+\tau_SI_p)^{-2}.
\]
Therefore,
\begin{equation}
D_S-2A_S
=
-
\frac{\tau_S^2}{1-\gamma_S}
\Sigma
(\Sigma+\tau_SI_p)^{-2}.
\label{eq:DS-minus-2AS}
\end{equation}

The first and third terms of
\eqref{eq:w2sg-after-stage1-DE}
can therefore be combined as
\begin{align}
&
\Tr
\left[
D_S\Sigma(\Sigma+\tau_TI_p)^{-1}
\right]
-
2
\Tr
\left[
A_S\Sigma(\Sigma+\tau_TI_p)^{-1}
\right]
=
\Tr
\left[
(D_S-2A_S)
\Sigma
(\Sigma+\tau_TI_p)^{-1}
\right]
\notag\\
=&
-
\frac{\tau_S^2}{1-\gamma_S}
\Tr
\left[
\Sigma^2
(\Sigma+\tau_SI_p)^{-2}
(\Sigma+\tau_TI_p)^{-1}
\right].
\label{eq:w2sg-signal-DE}
\end{align}

For the noise contribution, substituting the explicit form of $D_S$
into \eqref{eq:DS-noise-DE} gives
\begin{align}
&
\frac{\sigma^2}{n(1-\gamma_T)}
\Tr
\left[
D_S
\Sigma
(\Sigma+\tau_TI_p)^{-2}
\right]
\notag\\
=&
\frac{2\sigma^2\tau_S}{n(1-\gamma_T)}
\Tr
\left[
\Sigma^2
(\Sigma+\tau_SI_p)^{-1}
(\Sigma+\tau_TI_p)^{-2}
\right]
-
\frac{\sigma^2\tau_S^2}
{n(1-\gamma_T)(1-\gamma_S)}
\Tr
\left[
\Sigma^2
(\Sigma+\tau_SI_p)^{-2}
(\Sigma+\tau_TI_p)^{-2}
\right].
\label{eq:w2sg-noise-DE}
\end{align}

Combining
\eqref{eq:w2sg-signal-DE}
and
\eqref{eq:w2sg-noise-DE}
gives
\begin{align}
\mathbb E_{\beta^\star}
\!\left[
\mathsf{W2SG}_p
\right]
\overset{\mathrm{DE}}{\simeq}
{}&
-
\frac{\tau_S^2}{1-\gamma_S}
\Tr
\left[
\Sigma^2
(\Sigma+\tau_SI_p)^{-2}
(\Sigma+\tau_TI_p)^{-1}
\right]
+
\frac{2\sigma^2\tau_S}{n(1-\gamma_T)}
\Tr
\left[
\Sigma^2
(\Sigma+\tau_SI_p)^{-1}
(\Sigma+\tau_TI_p)^{-2}
\right]
\notag\\
-&
\frac{\sigma^2\tau_S^2}
{n(1-\gamma_T)(1-\gamma_S)}
\Tr
\left[
\Sigma^2
(\Sigma+\tau_SI_p)^{-2}
(\Sigma+\tau_TI_p)^{-2}
\right].
\end{align}
The right-hand side is precisely
$\mathsf{W2SG}_p^{\mathrm{DE}}$ in
\eqref{eq:w2sg-fully-deterministic}.

Finally, because every matrix in this expression is a function of
$\Sigma$, diagonalizing $\Sigma$ gives
\eqref{eq:w2sg-eigenvalue-final}.
\end{proof}
\begin{lemma}[Average spectral preference]
\label{lem:stage-II-rank-spectral-filter}
Assume $1\le m<p$. Let
\[
\Sigma
=
V\operatorname{diag}(\lambda_1,\ldots,\lambda_p)V^\top,
\qquad
\lambda_1\ge\cdots\ge\lambda_p>0,
\]
where $V=[v_1,\ldots,v_p]$ is orthogonal. Define
\[
P_m:=\widetilde X^\dagger\widetilde X,
\qquad
\widehat\beta_S:=P_m\widehat\beta_T.
\]
Then the following statements hold: The expected projector satisfies
\[
\mathbb E_{\widetilde X}[P_m]
=
V\operatorname{diag}(\ell_1,\ldots,\ell_p)V^\top,
\]
where
\[
\ell_i
:=
\mathbb E_{\widetilde X}[v_i^\top P_m v_i]
=
\mathbb E_{\widetilde X}\|P_m v_i\|_2^2
\]
obey
\[
1>\ell_1\ge\cdots\ge\ell_p>0,
\qquad
\sum_{i=1}^{p}\ell_i=m.
\]
Thus higher-variance covariance eigendirections have greater
expected retention, but are not subject to an exact coordinate cutoff.
\end{lemma}

\begin{proof}
Write
$\widetilde XV=G\Lambda^{1/2}$, where
$\Lambda=\operatorname{diag}(\lambda_1,\ldots,\lambda_p)$
and $G=[g_1,\ldots,g_p]$ has independent
$\mathcal N(0,I_m)$ columns. In the covariance eigenbasis,
\[
V^\top P_mV
=
\Lambda^{1/2}G^\top
(G\Lambda G^\top)^{-1}G\Lambda^{1/2}.
\]
Sign symmetry of each Gaussian column implies that every
off-diagonal entry of this matrix has zero expectation.

For a diagonal entry, define
\[
A_{-i}:=\sum_{j\ne i}\lambda_jg_jg_j^\top,
\qquad
a_i:=g_i^\top A_{-i}^{-1}g_i.
\]
Because $p-1\ge m$, $A_{-i}$ is positive definite almost surely.
The Sherman--Morrison identity gives
\[
(V^\top P_mV)_{ii}
=
\frac{\lambda_i a_i}{1+\lambda_i a_i}.
\]
Hence $0<\ell_i<1$. For fixed $G$, this expression increases
with $\lambda_i$ and decreases with any $\lambda_j$, $j\ne i$.
Comparing two configurations with $\lambda_i$ and $\lambda_j$
interchanged, and using exchangeability of the Gaussian columns,
therefore gives
$\lambda_i\ge\lambda_j\Rightarrow\ell_i\ge\ell_j$.
Finally,
\[
\sum_{i=1}^{p}\ell_i
=
\mathbb E[\operatorname{Tr}(P_m)]
=
m,
\]
because $P_m$ is a rank-$m$ orthogonal projector.
\end{proof}
\paragraph{Interpretation.}
The pseudo-labeled sample size $m$ determines the dimension of the
teacher-parameter subspace retained by Stage-II training: an
$m$-dimensional, sample-dependent subspace is preserved, while its
$(p-m)$-dimensional orthogonal complement is eliminated. The retained
subspace is the row space of $\widetilde X$, whereas the eliminated
subspace is its null space. Importantly, these are sample-dependent
directions and do not generally coincide with the covariance
eigenvectors. In the covariance eigenbasis, the same mechanism appears
as preferential retention of higher-variance directions in expectation,
with total expected retention
\[
\sum_{i=1}^p \ell_i=m.
\]
Thus, $m$ determines the dimension of the retained subspace, while the
sampled covariates and the population covariance determine its
orientation.
\begin{assumption}[Power-law covariance]
\label{ass:powerlaw-covariance}

We assume that the population covariance follows the power-law model
\begin{equation}
\Sigma
=
\operatorname{diag}
\left(
1,2^{-\alpha},\ldots,p^{-\alpha}
\right),
\qquad
\lambda_i=i^{-\alpha},
\qquad
\alpha>1,
\label{eq:powerlaw-covariance}
\end{equation}
where $\{\lambda_i\}_{i=1}^p$ are the eigenvalues of $\Sigma$.
\end{assumption}

\begin{assumption}[Proportional high-dimensional asymptotics]
\label{ass:proportional-asymptotics}

We work in the proportional high-dimensional regime
\begin{equation}
\phi_{T,p}
:=
\frac{p}{n}
\longrightarrow
\phi_T>1,
\qquad
\phi_{S,p}
:=
\frac{p}{m}
\longrightarrow
\phi_S>1,
\label{eq:powerlaw-proportional-ratios}
\end{equation}
as $p,n,m\to\infty$.
Thus both Stage~I and Stage~II remain overparameterized, with
asymptotically nonzero sample-to-dimension ratios. In this regime, the deterministic-equivalence framework of
Lemma~\ref{lem:strong-deterministic-equivalence} applies under the
corresponding proportional-asymptotic assumptions.
\end{assumption}

The preceding spectral-preference result provides intuition for how Stage~II can modify the Stage-I estimator even in the absence of explicit regularization: the Stage-II row-space projection preferentially retains directions associated with larger population variance.

We now turn from this qualitative interpretation to a quantitative
analysis. Under Assumptions~\ref{ass:powerlaw-covariance} and
\ref{ass:proportional-asymptotics}, we specialize the
deterministic-equivalent expression
\eqref{eq:w2sg-after-stage1-DE} to the power-law covariance model.
\subsection{Proportional high-dimensional asymptotics}
\begin{theorem}[Power-law limit of the deterministic-equivalent risk improvement]
\label{thm:powerlaw-risk-improvement}

Under Assumptions~\ref{ass:powerlaw-covariance} and~\ref{ass:proportional-asymptotics}, let
$\mathsf{W2SG}_p^{\mathrm{DE}}$ denote the finite-$p$
deterministic-equivalent Stage-I-to-Stage-II risk improvement. Then
\begin{align}
\mathsf{W2SG}_p^{\mathrm{DE}}
\longrightarrow{}&
\frac{\sigma^2\phi_Ts_S}
{1-\gamma_T}
\int_0^1
\frac{
x^\alpha
\left[
2+
\left(
2-\frac{1}{1-\gamma_S}
\right)
s_Sx^\alpha
\right]
}{
(1+s_Sx^\alpha)^2
(1+s_Tx^\alpha)^2
}
\,dx.
\label{eq:powerlaw-final-limit-single-integral}
\end{align}
Here $s_T,s_S>0$ are the unique solutions of
\begin{equation}
\int_0^1
\frac{dx}{1+s_ax^\alpha}
=
\frac1{\phi_a},
\qquad
a\in\{T,S\},
\label{eq:powerlaw-final-sj}
\end{equation}
and
\begin{equation}
\gamma_a
=
\phi_a
\int_0^1
\frac{dx}
{(1+s_ax^\alpha)^2},
\qquad
a\in\{T,S\}.
\label{eq:powerlaw-final-gammaj}
\end{equation}
\end{theorem}
\begin{proof}
For clarity, throughout this proof we use
$
\tau_{a,p},
\gamma_{a,p},
a\in\{T,S\},
$
to denote the corresponding finite-$p$ quantities. Below, we define
the rescaled ridge
$
s_{a,p}:=p^\alpha\tau_{a,p},
$
whose limit is denoted by $s_a$, while $\gamma_a$ denotes the limit
of $\gamma_{a,p}$.
Since the ridgeless fixed-point equation is
\[
\sum_{i=1}^p
\frac{i^{-\alpha}}
{i^{-\alpha}+\tau_{a,p}}
=
\frac{p}{\phi_{a,p}},
\]
using $s_{a,p}=p^\alpha\tau_{a,p}$ and dividing by $p$ gives
\begin{align}
\frac1p
\sum_{i=1}^p
\frac{i^{-\alpha}}
{i^{-\alpha}+\tau_{a,p}}
=
\frac1p
\sum_{i=1}^p
\frac{1}
{1+s_{a,p}(i/p)^\alpha}
=
\frac1{\phi_{a,p}}.
\label{eq:sjp-fixed-point}
\end{align}
Hence
$s_{a,p}\longrightarrow s_a,
~ a\in\{T,S\},$
where $s_a>0$ is the unique solution of
\begin{equation}
\int_0^1
\frac{dx}{1+s_ax^\alpha}
=
\frac1{\phi_a}.
\label{eq:sj-limit-fixed-point}
\end{equation}
Consequently,
\begin{equation}
\tau_{a,p}
=
s_a\,p^{-\alpha}
+
o(p^{-\alpha}).
\label{eq:tauj-asymptotic}
\end{equation}
Similarly,
\begin{align}
\gamma_{a,p}
=
\frac{\phi_{a,p}}{p}
\sum_{i=1}^p
\left(
\frac{i^{-\alpha}}
{i^{-\alpha}+\tau_{a,p}}
\right)^2
=
\phi_{a,p}
\frac1p
\sum_{i=1}^p
\frac{1}
{
\left(
1+s_{a,p}(i/p)^\alpha
\right)^2
}.
\end{align}
Therefore,
\begin{equation}
\gamma_{a,p}
\longrightarrow
\gamma_a
:=
\phi_a
\int_0^1
\frac{dx}
{(1+s_ax^\alpha)^2}.
\label{eq:gammaj-powerlaw-limit}
\end{equation}
Using \eqref{eq:sj-limit-fixed-point},
\begin{align}
1-\gamma_a
=
\phi_a
\int_0^1
\left[
\frac{1}{1+s_ax^\alpha}
-
\frac{1}{(1+s_ax^\alpha)^2}
\right]
dx
=
\phi_as_a
\int_0^1
\frac{x^\alpha}
{(1+s_ax^\alpha)^2}
\,dx
>0.
\label{eq:one-minus-gammaj}
\end{align}
Hence $(1-\gamma_{a,p})^{-1}$ remains bounded for all sufficiently
large $p$. For the remainder of the proof, let
$
x_{i,p}:=\frac{i}{p}.
$
Then
$
i^{-\alpha}
=
p^{-\alpha}x_{i,p}^{-\alpha},
$
and
$
\frac{i^{-\alpha}}
{i^{-\alpha}+\tau_{a,p}}
=
\frac{1}
{1+s_{a,p}x_{i,p}^{\alpha}}.
$
Consider
\[
\Tr
\left[
D_S
\Sigma
(\Sigma+\tau_{T,p}I_p)^{-1}
\right].
\]
Since all matrices are diagonal in the eigenbasis of $\Sigma$,
\begin{align}
&
\Tr
\left[
D_S
\Sigma
(\Sigma+\tau_{T,p}I_p)^{-1}
\right]
=
2\tau_{S,p}
\sum_{i=1}^p
\frac{i^{-\alpha}}
{i^{-\alpha}+\tau_{S,p}}
\frac{i^{-\alpha}}
{i^{-\alpha}+\tau_{T,p}}
\notag\\
-&
\frac{\tau_{S,p}^2}
{1-\gamma_{S,p}}
\sum_{i=1}^p
\frac{i^{-\alpha}}
{(i^{-\alpha}+\tau_{S,p})^2}
\frac{i^{-\alpha}}
{i^{-\alpha}+\tau_{T,p}}.
\label{eq:first-trace-expanded}
\end{align}
The two summands satisfy
\begin{align}
\tau_{S,p}
\frac{i^{-\alpha}}
{i^{-\alpha}+\tau_{S,p}}
\frac{i^{-\alpha}}
{i^{-\alpha}+\tau_{T,p}}
=&
p^{-\alpha}
\frac{s_{S,p}}
{
(1+s_{S,p}x_{i,p}^\alpha)
(1+s_{T,p}x_{i,p}^\alpha)
},
\label{eq:first-trace-first-rescale}
\end{align}
and
\begin{align}
\tau_{S,p}^2
\frac{i^{-\alpha}}
{(i^{-\alpha}+\tau_{S,p})^2}
\frac{i^{-\alpha}}
{i^{-\alpha}+\tau_{T,p}}
=&
p^{-\alpha}
\frac{
s_{S,p}^2x_{i,p}^\alpha
}{
(1+s_{S,p}x_{i,p}^\alpha)^2
(1+s_{T,p}x_{i,p}^\alpha)
}.
\label{eq:first-trace-second-rescale}
\end{align}
Therefore,
\begin{align}
&
\Tr
\left[
D_S
\Sigma
(\Sigma+\tau_{T,p}I_p)^{-1}
\right]
\notag\\
=&
p^{1-\alpha}
\Bigg[
2s_{S,p}
\frac1p
\sum_{i=1}^p
\frac{1}
{
(1+s_{S,p}x_{i,p}^\alpha)
(1+s_{T,p}x_{i,p}^\alpha)
}
-
\frac{s_{S,p}^2}
{1-\gamma_{S,p}}
\frac1p
\sum_{i=1}^p
\frac{
x_{i,p}^\alpha
}{
(1+s_{S,p}x_{i,p}^\alpha)^2
(1+s_{T,p}x_{i,p}^\alpha)
}
\Bigg].
\end{align}
By Riemann-sum convergence,
\begin{align}
=
p^{1-\alpha}
\Bigg[
2s_S
\int_0^1
\frac{dx}
{
(1+s_Sx^\alpha)
(1+s_Tx^\alpha)
}
-
\frac{s_S^2}
{1-\gamma_S}
\int_0^1
\frac{
x^\alpha
}{
(1+s_Sx^\alpha)^2
(1+s_Tx^\alpha)
}
\,dx
\Bigg]
+
o(p^{1-\alpha}).
\label{eq:first-trace-asymptotic}
\end{align}
Since $\alpha>1$, this term vanishes as $p\to\infty$.

We next consider
\[
\frac{\sigma^2}
{n(1-\gamma_{T,p})}
\Tr
\left[
D_S
\Sigma
(\Sigma+\tau_{T,p}I_p)^{-2}
\right].
\]
Expanding $D_S$ gives
\begin{align}
&
\frac{\sigma^2}
{n(1-\gamma_{T,p})}
\Tr
\left[
D_S
\Sigma
(\Sigma+\tau_{T,p}I_p)^{-2}
\right]
=
\frac{2\sigma^2\tau_{S,p}}
{n(1-\gamma_{T,p})}
\sum_{i=1}^p
\frac{i^{-\alpha}}
{i^{-\alpha}+\tau_{S,p}}
\frac{i^{-\alpha}}
{(i^{-\alpha}+\tau_{T,p})^2}
\notag\\
-&
\frac{\sigma^2\tau_{S,p}^2}
{n(1-\gamma_{T,p})(1-\gamma_{S,p})}
\sum_{i=1}^p
\frac{i^{-\alpha}}
{(i^{-\alpha}+\tau_{S,p})^2}
\frac{i^{-\alpha}}
{(i^{-\alpha}+\tau_{T,p})^2}.
\label{eq:noise-trace-expanded}
\end{align}
Since
\begin{equation}
\frac{i^{-\alpha}}
{(i^{-\alpha}+\tau_{T,p})^2}
=
p^\alpha
\frac{x_{i,p}^\alpha}
{
(1+s_{T,p}x_{i,p}^\alpha)^2
},
\label{eq:stage1-noise-factor-rescale}
\end{equation}
we have
\begin{align}
\tau_{S,p}
\frac{i^{-\alpha}}
{i^{-\alpha}+\tau_{S,p}}
\frac{i^{-\alpha}}
{(i^{-\alpha}+\tau_{T,p})^2}
=&
s_{S,p}
\frac{
x_{i,p}^\alpha
}{
(1+s_{S,p}x_{i,p}^\alpha)
(1+s_{T,p}x_{i,p}^\alpha)^2
},
\label{eq:noise-first-rescale}
\end{align}
and
\begin{align}
\tau_{S,p}^2
\frac{i^{-\alpha}}
{(i^{-\alpha}+\tau_{S,p})^2}
\frac{i^{-\alpha}}
{(i^{-\alpha}+\tau_{T,p})^2}
=&
s_{S,p}^2
\frac{
x_{i,p}^{2\alpha}
}{
(1+s_{S,p}x_{i,p}^\alpha)^2
(1+s_{T,p}x_{i,p}^\alpha)^2
}.
\label{eq:noise-second-rescale}
\end{align}

Using
$
\frac1n=\frac{\phi_{T,p}}p,
$
we obtain
\begin{align}
&
\frac{\sigma^2}
{n(1-\gamma_{T,p})}
\Tr
\left[
D_S
\Sigma
(\Sigma+\tau_{T,p}I_p)^{-2}
\right]
=
\frac{2\sigma^2\phi_{T,p}s_{S,p}}
{1-\gamma_{T,p}}
\frac1p
\sum_{i=1}^p
\frac{
x_{i,p}^\alpha
}{
(1+s_{S,p}x_{i,p}^\alpha)
(1+s_{T,p}x_{i,p}^\alpha)^2
}
\notag\\
&\qquad
-
\frac{
\sigma^2\phi_{T,p}s_{S,p}^2
}{
(1-\gamma_{T,p})(1-\gamma_{S,p})
}
\frac1p
\sum_{i=1}^p
\frac{
x_{i,p}^{2\alpha}
}{
(1+s_{S,p}x_{i,p}^\alpha)^2
(1+s_{T,p}x_{i,p}^\alpha)^2
}.
\label{eq:noise-trace-rescaled}
\end{align}
Taking $p\to\infty$ yields
\begin{align}
&
\frac{\sigma^2}
{n(1-\gamma_{T,p})}
\Tr
\left[
D_S
\Sigma
(\Sigma+\tau_{T,p}I_p)^{-2}
\right]
\longrightarrow
\frac{2\sigma^2\phi_Ts_S}
{1-\gamma_T}
\int_0^1
\frac{
x^\alpha
}{
(1+s_Sx^\alpha)
(1+s_Tx^\alpha)^2
}
\,dx
\notag\\
&\qquad
-
\frac{
\sigma^2\phi_Ts_S^2
}{
(1-\gamma_T)(1-\gamma_S)
}
\int_0^1
\frac{
x^{2\alpha}
}{
(1+s_Sx^\alpha)^2
(1+s_Tx^\alpha)^2
}
\,dx.
\label{eq:noise-trace-limit}
\end{align}
Thus, unlike the signal-dependent contribution, the Stage-I noise
contribution remains of order one.

Finally,
\begin{align}
&
-2
\Tr
\left[
A_S
\Sigma
(\Sigma+\tau_{T,p}I_p)^{-1}
\right]
=
-2\tau_{S,p}
\sum_{i=1}^p
\frac{i^{-\alpha}}
{i^{-\alpha}+\tau_{S,p}}
\frac{i^{-\alpha}}
{i^{-\alpha}+\tau_{T,p}}
=
-2p^{1-\alpha}s_{S,p}
\left[
\frac1p
\sum_{i=1}^p
\frac{1}
{
(1+s_{S,p}x_{i,p}^\alpha)
(1+s_{T,p}x_{i,p}^\alpha)
}
\right].
\label{eq:third-trace-rescaled}
\end{align}
Hence
\begin{align}
-2
\Tr
\left[
A_S
\Sigma
(\Sigma+\tau_{T,p}I_p)^{-1}
\right]
=
-2s_Sp^{1-\alpha}
\int_0^1
\frac{dx}
{
(1+s_Sx^\alpha)
(1+s_Tx^\alpha)
}
+
o(p^{1-\alpha}),
\label{eq:third-trace-asymptotic}
\end{align}
and therefore this term also vanishes because $\alpha>1$. Moreover, the leading $2s_S$ contribution in
\eqref{eq:first-trace-asymptotic} cancels with
\eqref{eq:third-trace-asymptotic}. Thus, their combined contribution is
\begin{align}
&
\Tr
\left[
D_S
\Sigma
(\Sigma+\tau_{T,p}I_p)^{-1}
\right]
-
2
\Tr
\left[
A_S
\Sigma
(\Sigma+\tau_{T,p}I_p)^{-1}
\right]
\notag\\
=&
-
\frac{s_S^2}
{1-\gamma_S}
p^{1-\alpha}
\int_0^1
\frac{
x^\alpha
}{
(1+s_Sx^\alpha)^2
(1+s_Tx^\alpha)
}
\,dx
+
o(p^{1-\alpha}).
\label{eq:combined-signal-powerlaw}
\end{align}
Since $\alpha>1$, the entire signal-dependent contribution vanishes in
the proportional limit.

\subsubsection{Limiting Stage-I-to-Stage-II risk improvement}

Combining the preceding calculations gives
\begin{align}
\mathsf{W2SG}_p^{\mathrm{DE}}
={}&
\frac{2\sigma^2\phi_Ts_S}
{1-\gamma_T}
\int_0^1
\frac{
x^\alpha
}{
(1+s_Sx^\alpha)
(1+s_Tx^\alpha)^2
}
\,dx
-
\frac{
\sigma^2\phi_Ts_S^2
}{
(1-\gamma_T)(1-\gamma_S)
}
\int_0^1
\frac{
x^{2\alpha}
}{
(1+s_Sx^\alpha)^2
(1+s_Tx^\alpha)^2
}
\,dx
\notag\\
-&
\frac{s_S^2}
{1-\gamma_S}
p^{1-\alpha}
\int_0^1
\frac{
x^\alpha
}{
(1+s_Sx^\alpha)^2
(1+s_Tx^\alpha)
}
\,dx
+
o(1).
\label{eq:powerlaw-full-asymptotic}
\end{align}
Because $\alpha>1$,
$
p^{1-\alpha}\longrightarrow0,
$
and therefore
\begin{align}
\mathsf{W2SG}_p^{\mathrm{DE}}
\longrightarrow{}&
\frac{2\sigma^2\phi_Ts_S}
{1-\gamma_T}
\int_0^1
\frac{
x^\alpha
}{
(1+s_Sx^\alpha)
(1+s_Tx^\alpha)^2
}
\,dx
-
\frac{
\sigma^2\phi_Ts_S^2
}{
(1-\gamma_T)(1-\gamma_S)
}
\int_0^1
\frac{
x^{2\alpha}
}{
(1+s_Sx^\alpha)^2
(1+s_Tx^\alpha)^2
}
\,dx.
\label{eq:powerlaw-risk-improvement-limit}
\end{align}
Combining the two terms over the common denominator
$(1+s_Sx^\alpha)^2$ gives
\eqref{eq:powerlaw-final-limit-single-integral}.
Define
\begin{equation}
A_y
:=
\int_0^1
\frac{x^{y\alpha}}
{(1+s_Sx^\alpha)^2
(1+s_Tx^\alpha)^2}
\,dx.
\label{eq:Ay-definition}
\end{equation}
Then
\begin{align}
\mathsf{W2SG}_p^{\mathrm{DE}}
\longrightarrow
\mathsf{W2SG}^{\mathrm{DE}}
:={}&
\frac{\sigma^2\phi_Ts_S}
{1-\gamma_T}
\left[
2A_1
+
\left(
2-\frac{1}{1-\gamma_S}
\right)
s_SA_2
\right].
\label{eq:powerlaw-final-limit-Ay}
\end{align}

Define
\begin{equation}
J_\alpha
:=
2A_1
+
\left(
2-\frac{1}{1-\gamma_S}
\right)
s_SA_2.
\label{eq:conditionofw2sg}
\end{equation}
Since
$
\phi_T>1,
s_S>0,
1-\gamma_T>0,
$
the sign of the limiting deterministic-equivalent gain satisfies
$
\operatorname{sgn}
\bigl(
\mathsf{W2SG}^{\mathrm{DE}}
\bigr)
=
\operatorname{sgn}(J_\alpha)
\qquad
\text{whenever }\sigma^2>0.
$
Equivalently, for $\sigma^2>0$,
\[
\mathsf{W2SG}^{\mathrm{DE}}>0
\quad\Longleftrightarrow\quad
J_\alpha>0.
\]
If $\sigma^2=0$, then instead
$
\mathsf{W2SG}^{\mathrm{DE}}=0.
$

Under the power-law covariance model, the proportional limit of the
deterministic-equivalent risk improvement is therefore determined by
$\alpha$, $\phi_T$, $\phi_S$, and $\sigma^2$ through
$s_T$, $s_S$, $\gamma_T$, and $\gamma_S$.
\end{proof}
\begin{lemma}[Monotonicity and limiting values of $\gamma_S$]
\label{lem:threecase-gamma-monotone}
For every $\alpha>1$, the function $\gamma_S$ is strictly increasing as a
function of
$
\phi_S^{-1}\in(0,1).
$
Moreover,
\begin{equation}
\label{eq:threecase-gamma-limits}
\lim_{\phi_S^{-1}\downarrow0}\gamma_S
=
\frac{\alpha-1}{\alpha},
\qquad
\lim_{\phi_S^{-1}\uparrow1}\gamma_S
=
1.
\end{equation}
Consequently:
\begin{enumerate}
\item if $1<\alpha<2$, there exists a unique
$\phi_{S,\gamma}^{-1}(\alpha)\in(0,1)$ such that
\[
\gamma_S<\frac12
\quad\text{for }0<\phi_S^{-1}<\phi_{S,\gamma}^{-1}(\alpha),
\]
and
\[
\gamma_S>\frac12
\quad\text{for }\phi_{S,\gamma}^{-1}(\alpha)<\phi_S^{-1}<1;
\]

\item if $\alpha=2$,
\[
\gamma_S>\frac12
\qquad
\forall\,\phi_S^{-1}\in(0,1),
\]
while
\[
\gamma_S\downarrow\frac12
\qquad
(\phi_S^{-1}\downarrow0);
\]

\item if $\alpha>2$,
\[
\gamma_S>\frac12
\qquad
\forall\,\phi_S^{-1}\in(0,1).
\]
\end{enumerate}
\end{lemma}

\begin{proof}
Recall that $s_S>0$ is determined implicitly by
\begin{equation}
\int_0^1
\frac{dx}{1+s_Sx^\alpha}
=
\frac{1}{\phi_S},
\label{eq:gammaS-fixed-point-proof}
\end{equation}
and
\begin{equation}
\gamma_S
=
\phi_S
\int_0^1
\frac{dx}{(1+s_Sx^\alpha)^2}.
\label{eq:gammaS-definition-proof}
\end{equation}

We first derive a convenient closed-form relation for $\gamma_S$.
Observe that
\[
\frac{d}{dx}
\left(
\frac{x}{1+sx^\alpha}
\right)
=
\frac{1}{1+sx^\alpha}
-
\frac{\alpha sx^\alpha}{(1+sx^\alpha)^2}.
\]
Integrating over $x\in[0,1]$ gives
\[
\frac{1}{1+s}
=
\int_0^1
\frac{dx}{1+sx^\alpha}
-
\alpha
\int_0^1
\frac{sx^\alpha}{(1+sx^\alpha)^2}\,dx.
\]
Using
$
\frac{sx^\alpha}{(1+sx^\alpha)^2}
=
\frac{1}{1+sx^\alpha}
-
\frac{1}{(1+sx^\alpha)^2},
$
we obtain
\begin{equation}
\int_0^1
\frac{dx}{(1+sx^\alpha)^2}
=
\frac{\alpha-1}{\alpha}
\int_0^1
\frac{dx}{1+sx^\alpha}
+
\frac{1}{\alpha(1+s)}.
\label{eq:gammaS-integral-identity}
\end{equation}

Setting $s=s_S$ and using
\eqref{eq:gammaS-fixed-point-proof} and
\eqref{eq:gammaS-definition-proof}, we find
\begin{align}
\gamma_S
=&
\phi_S
\left[
\frac{\alpha-1}{\alpha}
\frac{1}{\phi_S}
+
\frac{1}{\alpha(1+s_S)}
\right]
=
\frac{\alpha-1}{\alpha}
+
\frac{\phi_S}{\alpha(1+s_S)}.
\label{eq:threecase-gamma-closed}
\end{align}
Equivalently, since
$
\phi_S
=
\left(
\int_0^1
\frac{dx}{1+s_Sx^\alpha}
\right)^{-1},
$
we may write
\begin{equation}
\gamma_S
=
\frac{\alpha-1}{\alpha}
+
\frac{1}{\alpha}
\left[
(1+s_S)
\int_0^1
\frac{dx}{1+s_Sx^\alpha}
\right]^{-1}.
\label{eq:gammaS-direct-s}
\end{equation}

We next establish monotonicity. First,
\[
\frac{d}{ds}
\int_0^1
\frac{dx}{1+sx^\alpha}
=
-
\int_0^1
\frac{x^\alpha}{(1+sx^\alpha)^2}\,dx
<0.
\]
Hence
$
\phi_S^{-1}
=
\int_0^1
\frac{dx}{1+s_Sx^\alpha}
$
is a strictly decreasing function of $s_S$.
It therefore remains to show that $\gamma_S$ is also strictly
decreasing as a function of $s_S$. By
\eqref{eq:gammaS-direct-s}, it is enough to show that
\[
\left[
(1+s)
\int_0^1
\frac{dx}{1+sx^\alpha}
\right]^{-1}
\]
is strictly decreasing in $s$.
Its logarithmic derivative is
\begin{align}
&\frac{d}{ds}
\log
\left[
\left(
(1+s)
\int_0^1
\frac{dx}{1+sx^\alpha}
\right)^{-1}
\right]
=
-\frac{1}{1+s}
+
\frac{
\displaystyle
\int_0^1
\frac{x^\alpha}{(1+sx^\alpha)^2}\,dx
}{
\displaystyle
\int_0^1
\frac{dx}{1+sx^\alpha}
}.
\label{eq:gammaS-log-derivative}
\end{align}
Now,
\begin{align}
&
\int_0^1
\frac{dx}{1+sx^\alpha}
-
(1+s)
\int_0^1
\frac{x^\alpha}{(1+sx^\alpha)^2}\,dx
=
\int_0^1
\frac{
1+sx^\alpha-(1+s)x^\alpha
}{
(1+sx^\alpha)^2
}
\,dx
=
\int_0^1
\frac{
1-x^\alpha
}{
(1+sx^\alpha)^2
}
\,dx
>0.
\end{align}
Therefore,
\[
\frac{
\displaystyle
\int_0^1
\frac{x^\alpha}{(1+sx^\alpha)^2}\,dx
}{
\displaystyle
\int_0^1
\frac{dx}{1+sx^\alpha}
}
<
\frac{1}{1+s}.
\]
Substituting this into
\eqref{eq:gammaS-log-derivative} shows that the logarithmic derivative
is strictly negative. Hence $\gamma_S$ is strictly decreasing in
$s_S$.
Since both $\gamma_S$ and $\phi_S^{-1}$ are strictly decreasing
functions of $s_S$, it follows that $\gamma_S$ is a strictly increasing
function of
$
\phi_S^{-1}\in(0,1).
$

We now determine the endpoint limits. As $s_S\downarrow0$,
$
\int_0^1
\frac{dx}{1+s_Sx^\alpha}
\longrightarrow1,
$
so
$
\phi_S^{-1}\uparrow1.
$
Moreover, from \eqref{eq:gammaS-definition-proof},
$
\gamma_S
=
\phi_S
\int_0^1
\frac{dx}{(1+s_Sx^\alpha)^2}
\longrightarrow1.$
Thus,
$
\lim_{\phi_S^{-1}\uparrow1}\gamma_S
=
1.
$
Next consider $s_S\to\infty$. With the substitution
$
u=s_S^{1/\alpha}x,
$
we have
\begin{align}
\int_0^1
\frac{dx}{1+s_Sx^\alpha}
=&
s_S^{-1/\alpha}
\int_0^{s_S^{1/\alpha}}
\frac{du}{1+u^\alpha}.
\end{align}
Since $\alpha>1$, Define
$
C_\alpha
:=
\int_0^\infty
\frac{du}{1+u^\alpha}
<\infty,
$
and therefore
\[
\int_0^1
\frac{dx}{1+s_Sx^\alpha}
\sim
C_\alpha s_S^{-1/\alpha}.
\]
In particular,
$
\phi_S^{-1}\downarrow0.
$
Furthermore,
\[
(1+s_S)
\int_0^1
\frac{dx}{1+s_Sx^\alpha}
\sim
C_\alpha s_S^{\,1-1/\alpha}
\longrightarrow\infty,
\]
because $\alpha>1$. Hence, by
\eqref{eq:gammaS-direct-s},
$
\gamma_S
\longrightarrow
\frac{\alpha-1}{\alpha}.
$
Thus,
$
\lim_{\phi_S^{-1}\downarrow0}\gamma_S
=
\frac{\alpha-1}{\alpha}.
$

It remains to distinguish the three cases.
If $1<\alpha<2$, then
$
\frac{\alpha-1}{\alpha}
<
\frac12
<
1.
$
Since $\gamma_S$ is continuous and strictly increasing in
$\phi_S^{-1}$, there exists a unique
$\phi_{S,\gamma}^{-1}(\alpha)\in(0,1)$ such that
$
\gamma_S=\frac12.
$
Consequently,
\[
\gamma_S<\frac12
\qquad
\text{for }
0<\phi_S^{-1}
<
\phi_{S,\gamma}^{-1}(\alpha),
\]
whereas
\[
\gamma_S>\frac12
\qquad
\text{for }
\phi_{S,\gamma}^{-1}(\alpha)
<
\phi_S^{-1}<1.
\]

If $\alpha=2$, then
$
\lim_{\phi_S^{-1}\downarrow0}\gamma_S
=
\frac12.
$
Strict monotonicity therefore implies
\[
\gamma_S>\frac12
\qquad
\forall\,\phi_S^{-1}\in(0,1),
\]
with
\[
\gamma_S\downarrow\frac12
\qquad
\text{as }\phi_S^{-1}\downarrow0.
\]

Finally, if $\alpha>2$, then
$
\frac{\alpha-1}{\alpha}
>
\frac12.
$
Since this is already the lower endpoint limit of $\gamma_S$, strict
monotonicity gives
\[
\gamma_S>\frac12
\qquad
\forall\,\phi_S^{-1}\in(0,1).
\]

This proves all three cases.
\end{proof}

\begin{lemma}[Removal of the Stage-I weight]
\label{lem:threecase-remove-stage1}
For every $s_T,s_S>0$,
\begin{equation}
\label{eq:threecase-A-ratio-bound}
\frac{A_2}{A_1}
\le
\frac{B_2(s_S)}{B_1(s_S)}.
\end{equation}
where
\begin{equation}
\label{eq:threecase-By-definition}
B_y(s_S)
=
\int_0^1
\frac{x^{y\alpha}}
{\left(1+s_S x^\alpha\right)^2}
\,dx
\end{equation}
\end{lemma}

\begin{proof}
Define
\[
w(x)
:=
\frac{x^\alpha}{(1+s_Sx^\alpha)^2},
\qquad
h(x)
:=
\frac{1}{(1+s_Tx^\alpha)^2},
\]
and let $\mu$ be the probability measure on $[0,1]$ defined by
\[
d\mu(x)
=
\frac{w(x)\,dx}
{\displaystyle\int_0^1 w(t)\,dt}.
\]
Then, by the definitions of $A_1$ and $A_2$,
\[
\frac{A_2}{A_1}
=
\frac{
\mathbb E_\mu\!\left[x^\alpha h(x)\right]
}{
\mathbb E_\mu\!\left[h(x)\right]
}.
\]

Since $x^\alpha$ is increasing on $[0,1]$, whereas $h(x)$ is
decreasing, the two functions are oppositely monotone. Hence
$
\operatorname{Cov}_\mu
\!\left(
x^\alpha,h(x)
\right)
\le 0,
$
and therefore
$
\mathbb E_\mu
\!\left[
x^\alpha h(x)
\right]
\le
\mathbb E_\mu
\!\left[
x^\alpha
\right]
\mathbb E_\mu
\!\left[
h(x)
\right].
$
Since $\mathbb E_\mu[h(x)]>0$, division gives
\[
\frac{A_2}{A_1}
\le
\mathbb E_\mu[x^\alpha].
\]
Finally,
\[
\mathbb E_\mu[x^\alpha]
=
\frac{
\displaystyle
\int_0^1
\frac{x^{2\alpha}}
{(1+s_Sx^\alpha)^2}\,dx
}{
\displaystyle
\int_0^1
\frac{x^\alpha}
{(1+s_Sx^\alpha)^2}\,dx
}
=
\frac{B_2(s_S)}{B_1(s_S)}.
\]
This proves \eqref{eq:threecase-A-ratio-bound}.
\end{proof}
\begin{lemma}[Scalar positivity up to the critical exponent]
\label{lem:threecase-scalar-positive}
Let
$
1<\alpha\le2.
$
Then, for every $s>0$,
\begin{align}
G_\alpha(s)
:={}&
\int_0^1\frac{dx}{1+sx^\alpha}
-
\left(
\int_0^1\frac{dx}{(1+sx^\alpha)^2}
\right)
\left(
2-
\int_0^1\frac{dx}{1+sx^\alpha}
\right)
>0.
\label{eq:threecase-G-positive}
\end{align}
\end{lemma}

\begin{proof}
We first derive an identity relating the two integrals appearing in
$G_\alpha$. Since
\[
\frac{d}{dx}
\left(
\frac{x}{1+sx^\alpha}
\right)
=
\frac{1}{1+sx^\alpha}
-
\frac{\alpha sx^\alpha}{(1+sx^\alpha)^2},
\]
integration over $[0,1]$ gives
\[
\frac{1}{1+s}
=
\int_0^1\frac{dx}{1+sx^\alpha}
-
\alpha
\int_0^1
\frac{sx^\alpha}{(1+sx^\alpha)^2}\,dx.
\]
Using
\[
\frac{sx^\alpha}{(1+sx^\alpha)^2}
=
\frac{1}{1+sx^\alpha}
-
\frac{1}{(1+sx^\alpha)^2},
\]
we obtain
\begin{equation}
\int_0^1
\frac{dx}{(1+sx^\alpha)^2}
=
\frac{\alpha-1}{\alpha}
\int_0^1
\frac{dx}{1+sx^\alpha}
+
\frac{1}{\alpha(1+s)}.
\label{eq:threecase-integral-identity}
\end{equation}

Moreover,
\[
\frac{d}{ds}
\int_0^1
\frac{dx}{1+sx^\alpha}
=
-
\int_0^1
\frac{x^\alpha}{(1+sx^\alpha)^2}\,dx,
\]
and hence
\begin{equation}
s\frac{d}{ds}
\int_0^1
\frac{dx}{1+sx^\alpha}
=
\int_0^1
\frac{dx}{(1+sx^\alpha)^2}
-
\int_0^1
\frac{dx}{1+sx^\alpha}.
\label{eq:threecase-integral-derivative}
\end{equation}
Combining this with
\eqref{eq:threecase-integral-identity} gives
\begin{equation}
s\frac{d}{ds}
\int_0^1
\frac{dx}{1+sx^\alpha}
=
\frac1\alpha
\left[
\frac1{1+s}
-
\int_0^1
\frac{dx}{1+sx^\alpha}
\right].
\label{eq:threecase-integral-derivative-closed}
\end{equation}

We first examine $G_\alpha$ near $s=0$. Expanding the integrands gives
\begin{align}
\int_0^1
\frac{dx}{1+sx^\alpha}
=&
1-\frac{s}{\alpha+1}
+\frac{s^2}{2\alpha+1}
+O(s^3),
\label{eq:threecase-first-integral-small}
\\
\int_0^1
\frac{dx}{(1+sx^\alpha)^2}
=&
1-\frac{2s}{\alpha+1}
+\frac{3s^2}{2\alpha+1}
+O(s^3).
\label{eq:threecase-second-integral-small}
\end{align}
Therefore,
\begin{align}
G_\alpha(s)
=&
\left[
\frac{2}{(\alpha+1)^2}
-
\frac{1}{2\alpha+1}
\right]s^2
+O(s^3)
=
\frac{
2-(\alpha-1)^2
}{
(\alpha+1)^2(2\alpha+1)
}
s^2
+O(s^3).
\label{eq:threecase-Gsmall}
\end{align}
For $1<\alpha\le2$,
$
2-(\alpha-1)^2>0,
$
and consequently
\begin{equation}
G_\alpha(s)>0
\qquad
\text{for all sufficiently small }s>0.
\label{eq:threecase-Gsmall-positive}
\end{equation}

Suppose, toward a contradiction, that $G_\alpha$ has a positive zero,
and let $s_\star>0$ denote its first positive zero. Then
\[
G_\alpha(s_\star)=0,
\qquad
G_\alpha(s)>0
\quad
\text{for }0<s<s_\star.
\]
Define
$
u
:=
1-
\int_0^1
\frac{dx}{1+s_\star x^\alpha}.
$
Since $s_\star>0$,
$
0<u<1.
$
Thus,
$
\int_0^1
\frac{dx}{1+s_\star x^\alpha}
=
1-u.
$
The condition $G_\alpha(s_\star)=0$ then implies
\begin{equation}
\int_0^1
\frac{dx}{(1+s_\star x^\alpha)^2}
=
\frac{1-u}{1+u}.
\label{eq:threecase-second-integral-root}
\end{equation}

Define
$
b
:=
\frac{1-(\alpha-1)u}{\alpha}.
$
Combining
\eqref{eq:threecase-integral-identity}
and
\eqref{eq:threecase-second-integral-root}
gives
\begin{equation}
1+s_\star
=
\frac{1+u}
{\alpha(1-u)b}.
\label{eq:threecase-one-plus-s-root}
\end{equation}
Equivalently,
\begin{equation}
s_\star
=
\frac{u(1+b)}
{(1-u)b}.
\label{eq:threecase-s-root}
\end{equation}

We next note that $b>0$. If $1<\alpha<2$, then
\[
b
=
\frac{1-(\alpha-1)u}{\alpha}
>
\frac{2-\alpha}{\alpha}
>0,
\]
whereas, if $\alpha=2$,
$
b
=
\frac{1-u}{2}
>0.
$
Hence
\begin{equation}
b>0
\qquad
\text{for all }1<\alpha\le2.
\label{eq:threecase-b-positive}
\end{equation}

Differentiating the definition of $G_\alpha$ and using
\eqref{eq:threecase-integral-identity},
\eqref{eq:threecase-integral-derivative-closed},
\eqref{eq:threecase-second-integral-root},
\eqref{eq:threecase-one-plus-s-root}, and
\eqref{eq:threecase-s-root}, a direct simplification yields
\begin{equation}
G_\alpha'(s_\star)
=
\frac{
u(1-u)^2
b
\bigl(\alpha b^2+2-\alpha\bigr)
}{
(1+u)^2(1+b)
}.
\label{eq:threecase-Gprime-root}
\end{equation}

If $1<\alpha<2$, then
$
\alpha b^2+2-\alpha>0.
$
If $\alpha=2$, then
$
\alpha b^2+2-\alpha
=
2b^2>0.
$
Together with
$
0<u<1,
b>0,
$
this implies
$
G_\alpha'(s_\star)>0.
$
On the other hand, $s_\star$ is the first zero reached from a region
where $G_\alpha(s)>0$. Since $G_\alpha$ is differentiable, this requires
$
G_\alpha'(s_\star)\le0,
$
which is a contradiction.

Therefore $G_\alpha$ has no positive zero. Since it is positive for all
sufficiently small $s>0$, we conclude that
\[
G_\alpha(s)>0
\qquad
\forall\,s>0.
\]
\end{proof}
\begin{theorem}[Global proportional positivity for $1<\alpha\le2$]
\label{thm:threecase-global-positive}
Let
$
1<\alpha\le2,
\phi_T>1,
\phi_S>1,
\sigma^2>0.
$
Then
\begin{equation}
\label{eq:threecase-global-positive-result}
\mathsf{W2SG}^{\mathrm{DE}}>0.
\end{equation}
for every nontrivial proportional Stage-II sampling fraction
$
0<\phi_S^{-1}<1.
$
\end{theorem}

\begin{proof}
We distinguish the two possible signs of the Stage-II coefficient.

If $\gamma_S\le1/2$, then
$
\frac{1-2\gamma_S}{1-\gamma_S}\ge0.
$
Since $A_1>0$ and $A_2>0$, equation
\eqref{eq:conditionofw2sg} immediately gives
$
\mathsf{W2SG}^{\mathrm{DE}}>0
$

It remains to consider the case $\gamma_S>1/2$. In this case,
\eqref{eq:conditionofw2sg} shows that
$\mathsf{W2SG}^{\mathrm{DE}}(\phi_T,\phi_S)>0$ is equivalent to
\begin{equation}
\label{eq:threecase-theorem-target}
\frac{2\gamma_S-1}{1-\gamma_S}
\,s_S\frac{A_2}{A_1}
<2.
\end{equation}
By Lemma~\ref{lem:threecase-remove-stage1},
\[
\frac{A_2}{A_1}
\le
\frac{B_2(s_S)}{B_1(s_S)}.
\]

Recall that
$
\phi_S^{-1}
=
\int_0^1
\frac{dx}{1+s_Sx^\alpha},
$
and
$
\gamma_S
=
\phi_S
\int_0^1
\frac{dx}{(1+s_Sx^\alpha)^2}.
$
Therefore,
\begin{equation}
\gamma_S
=
\frac{
\displaystyle
\int_0^1
\frac{dx}{(1+s_Sx^\alpha)^2}
}{
\displaystyle
\int_0^1
\frac{dx}{1+s_Sx^\alpha}
}.
\label{eq:threecase-gamma-integral-ratio}
\end{equation}
It follows that
\begin{align}
\frac{2\gamma_S-1}{1-\gamma_S}
=
\frac{
\displaystyle
2\int_0^1
\frac{dx}{(1+s_Sx^\alpha)^2}
-
\int_0^1
\frac{dx}{1+s_Sx^\alpha}
}{
\displaystyle
\int_0^1
\frac{dx}{1+s_Sx^\alpha}
-
\int_0^1
\frac{dx}{(1+s_Sx^\alpha)^2}
}.
\label{eq:threecase-theorem-gamma-ratio}
\end{align}

Next, directly from the definitions of $B_1$ and $B_2$,
\begin{align}
\int_0^1
\frac{dx}{1+s_Sx^\alpha}
-
\int_0^1
\frac{dx}{(1+s_Sx^\alpha)^2}
=&
s_SB_1(s_S),
\label{eq:threecase-theorem-B0}
\\
1
-
2\int_0^1
\frac{dx}{1+s_Sx^\alpha}
+
\int_0^1
\frac{dx}{(1+s_Sx^\alpha)^2}
=&
s_S^2B_2(s_S).
\label{eq:threecase-theorem-B1}
\end{align}
Hence
\begin{align}
s_S\frac{B_2(s_S)}{B_1(s_S)}
=
\frac{
\displaystyle
1
-
2\int_0^1
\frac{dx}{1+s_Sx^\alpha}
+
\int_0^1
\frac{dx}{(1+s_Sx^\alpha)^2}
}{
\displaystyle
\int_0^1
\frac{dx}{1+s_Sx^\alpha}
-
\int_0^1
\frac{dx}{(1+s_Sx^\alpha)^2}
}.
\label{eq:threecase-B-ratio-integral}
\end{align}

Combining
\eqref{eq:threecase-theorem-gamma-ratio},
\eqref{eq:threecase-B-ratio-integral},
and Lemma~\ref{lem:threecase-remove-stage1}, we obtain
\begin{align}
&
\frac{2\gamma_S-1}{1-\gamma_S}
\,s_S\frac{A_2}{A_1}
\nonumber\\
&\qquad\le
\frac{
\left(
2\displaystyle\int_0^1
\frac{dx}{(1+s_Sx^\alpha)^2}
-
\displaystyle\int_0^1
\frac{dx}{1+s_Sx^\alpha}
\right)
\left(
1
-
2\displaystyle\int_0^1
\frac{dx}{1+s_Sx^\alpha}
+
\displaystyle\int_0^1
\frac{dx}{(1+s_Sx^\alpha)^2}
\right)
}{
\left(
\displaystyle\int_0^1
\frac{dx}{1+s_Sx^\alpha}
-
\displaystyle\int_0^1
\frac{dx}{(1+s_Sx^\alpha)^2}
\right)^2
}.
\label{eq:threecase-theorem-upper-bound}
\end{align}

We now apply Lemma~\ref{lem:threecase-scalar-positive}. A direct
algebraic expansion gives
\begin{align}
&2
\left(
\int_0^1
\frac{dx}{1+s_Sx^\alpha}
-
\int_0^1
\frac{dx}{(1+s_Sx^\alpha)^2}
\right)^2
\nonumber\\
&\quad-
\left(
2\int_0^1
\frac{dx}{(1+s_Sx^\alpha)^2}
-
\int_0^1
\frac{dx}{1+s_Sx^\alpha}
\right)
\left(
1
-
2\int_0^1
\frac{dx}{1+s_Sx^\alpha}
+
\int_0^1
\frac{dx}{(1+s_Sx^\alpha)^2}
\right)
\nonumber\\
=&
\int_0^1
\frac{dx}{1+s_Sx^\alpha}
-
\left(
\int_0^1
\frac{dx}{(1+s_Sx^\alpha)^2}
\right)
\left(
2-
\int_0^1
\frac{dx}{1+s_Sx^\alpha}
\right).
\label{eq:threecase-theorem-identity}
\end{align}
By Lemma~\ref{lem:threecase-scalar-positive}, applied at $s=s_S$,
the right-hand side is strictly positive. Hence
\begin{align}
&
\frac{
\left(
2\displaystyle\int_0^1
\frac{dx}{(1+s_Sx^\alpha)^2}
-
\displaystyle\int_0^1
\frac{dx}{1+s_Sx^\alpha}
\right)
\left(
1
-
2\displaystyle\int_0^1
\frac{dx}{1+s_Sx^\alpha}
+
\displaystyle\int_0^1
\frac{dx}{(1+s_Sx^\alpha)^2}
\right)
}{
\left(
\displaystyle\int_0^1
\frac{dx}{1+s_Sx^\alpha}
-
\displaystyle\int_0^1
\frac{dx}{(1+s_Sx^\alpha)^2}
\right)^2
}
<2.
\end{align}
Together with
\eqref{eq:threecase-theorem-upper-bound}, this proves
\eqref{eq:threecase-theorem-target}. Therefore,
$
\mathsf{W2SG}^{\mathrm{DE}}(\phi_T,\phi_S)>0
$
also in the case $\gamma_S>1/2$.
Thus,
\[
\mathsf{W2SG}^{\mathrm{DE}}>0
\qquad
\text{for all }
1<\alpha\le2,
\quad
\phi_T>1,
\quad
\phi_S>1.
\]
\end{proof}
% ============================================================

The situation changes qualitatively when $\alpha>2$. By
Lemma~\ref{lem:threecase-gamma-monotone},
\[
\gamma_S>\frac12
\qquad
\forall\,\phi_S^{-1}\in(0,1),
\]
and hence
\begin{equation}
\label{eq:alpha-gt2-condition}
\mathsf{W2SG}^{\mathrm{DE}}>0
\iff
2A_1
>
\frac{2\gamma_S-1}{1-\gamma_S}
s_SA_2.
\end{equation}
Unlike the case $1<\alpha\le2$, this inequality is not automatically
satisfied.

We first prove that it necessarily fails near the sparse edge.

\begin{proposition}[Sparse proportional negativity for $\alpha>2$]
\label{prop:alpha-gt2-sparse-negative}

Let
$
\alpha>2,
\phi_T>1,
\sigma^2>0.
$
Then there exists
$
\varepsilon_0
=
\varepsilon_0(\alpha,\phi_T)>0
$
such that
\begin{equation}
0<\phi_S^{-1}<\varepsilon_0
\quad\Longrightarrow\quad
\mathsf{W2SG}^{\mathrm{DE}}(\phi_T,\phi_S)<0.
\label{eq:alpha-gt2-small-inverse-phi-negative}
\end{equation}
\end{proposition}
\begin{proof}
The sparse proportional limit corresponds to
\[
\phi_S^{-1}\downarrow0
\qquad\Longleftrightarrow\qquad
s_S\to\infty.
\]
From Lemma~\ref{lem:threecase-gamma-monotone},
$
\gamma_S
\longrightarrow
\frac{\alpha-1}{\alpha},
$
and hence
\begin{equation}
\label{eq:stage2-coefficient-sparse-limit}
\frac{1-2\gamma_S}{1-\gamma_S}
\longrightarrow
2-\alpha<0.
\end{equation}

Write
$
h_1(x)
:=
\frac1{(1+s_Tx^\alpha)^2}.
$
Then
\[
A_1
=
\int_0^1
\frac{x^\alpha}{(1+s_Sx^\alpha)^2}
h_1(x)\,dx.
\]
Multiplying by $s_S$,
\[
s_SA_1
=
\int_0^1
\frac{s_Sx^\alpha}
{(1+s_Sx^\alpha)^2}
h_1(x)\,dx.
\]
For every $x>0$,
$
\frac{s_Sx^\alpha}
{(1+s_Sx^\alpha)^2}
\longrightarrow0,
$
and
$
0
\le
\frac{s_Sx^\alpha}
{(1+s_Sx^\alpha)^2}
\le
\frac14.
$
Thus dominated convergence yields
\begin{equation}
\label{eq:sA0-zero}
s_SA_1\longrightarrow0.
\end{equation}

Similarly,
\[
s_S^2A_2
=
\int_0^1
\frac{(s_Sx^\alpha)^2}
{(1+s_Sx^\alpha)^2}
h_1(x)\,dx.
\]
For every $x>0$,
$
\frac{(s_Sx^\alpha)^2}
{(1+s_Sx^\alpha)^2}
\longrightarrow1,
$
while the ratio is bounded by $1$. Therefore
\begin{equation}
\label{eq:s2A2-limit}
s_S^2A_2
\longrightarrow
C_1(s_T)
:=
\int_0^1
\frac{dx}{(1+s_Tx^\alpha)^2}
>
0.
\end{equation}

Multiplying by $s_S$,
\[
s_SJ_\alpha
=
2s_SA_1
+
\frac{1-2\gamma_S}{1-\gamma_S}
\,s_S^2A_2.
\]
Using
\eqref{eq:stage2-coefficient-sparse-limit},
\eqref{eq:sA0-zero}, and
\eqref{eq:s2A2-limit},
\[
s_SJ_\alpha
\longrightarrow
(2-\alpha)C_1(s_T)<0.
\]
Therefore $J_\alpha<0$ for all sufficiently large $s_S$.
Since $\phi_S^{-1}$ is a strictly decreasing function of $s_S$,
there exists $\varepsilon_0=\varepsilon_0(\alpha,\phi_T)>0$ such that
\[
0<\phi_S^{-1}<\varepsilon_0
\quad\Longrightarrow\quad
J_\alpha<0.
\]
Because $\sigma^2>0$, the sign relation
\eqref{eq:conditionofw2sg} then gives
$
\mathsf{W2SG}^{\mathrm{DE}}(\phi_T,\phi_S)<0.
$
\end{proof}

Thus for every $\alpha>2$, there always exist proportional ratios
$p/m$ sufficiently large for which the W2SG condition fails.

We next determine when a positive region necessarily occurs near the
opposite boundary $\phi_S^{-1}\uparrow1$.

Define the near-square quantities
\begin{align}
B_y(s_T)
:=&
\int_0^1
\frac{x^{y\alpha}}{(1+s_Tx^\alpha)^2}\,dx,
\label{eq:threecase-By-square}
\end{align}
and
\begin{equation}
\label{eq:threecase-L-def}
L_\alpha(\phi_T)
:=
2B_1(s_T)
-
(\alpha+1)B_2(s_T).
\end{equation}

\begin{proposition}[Near-square sign for $\alpha>2$]
\label{prop:alpha-gt2-near-square}
Let
$
\alpha>2,
\phi_T>1.
$
Then
\begin{equation}
\label{eq:J-near-square-limit}
\lim_{\phi_S^{-1}\uparrow1}
J_\alpha(\phi_T,\phi_S)
=
L_\alpha(\phi_T).
\end{equation}
Consequently:
\begin{enumerate}
\item if $L_\alpha(\phi_T)>0$, then
$
J_\alpha>0
$
for all $\phi_S^{-1}<1$ sufficiently close to $1$;
\item if $L_\alpha(\phi_T)<0$, then
$
J_\alpha<0
$
for all $\phi_S^{-1}<1$ sufficiently close to $1$;

\item if $L_\alpha(\phi_T)=0$, the leading near-square limit is
degenerate and a higher-order expansion is required.
\end{enumerate}
\end{proposition}

\begin{proof}
As
$
\phi_S^{-1}\uparrow1,
$
the Stage-II fixed point gives
$
s_S\downarrow0.
$
From
\[
F_\alpha(s)
=
1-\frac{s}{\alpha+1}+O(s^2),
\qquad
H_\alpha(s)
=
1-\frac{2s}{\alpha+1}+O(s^2),
\]
we obtain
\[
\gamma_S
=
\frac{H_\alpha(s_S)}{F_\alpha(s_S)}
=
1-\frac{s_S}{\alpha+1}
+O(s_S^2).
\]
Thus
\begin{equation}
\label{eq:near-square-coeff-limit}
s_S
\frac{2\gamma_S-1}{1-\gamma_S}
\longrightarrow
\alpha+1.
\end{equation}

Moreover, dominated convergence gives
$
A_1\longrightarrow B_1(s_T)
,
A_2\longrightarrow B_2(s_T)
.
$
Since $\gamma_S>1/2$ for $\alpha>2$,
$
J_\alpha
=
2A_1
-
\frac{2\gamma_S-1}{1-\gamma_S}
s_SA_2.
$
Hence
\[
J_\alpha
\longrightarrow
2B_1(s_T)
-
(\alpha+1)B_2(s_T)
=
L_\alpha(\phi_T).
\]
The sign conclusions follow by continuity.
\end{proof}

Combining the sparse and near-square results immediately gives a crossing
criterion.

\begin{theorem}[Existence of both W2SG and non-W2SG proportional ratios for $\alpha>2$]
\label{thm:alpha-gt2-both-signs}
Assume
$
\alpha>2,
\phi_T>1,
\sigma^2>0,
$
and suppose
$
L_\alpha(\phi_T)>0.
$
Then there exist
$
0<\phi_{S,-}^{-1}<\phi_{S,+}^{-1}<1
$
such that
\[
\mathsf{W2SG}^{\mathrm{DE}}(\phi_T,\phi_{S,-})<0,
\qquad
\mathsf{W2SG}^{\mathrm{DE}}(\phi_T,\phi_{S,+})>0.
\]
Equivalently, there exists a proportional ratio $p/m$ for which the
W2SG condition fails and another proportional ratio $p/m$ for which it
holds.
Moreover, by continuity, there exists at least one
$
\phi_{S,\star}^{-1}\in(0,1)
$
such that
$
\mathsf{W2SG}^{\mathrm{DE}}(\phi_T,\phi_{S,\star})=0.
$
\end{theorem}

\begin{proof}
By Proposition~\ref{prop:alpha-gt2-sparse-negative},
$
\mathsf{W2SG}^{\mathrm{DE}}(\phi_T,\phi_S)<0
$
for all sufficiently small $\phi_S^{-1}>0$.
By Proposition~\ref{prop:alpha-gt2-near-square} and
$L_\alpha(\phi_T)>0$,
$
\mathsf{W2SG}^{\mathrm{DE}}(\phi_T,\phi_S)>0
$
for all $\phi_S^{-1}<1$ sufficiently close to $1$.

Choose one point from each region. Since
$\mathsf{W2SG}^{\mathrm{DE}}(\phi_T,\phi_S)$ is continuous as a
function of $\phi_S^{-1}\in(0,1)$, the intermediate value theorem
gives at least one zero between them.
\end{proof}

The condition $L_\alpha(\phi_T)>0$ is automatically satisfied for a
nontrivial range above the critical exponent.
\begin{corollary}[An unconditional crossing range above $\alpha=2$]
\label{cor:alpha-gt2-unconditional-range}

If
$
2<\alpha\le1+\sqrt2,$
then
$
L_\alpha(\phi_T)>0
$
for every $\phi_T>1$.
Consequently, if in addition $\sigma^2>0$, then for every
$
2<\alpha\le1+\sqrt2
$
and every $\phi_T>1$, there exist both positive-W2SG and
negative-W2SG proportional Stage-II ratios.
\end{corollary}

\begin{proof}
Consider the probability measure
\[
d\nu(x)
=
\frac{x^\alpha\,dx}
{\int_0^1t^\alpha\,dt}.
\]
Then
\[
\frac{B_2(s_T)
}{B_1(s_T)}
=
\frac{
\mathbb E_\nu
\left[
x^\alpha(1+s_Tx^\alpha)^{-2}
\right]
}{
\mathbb E_\nu
\left[
(1+s_Tx^\alpha)^{-2}
\right]
}.
\]
Since $x^\alpha$ is increasing and
$
(1+s_Tx^\alpha)^{-2}
$ is strictly decreasing for $s_T>0$, the same covariance argument as in
Lemma~\ref{lem:threecase-remove-stage1} gives the strict inequality
\begin{equation}
\label{eq:square-ratio-bound}
\frac{B_2(s_T)}{B_1(s_T)}
<
\frac{
\int_0^1x^{2\alpha}\,dx
}{
\int_0^1x^\alpha\,dx
}
=
\frac{\alpha+1}{2\alpha+1}.
\end{equation}
Therefore
\begin{align}
\frac{L_\alpha(\phi_T)}{B_1(s_T)}
=&
2
-
(\alpha+1)
\frac{B_2(s_T)}{B_1(s_T)}
>
2
-
\frac{(\alpha+1)^2}{2\alpha+1}
=
\frac{
2-(\alpha-1)^2
}{
2\alpha+1}.
\label{eq:L-lower-bound}
\end{align}
For
$
2<\alpha<1+\sqrt2,
$
the right-hand side is strictly positive.
At
$
\alpha=1+\sqrt2,
$
the final non-strict lower bound equals zero, but
\eqref{eq:square-ratio-bound} is strict because
$s_T>0$. Hence $L_\alpha(\phi_T)>0$ also at the endpoint.
\end{proof}
For
$
\alpha>1+\sqrt2,
$
the condition $L_\alpha(\phi_T)>0$ may still hold, depending on the
Stage-I aspect ratio $\phi_T$. The preceding theorem therefore remains
applicable, but positivity near the square boundary is no longer implied
by $\alpha$ alone. If
$
L_\alpha(\phi_T)\le0,
$
the sparse region remains rigorously negative, but the present argument
does not rule out the possibility of a positive component at an
intermediate proportional ratio.

% ============================================================
We next strengthen the preceding existence result by proving uniqueness
of the proportional W2SG transition whenever
$L_\alpha(\phi_T)>0$. Specifically, we show that the gain crosses zero
exactly once as the Stage-II sampling fraction varies from the sparse
boundary to the near-square boundary.

The argument relies on two auxiliary results: a Stieltjes
representation for the squared Stage-II fixed-point map and a strict
curvature property for a corresponding Stieltjes quotient. These are
established in Lemmas~\ref{lem:fp-stieltjes-square} and
\ref{lem:fp-stieltjes-curvature}, respectively, and are then used to
prove Theorem~\ref{thm:fp-unique-transition}.
\begin{lemma}[Stieltjes representation of the squared Stage-II fixed-point map]
\label{lem:fp-stieltjes-square}
Let
$
0<\ell<\frac12,
$
and define
\[
F_\ell(s)
:=
\ell\int_0^1\frac{y^{\ell-1}}{1+sy}\,dy,
\qquad s>0.
\]
Set
\begin{align}
C_\ell
:=&
\int_0^1
\frac{u^{\ell-1}-u^{-\ell}}{1-u}\,du,
\label{eq:fp-Cell}
\\
q_\ell(y)
:=&
C_\ell y^\ell
+
\sum_{k=0}^{\infty}
\frac{y^{k+1}}{k+1-\ell},
\qquad 0<y<1,
\label{eq:fp-qell}
\\
w_\ell(y)
:=&
2\ell^2 y^{\ell-1}q_\ell(y).
\label{eq:fp-well}
\end{align}
Then
\[
C_\ell>0,
\qquad
w_\ell(y)>0
\quad (0<y<1),
\]
and
\[
\int_0^1 w_\ell(y)\,dy=1.
\]
Moreover,
\begin{equation}
\label{eq:fp-square-representation}
F_\ell(s)^2
=
\int_0^1
\frac{w_\ell(y)}{1+sy}\,dy,
\qquad s>0.
\end{equation}
Thus \(F_\ell^2\) is itself a Stieltjes transform of a probability density on \((0,1)\).
\end{lemma}

\begin{proof}
The proof has three steps: first we establish positivity of the density,
then we compute all moments of \(w_\ell\), and finally we compare those
moments with the power-series coefficients of \(F_\ell^2\).

\medskip
\noindent
\textbf{Step 1: Positivity of \(C_\ell\) and \(w_\ell\).}
Since \(0<\ell<1/2\), we have \(1-2\ell>0\). For \(0<u<1\),
\[
u^{\ell-1}-u^{-\ell}
=
u^{\ell-1}\bigl(1-u^{1-2\ell}\bigr)>0.
\]
Hence the integrand in \eqref{eq:fp-Cell} is positive on \((0,1)\).

It is also integrable at both endpoints. Therefore
$
C_\ell>0.
$
Using
$
\frac1{1-u}=\sum_{k=0}^\infty u^k,
\qquad 0<u<1,
$
and monotone convergence, we obtain the useful series representation
\begin{equation}
\label{eq:fp-Cell-series}
C_\ell
=
\sum_{k=0}^{\infty}
\left(
\frac{1}{k+\ell}
-
\frac{1}{k+1-\ell}
\right)
>0.
\end{equation}
Every term in the definition of \(q_\ell(y)\) is therefore positive,
so
\[
q_\ell(y)>0,
\qquad 0<y<1.
\]
Since \(2\ell^2y^{\ell-1}>0\), it follows that
\[
w_\ell(y)>0,
\qquad 0<y<1.
\]

\medskip
\noindent
\textbf{Step 2: Moments of \(w_\ell\).}
For \(n\ge 0\), define
$
M_n
:=
\int_0^1 y^n w_\ell(y)\,dy.
$
By \eqref{eq:fp-well} and Tonelli's theorem,
\begin{align}
M_n
=&
2\ell^2
\int_0^1
y^{n+\ell-1}
\left(
C_\ell y^\ell
+
\sum_{k=0}^{\infty}
\frac{y^{k+1}}{k+1-\ell}
\right)\,dy=
2\ell^2
\left[
\frac{C_\ell}{n+2\ell}
+
\sum_{k=0}^{\infty}
\frac{1}
{(k+1-\ell)(n+k+\ell+1)}
\right].
\label{eq:fp-well-moment-start}
\end{align}
Now use the partial-fraction identity
\begin{equation}
\label{eq:fp-partial-fraction}
\frac{1}
{(k+1-\ell)(n+k+\ell+1)}
=
\frac{1}{n+2\ell}
\left(
\frac{1}{k+1-\ell}
-
\frac{1}{n+k+\ell+1}
\right).
\end{equation}
Substituting \eqref{eq:fp-partial-fraction} into
\eqref{eq:fp-well-moment-start} gives
\begin{align}
M_n
=&
\frac{2\ell^2}{n+2\ell}
\left[
C_\ell
+
\sum_{k=0}^{\infty}
\left(
\frac{1}{k+1-\ell}
-
\frac{1}{n+k+\ell+1}
\right)
\right].
\end{align}
Using \eqref{eq:fp-Cell-series}, the terms involving
\((k+1-\ell)^{-1}\) cancel, so
\begin{align}
M_n
=&
\frac{2\ell^2}{n+2\ell}
\sum_{k=0}^{\infty}
\left(
\frac{1}{k+\ell}
-
\frac{1}{n+k+\ell+1}
\right)
=
\frac{2\ell^2}{n+2\ell}
\sum_{k=0}^{n}
\frac{1}{k+\ell}.
\label{eq:fp-well-moment-telescope}
\end{align}
The last expression can be symmetrized by observing that
\[
\frac{1}{(k+\ell)(n-k+\ell)}
=
\frac{1}{n+2\ell}
\left(
\frac{1}{k+\ell}
+
\frac{1}{n-k+\ell}
\right).
\]
Summing over \(k=0,\dots,n\) yields
\[
\sum_{k=0}^{n}
\frac{1}{(k+\ell)(n-k+\ell)}
=
\frac{2}{n+2\ell}
\sum_{k=0}^{n}
\frac{1}{k+\ell}.
\]
Hence
\begin{equation}
\label{eq:fp-well-moments}
M_n
=
\ell^2
\sum_{k=0}^{n}
\frac{1}
{(k+\ell)(n-k+\ell)}.
\end{equation}
In particular, setting \(n=0\) gives
$
\int_0^1 w_\ell(y)\,dy
=
M_0
=
\ell^2\frac{1}{\ell^2}
=
1.
$
Thus \(w_\ell\) is a probability density on \((0,1)\).

\medskip
\noindent
\textbf{Step 3: Matching the power series.}
For \(|s|<1\),
$
\frac{1}{1+sy}
=
\sum_{n=0}^{\infty}(-s)^n y^n.
$
Therefore
\begin{align}
F_\ell(s)
=&
\ell
\int_0^1
y^{\ell-1}
\sum_{n=0}^{\infty}(-s)^n y^n\,dy
=
\sum_{n=0}^{\infty}
\frac{\ell(-s)^n}{n+\ell}.
\label{eq:fp-Fell-series}
\end{align}
Squaring \eqref{eq:fp-Fell-series} and using the Cauchy product gives
\begin{align}
F_\ell(s)^2
=&
\sum_{n=0}^{\infty}
(-s)^n
\left[
\ell^2
\sum_{k=0}^{n}
\frac{1}
{(k+\ell)(n-k+\ell)}
\right]
=
\sum_{n=0}^{\infty}
(-s)^n M_n,
\end{align}
where the last equality follows from \eqref{eq:fp-well-moments}.
Since
$
M_n=\int_0^1 y^n w_\ell(y)\,dy,
$
we obtain
\begin{align}
F_\ell(s)^2
=&
\int_0^1
w_\ell(y)
\sum_{n=0}^{\infty}
(-sy)^n\,dy
=
\int_0^1
\frac{w_\ell(y)}{1+sy}\,dy,
\qquad |s|<1.
\end{align}

Finally, both sides of \eqref{eq:fp-square-representation} are real
analytic for \(s>-1\). Since they agree on the nonempty interval
\((-1,1)\), the identity theorem extends the equality to the whole
connected domain \(s>-1\), and hence in particular to every \(s>0\).
\end{proof}

\begin{lemma}[Strict curvature at stationary points of a Stieltjes quotient]
\label{lem:fp-stieltjes-curvature}
Let $w$ be integrable and positive almost everywhere on $(0,1)$.
Suppose $r$ is continuous on $[0,1]$, satisfies
\[
r(0)=r(1)=0,
\qquad
r(y)>0
\quad
(0<y<1),
\]
and is strictly increasing up to a unique point in $(0,1)$ and strictly
decreasing afterward. Define
\[
M(s)
:=
\int_0^1
\frac{w(y)}{1+sy}\,dy,
\qquad
N(s)
:=
\int_0^1
\frac{r(y)w(y)}{1+sy}\,dy,
\]
and
$
Q(s):=\frac{N(s)}{M(s)}.
$
Then every stationary point $s_0>0$ of $Q$ satisfies
$
Q''(s_0)<0.
$
Consequently, $Q$ has at most one stationary point on $(0,\infty)$.
\end{lemma}

\begin{proof}
Suppose $Q'(s_0)=0$, and set $c:=Q(s_0)$. Since $c$ is a weighted
average of $r$ with a positive weight,
$
0<c<\max_{[0,1]}r.
$
By strict unimodality of $r$, there are exactly two points
$0<u<v<1$ such that $r(u)=r(v)=c$, and $r-c$ has sign pattern
\[
-\,,\quad+\,,\quad-
\]
on the three intervals determined by $u$ and $v$.

Put
\[
g(y):=(r(y)-c)w(y),
\qquad
T(s):=N(s)-cM(s).
\]
Since $T=M(Q-c)$ and $Q'(s_0)=0$,
$
T(s_0)=T'(s_0)=0.
$
Hence
\begin{equation}
\label{eq:fp-orthogonality}
\int_0^1
\frac{g(y)}{1+s_0y}\,dy
=
0,
\qquad
\int_0^1
\frac{yg(y)}{(1+s_0y)^2}\,dy
=
0.
\end{equation}

The sign pattern of $g$ implies
\begin{equation}
\label{eq:fp-sign-integral}
\int_0^1
\frac{
g(y)(y-u)(y-v)
}{
(1+s_0y)^3
}\,dy
<0.
\end{equation}
Using the identity
\begin{align*}
\frac{(y-u)(y-v)}{(1+s_0y)^3}
={}&
\frac{uv}{1+s_0y}
-
\frac{
(u+v+2s_0uv)y
}{
(1+s_0y)^2
}
+
\frac{
(1+s_0u)(1+s_0v)y^2
}{
(1+s_0y)^3
},
\end{align*}
the first two terms vanish after integration by
\eqref{eq:fp-orthogonality}. Since
$(1+s_0u)(1+s_0v)>0$, \eqref{eq:fp-sign-integral} therefore gives
\[
T''(s_0)
=
2\int_0^1
\frac{y^2g(y)}{(1+s_0y)^3}\,dy
<0.
\]
Differentiating $T=M(Q-c)$ twice at $s_0$ yields
$
M(s_0)Q''(s_0)=T''(s_0)<0.
$
Thus $Q''(s_0)<0$.

If $Q$ had two stationary points, both would be strict local maxima.
The minimum of $Q$ on the closed interval joining them would then be
attained at an interior point, producing a stationary point with
nonnegative second derivative, a contradiction. Hence $Q$ has at most
one stationary point.
\end{proof}

\begin{theorem}[Uniqueness of the proportional transition for $\alpha>2$]
\label{thm:fp-unique-transition}
Assume
$
\alpha>2,
\phi_T>1,
\sigma^2>0,
L_\alpha(\phi_T)>0.
$
Let $s_T>0$ be the Stage-I proportional fixed-point parameter. Then
there exists a unique
$
s_{S,\mathrm{crit}}>0
$
such that
\begin{equation}
\label{eq:fp-unique-scrit}
J_\alpha(\phi_T,\phi_S)=0
\qquad\Longleftrightarrow\qquad
s_S=s_{S,\mathrm{crit}}.
\end{equation}

Equivalently, there exists a unique proportional critical ratio
\begin{equation}
\label{eq:fp-phimcrit-definition}
\phi_{m,\mathrm{crit}}
:=
\left[
\int_0^1
\frac{dx}{1+s_{S,\mathrm{crit}}x^\alpha}
\right]^{-1}
>1
\end{equation}
such that
\begin{equation}
\label{eq:fp-transition-sign}
\mathsf{W2SG}^{\mathrm{DE}}(\phi_T,\phi_S)
\begin{cases}
>0,
&
1<\phi_S<\phi_{m,\mathrm{crit}},
\\[1mm]
=0,
&
\phi_S=\phi_{m,\mathrm{crit}},
\\[1mm]
<0,
&
\phi_S>\phi_{m,\mathrm{crit}}.
\end{cases}
\end{equation}

Equivalently, in terms of the Stage-II sampling fraction,
\begin{equation}
\label{eq:fp-transition-sign-inverse}
\mathsf{W2SG}^{\mathrm{DE}}(\phi_T,\phi_S)>0
\quad\Longleftrightarrow\quad
\phi_S^{-1}>
\phi_{m,\mathrm{crit}}^{-1}.
\end{equation}
\end{theorem}

\begin{proof}
Fix $\alpha$ and $\phi_T$, and write
$
t:=s_T,
s:=s_S.
$
Define
$
F(s)
:=
\int_0^1
\frac{dx}{1+sx^\alpha}
$
and
\begin{equation}
\label{eq:fp-B-def}
B_t(s)
:=
\int_0^1
\frac{x^\alpha}
{(1+sx^\alpha)(1+tx^\alpha)^2}\,dx.
\end{equation}
By the Stage-II fixed-point equation,
\begin{equation}
\label{eq:fp-phi-F}
F(s)=\phi_S^{-1}.
\end{equation}
Moreover,
\[
F(s)+sF'(s)
=
\int_0^1
\frac{dx}{(1+sx^\alpha)^2},
\]
and therefore
\begin{equation}
\label{eq:fp-gamma-F}
\gamma_S
=
\frac{F+sF'}{F},
\qquad
1-\gamma_S
=
-\frac{sF'}{F}.
\end{equation}
In particular, $F'(s)<0$.
Differentiating \eqref{eq:fp-B-def} under the integral sign gives
\begin{equation}
\label{eq:fp-B-identities}
B_t'(s)=-A_2,
\qquad
B_t(s)+sB_t'(s)=A_1.
\end{equation}
Using \eqref{eq:conditionofw2sg},
\eqref{eq:fp-gamma-F}, and \eqref{eq:fp-B-identities}, we obtain
\begin{align}
J_\alpha =
2B_t
+
\frac{s}{1-\gamma_S}B_t'
=
2B_t
-
\frac{F}{F'}B_t'.
\label{eq:fp-J-BF}
\end{align}
Now define
\begin{equation}
\label{eq:fp-R-def}
R_t(s)
:=
\frac{B_t(s)}{F(s)^2}.
\end{equation}
A direct differentiation of \eqref{eq:fp-R-def} and comparison with
\eqref{eq:fp-J-BF} gives
\begin{equation}
\label{eq:fp-J-Rprime}
J_\alpha
=
-\frac{F(s)^3}{F'(s)}R_t'(s).
\end{equation}
Since $F>0$ and $F'<0$, the prefactor in
\eqref{eq:fp-J-Rprime} is strictly positive. Thus
\begin{equation}
\label{eq:fp-J-R-sign}
\operatorname{sign}J_\alpha
=
\operatorname{sign}R_t'.
\end{equation}

It remains to prove that $R_t'$ has at most one zero. Put
$
\ell:=\frac1\alpha\in(0,1/2).
$
The change of variables $y=x^\alpha$ gives
\begin{equation}
\label{eq:fp-F-y}
F(s)
=
\ell\int_0^1
\frac{y^{\ell-1}}{1+sy}\,dy
\end{equation}
and
\begin{equation}
\label{eq:fp-B-y}
B_t(s)
=
\ell\int_0^1
\frac{y^\ell}
{(1+ty)^2(1+sy)}\,dy.
\end{equation}
Let $q_\ell$ and $w_\ell$ be defined by
\eqref{eq:fp-qell}--\eqref{eq:fp-well}. By
Lemma~\ref{lem:fp-stieltjes-square},
\[
F(s)^2
=
\int_0^1
\frac{w_\ell(y)}{1+sy}\,dy.
\]
Define
\begin{equation}
\label{eq:fp-r-def}
r_t(y)
:=
\frac{
y
}{
2\ell(1+ty)^2q_\ell(y)
},
\qquad 0<y<1.
\end{equation}
Then \eqref{eq:fp-B-y} becomes
\[
B_t(s)
=
\int_0^1
\frac{r_t(y)w_\ell(y)}{1+sy}\,dy,
\]
and hence
\begin{equation}
\label{eq:fp-R-stieltjes}
R_t(s)
=
\frac{
\displaystyle
\int_0^1
\frac{r_t(y)w_\ell(y)}{1+sy}\,dy
}{
\displaystyle
\int_0^1
\frac{w_\ell(y)}{1+sy}\,dy
}.
\end{equation}

We next show that $r_t$ is strictly unimodal. From
\eqref{eq:fp-qell},
\[
q_\ell(y)\sim C_\ell y^\ell
\qquad
(y\downarrow0),
\]
while
\[
q_\ell(y)\longrightarrow\infty
\qquad
(y\uparrow1),
\]
because the series in \eqref{eq:fp-qell} diverges at $y=1$.
Thus $r_t$ extends continuously to $[0,1]$ with
$
r_t(0)=r_t(1)=0,
r_t(y)>0
(0<y<1).
$
Write $y=e^v$, $v<0$. Then
\[
q_\ell(e^v)
=
C_\ell e^{\ell v}
+
\sum_{k=0}^{\infty}
\frac{e^{(k+1)v}}{k+1-\ell}.
\]
The series and its derivatives converge locally uniformly for $v<0$.
Writing its positive terms as $a_je^{\lambda_jv}$ and normalizing them
to probabilities
$
\pi_j(v)
=
\frac{
a_je^{\lambda_jv}
}{
q_\ell(e^v)
},
$
we obtain
\[
\frac{d^2}{dv^2}
\log q_\ell(e^v)
=
\sum_j\pi_j(v)\lambda_j^2
-
\left(
\sum_j\pi_j(v)\lambda_j
\right)^2
>0,
\]
because the exponents $\ell,1,2,\ldots$ are distinct. Consequently,
\begin{equation}
\label{eq:fp-log-r-concave}
\frac{d^2}{dv^2}
\log r_t(e^v)
=
-\frac{2te^v}{(1+te^v)^2}
-
\frac{d^2}{dv^2}
\log q_\ell(e^v)
<0.
\end{equation}
Hence $\log r_t(e^v)$ is strictly concave. Since $r_t$ is positive on
$(0,1)$ and vanishes at both endpoints, it has a unique interior
maximum and is strictly increasing before that point and strictly
decreasing afterward.

Lemma~\ref{lem:fp-stieltjes-curvature}, applied to
\eqref{eq:fp-R-stieltjes}, now implies
\begin{equation}
\label{eq:fp-R-curvature}
R_t'(s_0)=0
\quad\Longrightarrow\quad
R_t''(s_0)<0.
\end{equation}
In particular, $R_t$ has at most one stationary point, and by
\eqref{eq:fp-J-R-sign}, $J_\alpha$ has at most one zero as a function
of $s_S$.

It remains only to establish existence. By the assumption
$
L_\alpha(\phi_T)>0,
$
Proposition~\ref{prop:alpha-gt2-near-square} gives
$
J_\alpha>0
$
for $s_S>0$ sufficiently small. On the other hand,
Proposition~\ref{prop:alpha-gt2-sparse-negative} gives
$
J_\alpha<0
$
for $s_S$ sufficiently large. Continuity therefore gives at least one
zero $s_{S,\mathrm{crit}}$, and the preceding argument shows that this
zero is unique. Moreover,
\begin{equation}
\label{eq:fp-J-s-sign}
J_\alpha
\begin{cases}
>0,
&
0<s_S<s_{S,\mathrm{crit}},
\\[1mm]
=0,
&
s_S=s_{S,\mathrm{crit}},
\\[1mm]
<0,
&
s_S>s_{S,\mathrm{crit}}.
\end{cases}
\end{equation}

Finally, by \eqref{eq:fp-phi-F} and $F'(s)<0$,
$\phi_S^{-1}$ is strictly decreasing in $s_S$, equivalently
$\phi_S$ is strictly increasing in $s_S$. Thus
\eqref{eq:fp-J-s-sign} is exactly the sign pattern in
\eqref{eq:fp-transition-sign}. The prefactor relating
$\mathsf{W2SG}^{\mathrm{DE}}(\phi_T,\phi_S)$ and $J_\alpha$ in
\eqref{eq:conditionofw2sg} is strictly positive for $\sigma^2>0$, so the
same sign pattern holds for the proportional deterministic-equivalent
W2SG improvement.
\end{proof}

% ============================================================
\begin{theorem}[Proportional W2SG phase diagram] \label{thm:three-regime-phase-diagram} Assume
$
\alpha>1, \phi_T>1, \sigma^2>0, 0<\phi_S^{-1}<1.
$
Then the proportional deterministic-equivalent W2SG limit has the following behavior.

\begin{enumerate}

\item[\textnormal{(i)}]
\textbf{Subcritical and critical decay regime: $1<\alpha\le2$.}

If $1<\alpha<2$, there exists a unique
$
\phi_{S,\gamma}^{-1}(\alpha)\in(0,1)
$
such that
$
\gamma_S=\frac12.
$
For
$
0<\phi_S^{-1}
\le
\phi_{S,\gamma}^{-1}(\alpha),
$
the coefficient
\[
\frac{1-2\gamma_S}{1-\gamma_S}
\]
is nonnegative, so positivity follows directly from
\eqref{eq:conditionofw2sg}. For
$
\phi_{S,\gamma}^{-1}(\alpha)
<
\phi_S^{-1}
<
1,
$
one has $\gamma_S>1/2$, but
Theorem~\ref{thm:threecase-global-positive} still gives
\[
2A_1
>
\frac{2\gamma_S-1}{1-\gamma_S}
s_SA_2.
\]

At the critical exponent $\alpha=2$, there is no interior
$\gamma_S=1/2$ crossing. Instead,
\[
\gamma_S>\frac12
\qquad
\forall\,\phi_S^{-1}\in(0,1),
\]
with
\[
\gamma_S\downarrow\frac12
\qquad
\text{as }\phi_S^{-1}\downarrow0.
\]
Nevertheless, Theorem~\ref{thm:threecase-global-positive} again yields
\[
2A_1
>
\frac{2\gamma_S-1}{1-\gamma_S}
s_SA_2.
\]

Consequently,
\begin{equation}
\label{eq:final-alpha-less2}
1<\alpha\le2
\quad\Longrightarrow\quad
\mathsf{W2SG}^{\mathrm{DE}}(\phi_T,\phi_S)>0
\qquad
\forall\,\phi_S^{-1}\in(0,1).
\end{equation}

\item[\textnormal{(ii)}]
\textbf{Supercritical unconditional transition regime:
$2<\alpha\le1+\sqrt2$.}

In this regime,
\[
\gamma_S>\frac12
\qquad
\forall\,\phi_S^{-1}\in(0,1),
\]
and W2SG is governed by
\[
2A_1
>
\frac{2\gamma_S-1}{1-\gamma_S}
s_SA_2.
\]

By Proposition~\ref{prop:alpha-gt2-sparse-negative},
$
\mathsf{W2SG}^{\mathrm{DE}}(\phi_T,\phi_S)<0
$
for all sufficiently small $\phi_S^{-1}>0$.
By Corollary~\ref{cor:alpha-gt2-unconditional-range},
\[
L_\alpha(\phi_T)>0
\qquad
\text{for every }\phi_T>1,
\]
including the endpoint $\alpha=1+\sqrt2$. Hence
Proposition~\ref{prop:alpha-gt2-near-square} gives
$
\mathsf{W2SG}^{\mathrm{DE}}(\phi_T,\phi_S)>0
$
for $\phi_S^{-1}$ sufficiently close to $1$.

Finally, Theorem~\ref{thm:fp-unique-transition} implies that there is a
unique proportional critical ratio
$
\phi_{m,\mathrm{crit}}>1
$
such that
\[
\mathsf{W2SG}^{\mathrm{DE}}(\phi_T,\phi_S)
\begin{cases}
>0,
&
1<\phi_S<\phi_{m,\mathrm{crit}},
\\[1mm]
=0,
&
\phi_S=\phi_{m,\mathrm{crit}},
\\[1mm]
<0,
&
\phi_S>\phi_{m,\mathrm{crit}}.
\end{cases}
\]

Equivalently,
\[
\mathsf{W2SG}^{\mathrm{DE}}(\phi_T,\phi_S)>0
\quad\Longleftrightarrow\quad
\phi_S^{-1}>
\phi_{m,\mathrm{crit}}^{-1}.
\]

\item[\textnormal{(iii)}]
\textbf{Conditional supercritical regime:
$\alpha>1+\sqrt2$.}

The sparse proportional region remains negative. If
$
L_\alpha(\phi_T)>0,
$
then Theorem~\ref{thm:fp-unique-transition} again yields a unique
negative-to-positive proportional transition.
If
$
L_\alpha(\phi_T)\le0,
$
the present analysis proves negativity near the sparse proportional
boundary but does not determine whether a positive W2SG region exists
at an intermediate proportional sampling fraction.

\end{enumerate}

\end{theorem}

\subsection{Sparse-growing Stage-II sample size}
\label{subsec:sparse-growing-extension}

\begin{assumption}[Sparse-growing Stage-II DE continuation]
\label{ass:adapted-sparse-DE}

Assume the power-law covariance model of
Assumption~\ref{ass:powerlaw-covariance}. Stage~I remains in the
proportional overparameterized regime,
$
\frac{p}{n}\longrightarrow\phi_T>1,~
\gamma_{T,p}\longrightarrow\gamma_T\in(0,1).
$
Stage~II follows the sparse-growing regime
$
p\to\infty,~
m=m_p\to\infty,~
\frac{m_p}{p}\to0.
$
Along this sequence, we assume that the finite-$p$
deterministic-equivalent expression
\eqref{eq:w2sg-fully-deterministic} continues to apply, with
$\tau_{S,p}>0$ determined by
\begin{equation}
\sum_{i=1}^p
\frac{\lambda_i}{\lambda_i+\tau_{S,p}}
=
m_p,
\label{eq:sparse-growing-fixed-point}
\end{equation}
and
\[
\gamma_{S,p}
=
\frac1{m_p}
\sum_{i=1}^p
\left(
\frac{\lambda_i}{\lambda_i+\tau_{S,p}}
\right)^2.
\]

This continuation assumption is not implied by the standard
proportional deterministic-equivalence theorem, because
$m_p/p\to0$.
\end{assumption}
\begin{theorem}[Sparse-growing DE W2SG limit]
\label{thm:adapted-sparse-DE}

Under Assumption~\ref{ass:adapted-sparse-DE},
\begin{equation}
\mathsf{W2SG}_p^{\mathrm{DE}}
\longrightarrow
(2-\alpha)
\sigma^2
\frac{\gamma_T}{1-\gamma_T}.
\label{eq:adapted-sparse-DE-limit}
\end{equation}

Consequently, if $\sigma^2>0$, the limiting gain is positive for
$1<\alpha<2$, zero for $\alpha=2$, and negative for $\alpha>2$.
If $\sigma^2=0$, the limiting gain is zero for every $\alpha>1$.
\end{theorem}

\begin{proof}
Let
\[
C_\alpha
:=
\int_0^\infty\frac{du}{1+u^\alpha},
\qquad
c_\alpha:=C_\alpha^\alpha.
\]
Rescaling \eqref{eq:sparse-growing-fixed-point} by $i=m_px$ gives
\begin{equation}
m_p^\alpha\tau_{S,p}\longrightarrow c_\alpha.
\label{eq:sparse-growing-tau-scaling}
\end{equation}
In particular,
$
\tau_{S,p}\sim c_\alpha m_p^{-\alpha}\longrightarrow0,
$
not $\infty$. The same Riemann-sum argument gives
\begin{equation}
\gamma_{S,p}\longrightarrow
\int_0^\infty
\frac{dx}{(1+c_\alpha x^\alpha)^2}
=
\frac{\alpha-1}{\alpha}.
\label{eq:sparse-growing-gamma-limit}
\end{equation}
The last equality follows from the identity used in
Lemma~\ref{lem:threecase-gamma-monotone}, after rescaling the integral
to $[0,\infty)$.

For the remainder of the proof, write
\[
r_{i,p}
:=
2\frac{\tau_{S,p}}{\lambda_i+\tau_{S,p}}
-
\frac{1}{1-\gamma_{S,p}}
\left(
\frac{\tau_{S,p}}{\lambda_i+\tau_{S,p}}
\right)^2
\]
and define the Stage-I noise weights
\begin{equation}
w_{i,p}
:=
\frac{\sigma^2}{n(1-\gamma_{T,p})}
\left(
\frac{\lambda_i}{\lambda_i+\tau_{T,p}}
\right)^2.
\label{eq:stage1-noise-weights}
\end{equation}
Define the limiting Stage-I noise mass by
$
V_T^{\mathrm{noise}}
:=
\sigma^2\frac{\gamma_T}{1-\gamma_T}.
$
Then
\begin{equation}
\sum_{i=1}^p w_{i,p}
=
\sigma^2\frac{\gamma_{T,p}}{1-\gamma_{T,p}}
\longrightarrow
V_T^{\mathrm{noise}}.
\label{eq:stage1-noise-weight-mass}
\end{equation}
Using \eqref{eq:w2sg-fully-deterministic}, the gain can be written as
\begin{equation}
\mathsf{W2SG}_p^{\mathrm{DE}}
=T_p^{\mathrm{teach}}
+
\sum_{i=1}^p r_{i,p}w_{i,p},
\label{eq:sparse-growing-decomposition}
\end{equation}
where
\[
T_p^{\mathrm{teach}}
=
-\frac{\tau_{S,p}^2}{1-\gamma_{S,p}}
\sum_{i=1}^p
\frac{\lambda_i^2}
{(\lambda_i+\tau_{S,p})^2(\lambda_i+\tau_{T,p})}.
\]

The teacher term vanishes. Indeed,
\[
0\le -T_p^{\mathrm{teach}}
\le
\frac{\tau_{S,p}^2}{1-\gamma_{S,p}}
\sum_{i=1}^p
\frac{\lambda_i}{(\lambda_i+\tau_{S,p})^2}.
\]
If
$a_{i,p}=\lambda_i/(\lambda_i+\tau_{S,p})$, then the fixed-point
equation implies
\[
m_p(1-\gamma_{S,p})
=
\tau_{S,p}
\sum_{i=1}^p
\frac{\lambda_i}{(\lambda_i+\tau_{S,p})^2}.
\]
Consequently,
\begin{equation}
0\le -T_p^{\mathrm{teach}}
\le m_p\tau_{S,p}
=O(m_p^{1-\alpha})
\longrightarrow0.
\label{eq:sparse-growing-teacher-vanishes}
\end{equation}

It remains to treat the noise term. The pointwise limit at a fixed
index is not useful here: because $\tau_{S,p}\to0$, it would give
$\tau_{S,p}/(\lambda_i+\tau_{S,p})\to0$. Instead, the weights
$w_{i,p}$ are asymptotically carried by indices of order $p$. More precisely,
\[
\lim_{\delta\downarrow0}
\limsup_{p\to\infty}
\sum_{i\le\delta p}w_{i,p}
=
0,
\]
by the Stage-I Riemann-sum limit. On the complementary set
$i\ge\delta p$, \eqref{eq:sparse-growing-tau-scaling} gives
\[
\sup_{i\ge\delta p}
\frac{\lambda_i}{\tau_{S,p}}
\le
O\!\left(
\left(\frac{m_p}{p}\right)^\alpha
\right)
\longrightarrow0.
\]
Hence, uniformly for $i\ge\delta p$,
\[
r_{i,p}
\longrightarrow
2-
\frac1{1-(\alpha-1)/\alpha}
=2-\alpha.
\]
The coefficients $r_{i,p}$ are uniformly bounded for all sufficiently
large $p$ because \eqref{eq:sparse-growing-gamma-limit} keeps
$1-\gamma_{S,p}$ bounded away from zero. Splitting the weighted sum at
$\delta p$, then sending $p\to\infty$ and $\delta\downarrow0$, yields
\begin{equation}
\sum_{i=1}^p r_{i,p}w_{i,p}
\longrightarrow
(2-\alpha)V_T^{\mathrm{noise}}.
\label{eq:sparse-growing-noise-transfer}
\end{equation}
Combining \eqref{eq:sparse-growing-decomposition},
\eqref{eq:sparse-growing-teacher-vanishes}, and
\eqref{eq:sparse-growing-noise-transfer} proves
\eqref{eq:adapted-sparse-DE-limit}.
\end{proof}
\subsection{Ultra-sparse Stage-II extensions}

We next consider the fixed-Stage-II-sample regime, which lies outside
the standard proportional asymptotic setting.
\begin{assumption}[Fixed-$m$ deterministic-equivalent continuation]
\label{ass:adapted-fixed-DE}

Assume the power-law covariance model of
Assumption~\ref{ass:powerlaw-covariance}. Stage~I remains in the
proportional overparameterized regime,
\[
\frac{p}{n}\longrightarrow\phi_T>1,
\qquad
\gamma_{T,p}\longrightarrow\gamma_T\in(0,1),
\]
while the Stage~II sample size $m\ge1$ is fixed as $p\to\infty$.

Along this sequence, we assume that the finite-$p$
deterministic-equivalent expression
\eqref{eq:w2sg-fully-deterministic} continues to apply, with
$\tau_{m,S,p}>0$ determined by
\begin{equation}
\sum_{i=1}^p
\frac{\lambda_i}{\lambda_i+\tau_{m,S,p}}
=
m,
\label{eq:fixed-m-finite-p-fixed-point}
\end{equation}
and
\[
\gamma_{m,S,p}
=
\frac1m
\sum_{i=1}^p
\left(
\frac{\lambda_i}{\lambda_i+\tau_{m,S,p}}
\right)^2.
\]

This continuation assumption is not implied by the proportional
deterministic-equivalence theorem, which requires a nonvanishing
Stage~II sample-to-dimension ratio.
\end{assumption}
\subsubsection{Fixed Stage-II sample count}
\label{subsec:fixed-m-extension}
\begin{lemma}[Fixed-$m$ spectral parameters]
\label{lem:fixed-m-spectral-parameters}
For every fixed $m\ge1$, there is a unique $\tau_{m,S}>0$ satisfying
\begin{equation}
\sum_{i=1}^{\infty}
\frac{\lambda_i}{\lambda_i+\tau_{m,S}}
=m.
\label{eq:fixed-m-infinite-fixed-point}
\end{equation}
Moreover,
\[
\tau_{m,S,p}\longrightarrow\tau_{m,S},
\qquad
\gamma_{m,S,p}\longrightarrow
\gamma_{m,S}
:=
\frac1m
\sum_{i=1}^{\infty}
\left(
\frac{\lambda_i}{\lambda_i+\tau_{m,S}}
\right)^2,
\]
where $0<\gamma_{m,S}<1$. Define
\begin{equation}
q_{m,S}
:=
2-\frac1{1-\gamma_{m,S}}
=
\frac{1-2\gamma_{m,S}}{1-\gamma_{m,S}}.
\label{eq:fixed-m-qm}
\end{equation}
\end{lemma}

\begin{proof}
Let
\[
f(\tau)
:=
\sum_{i=1}^{\infty}
\frac{\lambda_i}{\lambda_i+\tau}.
\]
Because $\alpha>1$, this series is finite for every $\tau>0$.
It is continuous and strictly decreasing, with
$f(\tau)\to\infty$ as $\tau\downarrow0$ and $f(\tau)\to0$ as
$\tau\to\infty$. Thus \eqref{eq:fixed-m-infinite-fixed-point} has a
unique solution. Monotone convergence of the finite partial sums,
together with uniqueness of the limiting root, gives
$\tau_{m,S,p}\to\tau_{m,S}$. The convergence of $\gamma_{m,S,p}$ follows by
dominated convergence.

Finally, with
$a_i=\lambda_i/(\lambda_i+\tau_{m,S})$, one has $0<a_i<1$ and
$\sum_i a_i=m$. Therefore
\[
0<\sum_i a_i^2<\sum_i a_i=m,
\]
which proves $0<\gamma_{m,S}<1$.
\end{proof}

\begin{theorem}[Fixed-$m$ deterministic-equivalent W2SG gain]
\label{thm:fixed-m-DE-gain}
Under Assumption~\ref{ass:adapted-fixed-DE}, for every fixed integer
$m\ge1$,
\begin{equation}
\mathsf{W2SG}_p^{\mathrm{DE}} \longrightarrow G_m := -m\tau_{m,S} + q_{m,S} \sigma^2\frac{\gamma_T}{1-\gamma_T}.
\label{eq:fixed-m-DE-gain}
\end{equation}
\end{theorem}

\begin{proof}
Write
\[
\mathsf{W2SG}_p^{\mathrm{DE}}
=T_{m,p}^{\mathrm{teach}}
+
\sum_{i=1}^p r_{m,i,p}w_{i,p},
\]
where
\begin{align*}
w_{i,p}
:=&
\frac{\sigma^2}
{n\left(1-\gamma_{T,p}\right)}
\left(
\frac{\lambda_i}
{\lambda_i+\tau_{T,p}}
\right)^2,
\\
T_{m,p}^{\mathrm{teach}}
=&
-\frac{\tau_{m,S,p}^2}{1-\gamma_{m,S,p}}
\sum_{i=1}^p
\frac{\lambda_i^2}
{(\lambda_i+\tau_{m,S,p})^2(\lambda_i+\tau_{T,p})},
\\
r_{m,i,p}
=&
2\frac{\tau_{m,S,p}}{\lambda_i+\tau_{m,S,p}}
-
\frac1{1-\gamma_{m,S,p}}
\left(
\frac{\tau_{m,S,p}}{\lambda_i+\tau_{m,S,p}}
\right)^2.
\end{align*}

Since $\tau_{T,p}\to0$ under the proportional Stage-I scaling
\eqref{eq:tauj-asymptotic}, Lemma~\ref{lem:fixed-m-spectral-parameters}
and dominated convergence give
\begin{align}
T_{m,p}^{\mathrm{teach}}
&\longrightarrow
-\frac{\tau_{m,S}^2}{1-\gamma_{m,S}}
\sum_{i=1}^{\infty}
\frac{\lambda_i}{(\lambda_i+\tau_{m,S})^2}.
\label{eq:fixed-m-teacher-preidentity}
\end{align}
The fixed-point identity implies
\begin{align}
m(1-\gamma_{m,S})
=&
\sum_{i=1}^{\infty}
\left[
\frac{\lambda_i}{\lambda_i+\tau_{m,S}}
-
\left(
\frac{\lambda_i}{\lambda_i+\tau_{m,S}}
\right)^2
\right]
=
\tau_{m,S}
\sum_{i=1}^{\infty}
\frac{\lambda_i}{(\lambda_i+\tau_{m,S})^2}.
\label{eq:fixed-m-identity}
\end{align}
Substitution into \eqref{eq:fixed-m-teacher-preidentity} yields
\begin{equation}
T_{m,p}^{\mathrm{teach}}
\longrightarrow -m\tau_{m,S}.
\label{eq:fixed-m-teacher-limit}
\end{equation}

For the noise term, $r_{m,i,p}$ is uniformly bounded for large $p$,
and Lemma~\ref{lem:fixed-m-spectral-parameters} gives
\[
\lim_{K\to\infty}
\limsup_{p\to\infty}
\sup_{K<i\le p}|r_{m,i,p}-q_{m,S}|=0,
\]
because $\lambda_i\to0$. For each fixed $K$,
\[
\sum_{i=1}^K w_{i,p}\longrightarrow0,
\qquad
\sum_{i=1}^{p} w_{i,p}
=
\sigma^2
\frac{\gamma_{T,p}}{1-\gamma_{T,p}}
\longrightarrow
V_T^{\mathrm{noise}}.
\] 
Splitting the sum at $K$
therefore yields
\begin{equation}
\sum_{i=1}^p r_{m,i,p}w_{i,p}
\longrightarrow
q_{m,S}V_T^{\mathrm{noise}}.
\label{eq:fixed-m-noise-limit}
\end{equation}
Combining \eqref{eq:fixed-m-teacher-limit} and
\eqref{eq:fixed-m-noise-limit} proves
\eqref{eq:fixed-m-DE-gain}.
\end{proof}

\subsubsection{Analytic sign results for the \texorpdfstring{fixed-$m$}{fixed-m} gain}
\label{fmphase:subsec:analytic-sign-results}

We now analyze the sign of the fixed-$m$ limiting gain as a continuous
function of the real variable $m>0$. Throughout this subsection, assume
\[
\alpha>1,
\qquad
\lambda_i=i^{-\alpha},
\qquad
i\in\mathbb N:=\{1,2,\ldots\}.
\]

Let
\begin{equation}
\label{fmphase:eq:teacher-factor}
0<\gamma_T<1,
\qquad
c_T:=\frac{\gamma_T}{1-\gamma_T}>0,
\qquad
V:=c_T\sigma^2,
\qquad
\sigma^2\ge0.
\end{equation}

For every real $m>0$, consider the fixed-point equation
\begin{equation}
\label{fmphase:eq:fixed-point}
f_\alpha(\tau)
:=
\sum_{i=1}^{\infty}
\frac{1}{1+\tau i^\alpha}
=
m.
\end{equation}
Lemma~\ref{fmphase:lem:regularity} below shows that this equation has a
unique solution $\tau>0$; we denote that solution by $\tau_m$.
Set
\begin{align}
a_i(m)
:=&
\frac{1}{1+\tau_m i^\alpha},
\nonumber\\
\gamma_m
:=&
\frac1m\sum_{i=1}^{\infty}a_i(m)^2,
&
\eta_m
:=&
\frac1m\sum_{i=1}^{\infty}a_i(m)^3,
\label{fmphase:eq:moments}\\
D_m
:=&
m(1-\gamma_m),
&
B_m
:=&
m\tau_m,
\label{fmphase:eq:D-B}\\
q_m
:=&
\frac{1-2\gamma_m}{1-\gamma_m}
=
2-\frac{m}{D_m},
&
G_m
:=&
-B_m+Vq_m.
\label{fmphase:eq:gain}
\end{align}

For integer $m$, uniqueness of the fixed-point solution implies
$
\tau_m=\tau_{m,S},~
\gamma_m=\gamma_{m,S},~
q_m=q_{m,S}.
$
Consequently, for integer $m$, $G_m$ coincides with the limiting
deterministic-equivalent gain in
Theorem~\ref{thm:fixed-m-DE-gain}.

\paragraph{Fixed-point regularity and exact identities}

\begin{lemma}[Existence, uniqueness, and differentiability]\label{fmphase:lem:regularity}
For every real $m>0$, \eqref{fmphase:eq:fixed-point} has a unique solution
$\tau_m>0$. The functions used in \eqref{fmphase:eq:moments}--\eqref{fmphase:eq:gain}
are continuously differentiable, and
\begin{equation}\label{fmphase:eq:gamma-range}
0<\gamma_m<1,\qquad D_m>0,\qquad q_m<1.
\end{equation}
Moreover,
\begin{align}
\tau_m'=&-\frac{\tau_m}{D_m},\label{fmphase:eq:tau-prime}\\
a_i'(m)=&\frac{a_i(m)(1-a_i(m))}{D_m},\label{fmphase:eq:a-prime}\\
B_m'=&-\frac{\tau_m\gamma_m}{1-\gamma_m}<0,\label{fmphase:eq:B-prime}\\
m\gamma_m'=&\frac{\gamma_m+\gamma_m^2-2\eta_m}{1-\gamma_m}.
\label{fmphase:eq:gamma-prime}
\end{align}
\end{lemma}

\begin{proof}
For each $\tau>0$, the defining series is bounded above by
$\tau^{-1}\sum_i i^{-\alpha}<\infty$.
It is continuous and strictly decreasing in $\tau$, tends to zero as
$\tau\to\infty$, and tends to infinity as $\tau\downarrow0$.
The last claim follows by retaining an arbitrarily large finite number
of terms, each tending to one.
These properties prove existence and uniqueness.

On any compact subinterval of $(0,\infty)$, the summands and their
$\tau$-derivatives are dominated by summable constant multiples of
$i^{-\alpha}$. Termwise differentiation and the inverse function theorem
therefore apply. Since $0<a_i<1$ and $\sum_i a_i=m$,
\[
0<\sum_i a_i^2<\sum_i a_i=m,
\]
which proves \eqref{fmphase:eq:gamma-range}. Also,
\[
-\tau_m f_\alpha'(\tau_m)
=\sum_i a_i(1-a_i)=m-\sum_i a_i^2=D_m.
\]
Differentiating \eqref{fmphase:eq:fixed-point} yields \eqref{fmphase:eq:tau-prime};
then the chain rule gives \eqref{fmphase:eq:a-prime} and
\[
B_m'=\tau_m+m\tau_m'
=\tau_m-\frac{\tau_m}{1-\gamma_m},
\]
which is \eqref{fmphase:eq:B-prime}.
Finally,
\[
\frac{\mathrm d}{\mathrm dm}\sum_i a_i^2
=\frac{2}{D_m}\sum_i a_i^2(1-a_i)
=\frac{2(\gamma_m-\eta_m)}{1-\gamma_m}.
\]
Differentiating $m\gamma_m=\sum_i a_i^2$ and rearranging proves
\eqref{fmphase:eq:gamma-prime}.
\end{proof}

\begin{definition}[Moment expression and noise thresholds]
Define
\begin{equation}\label{fmphase:eq:K}
K_m:=\gamma_m^2(2-\gamma_m)-\eta_m.
\end{equation}
On the set where $q_m>0$, define thresholds in $V$ and in $\sigma^2$ by
\begin{equation}\label{fmphase:eq:thresholds}
T_\alpha(m):=\frac{B_m}{q_m},\qquad
H_\alpha(m):=\frac{B_m}{c_T q_m}
=\frac{T_\alpha(m)}{c_T}.
\end{equation}
\end{definition}

\begin{proposition}[Exact sign criterion and threshold derivative]
\label{fmphase:prop:threshold}
For every $m\ge1$,
\begin{align}
q_m\le0&\quad\Longrightarrow\quad G_m<0
\quad\text{for every }\sigma^2\ge0,\label{fmphase:eq:qnegative}\\
q_m>0&\quad\Longrightarrow\quad
\operatorname{sgn} G_m=\operatorname{sgn}\bigl(\sigma^2-H_\alpha(m)\bigr).
\label{fmphase:eq:pointwise-sign}
\end{align}
Where $q_m>0$, one has
\begin{equation}\label{fmphase:eq:H-prime}
\frac{H_\alpha'(m)}{H_\alpha(m)}
=\frac{T_\alpha'(m)}{T_\alpha(m)}
=\frac{2K_m}{m(1-\gamma_m)^2(1-2\gamma_m)}.
\end{equation}
\end{proposition}

\begin{proof}
Since $B_m>0$, \eqref{fmphase:eq:qnegative} follows directly from
\eqref{fmphase:eq:gain}. If $q_m>0$, factor the gain as
$
G_m=c_T q_m\bigl(\sigma^2-H_\alpha(m)\bigr).
$
For the derivative identity, first rewrite \eqref{fmphase:eq:gamma-prime} as
\begin{equation}\label{fmphase:eq:gamma-K}
m\gamma_m'
=\gamma_m(1-2\gamma_m)+\frac{2K_m}{1-\gamma_m}.
\end{equation}
By \eqref{fmphase:eq:B-prime},
\[
\frac{B_m'}{B_m}=-\frac{\gamma_m}{m(1-\gamma_m)},
\qquad
\frac{q_m'}{q_m}
=-\frac{\gamma_m'}{(1-\gamma_m)(1-2\gamma_m)}.
\]
Subtracting the latter identity from the former and using
\eqref{fmphase:eq:gamma-K} proves \eqref{fmphase:eq:H-prime}.
\end{proof}

\paragraph{Integral comparisons and asymptotics}

Define
\begin{equation}\label{fmphase:eq:C-alpha}
C_\alpha:=\int_0^\infty\frac{\,\mathrm{d} x}{1+x^\alpha}
=\frac{\pi}{\alpha\sin(\pi/\alpha)}.
\end{equation}
The integral is evaluated by $u=x^\alpha$ and Euler's beta integral.

\begin{lemma}[Explicit fixed-point bounds]
\label{fmphase:lem:bounds}
For every $m>0$,
\begin{equation}
\label{fmphase:eq:tau-bounds}
\left(\frac{C_\alpha}{m+1}\right)^\alpha
<
\tau_m
<
\left(\frac{C_\alpha}{m}\right)^\alpha,
\end{equation}
and
\begin{equation}
\label{fmphase:eq:D-bounds}
\frac{m}{\alpha}-\frac14
<
D_m
<
\frac{m+1}{\alpha}+\frac14.
\end{equation}
Consequently,
\begin{equation}
\label{fmphase:eq:B-bounds}
C_\alpha^\alpha\frac{m}{(m+1)^\alpha}
<
B_m
<
C_\alpha^\alpha m^{1-\alpha}.
\end{equation}
\end{lemma}

\begin{proof}
Let
\[
g(x):=\frac{1}{1+\tau_m x^\alpha}.
\]
Since $g$ is strictly decreasing on $(0,\infty)$, the standard integral
comparison is strict:
\[
\sum_{i\ge1}g(i)
<
\int_0^\infty g(x)\,\mathrm{d}x
<
g(0)+\sum_{i\ge1}g(i).
\]
Using
\[
\sum_{i\ge1}g(i)=m,
\qquad
g(0)=1,
\qquad
\int_0^\infty g(x)\,\mathrm{d}x
=
C_\alpha\tau_m^{-1/\alpha},
\]
we obtain
\begin{equation}
\label{fmphase:eq:integral-fixedpoint}
m
<
C_\alpha\tau_m^{-1/\alpha}
<
m+1.
\end{equation}
Solving these inequalities for $\tau_m$ gives
\eqref{fmphase:eq:tau-bounds}. Since $B_m=m\tau_m$, multiplication
by $m>0$ gives \eqref{fmphase:eq:B-bounds}.

For the bounds on $D_m$, define
\[
h(x):=
\frac{\tau_m x^\alpha}
     {(1+\tau_m x^\alpha)^2}.
\]
The function $h$ increases from zero to its maximum $1/4$ and then
decreases to zero. For such an absolutely continuous unimodal function,
\begin{equation}
\label{fmphase:eq:unimodal-error}
\left|
\sum_{i\ge1}h(i)
-
\int_0^\infty h(x)\,\mathrm{d}x
\right|
\le
\max h
=
\frac14.
\end{equation}
To justify this bound, writing
$
\{x\}:=x-\lfloor x\rfloor
$
gives
\[
\sum_{i\ge1}h(i)
-
\int_0^\infty h(x)\,\mathrm{d}x
=
\int_0^\infty \{x\}h'(x)\,\mathrm{d}x.
\]
On the increasing part, the last integral contributes between zero and
$\max h$, while on the decreasing part it contributes between
$-\max h$ and zero. This proves
\eqref{fmphase:eq:unimodal-error}. Moreover,
$
h(x)=-\frac{xg'(x)}{\alpha}.
$
Hence integration by parts gives
\[
\int_0^\infty h(x)\,\mathrm{d}x
=
\frac{1}{\alpha}
\int_0^\infty g(x)\,\mathrm{d}x
=
\frac{C_\alpha}{\alpha}\tau_m^{-1/\alpha}.
\]
Since
\[
D_m
=
\sum_{i\ge1}h(i),
\]
equation \eqref{fmphase:eq:unimodal-error} yields
\[
\frac{C_\alpha}{\alpha}\tau_m^{-1/\alpha}-\frac14
\le
D_m
\le
\frac{C_\alpha}{\alpha}\tau_m^{-1/\alpha}+\frac14.
\]
Combining these inequalities with the strict comparison
\eqref{fmphase:eq:integral-fixedpoint},
\[
\frac{m}{\alpha}
<
\frac{C_\alpha}{\alpha}\tau_m^{-1/\alpha}
<
\frac{m+1}{\alpha},
\]
gives
\[
\frac{m}{\alpha}-\frac14
<
D_m
<
\frac{m+1}{\alpha}+\frac14,
\]
which proves \eqref{fmphase:eq:D-bounds}.
\end{proof}
\begin{lemma}[Moment and gain asymptotics]\label{fmphase:lem:asymptotics}
Fix $\alpha>1$ and write $b=1-1/\alpha$. As $m\to\infty$,
\begin{align}
\tau_m&\sim C_\alpha^\alpha m^{-\alpha},
&B_m&\sim C_\alpha^\alpha m^{1-\alpha},\label{fmphase:eq:tau-asymp}\\
\gamma_m&\longrightarrow b,
&\eta_m&\longrightarrow\frac{b(1+b)}2,\label{fmphase:eq:mom-asymp}\\
q_m&\longrightarrow2-\alpha,
&K_m&\longrightarrow\frac{b(1-b)(2b-1)}2.\label{fmphase:eq:q-K-asymp}
\end{align}
In particular,
\begin{equation}\label{fmphase:eq:G-limit}
\lim_{m\to\infty}G_m=(2-\alpha)V.
\end{equation}
\end{lemma}

\begin{proof}
The first asymptotic in \eqref{fmphase:eq:tau-asymp} follows by squeezing in
\eqref{fmphase:eq:tau-bounds}; the second follows by multiplication by $m$.
For the remaining moments, let $z_m=\tau_m^{-1/\alpha}\to\infty$ and
\[
S_k(\alpha,z)=\sum_{i\ge1}\frac1{(1+(i/z)^\alpha)^k},\qquad
C_{\alpha,k}=\int_0^\infty\frac{\,\mathrm{d} x}{(1+x^\alpha)^k}.
\]
Integral comparison gives
$
0\le zC_{\alpha,k}-S_k(\alpha,z)\le1.$
Integration of the derivative of $x(1+x^\alpha)^{-k}$ over $(0,\infty)$
yields
\[
C_{\alpha,k+1}=\left(1-\frac1{k\alpha}\right)C_{\alpha,k}.
\]
Since $C_{\alpha,1}=C_\alpha$, this gives
\[
C_{\alpha,2}=bC_\alpha,
\qquad
C_{\alpha,3}=\frac{b(1+b)}2C_\alpha.
\]
At $z=z_m$, $S_1=m$, $\gamma_m=S_2/S_1$, and $\eta_m=S_3/S_1$.
Taking ratios proves \eqref{fmphase:eq:mom-asymp}.
Substitution in the definitions of $q_m$ and $K_m$ gives
\eqref{fmphase:eq:q-K-asymp}. Finally, $B_m\to0$, proving \eqref{fmphase:eq:G-limit}.
\end{proof}

\paragraph{The initial noise threshold}

\begin{lemma}[Positivity of $q_1$ in the exponent range of interest]
\label{fmphase:lem:q1}
For $1<\alpha\le5/2$, one has $q_1>0$.
\end{lemma}

\begin{proof}
If $1<\alpha\le2$, then
\[
f_\alpha(1)=\sum_{i\ge1}\frac1{1+i^\alpha}
\ge\sum_{i\ge1}\frac1{1+i^2}
>\sum_{i\ge1}\frac1{i(i+1)}=1.
\]
The strict inequality holds termwise for $i\ge2$.
Hence $\tau_1>1$. It follows that $a_i(1)\le a_1(1)<1/2$, so
\[
\gamma_1=\sum_i a_i(1)^2\le a_1(1)\sum_i a_i(1)=a_1(1)<1/2.
\]

For $2<\alpha\le5/2$, we first prove $\tau_1>2/3$ using a finite
rational lower bound. Let
\[
(k_1,\ldots,k_{10})=(100,142,174,200,224,245,265,283,300,317).
\]
Each $k_i^2\ge10000i$, so $\sqrt{i}\le k_i/100$. Therefore
\[
f_\alpha(2/3)
\ge\sum_{i=1}^{10}\frac1{1+(2/3)i^{5/2}}
\ge\sum_{i=1}^{10}\frac{150}{150+i^2k_i}
>\frac{101}{100}>1.
\]
The penultimate inequality is an inequality between the displayed finite
sum of rational numbers and $101/100$, verified by bringing the fractions
to a common denominator. Thus $\tau_1>2/3$.
Consequently, $a_1(1)<3/5$ and $a_i(1)\le a_2(1)<3/11$ for $i\ge2$.
Using $\sum_i a_i(1)=1$ gives
\[
\gamma_1\le a_1(1)^2+\frac3{11}\bigl(1-a_1(1)\bigr)
\le\frac{129}{275}<\frac12.
\]
Here the quadratic is bounded by its maximum at the endpoints of
$[0,3/5]$. In both cases $\gamma_1<1/2$, hence $q_1>0$.
\end{proof}

\begin{definition}[Initial critical noise]\label{fmphase:def:sigc}
For $1<\alpha\le5/2$, define
\begin{equation}\label{fmphase:eq:sigc}
\sigma_c^2:=H_\alpha(1)=\frac{\tau_1}{c_T q_1}>0.
\end{equation}
This is the noise level at which $G_1=0$. It is not, by definition, a
threshold guaranteeing the same sign for every $m$.
\end{definition}

\begin{corollary}\label{fmphase:cor:initial-sign}
For $1<\alpha\le5/2$,
\[
G_1\begin{cases}
<0,&0\le\sigma^2<\sigma_c^2,\\
=0,&\sigma^2=\sigma_c^2,\\
>0,&\sigma^2>\sigma_c^2.
\end{cases}
\]
For $\sigma^2<\sigma_c^2$, continuity gives $G_m<0$ on some real interval
$[1,1+\delta)$, with $\delta>0$.
\end{corollary}
\begin{proof}
Use $G_1=c_T q_1(\sigma^2-\sigma_c^2)$ and $c_T q_1>0$.
\end{proof}

\paragraph{The regime $1<\alpha<2$}

\begin{theorem}[Negative initial gain and positive tail]
\label{fmphase:thm:subcritical-endpoints}
Suppose $1<\alpha<2$ and $0<\sigma^2<\sigma_c^2$.
Then $G_m<0$ near $m=1$, and $G_m>0$ for every sufficiently large $m$.
In particular, the continuous extension has at least one zero in
$(1,\infty)$.
\end{theorem}

\begin{proof}
The initial sign follows from Corollary~\ref{fmphase:cor:initial-sign}.
Since $V>0$ and $2-\alpha>0$, Lemma~\ref{fmphase:lem:asymptotics} gives
$G_m\to(2-\alpha)V>0$. The intermediate value theorem proves the
existence of a zero.
\end{proof}

\begin{remark}
Theorem~\ref{fmphase:thm:subcritical-endpoints} does not prove uniqueness of the
zero. It does not exclude additional changes of sign at intermediate $m$.
Positive noise is essential: for $\sigma^2=0$, $G_m=-B_m<0$ for every $m$.
\end{remark}

\begin{corollary}[An explicit sufficient positive region]
\label{fmphase:cor:explicit-positive}
Suppose $1<\alpha<2$ and $V>0$. If
\begin{equation}\label{fmphase:eq:positive-region}
m\ge\frac{\alpha(\alpha+2)}{4(2-\alpha)}
\quad\text{and}\quad
m>\left[\frac{2C_\alpha^\alpha}{(2-\alpha)V}\right]^{1/(\alpha-1)},
\end{equation}
then $G_m>0$.
\end{corollary}

\begin{proof}
The first inequality implies $m>\alpha/4$.
By \eqref{fmphase:eq:D-bounds},
\[
q_m=2-\frac{m}{D_m}
\ge2-\frac{\alpha m}{m-\alpha/4}
\ge\frac{2-\alpha}{2}.
\]
The last inequality is algebraically equivalent to the first condition
in \eqref{fmphase:eq:positive-region}. Using \eqref{fmphase:eq:B-bounds},
\[
G_m\ge-C_\alpha^\alpha m^{1-\alpha}
+\frac{2-\alpha}{2}V>0
\]
under the second condition.
\end{proof}

\begin{theorem}[Unique transition for sufficiently small positive noise]
\label{fmphase:thm:subcritical-unique}
Fix $1<\alpha<2$ and $c_T>0$.
There exists $0<\sigma_0^2<\sigma_c^2$ such that, for every
$0<\sigma^2\le\sigma_0^2$, there is a unique real $m_*>1$ with
\begin{equation}\label{fmphase:eq:subcritical-sign-pattern}
G_m\begin{cases}
<0,&1\le m<m_*,\\
=0,&m=m_*,\\
>0,&m>m_*.
\end{cases}
\end{equation}
A valid choice is
\begin{equation}\label{fmphase:eq:sigsmall}
\sigma_0^2=\frac{B_M}{c_T},
\end{equation}
where $M\ge1$ is any fixed number such that $q_m>0$ and $K_m<0$ for all
$m\ge M$. Such an $M$ exists and depends only on $\alpha$.
\end{theorem}

\begin{proof}
Let $b=1-1/\alpha\in(0,1/2)$.
By Lemma~\ref{fmphase:lem:asymptotics},
\[
q_m\to2-\alpha>0,\qquad
K_m\to\frac{b(1-b)(2b-1)}2<0.
\]
Thus $M$ with the stated properties exists.
On $[M,\infty)$, Proposition~\ref{fmphase:prop:threshold} implies
$T_\alpha'(m)<0$. Also,
\begin{equation}\label{fmphase:eq:T-tail}
T_\alpha(m)\sim
\frac{C_\alpha^\alpha}{2-\alpha}m^{1-\alpha}\longrightarrow0.
\end{equation}
Now take $0<V\le B_M$. Because $B_m$ is decreasing and $q_m<1$, for
$1\le m\le M$ one has
\[
G_m=-B_m+Vq_m<-B_m+V\le-B_M+V\le0.
\]
At $M$, $0<q_M<1$ gives $T_\alpha(M)>B_M\ge V$.
The continuous strictly decreasing function $T_\alpha$ therefore takes
the value $V$ exactly once on $(M,\infty)$.
Factoring $G_m=q_m[V-T_\alpha(m)]$ on this tail proves
\eqref{fmphase:eq:subcritical-sign-pattern}.
Finally, since $B_M\le B_1$ and $0<q_1<1$,
\[
\sigma_0^2=\frac{B_M}{c_T}\le\frac{B_1}{c_T}
<\frac{B_1}{c_T q_1}=\sigma_c^2.
\]
\end{proof}

\begin{corollary}[Small-noise location of the unique transition]
\label{fmphase:cor:crossing-asymptotic}
Under Theorem~\ref{fmphase:thm:subcritical-unique}, with $\alpha$ and $c_T$ fixed,
\begin{equation}\label{fmphase:eq:crossing-asymptotic}
m_*\sim
\left[\frac{C_\alpha^\alpha}
{(2-\alpha)c_T\sigma^2}\right]^{1/(\alpha-1)}
\qquad\text{as }\sigma^2\downarrow0.
\end{equation}
\end{corollary}

\begin{proof}
The equation $T_\alpha(m_*)=V$ and strict positivity of $T_\alpha$ at every
finite point on $[M,\infty)$ imply $m_*\to\infty$ as $V\downarrow0$.
Substitute $m=m_*$ into \eqref{fmphase:eq:T-tail} and solve the resulting
asymptotic relation for $m_*$.
\end{proof}
\paragraph{The critical fixed-$m$ regime at $\alpha=2$}
The case $\alpha=2$ is critical because the leading large-$m$ limit
$
G_m\longrightarrow(2-\alpha)V
$
vanishes at $\alpha=2$. Consequently, this leading-order limit does
not determine the eventual sign of $G_m$, and a higher-order analysis
is required. At $\alpha=2$, we can carry out this analysis explicitly
using closed-form summation identities.

Here, ``fixed-$m$'' refers to the order of limits: we first take
$p\to\infty$ with $m$ fixed and then analyze the resulting limiting
family as $m\to\infty$. At $\alpha=2$, retain the definitions
$
G_m=-B_m+Vq_m,
~
B_m=m\tau_m,
~
q_m=\frac{1-2\gamma_m}{1-\gamma_m},
$
where $\tau_m>0$ is determined by
\begin{equation}
\sum_{i=1}^{\infty}
\frac{1}{1+\tau_m i^2}
=
m.
\label{fmphase:eq:alpha2-fixed-point}
\end{equation}

\begin{lemma}[Exact hyperbolic representation at $\alpha=2$]
\label{fmphase:lem:alpha2-hyperbolic}

Let $\alpha=2$ and $m>0$. Define
$
z_m
:=
\frac{\pi}{\sqrt{\tau_m}},
s_m:=2m+1.
$
Then $z_m>0$ is the unique solution of
\begin{equation}
z_m\coth z_m=s_m.
\label{fmphase:eq:alpha2-coth-equation}
\end{equation}
Moreover, defining
\begin{equation}
\delta_m
:=
s_m^2-z_m^2
=
z_m^2\operatorname{csch}^2 z_m,
\label{fmphase:eq:alpha2-delta}
\end{equation}
one has
$
0<\delta_m<1,
$
and
\begin{align}
D_m
=&
m(1-\gamma_m)
=
\frac{s_m-\delta_m}{4},
\label{fmphase:eq:alpha2-D-exact}
\\
q_m
=&
\frac{2(1-\delta_m)}
{s_m-\delta_m},
\label{fmphase:eq:alpha2-q-exact}
\\
B_m
=&
\frac{\pi^2(s_m-1)}
{2(s_m^2-\delta_m)}.
\label{fmphase:eq:alpha2-B-exact}
\end{align}
In particular,
$
q_m>0
$
for every $m>0$.
\end{lemma}

\begin{proof}
For $a>0$, the classical identity
\begin{equation}
\sum_{i=1}^{\infty}
\frac{1}{i^2+a^2}
=
\frac{\pi a\coth(\pi a)-1}
{2a^2}
\label{fmphase:eq:mittag-leffler}
\end{equation}
implies, with $a=z_m/\pi$,
\begin{align}
\sum_{i=1}^{\infty}
\frac{1}{1+\tau_m i^2}
=
\frac{z_m^2}{\pi^2}
\sum_{i=1}^{\infty}
\frac{1}{i^2+(z_m/\pi)^2}
=
\frac{z_m\coth z_m-1}{2}.
\end{align}
Using \eqref{fmphase:eq:alpha2-fixed-point} therefore yields
$
z_m\coth z_m
=
2m+1
=
s_m,
$
which proves \eqref{fmphase:eq:alpha2-coth-equation}.

To establish uniqueness, define
$
f(z):=z\coth z,
~z>0.
$
Then
\begin{align}
f'(z)
=
\coth z
-
z\operatorname{csch}^2z
=
\frac{\sinh z\cosh z-z}{\sinh^2z}
>0,
\end{align}
because
$
\sinh(2z)>2z
$
for every
$z>0$.
Furthermore,
\[
\lim_{z\downarrow0}f(z)=1,
\qquad
\lim_{z\to\infty}f(z)=\infty.
\]
Since $s_m=2m+1>1$, equation
\eqref{fmphase:eq:alpha2-coth-equation}
therefore has a unique positive solution. Next, since
$
\coth^2z-1
=
\operatorname{csch}^2z,
$
we obtain
$
s_m^2-z_m^2
=
z_m^2\operatorname{csch}^2z_m
=
\delta_m.
$
Moreover, $\sinh z>z$ for every $z>0$, so
\[
0<
z^2\operatorname{csch}^2z
=
\frac{z^2}{\sinh^2z}
<1.
\]
Hence
$
0<\delta_m<1.
$
To obtain $D_m$, define
$
M(\tau)
:=
\sum_{i=1}^{\infty}
\frac{1}{1+\tau i^2}.
$
Then
\[
D_m
=
\sum_{i=1}^{\infty}
a_i(1-a_i)
=
-\tau_mM'(\tau_m),
\qquad
a_i:=\frac{1}{1+\tau_mi^2}.
\]
Since
$
M(\tau)
=
\dfrac{z\coth z-1}{2},
~
z=\pi\tau^{-1/2},
$
and
$
\dfrac{dz}{d\tau}
=
-\dfrac{z}{2\tau},
$
we obtain
\begin{align}
-\tau M'(\tau)
=&
\frac{z}{4}
\left(
\coth z-z\operatorname{csch}^2z
\right).
\end{align}
Evaluating at $z=z_m$ gives
\[
D_m
=
\frac{
z_m\coth z_m
-
z_m^2\operatorname{csch}^2z_m
}{4}
=
\frac{s_m-\delta_m}{4}.
\]
Since
$
q_m
=
2-\frac{m}{D_m},~
m=\frac{s_m-1}{2},
$
we obtain
\[
q_m
=
2-
\frac{2(s_m-1)}{s_m-\delta_m}
=
\frac{2(1-\delta_m)}
{s_m-\delta_m}.
\]
Because $0<\delta_m<1$, this also proves
$
q_m>0.
$
Finally,
\[
\tau_m
=
\frac{\pi^2}{z_m^2}
=
\frac{\pi^2}{s_m^2-\delta_m},
\]
and therefore
\[
B_m
=
m\tau_m
=
\frac{\pi^2(s_m-1)}
{2(s_m^2-\delta_m)}.
\]
\end{proof}

\begin{lemma}[Second-order large-$m$ expansion at $\alpha=2$]
\label{fmphase:lem:alpha2-second-order}

Under the assumptions of
Lemma~\ref{fmphase:lem:alpha2-hyperbolic},
as $m\to\infty$,
\begin{align}
q_m
=&
\frac1m
-
\frac{1}{2m^2}
+
O(m^{-3}),
\label{fmphase:eq:alpha2-q-expansion}
\\
B_m
=&
\frac{\pi^2}{4m}
-
\frac{\pi^2}{4m^2}
+
O(m^{-3}),
\label{fmphase:eq:alpha2-B-expansion}
\end{align}
and consequently
\begin{equation}
G_m
=
\frac{V-\pi^2/4}{m}
+
\frac{\pi^2/4-V/2}{m^2}
+
O(m^{-3}).
\label{fmphase:eq:alpha2-G-expansion}
\end{equation}
\end{lemma}

\begin{proof}
From
\eqref{fmphase:eq:alpha2-coth-equation},
\[
s_m-z_m
=
z_m(\coth z_m-1)
=
\frac{2z_m}{e^{2z_m}-1}.
\]
Since $s_m\to\infty$ and
$z\mapsto z\coth z$ is strictly increasing and unbounded,
we have $z_m\to\infty$. Hence
$
s_m-z_m\to0,
$
so
$
z_m\sim s_m\sim2m.
$
It follows that
$
s_m-z_m
=
O\!\left(m e^{-4m}\right),
$
and therefore
$
\delta_m
=
(s_m-z_m)(s_m+z_m)
=
O\!\left(m^2e^{-4m}\right).
$
In particular,
$\delta_m=o(m^{-k})$
for every fixed $k>0$.
Using
\eqref{fmphase:eq:alpha2-q-exact},
\begin{align}
q_m
=
\frac{2(1-\delta_m)}
{2m+1-\delta_m}
=
\frac{2}{2m+1}
+
o(m^{-3}).
\end{align}
Expanding $(2m+1)^{-1}$ gives
\[
q_m
=
\frac1m
-
\frac{1}{2m^2}
+
O(m^{-3}).
\]

Similarly, from
\eqref{fmphase:eq:alpha2-B-exact},
\begin{align}
B_m
=
\frac{\pi^2(2m)}
{2\bigl((2m+1)^2-\delta_m\bigr)}
=
\frac{\pi^2m}{(2m+1)^2}
+
o(m^{-3})
=
\frac{\pi^2}{4m}
-
\frac{\pi^2}{4m^2}
+
O(m^{-3}).
\end{align}
Finally, substituting these expansions into
$
G_m=-B_m+Vq_m
$
gives
\[
G_m
=
\frac{V-\pi^2/4}{m}
+
\frac{\pi^2/4-V/2}{m^2}
+
O(m^{-3}),
\]
as claimed.
\end{proof}

\begin{corollary}[Critical tail-noise threshold at $\alpha=2$]
\label{fmphase:cor:alpha2-critical-noise}
Let
$
V
=
\sigma^2\frac{\gamma_T}{1-\gamma_T},
$
and define
\begin{equation}
V_c
:=
\frac{\pi^2}{4},
\qquad
\sigma_{\mathrm{tail},c}^2
:=
\frac{\pi^2}{4}
\frac{1-\gamma_T}{\gamma_T}.
\label{fmphase:eq:alpha2-critical-noise}
\end{equation}
Then:
\begin{enumerate}
\item if
$
\sigma^2<\sigma_{\mathrm{tail},c}^2,
$
then
$
G_m<0
$
for all sufficiently large $m$;

\item if
$
\sigma^2>\sigma_{\mathrm{tail},c}^2,
$
then
$
G_m>0
$
for all sufficiently large $m$;

\item at the critical value
$
\sigma^2=\sigma_{\mathrm{tail},c}^2,
$
the first-order term vanishes and
\begin{equation}
G_m
=
\frac{\pi^2}{8m^2}
+
O(m^{-3}).
\label{fmphase:eq:alpha2-critical-positive}
\end{equation}
Consequently,
$
G_m>0
$
for all sufficiently large $m$.
\end{enumerate}
\end{corollary}

\begin{proof}
The first two claims follow immediately from the coefficient of
$m^{-1}$ in
\eqref{fmphase:eq:alpha2-G-expansion}. At the critical value
$
V=\frac{\pi^2}{4},
$
the coefficient of $m^{-1}$ vanishes, while the coefficient of
$m^{-2}$ becomes
\[
\frac{\pi^2}{4}
-
\frac12\frac{\pi^2}{4}
=
\frac{\pi^2}{8}>0.
\]
Hence
\[
G_m
=
\frac{\pi^2}{8m^2}
+
O(m^{-3}),
\]
and therefore
$
G_m>0
$
for all sufficiently large $m$.
\end{proof}

The preceding results show why $\alpha=2$ is a genuine critical
exponent. For every fixed positive proportional Stage-II sampling
fraction and $\sigma^2>0$, the proportional deterministic-equivalent
gain remains strictly positive. By contrast, in the sparse-growing
regime the order-one limit vanishes at $\alpha=2$, while in the
fixed-$m$ continuation the large-$m$ behavior is controlled by the
critical tail-noise level
$
\sigma_{\mathrm{tail},c}^2
=
\frac{\pi^2}{4}
\frac{1-\gamma_T}{\gamma_T}.
$
Thus, at the critical exponent the sign is determined by a
higher-order balance that is invisible in the leading sparse-growing
limit.
\paragraph{The regime $2<\alpha\le5/2$}

For $\alpha>2$, the large-$m$ behavior is qualitatively different from the
case $1<\alpha<2$.  Indeed, since
\[
q_m\longrightarrow 2-\alpha<0,
\]
the gain is eventually negative independently of the noise level.  The next
lemma gives an explicit finite point beyond which this occurs.

\begin{lemma}[Noise-independent negative tail]
\label{fmphase:lem:negative-tail}
For every $\alpha>2$, define
\begin{equation}
\label{fmphase:eq:M-alpha}
M_\alpha
:=
\frac{\alpha+4}{2(\alpha-2)}
=
\frac{2+\alpha/2}{\alpha-2}.
\end{equation}
Then, for every real $m\ge M_\alpha$,
\begin{equation}
\label{fmphase:eq:negative-tail}
q_m\le0,
\qquad
G_m<0
\quad
\text{for every }\sigma^2\ge0.
\end{equation}
\end{lemma}

\begin{proof}
From \eqref{fmphase:eq:D-bounds},
\[
D_m
\le
\frac{m+1}{\alpha}+\frac14
=
\frac{4m+\alpha+4}{4\alpha}.
\]
If $m\ge M_\alpha$, then
\[
\frac{4m+\alpha+4}{4\alpha}\le\frac m2,
\]
and therefore $D_m\le m/2$.  Hence
\[
q_m
=
2-\frac{m}{D_m}
\le0.
\]
The conclusion $G_m<0$ now follows from
\eqref{fmphase:eq:qnegative}.
\end{proof}

We next obtain an explicit lower bound for the threshold $T_\alpha(m)$
at the integer points for which $q_m>0$.

\begin{lemma}[Explicit integer threshold bound]
\label{fmphase:lem:integer-threshold-bound}
Suppose $2<\alpha\le5/2$. For every integer $m\ge2$ such that $q_m>0$,
\begin{equation}
\label{fmphase:eq:integer-T-bound}
T_\alpha(m)
>
C_\alpha^\alpha
\frac{\alpha+12}
     {3^{\alpha+1}(4-\alpha)}.
\end{equation}
\end{lemma}
\begin{proof}
By Lemma~\ref{fmphase:lem:negative-tail}, $q_m>0$ implies
$m<M_\alpha$. Moreover, the upper bound on $D_m$ gives
\[
q_m
=
2-\frac{m}{D_m}
\le
2-\frac{4\alpha m}{4m+\alpha+4}
=:
\overline q_\alpha(m).
\]
Thus
\begin{equation}
\label{fmphase:eq:qbar-integer}
0<q_m\le\overline q_\alpha(m),
\qquad
\overline q_\alpha(m)
=
\frac{2(\alpha+4)-4(\alpha-2)m}
     {4m+\alpha+4}.
\end{equation}
By the strict lower bound in \eqref{fmphase:eq:B-bounds},
$
B_m>
C_\alpha^\alpha\frac{m}{(m+1)^\alpha}.
$
Using also $0<q_m\le\overline q_\alpha(m)$, we obtain
\[
T_\alpha(m)
=
\frac{B_m}{q_m}
>
C_\alpha^\alpha
\frac{m}{(m+1)^\alpha}
\frac{1}{\overline q_\alpha(m)}
=:L_\alpha(m),
\]
where
\begin{equation}
\label{fmphase:eq:L-integer}
L_\alpha(m)
=
\frac{C_\alpha^\alpha}{4(\alpha-2)}
\frac{m(4m+\alpha+4)}
     {(m+1)^\alpha(M_\alpha-m)}.
\end{equation}

To locate the smallest value relevant to integer $m\ge2$, define
$
m_{0,\alpha}
:=
\frac{\alpha+4}{4(\alpha-1)}.
$
For $2<\alpha\le5/2$,
$
1<m_{0,\alpha}<\frac32<2.
$
Furthermore,
\begin{equation}
\label{fmphase:eq:L-log-derivative-integer}
\frac{L_\alpha'(m)}{L_\alpha(m)}
=
\frac1m
+\frac4{4m+\alpha+4}
+\frac1{M_\alpha-m}
-\frac{\alpha}{m+1}.
\end{equation}
Writing
$
d:=\alpha-2\in(0,1/2],
t:=m-m_{0,\alpha}\ge0,
$
and putting \eqref{fmphase:eq:L-log-derivative-integer} over its
positive common denominator, its numerator is, up to a positive factor,
\[
\begin{aligned}
P_d(t)
={}&64d(d+1)^3t^3
+16d(d+1)^2(d+2)(d+7)t^2
+4(2-d)(d+1)(d+2)(d+6)t
+(3-d)(d+2)^2(d+6)^2.
\end{aligned}
\]
Every coefficient is strictly positive for $0<d\le1/2$. Hence
$
L_\alpha'(m)>0
$ for
$m_{0,\alpha}<m<M_\alpha.$
Since every integer under consideration satisfies
$
2\le m<M_\alpha
,
2>m_{0,\alpha},
$
it follows that
$
L_\alpha(m)\ge L_\alpha(2).
$

Finally,
\[
\overline q_\alpha(2)
=
\frac{6(4-\alpha)}{\alpha+12},
\]
so
\[
L_\alpha(2)
=
C_\alpha^\alpha
\frac{2}{3^\alpha}
\frac{\alpha+12}{6(4-\alpha)}
=
C_\alpha^\alpha
\frac{\alpha+12}
     {3^{\alpha+1}(4-\alpha)}.
\]
Therefore
\[
T_\alpha(m)
>
L_\alpha(m)
\ge
L_\alpha(2)
=
C_\alpha^\alpha
\frac{\alpha+12}
     {3^{\alpha+1}(4-\alpha)},
\]
which proves \eqref{fmphase:eq:integer-T-bound}.
\end{proof}
\begin{theorem}[Explicit no-W2SG bound for integer sample counts]
\label{fmphase:thm:supercritical-integer}
Suppose $2<\alpha\le5/2$ and $c_T>0$. Define
\begin{equation}
\label{fmphase:eq:sigsafe}
\sigma_{\mathrm{safe}}^2
:=
\min\left\{
\sigma_c^2,\,
\frac{C_\alpha^\alpha}{c_T}
\frac{\alpha+12}
     {3^{\alpha+1}(4-\alpha)}
\right\}.
\end{equation}
Then
\begin{equation}
\label{fmphase:eq:global-nonpositive-integer}
0\le\sigma^2\le\sigma_{\mathrm{safe}}^2
\quad\Longrightarrow\quad
G_m\le0
\quad\text{for every }m\in\mathbb N.
\end{equation}
More precisely,
\begin{equation}
\label{fmphase:eq:global-strict-mge2}
G_m<0
\qquad
\text{for every integer }m\ge2
\end{equation}
throughout the whole range
$0\le\sigma^2\le\sigma_{\mathrm{safe}}^2$.
Furthermore,
\begin{equation}
\label{fmphase:eq:global-strict-interior}
0\le\sigma^2<\sigma_{\mathrm{safe}}^2
\quad\Longrightarrow\quad
G_m<0
\quad\text{for every }m\in\mathbb N.
\end{equation}
At the endpoint $\sigma^2=\sigma_{\mathrm{safe}}^2$, equality
$G_m=0$ can occur only at $m=1$, and only when
$\sigma_{\mathrm{safe}}^2=\sigma_c^2$.
\end{theorem}

\begin{proof}
First consider $m=1$. Since
$
\sigma_{\mathrm{safe}}^2\le\sigma_c^2,
$
the initial sign characterization gives
\[
G_1\le0
\qquad
\text{whenever }
0\le\sigma^2\le\sigma_{\mathrm{safe}}^2.
\]
Moreover,
$
G_1<0
$
if
$
\sigma^2<\sigma_{\mathrm{safe}}^2,
$
while at the endpoint
$\sigma^2=\sigma_{\mathrm{safe}}^2$ equality can occur only if
$
\sigma_{\mathrm{safe}}^2=\sigma_c^2.
$
Now let $m\ge2$ be an integer. If $q_m\le0$, then
$
G_m<0
$
for every $\sigma^2\ge0$ by
\eqref{fmphase:eq:qnegative}. It remains to consider $q_m>0$. By
Lemma~\ref{fmphase:lem:integer-threshold-bound},
\[
H_\alpha(m)
=
\frac{T_\alpha(m)}{c_T}
>
\frac{C_\alpha^\alpha}{c_T}
\frac{\alpha+12}
     {3^{\alpha+1}(4-\alpha)}.
\]
Since
\[
\sigma_{\mathrm{safe}}^2
\le
\frac{C_\alpha^\alpha}{c_T}
\frac{\alpha+12}
     {3^{\alpha+1}(4-\alpha)},
\]
we obtain the strict inequality
$
\sigma_{\mathrm{safe}}^2<H_\alpha(m).
$
Hence, for every
$0\le\sigma^2\le\sigma_{\mathrm{safe}}^2$,
$
\sigma^2<H_\alpha(m).
$
The pointwise sign criterion
\eqref{fmphase:eq:pointwise-sign} therefore yields
$
G_m<0.$ Thus $G_m<0$ for every integer $m\ge2$, while $G_1\le0$.
This proves
\eqref{fmphase:eq:global-nonpositive-integer} and
\eqref{fmphase:eq:global-strict-mge2}.
For
$\sigma^2<\sigma_{\mathrm{safe}}^2$, one also has $G_1<0$,
which proves
\eqref{fmphase:eq:global-strict-interior}.
The final endpoint statement follows from the defining property
of $\sigma_c^2$ at $m=1$.
\end{proof}

\begin{remark}
The bound $\sigma_{\mathrm{safe}}^2$ is an explicit sufficient
noise threshold that rules out positive fixed-$m$ W2SG gain for
every integer Stage~II sample count. At the endpoint
$\sigma^2=\sigma_{\mathrm{safe}}^2$, the gain remains strictly negative
for every $m\ge2$; the only possible zero is at $m=1$, when
$\sigma_{\mathrm{safe}}^2=\sigma_c^2$. No claim is made that
$\sigma_{\mathrm{safe}}^2$ is the largest possible threshold with this
property.
\end{remark}
\section{Two-Stage Random-Feature Interpolation Theory}
\label{app:B}

This section develops the two-stage random-feature theory for Case~II.
The teacher and student operate through independent random linear
feature maps, so the Stage-II predictor is constrained simultaneously
by the pseudo-labeled sample size and by the dimension of the student
feature space. The resulting weak-to-strong generalization behavior is
therefore governed not only by the quality of the Stage-I teacher, but
also by the interaction among the signal dimension, the teacher width,
the number of Stage-II samples, and the student width.

The analysis proceeds in two layers. We first derive a master
deterministic-equivalent description of the two-stage procedure. After
recording the two-point structured-Wishart deterministic equivalent
needed for second-order risk calculations, we formulate the
random-feature model and analyze Stage~I through its exact projection
geometry and deterministic-equivalent strong- and weak-block energies.
We then condition on the Stage-I teacher and derive the Stage-II
random-feature regression problem. Averaging successively over the
Stage-II samples and the student features yields the corresponding
effective-resolution equations, mimic and cross terms, and ultimately
a master deterministic equivalent for the teacher-minus-student
prediction-risk gap.

We next specialize this master expression to the positive-signal-
fraction regime. The resulting phase diagram separates the
sample-bottleneck, feature-bottleneck, and no-bottleneck branches and
identifies the parameter regions in which the deterministic-equivalent
weak-to-strong gain is positive or negative. In particular, the
analysis isolates the role of the Stage-I mismatch through a single
effective quantity and shows how the location of the active Stage-II
bottleneck relative to the signal dimension determines the form of the
phase boundary.

Finally, we examine the principal regimes excluded from the regular
fixed-ratio phase diagram. We give an exact finite-dimensional
classification of the Stage-II interpolation boundary, including its
integrability threshold and its termination at the ambient
no-bottleneck boundary. We then describe the $\sqrt p$ critical window
around the signal dimension. Within the master deterministic
equivalent, this analysis also reveals an asymmetry under joint
teacher--student criticality: the sample-bottleneck branch remains
unfavorable, whereas a positive weak-to-strong region may survive on
the feature-bottleneck branch.
\subsection{Deterministic equivalence}

In addition to Lemma~\ref{lem:strong-deterministic-equivalence} from
Appendix~\ref{app:A}, the derivation of the two-random-feature
ridgeless-regression results requires the following two-point
deterministic-equivalence lemma, introduced in
\cite{atanasov2025two}.

\begin{lemma}[Two-point deterministic equivalence for structured Wishart matrices]
\label{lem:two-point-deterministic-equivalence}
Let \(A_p\succeq0\) be deterministic with uniformly bounded operator
norm, let \(Z_p\in\mathbb R^{n\times p}\) have independent standard
Gaussian entries, and define
\[
q_p:=\frac{p}{n},
\qquad
W_p:=\frac1n Z_p^\top Z_p,
\qquad
\widehat A_p:=A_p^{1/2}W_pA_p^{1/2}.
\]
Assume that \(q_p\to q\in(0,\infty)\).
For \(\lambda,\lambda'>0\), let
\(\kappa_{p,\lambda},\kappa_{p,\lambda'}>0\) denote the
renormalized ridges determined by
\[
\kappa_{p,\lambda}
\left[
1-q_p\operatorname{tr}
\left(
A_p(A_p+\kappa_{p,\lambda}I_p)^{-1}
\right)
\right]
=\lambda,
\]
and the analogous equation with \(\lambda'\).  Define
\[
\widehat T_{p,\lambda}
:=\widehat A_p(\widehat A_p+\lambda I_p)^{-1},
\qquad
G_{p,\lambda}
:=(A_p+\kappa_{p,\lambda}I_p)^{-1},
\qquad
T_{p,\lambda}:=A_pG_{p,\lambda},
\]
with \(G_{p,\lambda'}\) and \(T_{p,\lambda'}\) defined similarly.
Then, for every deterministic matrix \(M_p\) satisfying
\(\sup_p\|M_p\|_{\mathrm{op}}<\infty\), the two-point
deterministic equivalent of
\cite[Eq.~(A.3) and the subsequent free-product specialization] {atanasov2025two}
specializes to
\begin{align}
\widehat T_{p,\lambda}M_p\widehat T_{p,\lambda'}
&\simeq
T_{p,\lambda}M_pT_{p,\lambda'}
+
\kappa_{p,\lambda}\kappa_{p,\lambda'}
G_{p,\lambda}A_pG_{p,\lambda'}
\frac{
q_p\operatorname{tr}
\left[
A_pG_{p,\lambda}M_pG_{p,\lambda'}
\right]
}{
1-q_p\operatorname{tr}
\left[
A_p^2G_{p,\lambda}G_{p,\lambda'}
\right]
}.
\label{eq:two-point-DE-general}
\end{align}
Here \(\simeq\) has the strong test-matrix meaning of
Lemma~\ref{lem:strong-deterministic-equivalence}, and the denominator
is assumed to remain bounded away from zero.

In particular, for \(\lambda'=\lambda\),
\begin{align}
\widehat T_{p,\lambda}M_p\widehat T_{p,\lambda}
&\simeq
T_{p,\lambda}M_pT_{p,\lambda}
+
\kappa_{p,\lambda}^{2}
G_{p,\lambda}A_pG_{p,\lambda}
\frac{
q_p\operatorname{tr}
\left[
A_pG_{p,\lambda}M_pG_{p,\lambda}
\right]
}{
1-q_p\operatorname{tr}
\left[
A_p^2G_{p,\lambda}^{2}
\right]
}.
\label{eq:two-point-DE-equal-ridge}
\end{align}
Since \(0\preceq\widehat T_{p,\lambda}\preceq I_p\), the same
normalized-trace deterministic equivalent remains valid after taking
expectation over \(W_p\).
\end{lemma}
\subsection{Defining Random-Feature Model}
\label{sec:problem2Def}

In the second setting, we consider a linear random-feature regression
model in which the teacher and student operate through random feature
representations of possibly different dimensions. Unlike Case I, where the
teacher and student share the same hypothesis space and the difference
arises only from the two-stage training procedure, the weak-to-strong
generalization phenomenon in this setting is affected by the interaction
between the random feature representations and the Stage-II training
procedure.

Let the input distribution in the preliminaries be
\(x\sim P_x=\mathcal N(0,\Sigma_p)\), where

\[
\Sigma_p
=
\operatorname{diag}
\left(
\underbrace{1,\ldots,1}_{k},
\underbrace{\delta_p,\ldots,\delta_p}_{p-k}
\right),
\qquad
0<\delta_p<1 .
\]

Let $\beta^\star\in\mathbb R^p$ be supported on the first $k$
coordinates, so \(\beta_i^\star=0\) for \(i>k\).

The target function and noiseless observation model are
\(\psi^\star(x)=x^\top\beta^\star\) and \(y=\psi^\star(x)\).
Thus, for the labeled dataset
$\mathcal D_n=\{(x_i,y_i)\}_{i=1}^n$ introduced in the preliminaries,
$y_i=x_i^\top\beta^\star$. We normalize the target so that
\(\mathbb E[y^2]=(\beta^\star)^\top\Sigma_p\beta^\star=1\).
The teacher and student feature extractors are random linear maps, which
are defined through independent random matrices $U_T$ and $U_S$,
respectively. Case II studies its population-limit analogue:
conditional on $U_T$, the teacher minimizes population risk. Consequently,
$n$ does not enter the asymptotic analysis below.

\noindent
\textbf{Stage I (teacher).}
Let
\[
U_T
:=
[u_{T,1}\vert\ldots\vert u_{T,N_T}]^\top
\in\mathbb{R}^{N_T\times p},
\]
where
\(u_{T,1},\ldots,u_{T,N_T}\iid\mathcal{N}(\mathbf{0},I_p)\).
The matrix $U_T$ is randomly drawn and is not learnable. Conditional on
$U_T$, the teacher feature extractor is
\(f_T(x)=U_Tx\in\mathbb{R}^{N_T}\), and the conditional teacher
hypothesis class is
\[
    \mathcal H_T(U_T)
    :=
    \left\{
        h_{T,w}:x\mapsto\langle w,f_T(x)\rangle
        =x^\top U_T^\top w
        :w\in\mathbb R^{N_T}
    \right\}.
\]
The population version of the teacher learning rule $\mathcal A_T$
returns the minimum-Euclidean-norm population-risk minimizer:
\[
w_T^\star
=
\argmin_{w\in\mathbb{R}^{N_T}}\|w\|_2
\qquad
\mathrm{subject~to}
\qquad
w\in
\argmin_{v\in\mathbb{R}^{N_T}}
\mathbb{E}_{x}
\!\left[
    \left(x^\top U_T^\top v-x^\top\beta^\star\right)^2
\right].
\]
Define the induced coefficient in the original input space by
\(\widehat{\beta}_T=U_T^\top w_T^\star\). The resulting teacher predictor
is \(\widehat h_T=h_{T,w_T^\star}\), with
\(\widehat h_T(x)=x^\top\widehat\beta_T\).

\noindent
\textbf{Stage II (student).}
Independently of $U_T$, generate fresh feature vectors
\(\widetilde{x}_1,\ldots,\widetilde{x}_m
\iid\mathcal{N}(\mathbf{0},\Sigma_p)\), and assign pseudo-labels using the
teacher predictor:
\[
\widetilde{y}_i
\triangleq
\widehat h_T(\widetilde x_i)
=
\widetilde{x}_i^\top\widehat{\beta}_T,
\qquad
i\in[m].
\]

The resulting Stage-II dataset is the pseudo-labeled sample from the
preliminaries,
\(\widetilde{\mathcal D}_m
=\{(\widetilde{x}_i,\widetilde{y}_i)\}_{i\in[m]}\). Define the
corresponding input matrix by
\(\widetilde{X}\triangleq
[\widetilde{x}_1\vert\ldots\vert\widetilde{x}_m]^\top
\in\mathbb{R}^{m\times p}\) and label vector
\(\widetilde{y}\triangleq
[\widetilde{y}_1,\ldots,\widetilde{y}_m]^\top\in\mathbb{R}^m\).

The student uses an independent random feature matrix
\[
U_S
\triangleq
[u_{S,1}\vert\ldots\vert u_{S,N_S}]^\top
\in\mathbb{R}^{N_S\times p},
\]
where
\(u_{S,1},\ldots,u_{S,N_S}\iid\mathcal{N}(\mathbf{0},I_p)\).
The matrix $U_S$ is independent of $U_T$ and the Stage-II samples.
Conditional on $U_S$, the student feature extractor is
\(f_S(x)=U_Sx\in\mathbb{R}^{N_S}\).
The conditional student hypothesis class is
\[
    \mathcal H_S(U_S)
    :=
    \left\{
        h_{S,w}:x\mapsto\langle w,f_S(x)\rangle
        =x^\top U_S^\top w
        :w\in\mathbb R^{N_S}
    \right\}.
\]

The Stage-II training rule $\mathcal{A}_S$ corresponds to ridgeless
least-squares regression in the random-feature space. We take the
minimum-norm least-squares solution
$
\widehat w_S
=
\left(\widetilde{X}U_S^\top\right)^\dagger
\widetilde{y}.
$
Equivalently,
\[
\widehat w_S
=
\argmin_{w\in\mathbb{R}^{N_S}}
\|w\|_2
\qquad
\mathrm{subject~to}
\qquad
w\in
\argmin_{v\in\mathbb{R}^{N_S}}
\left\|
\widetilde{X}U_S^\top v-\widetilde{y}
\right\|_2^2.
\]
No direct interpolation constraint is imposed here, since
$\widetilde y$ need not lie in
$\operatorname{range}(\widetilde XU_S^\top)$. The induced coefficient and student predictor are
\(\widehat{\beta}_S=U_S^\top \widehat w_S\),
\(\widehat h_S=h_{S,\widehat w_S}\), and
\(\widehat h_S(x)=x^\top\widehat\beta_S\). For a linear predictor $h_\beta(x):=x^\top\beta$, the squared-loss
population risk in \eqref{eq:general-population-risk} specializes to

\[
\mathcal R(h_\beta)
=
\mathbb E_x
\left[
(x^\top\beta-x^\top\beta^\star)^2
\right]
=
(\beta-\beta^\star)^\top
\Sigma_p
(\beta-\beta^\star).
\]

Accordingly, in the notation of the preliminaries,
\(\overline{\mathcal R}_{T,p}
=\mathbb E\!\left[\mathcal R(\widehat h_T)\right]\) and
\(\overline{\mathcal R}_{S,p}
=\mathbb E\!\left[\mathcal R(\widehat h_S)\right]\). The expectation
in \(\overline{\mathcal R}_{T,p}\) is over \(U_T\), whereas the
expectation in \(\overline{\mathcal R}_{S,p}\) is over \(U_T\),
\(U_S\), and \(\widetilde{\mathcal D}_m\). The finite-dimensional W2SG
gain is
\begin{equation}
\label{eq:case-II-w2sg}
    \mathsf{W2SG}_p
    :=
    \overline{\mathcal R}_{T,p}-
    \overline{\mathcal R}_{S,p}.
\end{equation}
Thus, W2SG occurs whenever \(\mathsf{W2SG}_p>0\).
\subsection{Theory of DE for stage I}

\begin{lemma}[Exact Stage-I estimator and projection geometry]
\label{lem:case-II-teacher-exact}
Let
\[
A_{T,p}:=U_T\Sigma_p^{1/2},
\qquad
P_{T,p}:=A_{T,p}^\top(A_{T,p}A_{T,p}^\top)^\dagger A_{T,p}.
\]
The minimum-norm population estimator and its predictor are
\[
\widehat\beta_T
=U_T^\top(U_T\Sigma_pU_T^\top)^\dagger
U_T\Sigma_p\beta^\star,
\qquad
\widehat h_T(x):=x^\top\widehat\beta_T.
\]
In whitened coordinates,
\[
\theta_T:=\Sigma_p^{1/2}\widehat\beta_T
=P_{T,p}\theta^\star,
\qquad
\theta^\star:=\Sigma_p^{1/2}\beta^\star,
\]
where \(P_{T,p}\) is the Euclidean orthogonal projector onto
\(\operatorname{row}(A_{T,p})\). Consequently,
$
\overline{\mathcal R}_{T,p}
=1-\mathbb E\|\theta_T\|_2^2
=1-S_{T,p}-V_{T,p},
$
where
$
S_{T,p}:=\mathbb E\|\theta_{T,\mathrm{str}}\|_2^2,
V_{T,p}:=\mathbb E\|\theta_{T,\mathrm{wk}}\|_2^2.
$
Moreover,
\[
\mathbb E[(\beta^\star_{\mathrm{str}})^\top
\theta_{T,\mathrm{str}}]=S_{T,p}+V_{T,p}.
\]
\end{lemma}

\begin{proof}

% ============================================================
\medskip\noindent\textbf{Exact Stage-I population estimator.}
\label{subsec:master-stage1-exact}
% ============================================================

For a candidate coefficient
$
\beta=U_T^\top w,
$
the population squared loss is
\begin{align}
\mathcal L_T(w)
=
\mathbb E_x
\left[
\left(
x^\top U_T^\top w-x^\top\beta^\star
\right)^2
\right]
=
\mathbb E_x
\left[
\left(
x^\top(U_T^\top w-\beta^\star)
\right)^2
\right]
=
(U_T^\top w-\beta^\star)^\top
\Sigma_p
(U_T^\top w-\beta^\star),
\label{eq:master-stage1-pop-objective}
\end{align}
where we used
$
\mathbb E[xx^\top]=\Sigma_p.
$
Differentiating \eqref{eq:master-stage1-pop-objective} with respect to
\(w\) gives
$
\nabla_w\mathcal L_T(w)
=
2U_T\Sigma_p
(U_T^\top w-\beta^\star).
$
Hence every minimizer satisfies the normal equation
\begin{equation}
\label{eq:master-stage1-normal}
U_T\Sigma_pU_T^\top w
=
U_T\Sigma_p\beta^\star.
\end{equation}

Choosing the minimum-Euclidean-norm solution gives
$
w_T^\star
=
(U_T\Sigma_pU_T^\top)^\dagger
U_T\Sigma_p\beta^\star.
$
Therefore the exact Stage-I coefficient in the original input
coordinates is
\begin{equation}
\label{eq:master-beta1}
\widehat\beta_T
=
U_T^\top
(U_T\Sigma_pU_T^\top)^\dagger
U_T\Sigma_p\beta^\star.
\end{equation}

Equation \eqref{eq:master-beta1} is a finite-\(p\) identity; no
deterministic-equivalent approximation has yet been used.
Introduce whitened coordinates
$
\theta
:=
\Sigma_p^{1/2}\beta,
\theta^\star
:=
\Sigma_p^{1/2}\beta^\star.
$
Since the teacher is supported on the strong block,
\[
\theta^\star
=
\begin{pmatrix}
\beta^\star_{\mathrm{str}}\\
0
\end{pmatrix},
\qquad
\|\theta^\star\|_2^2=1.
\]

The population prediction risk becomes ordinary Euclidean squared
distance:
\begin{align}
\mathcal R(h_\beta)
:=&
(\beta-\beta^\star)^\top
\Sigma_p
(\beta-\beta^\star)
=
\left\|
\Sigma_p^{1/2}(\beta-\beta^\star)
\right\|_2^2
=
\|\theta-\theta^\star\|_2^2.
\label{eq:master-risk-whitened}
\end{align}

Define the whitened Stage-I feature matrix
\begin{equation}
\label{eq:master-AT}
A_{T,p}
:=
U_T\Sigma_p^{1/2}
\in\mathbb R^{N_T\times p}.
\end{equation}
Multiplying \eqref{eq:master-beta1} by \(\Sigma_p^{1/2}\) gives
\begin{align}
\theta_T
:=
\Sigma_p^{1/2}\widehat\beta_T
=&
\Sigma_p^{1/2}
U_T^\top
(U_T\Sigma_pU_T^\top)^\dagger
U_T\Sigma_p\beta^\star
=
A_{T,p}^\top
(A_{T,p}A_{T,p}^\top)^\dagger
A_{T,p}\theta^\star.
\end{align}
Thus
\begin{equation}
\label{eq:master-theta1-proj}
\theta_T=P_{T,p}\theta^\star,
\qquad
P_{T,p}
:=
A_{T,p}^\top
(A_{T,p}A_{T,p}^\top)^\dagger
A_{T,p}.
\end{equation}

The matrix \(P_{T,p}\) is exactly the Euclidean orthogonal projector
onto
$
\operatorname{row}(A_{T,p})
\subseteq\mathbb R^p.
$
Indeed,
$
P_{T,p}^\top=P_{T,p},
P_{T,p}^2=P_{T,p},
$
by the Moore--Penrose identities. Consequently,
$
\theta^\star-\theta_T
\perp
\theta_T,
$
and therefore
\begin{equation}
\label{eq:master-stage1-projection-identity}
(\theta^\star-\theta_T)^\top\theta_T=0.
\end{equation}
By Pythagoras,
\begin{align}
\|\theta^\star-\theta_T\|_2^2
=&
\|\theta^\star\|_2^2-\|\theta_T\|_2^2
=
1-\|\theta_T\|_2^2.
\label{eq:master-stage1-pythagoras}
\end{align}
Taking expectation over \(U_T\), the exact Stage-I risk is therefore
\begin{equation}
\label{eq:master-stage1-risk-exact}
\overline{\mathcal R}_{T,p}
=
1-\mathbb E\|\theta_T\|_2^2.
\end{equation}

This exact projection identity will later allow us to determine the total Stage-I second moment without a separate random matrix calculation. Split the whitened Stage-I estimator according to the two covariance
blocks:
$$
\theta_T
=
\begin{pmatrix}
\theta_{T,\mathrm{str}}\\
\theta_{T,\mathrm{wk}}
\end{pmatrix},
\theta_{T,\mathrm{str}}\in\mathbb R^k,
\theta_{T,\mathrm{wk}}\in\mathbb R^{p-k}.
$$ 
Define the aggregate strong and weak prediction energies
\begin{equation}
\label{eq:master-SV-def}
S_{T,p}
:=
\mathbb E\|\theta_{T,\mathrm{str}}\|_2^2,
\qquad
V_{T,p}
:=
\mathbb E\|\theta_{T,\mathrm{wk}}\|_2^2.
\end{equation}

Because
$
\|\theta_T\|_2^2
=
\|\theta_{T,\mathrm{str}}\|_2^2+\|\theta_{T,\mathrm{wk}}\|_2^2,
$
equation \eqref{eq:master-stage1-risk-exact} yields the exact
finite-\(p\) identity
\begin{equation}
\label{eq:master-R1-SV}
\overline{\mathcal R}_{T,p}
=
1-S_{T,p}-V_{T,p}.
\end{equation}

Moreover, expanding
\eqref{eq:master-stage1-projection-identity} blockwise gives
$
(\beta^\star_{\mathrm{str}}-\theta_{T,\mathrm{str}})^\top \theta_{T,\mathrm{str}}
-
\|\theta_{T,\mathrm{wk}}\|_2^2
=
0.
$
Hence
\begin{equation}
\label{eq:master-stage1-block-proj-id}
(\beta^\star_{\mathrm{str}}-\theta_{T,\mathrm{str}})^\top \theta_{T,\mathrm{str}}
=
\|\theta_{T,\mathrm{wk}}\|_2^2.
\end{equation}

Equivalently,
$
(\beta^\star_{\mathrm{str}})^\top \theta_{T,\mathrm{str}}
=
\|\theta_{T,\mathrm{str}}\|_2^2+\|\theta_{T,\mathrm{wk}}\|_2^2.
$
Taking expectation gives
\begin{equation}
\label{eq:master-stage1-cross-exact}
\mathbb E[(\beta^\star_{\mathrm{str}})^\top \theta_{T,\mathrm{str}}]
=
S_{T,p}+V_{T,p}.
\end{equation}

The identities
\eqref{eq:master-R1-SV} and
\eqref{eq:master-stage1-cross-exact}
are exact and will be repeatedly used in the Stage-II analysis.

\end{proof}

\begin{theorem}[Stage-I deterministic equivalents]
\label{thm:case-II-teacher-DE}
Assume the proportional high-dimensional regime specified in the
preliminaries and the regularity conditions required by the one-point
and two-point structured-Wishart deterministic-equivalent results.
Let \(\tau_{T,p}=0\) when \(N_T\ge p\); when \(N_T<p\), let
\(\tau_{T,p}>0\) be the unique solution of
\[
\operatorname{df}_p(\tau_{T,p})=\frac{N_T}{p},
\qquad
\operatorname{df}_p(\tau)
:=\frac1p\left[
\frac{k}{1+\tau}
+\frac{(p-k)\delta_p}{\delta_p+\tau}
\right].
\]
Define
$
a_{T,p}:=\frac1{1+\tau_{T,p}},
c_{T,p}:=a_{T,p}-a_{T,p}^2,
$
and
\[
\Theta_{T,p}
:=
\left[
1+
\frac{k(\delta_p+\tau_{T,p})^2}
{(p-k)\delta_p(1+\tau_{T,p})^2}
\right]^{-1}.
\]
Then
\[
\overline{\mathcal R}_{T,p}
\overset{\mathrm{DE}}{\simeq}
\overline{\mathcal R}_{T,p}^{\mathrm{DE}}
:=\frac{\tau_{T,p}}{1+\tau_{T,p}},
\]
\[
\mathbb E\widehat\beta_T
\overset{\mathrm{DE}}{\simeq}
\begin{pmatrix}
a_{T,p}\beta^\star_{\mathrm{str}}\\
0
\end{pmatrix},
\]
and
\[
V_{T,p}\overset{\mathrm{DE}}{\simeq}
V_{T,p}^{\mathrm{DE}}:=c_{T,p}\Theta_{T,p},
\qquad
S_{T,p}\overset{\mathrm{DE}}{\simeq}
S_{T,p}^{\mathrm{DE}}:=a_{T,p}-c_{T,p}\Theta_{T,p}.
\]
\end{theorem}

\begin{proof}

% ============================================================
\medskip\noindent\textbf{Structured-Wishart resolvent and effective ridge.}
\label{subsec:master-stage1-DE}
% ============================================================

We now pass from the exact projector representation to its
high-dimensional deterministic equivalent.

For \(j=1,\ldots,N_T\), define the effective teacher-feature vector
\[
z_{T,j}:=\Sigma_p^{1/2}u_{T,j}.
\]
Since \(u_{T,j}\sim\mathcal N(0,I_p)\),
$
z_{T,j}\sim\mathcal N(0,\Sigma_p).
$
Moreover, the \(j\)-th row of
\(A_{T,p}=U_T\Sigma_p^{1/2}\) is \(z_{T,j}^{\top}\). Therefore,
\begin{equation}
\label{eq:master-Sigmahat-T}
\widehat\Sigma_{T,p}
:=
\frac1{N_T}A_{T,p}^{\top}A_{T,p}
=
\frac1{N_T}\sum_{j=1}^{N_T}z_{T,j}z_{T,j}^{\top}
=
\Sigma_p^{1/2}
\left(
\frac1{N_T}U_T^{\top}U_T
\right)
\Sigma_p^{1/2}.
\end{equation}
Thus \(\widehat\Sigma_{T,p}\) is the empirical covariance of
\(N_T\) independent samples from \(\mathcal N(0,\Sigma_p)\). In
particular,
$
\mathbb E\widehat\Sigma_{T,p}=\Sigma_p,
$
and \(\widehat\Sigma_{T,p}\) is a structured Wishart matrix with
population covariance \(\Sigma_p\). To verify the projector identity algebraically, let
$
A_{T,p}=Q_{T,p}D_{T,p}R_{T,p}^{\top}
$
be a thin singular-value decomposition, where
$
D_{T,p}=\operatorname{diag}(d_1,\ldots,d_r),
r=\operatorname{rank}(A_{T,p}),
$
and every \(d_j>0\). Then
\begin{align}
P_{T,p}
=&
A_{T,p}^{\top}(A_{T,p}A_{T,p}^{\top})^{\dagger}A_{T,p}
=
R_{T,p}D_{T,p}Q_{T,p}^{\top}
\left(Q_{T,p}D_{T,p}^{2}Q_{T,p}^{\top}\right)^{\dagger}
Q_{T,p}D_{T,p}R_{T,p}^{\top}
\nonumber\\
=&
R_{T,p}D_{T,p}D_{T,p}^{-2}D_{T,p}R_{T,p}^{\top}
=
R_{T,p}R_{T,p}^{\top}.
\label{eq:master-projector-svd}
\end{align}
Moreover,
$
\widehat\Sigma_{T,p}
=
R_{T,p}\left(\frac{D_{T,p}^{2}}{N_T}\right)R_{T,p}^{\top},
$
so for every fixed bare ridge \(\lambda>0\),
\begin{align}
\widehat\Sigma_{T,p}
(\widehat\Sigma_{T,p}+\lambda I_p)^{-1}
=&
R_{T,p}
\operatorname{diag}
\left(
\frac{d_1^2}{d_1^2+N_T\lambda},\ldots,
\frac{d_r^2}{d_r^2+N_T\lambda}
\right)
R_{T,p}^{\top}.
\label{eq:master-regularized-projector-svd}
\end{align}
Since
\[
\lim_{\lambda\downarrow0}
\frac{d_j^2}{d_j^2+N_T\lambda}=1,
\qquad j=1,\ldots,r,
\]
equations \eqref{eq:master-projector-svd} and
\eqref{eq:master-regularized-projector-svd} give
\begin{equation}
\label{eq:master-projector-ridge-limit}
P_{T,p}
=
\lim_{\lambda\downarrow0}
\widehat\Sigma_{T,p}
(\widehat\Sigma_{T,p}+\lambda I_p)^{-1}.
\end{equation}

We now apply
Lemma~\ref{lem:strong-deterministic-equivalence}(equation~\ref{eq:shrinkage-DE}).
The population covariance in this application is \(\Sigma_p\).
Consequently, the empirical teacher-feature covariance resolvent has
the deterministic equivalent
\begin{align}
\widehat\Sigma_{T,p}
(\widehat\Sigma_{T,p}+\lambda I_p)^{-1}
&\overset{\mathrm{DE}}{\simeq}
\Sigma_p
(\Sigma_p+\kappa_{T,\lambda}I_p)^{-1}.
\label{eq:master-Sigmahat-T-DE}
\end{align}
where the renormalized ridge
\(\kappa_{T,\lambda}>0\) is determined by
\begin{equation}
\label{eq:master-ridge-fixed-point}
\kappa_{T,\lambda}
\left[
1-
\frac{p}{N_T}
\operatorname{df}_p(\kappa_{T,\lambda})
\right]
=
\lambda.
\end{equation}
Here
\begin{equation}
\label{eq:master-df1-general}
\operatorname{df}_p(\tau)
:=
\operatorname{tr}
\left[
\Sigma_p(\Sigma_p+\tau I_p)^{-1}
\right],
\end{equation}

For the two-block covariance,
$
\Sigma_p
=
\operatorname{diag}
\left(
I_k,\delta_pI_{p-k}
\right),
$
so
\begin{equation}
\label{eq:master-df1}
\operatorname{df}_p(\tau)
=
\frac1p
\left[
\frac{k}{1+\tau}
+
\frac{(p-k)\delta_p}{\delta_p+\tau}
\right].
\end{equation}

We will also need the squared shrinkage trace
\begin{equation}
\label{eq:master-df2-general}
\operatorname{df}_p^{(2)}(\tau)
:=
\operatorname{tr}
\left[
\Sigma_p^2(\Sigma_p+\tau I_p)^{-2}
\right],
\end{equation}
which in the present two-block model becomes
\begin{equation}
\label{eq:master-df2}
\operatorname{df}_p^{(2)}(\tau)
=
\frac1p
\left[
\frac{k}{(1+\tau)^2}
+
\frac{(p-k)\delta_p^2}{(\delta_p+\tau)^2}
\right].
\end{equation}

% ============================================================
Suppose first that
$
N_T<p.
$
Then the ridgeless limit remains on the positive effective-ridge
branch.  Sending \(\lambda\downarrow0\) in
\eqref{eq:master-ridge-fixed-point} gives
\[
\kappa_{T,\lambda}
\left[
1-
\frac{p}{N_T}
\operatorname{df}_p(\kappa_{T,\lambda})
\right]
\longrightarrow0.
\]
Since the limiting effective ridge is positive in the
underparameterized branch, its limit \(\tau_{T,p}>0\) must satisfy
\[
1-
\frac{p}{N_T}
\operatorname{df}_p(\tau_{T,p})
=
0.
\]
Therefore,
\begin{equation}
\label{eq:master-tau1-two-block}
\frac{k}{1+\tau_{T,p}}
+
\frac{(p-k)\delta_p}
{\delta_p+\tau_{T,p}}
=
N_T, \qquad N_T<p
\end{equation}

Because \(\operatorname{df}_p(\tau)\) is continuous and strictly
decreasing from \(1\) to \(0\), equation
\eqref{eq:master-tau1-two-block} has a unique positive solution.

If instead
$
N_T\ge p,
$
then \(A_{T,p}\) has full column rank almost surely, so
$
P_{T,p}=I_p
$
and Stage I recovers the target coefficient \(\beta^\star\) exactly.
In this branch we set
$\tau_{T,p}=0.$ From the exact projection representation,
\[
\overline{\mathcal R}_{T,p}
=
1-
\mathbb E
\left[
(\theta^\star)^\top
P_{T,p}
\theta^\star
\right].
\]

Using
\eqref{eq:master-projector-ridge-limit} and the structured-Wishart
deterministic equivalent, the ridgeless projector is replaced at
leading order by
$
\Sigma_p
(\Sigma_p+\tau_{T,p}I_p)^{-1}.
$ Therefore
\begin{align}
\overline{\mathcal R}_{T,p}
&\overset{\mathrm{DE}}{\simeq}
1-
(\theta^\star)^\top
\Sigma_p
(\Sigma_p+\tau_{T,p}I_p)^{-1}
\theta^\star.
\label{eq:master-R1-DE-step}
\end{align}

The target coefficient lies entirely in the unit-eigenvalue block, hence
$
(\theta^\star)^\top
\Sigma_p
(\Sigma_p+\tau I_p)^{-1}
\theta^\star
=
\frac{1}{1+\tau}
\|\theta^\star\|_2^2
=
\frac1{1+\tau}.
$

So we have
\begin{equation}
\label{eq:master-stage1-risk-DE}
\overline{\mathcal R}_{T,p}
\overset{\mathrm{DE}}{\simeq}
\overline{\mathcal R}_{T,p}^{\mathrm{DE}}
:=
\frac{\tau_{T,p}}
{1+\tau_{T,p}}.
\end{equation}

It is convenient to define
$
a_{T,p}
:=
\frac1{1+\tau_{T,p}}.
$
Then
\begin{equation}
\label{eq:master-R1-a}
\overline{\mathcal R}_{T,p}^{\mathrm{DE}}
=
1-a_{T,p}.
\end{equation}

Combining this with the exact identity
$
S_{T,p}+V_{T,p}=1-\overline{\mathcal R}_{T,p}
$
already shows that the total Stage-I prediction energy satisfies
\begin{equation}
\label{eq:master-total-energy-DE}
S_{T,p}+V_{T,p}
\overset{\mathrm{DE}}{\simeq}
a_{T,p}.
\end{equation}

We first determine the form of the mean exactly from the rotational
symmetries of the model.  The deterministic equivalent will only be
needed afterward to determine the remaining scalar coefficient. For clarity, write the Stage-I estimator as a function of the random
feature matrix:
\[
\widehat\beta_T(U_T)
=
U_T^\top
(U_T\Sigma_pU_T^\top)^\dagger
U_T\Sigma_p\beta^\star.
\]
Let \(D\in O(p)\) be any orthogonal matrix satisfying
$
D\Sigma_pD^\top=\Sigma_p,
D\beta^\star=\beta^\star.
$
The first identity implies that \(D\Sigma_p=\Sigma_pD\).  Evaluating
the estimator at \(U_TD\) gives
\begin{align}
\widehat\beta_T(U_TD)
=&
D^\top U_T^\top
\left(
U_TD\Sigma_pD^\top U_T^\top
\right)^\dagger
U_TD\Sigma_p\beta^\star
\nonumber\\
=&
D^\top U_T^\top
(U_T\Sigma_pU_T^\top)^\dagger
U_T\Sigma_pD\beta^\star
\nonumber\\
=&
D^\top U_T^\top
(U_T\Sigma_pU_T^\top)^\dagger
U_T\Sigma_p\beta^\star
\nonumber\\
=&
D^\top\widehat\beta_T(U_T).
\label{eq:master-estimator-equivariance}
\end{align}
Because \(U_T\) has independent isotropic Gaussian rows,
$
U_TD\overset{d}=U_T.
$
Combining this distributional invariance with
\eqref{eq:master-estimator-equivariance} yields
\begin{equation}
\label{eq:master-estimator-distributional-invariance}
\widehat\beta_T(U_T)
\overset{d}=
D^\top\widehat\beta_T(U_T).
\end{equation}

We now apply this identity to rotations of the weak block.  For any
\(O\in O(p-k)\), let
$
D_O
:=
\begin{pmatrix}
I_k&0\\
0&O
\end{pmatrix}.
$
Since
$
\Sigma_p
=
\begin{pmatrix}
I_k&0\\
0&\delta_pI_{p-k}
\end{pmatrix},
\beta^\star
=
\begin{pmatrix}
\beta^\star_{\mathrm{str}}\\
0
\end{pmatrix},
$
we have
$
D_O\Sigma_pD_O^\top=\Sigma_p,
D_O\beta^\star=\beta^\star.
$
Define
$
\mu_{T,p}
:=
\mathbb E\widehat\beta_T
=
\begin{pmatrix}
\mu_{T,\mathrm{str},p}\\
\mu_{T,\mathrm{wk},p}
\end{pmatrix}.
$
Taking expectations in
\eqref{eq:master-estimator-distributional-invariance} gives
$
\mu_{T,p}=D_O^\top\mu_{T,p},
$
and therefore
\[
\mu_{T,\mathrm{wk},p}
=
O^\top\mu_{T,\mathrm{wk},p}
\qquad
\text{for every }O\in O(p-k).
\]
Choosing \(O=-I_{p-k}\) gives
$
\mu_{T,\mathrm{wk},p}
=
-\mu_{T,\mathrm{wk},p},
$
so the weak component of the mean vanishes exactly:
\begin{equation}
\label{eq:master-weak-mean-zero}
\mathbb E[\widehat\beta_{T,\mathrm{wk}}]
=
\mu_{T,\mathrm{wk},p}
=
0.
\end{equation}

To determine the direction of the strong mean, consider
$
D_R
:=
\begin{pmatrix}
R&0\\
0&I_{p-k}
\end{pmatrix},
$
where \(R\in O(k)\) satisfies
$
R\beta^\star_{\mathrm{str}}
=
\beta^\star_{\mathrm{str}}.
$
Because the strong covariance block equals \(I_k\), we again have
\[
D_R\Sigma_pD_R^\top=\Sigma_p,
\qquad
D_R\beta^\star=\beta^\star.
\]
Thus
$
\mu_{T,\mathrm{str},p}
=
R^\top\mu_{T,\mathrm{str},p}
$
for every strong-block rotation that fixes
\(\beta^\star_{\mathrm{str}}\).  Such rotations act arbitrarily on the
orthogonal complement of \(\beta^\star_{\mathrm{str}}\).  Hence their
common fixed subspace is
$
\operatorname{span}\{\beta^\star_{\mathrm{str}}\}.
$
It follows that, for some scalar \(b_{T,p}\),
\begin{equation}
\label{eq:master-mean-exact-form}
\mathbb E\widehat\beta_T
=
\begin{pmatrix}
b_{T,p}\beta^\star_{\mathrm{str}}\\
0
\end{pmatrix}.
\end{equation}
This conclusion is exact for every finite \(p\): symmetry determines
the direction of the mean, but not yet the scalar \(b_{T,p}\). We next determine \(b_{T,p}\).  Since
$
\|\beta^\star_{\mathrm{str}}\|_2^2
=
(\beta^\star)^\top\Sigma_p\beta^\star
=
1,
$
equation \eqref{eq:master-mean-exact-form} implies
\begin{align}
b_{T,p}
=&
(\beta^\star_{\mathrm{str}})^\top
\mathbb E[\widehat\beta_{T,\mathrm{str}}]
=
\mathbb E
\left[
(\theta^\star)^\top\theta_T
\right]
=
\mathbb E
\left[
(\theta^\star)^\top
P_{T,p}\theta^\star
\right].
\label{eq:master-b-projector-form}
\end{align}
Following the ridgeless branch in
\eqref{eq:master-projector-ridge-limit} and
\eqref{eq:master-Sigmahat-T-DE} gives
\[
P_{T,p}
\overset{\mathrm{DE}}{\simeq}
\Sigma_p(\Sigma_p+\tau_{T,p}I_p)^{-1}.
\]
In the strong deterministic-equivalence definition, choose the
deterministic test matrix
$
M_p:=\theta^\star(\theta^\star)^\top.
$
Then
$
\operatorname{Tr}(P_{T,p}M_p)
=
(\theta^\star)^\top P_{T,p}\theta^\star.
$
Because \(P_{T,p}\) is an orthogonal projector,
$
0
\leq
(\theta^\star)^\top P_{T,p}\theta^\star
\leq
\|\theta^\star\|_2^2
=
1,
$
so the corresponding convergence in probability also passes to the
expectation.  Therefore
\begin{align}
b_{T,p}
&\overset{\mathrm{DE}}{\simeq}
(\theta^\star)^\top
\Sigma_p(\Sigma_p+\tau_{T,p}I_p)^{-1}
\theta^\star
=
\frac1{1+\tau_{T,p}}
\|\theta^\star\|_2^2
=
\frac1{1+\tau_{T,p}}
=
a_{T,p}.
\label{eq:master-b-DE}
\end{align}
Equivalently, the same conclusion follows entirely from the exact
cross identity:
\[
b_{T,p}
=
\mathbb E[
(\beta^\star_{\mathrm{str}})^\top
\theta_{T,\mathrm{str}}]
=
S_{T,p}+V_{T,p}
\overset{\mathrm{DE}}{\simeq}
a_{T,p}.
\]

Consequently,
\begin{equation}
\label{eq:master-stage1-mean-DE}
\mathbb E\widehat\beta_T
\overset{\mathrm{DE}}{\simeq}
\begin{pmatrix}
a_{T,p}\beta^\star_{\mathrm{str}}\\
0
\end{pmatrix}.
\end{equation}

Finally, using the exact mean form,
\begin{align}
(\mathbb E\widehat\beta_T)^\top
\Sigma_p
(\mathbb E\widehat\beta_T)
=&
b_{T,p}^2
(\beta^\star_{\mathrm{str}})^\top
I_k
\beta^\star_{\mathrm{str}}
=
b_{T,p}^2.
\end{align}
Together with \(b_{T,p}\overset{\mathrm{DE}}{\simeq}a_{T,p}\), this
gives
\begin{equation}
\label{eq:master-mean-energy}
(\mathbb E\widehat\beta_T)^\top
\Sigma_p
(\mathbb E\widehat\beta_T)
\overset{\mathrm{DE}}{\simeq}
a_{T,p}^2.
\end{equation}

Define the exact second-moment matrix
$
\mathsf M_{T,p}
:=
\mathbb E
\left[
\widehat\beta_T\widehat\beta_T^\top
\right].
$
Partition it according to the strong and weak covariance blocks:
\[
\mathsf M_{T,p}
=
\begin{pmatrix}
\mathsf M_{\mathrm{ss},p}
&
\mathsf M_{\mathrm{sw},p}
\\
\mathsf M_{\mathrm{ws},p}
&
\mathsf M_{\mathrm{ww},p}
\end{pmatrix}.
\]

For every \(O\in O(p-k)\), the matrix \(D_O\) defined above satisfies
$
\widehat\beta_T(U_TD_O)
=
D_O^\top\widehat\beta_T(U_T),
U_TD_O\overset{d}=U_T.
$
Therefore
\begin{align}
\mathsf M_{T,p}
=&
\mathbb E
\left[
\widehat\beta_T(U_TD_O)
\widehat\beta_T(U_TD_O)^\top
\right]
=
D_O^\top
\mathbb E
\left[
\widehat\beta_T(U_T)
\widehat\beta_T(U_T)^\top
\right]
D_O=
D_O^\top\mathsf M_{T,p}D_O.
\label{eq:master-second-moment-invariance}
\end{align}
Expanding the right-hand side blockwise gives
\[
D_O^\top\mathsf M_{T,p}D_O
=
\begin{pmatrix}
\mathsf M_{\mathrm{ss},p}
&
\mathsf M_{\mathrm{sw},p}O
\\
O^\top\mathsf M_{\mathrm{ws},p}
&
O^\top\mathsf M_{\mathrm{ww},p}O
\end{pmatrix}.
\]
Comparison with \(\mathsf M_{T,p}\) yields
\[
\mathsf M_{\mathrm{sw},p}O
=
\mathsf M_{\mathrm{sw},p},
\qquad
O^\top\mathsf M_{\mathrm{ws},p}
=
\mathsf M_{\mathrm{ws},p},
\]
and
\[
O^\top\mathsf M_{\mathrm{ww},p}O
=
\mathsf M_{\mathrm{ww},p}
\qquad
\text{for every }O\in O(p-k).
\]
Taking \(O=-I_{p-k}\) in the first two identities gives
$
\mathsf M_{\mathrm{sw},p}
=
-\mathsf M_{\mathrm{sw},p},
\mathsf M_{\mathrm{ws},p}
=
-\mathsf M_{\mathrm{ws},p},
$
so
$
\mathsf M_{\mathrm{sw},p}
=
\mathsf M_{\mathrm{ws},p}
=
0.
$

The last invariance identity states that the weak--weak block is
unchanged by conjugation with every orthogonal matrix.  Sign-flip
matrices force all off-diagonal entries to vanish, and permutation
matrices force all diagonal entries to be equal.  Therefore there is a scalar \(\nu_{T,p}\geq0\) such that
$
\mathsf M_{\mathrm{ww},p}
=
\nu_{T,p}I_{p-k}.
$
It remains to determine \(\nu_{T,p}\).  On the weak block,
$
\theta_{T,\mathrm{wk}}
=
\delta_p^{1/2}\widehat\beta_{T,\mathrm{wk}},
$
and hence
\begin{align}
V_{T,p}
=&
\mathbb E\|\theta_{T,\mathrm{wk}}\|_2^2
=
\delta_p
\mathbb E\|\widehat\beta_{T,\mathrm{wk}}\|_2^2
=
\delta_p\operatorname{Tr}
(\mathsf M_{\mathrm{ww},p})
=
\delta_p\nu_{T,p}(p-k).
\end{align}
Thus
$
\nu_{T,p}
=
\frac{V_{T,p}}{(p-k)\delta_p}.
$ On the strong block, whitening has no effect because the corresponding
eigenvalue of \(\Sigma_p\) is \(1\):
\[
\theta_{T,\mathrm{str}}
=
\widehat\beta_{T,\mathrm{str}}.
\]
Therefore
$
S_{T,p}
=
\mathbb E\|\widehat\beta_{T,\mathrm{str}}\|_2^2
=
\operatorname{Tr}(\mathsf M_{\mathrm{ss},p}).
$ Writing
$
\mathsf M_{T,\mathrm{str},p}
:=
\mathsf M_{\mathrm{ss},p},
$
we arrive at the exact finite-\(p\) block structure
\begin{equation}
\label{eq:master-M1-block}
\mathsf M_{T,p}
=
\begin{pmatrix}
\mathsf M_{T,\mathrm{str},p}&0\\[1mm]
0&
\dfrac{V_{T,p}}{(p-k)\delta_p}
I_{p-k}
\end{pmatrix},
\qquad
\operatorname{Tr}\mathsf M_{T,\mathrm{str},p}
=
S_{T,p}.
\end{equation}

No assumption such as \(k=o(p)\) has been used.  The equivariance,
vanishing cross blocks, weak-block isotropy, and normalization above
are exact for every finite \(p\).

Let
$
\mu_{T,p}
:=
\mathbb E\widehat\beta_T
$
and define the covariance of the Stage-I estimator by
$
\Gamma_{T,p}
:=
\mathsf M_{T,p}
-
\mu_{T,p}\mu_{T,p}^\top.
$
Its total prediction-norm trace is
\begin{align}
\operatorname{Tr}
(\Sigma_p\Gamma_{T,p})
=&
\mathbb E[
\widehat\beta_T^\top
\Sigma_p
\widehat\beta_T]
-
\mu_{T,p}^\top
\Sigma_p
\mu_{T,p}.
\label{eq:master-fluctuation-total-exact}
\end{align}

The first term equals
\[
S_{T,p}+V_{T,p}
=
1-\overline{\mathcal R}_{T,p}
\]
exactly.  Using
\eqref{eq:master-total-energy-DE} and
\eqref{eq:master-mean-energy}, we obtain
\[
\operatorname{Tr}
(\Sigma_p\Gamma_{T,p})
\overset{\mathrm{DE}}{\simeq}
a_{T,p}-a_{T,p}^2.
\]

Define
$
c_{T,p}
:=
a_{T,p}-a_{T,p}^2
=
\frac{\tau_{T,p}}
{(1+\tau_{T,p})^2}.
$
Thus \(c_{T,p}\) is the deterministic-equivalent total
projection-induced fluctuation energy.

Knowing the scalar total fluctuation energy \(c_{T,p}\) does not yet
identify its location.  To separate the strong and weak contributions,
we first write down the precise random quantity that must be evaluated;
only after that do we invoke the two-point deterministic-equivalent
lemma.

Let \(D_p\succeq0\) be a deterministic prediction-energy weight and
define
\begin{equation}
\label{eq:master-weighted-fluctuation-def}
\mathcal F_{D,p}
:=
\mathbb E\!\left[
(\widehat\beta_T-\mu_{T,p})^\top
D_p
(\widehat\beta_T-\mu_{T,p})
\right],
\qquad
\mu_{T,p}:=\mathbb E\widehat\beta_T.
\end{equation}
Thus \(\mathcal F_{D,p}\) is the amount of Stage-I covariance measured
in the quadratic norm generated by \(D_p\).  Expanding the centered
quadratic form gives the exact identity
\begin{equation}
\label{eq:master-weighted-fluctuation-expand}
\mathcal F_{D,p}
=
\mathbb E[\widehat\beta_T^\top D_p\widehat\beta_T]
-
\mu_{T,p}^\top D_p\mu_{T,p}.
\end{equation}

We now express the first term through the random projector.  Recall that
$
\widehat\beta_T
=
\Sigma_p^{-1/2}P_{T,p}\theta^\star.
$
Define
$
M_{D,p}
:=
\Sigma_p^{-1/2}D_p\Sigma_p^{-1/2}.
$
Then, using \(P_{T,p}^\top=P_{T,p}\),
\begin{align}
\widehat\beta_T^\top D_p\widehat\beta_T
=&
(\theta^\star)^\top
P_{T,p}\Sigma_p^{-1/2}D_p\Sigma_p^{-1/2}
P_{T,p}\theta^\star
=
(\theta^\star)^\top
P_{T,p}M_{D,p}P_{T,p}\theta^\star.
\label{eq:master-random-weighted-second-moment}
\end{align}
Consequently, the random matrix that must be understood is
$
\mathbb E[P_{T,p}M_{D,p}P_{T,p}].
$
This is a two-point object: the same random projector appears on both
sides of \(M_{D,p}\).  A one-point equivalent for \(P_{T,p}\) is not
sufficient, because in general
\[
\mathbb E[P_{T,p}M_{D,p}P_{T,p}]
\neq
(\mathbb E P_{T,p})M_{D,p}(\mathbb E P_{T,p}).
\]
The difference between these two matrices is exactly the contribution
of the projector fluctuations. To place this random object in the resolvent setting, regularize the
projector.  For \(\lambda>0\), define
$
\widehat T_{T,\lambda}
:=
\widehat\Sigma_{T,p}
(\widehat\Sigma_{T,p}+\lambda I_p)^{-1}.
$
By \eqref{eq:master-projector-ridge-limit},
\[
\widehat T_{T,\lambda}\longrightarrow P_{T,p}
\qquad\text{as }\lambda\downarrow0.
\]
Define the associated regularized coefficient and its mean by
$
\widehat\beta_{T,\lambda}
:=
\Sigma_p^{-1/2}\widehat T_{T,\lambda}\theta^\star,
\mu_{T,p,\lambda}
:=
\mathbb E\widehat\beta_{T,\lambda}.
$
Then the regularized weighted fluctuation is exactly
\begin{align}
\mathcal F_{D,p}(\lambda)
:=&
\mathbb E\!\left[
(\widehat\beta_{T,\lambda}-\mu_{T,p,\lambda})^\top
D_p
(\widehat\beta_{T,\lambda}-\mu_{T,p,\lambda})
\right]
=
(\theta^\star)^\top
\mathbb E\!\left[
\widehat T_{T,\lambda}M_{D,p}\widehat T_{T,\lambda}
\right]
\theta^\star
-
\mu_{T,p,\lambda}^\top D_p\mu_{T,p,\lambda}.
\label{eq:master-weighted-fluctuation-regularized}
\end{align}

We now invoke Lemma~\ref{lem:two-point-deterministic-equivalence}.
Suppose first that \(N_T<p\), and set
$
q_{T,p}:=\frac{p}{N_T}.
$
For fixed \(\lambda>0\), write
$
G_{T,\lambda}
:=
(\Sigma_p+\kappa_{T,\lambda}I_p)^{-1},
T_{T,\lambda}
:=
\Sigma_pG_{T,\lambda}.
$
The equal-ridge specialization of the two-point lemma gives, for every
deterministic \(M_p\) of uniformly bounded operator norm,
\begin{align}
\mathbb E\!\left[
\widehat T_{T,\lambda}M_{D,p}\widehat T_{T,\lambda}
\right]
&\overset{\mathrm{DE}}{\simeq}
T_{T,\lambda}M_{D,p}T_{T,\lambda}
+
\kappa_{T,\lambda}^{2}
G_{T,\lambda}\Sigma_pG_{T,\lambda}
\frac{
q_{T,p}\operatorname{tr}
\left[
\Sigma_pG_{T,\lambda}M_{D,p}G_{T,\lambda}
\right]
}{
1-q_{T,p}\operatorname{tr}
\left[
\Sigma_p^2G_{T,\lambda}^2
\right]
}.
\label{eq:master-two-point-specialization}
\end{align}
The first term on the
right-hand side is the squared-mean contribution.  Indeed, the
one-point deterministic equivalent gives
\[
\mu_{T,p,\lambda}
\overset{\mathrm{DE}}{\simeq}
\Sigma_p^{-1/2}T_{T,\lambda}\theta^\star,
\]
and hence
\begin{align}
\mu_{T,p,\lambda}^\top D_p\mu_{T,p,\lambda}
&\overset{\mathrm{DE}}{\simeq}
(\theta^\star)^\top
T_{T,\lambda}
\Sigma_p^{-1/2}D_p\Sigma_p^{-1/2}
T_{T,\lambda}\theta^\star
=
(\theta^\star)^\top
T_{T,\lambda}M_{D,p}T_{T,\lambda}\theta^\star.
\label{eq:master-two-point-mean-cancellation}
\end{align}
Therefore the first term in
\eqref{eq:master-two-point-specialization} cancels the squared-mean term
in \eqref{eq:master-weighted-fluctuation-regularized}.  The second term
is precisely the leading covariance contribution. We apply the two-point deterministic equivalent only to the strong- and
weak-block prediction-energy weights defined below. Both commute with
\(\Sigma_p\), and their whitened weights
\(M_{D,p}=\Sigma_p^{-1/2}D_p\Sigma_p^{-1/2}\)
are orthogonal coordinate projectors. In particular,
\[
\|M_{D,p}\|_{\mathrm{op}}=1
\]
uniformly in \(p\). We first take the
high-dimensional deterministic equivalent at fixed \(\lambda>0\), and
then send \(\lambda\downarrow0\).  Write
$
G_{T,p}:=(\Sigma_p+\tau_{T,p}I_p)^{-1},
\gamma_{T,p}
:=q_{T,p}\operatorname{tr}
\left[
\Sigma_p^2G_{T,p}^2
\right].
$
The surviving covariance term gives
\begin{align}
\mathcal F_{D,p}
&\overset{\mathrm{DE}}{\simeq}
\tau_{T,p}^2
(\theta^\star)^\top
G_{T,p}\Sigma_pG_{T,p}\theta^\star
\frac{
q_{T,p}\operatorname{tr}
\left[
\Sigma_pG_{T,p}M_{D,p}G_{T,p}
\right]
}{1-\gamma_{T,p}}.
\label{eq:master-weighted-cov-raw}
\end{align}

We now simplify every factor.  Since \(\theta^\star\) is supported on
the strong block, on which \(\Sigma_p=I_k\),
$
(\theta^\star)^\top
G_{T,p}\Sigma_pG_{T,p}\theta^\star
=
\frac1{(1+\tau_{T,p})^2}.
$
Since \(D_p\), \(\Sigma_p\), and \(G_{T,p}\) commute,
$
\operatorname{tr}
\left[
\Sigma_pG_{T,p}M_{D,p}G_{T,p}
\right]
=
\operatorname{tr}[D_pG_{T,p}^2].
$ 
It remains to simplify the denominator.  The ridgeless fixed-point
equation is
$
q_{T,p}\operatorname{tr}[\Sigma_pG_{T,p}]=1.
$
Using
\[
\Sigma_pG_{T,p}
=
\Sigma_p^2G_{T,p}^2
+
\tau_{T,p}\Sigma_pG_{T,p}^2,
\]
we obtain
\begin{align*}
1
=&
q_{T,p}\operatorname{tr}[\Sigma_pG_{T,p}]
=
q_{T,p}\operatorname{tr}[\Sigma_p^2G_{T,p}^2]
+
q_{T,p}\tau_{T,p}\operatorname{tr}[\Sigma_pG_{T,p}^2]
=
\gamma_{T,p}
+
q_{T,p}\tau_{T,p}\operatorname{tr}[\Sigma_pG_{T,p}^2].
\end{align*}
Therefore
$
1-\gamma_{T,p}
=
q_{T,p}\tau_{T,p}
\operatorname{tr}
\left[
\Sigma_pG_{T,p}^2
\right].
$ Substituting these three simplifications into
\eqref{eq:master-weighted-cov-raw} yields
\begin{align*}
\mathcal F_{D,p}
&\overset{\mathrm{DE}}{\simeq}
\tau_{T,p}^2
\frac1{(1+\tau_{T,p})^2}
\frac{
q_{T,p}\operatorname{tr}[D_pG_{T,p}^2]
}{
q_{T,p}\tau_{T,p}\operatorname{tr}[\Sigma_pG_{T,p}^2]
}
=
\frac{\tau_{T,p}}{(1+\tau_{T,p})^2}
\frac{
\operatorname{tr}[D_pG_{T,p}^2]
}{
\operatorname{tr}[\Sigma_pG_{T,p}^2]
}.
\end{align*}
Since
$
c_{T,p}=\frac{\tau_{T,p}}{(1+\tau_{T,p})^2},
G_{T,p}^2=(\Sigma_p+\tau_{T,p}I_p)^{-2},
$
the desired weighted covariance formula is
\begin{equation}
\label{eq:master-weighted-cov-DE}
\mathcal F_{D,p}
\overset{\mathrm{DE}}{\simeq}
c_{T,p}
\frac{
\operatorname{tr}
\left[
D_p(\Sigma_p+\tau_{T,p}I_p)^{-2}
\right]
}{
\operatorname{tr}
\left[
\Sigma_p(\Sigma_p+\tau_{T,p}I_p)^{-2}
\right]
}.
\end{equation}

Equation \eqref{eq:master-weighted-cov-DE} says that the total
fluctuation energy is \(c_{T,p}\), while the trace ratio determines the
fraction assigned to the directions selected by \(D_p\).  If
\(N_T\ge p\), then \(P_{T,p}=I_p\) almost surely and \(c_{T,p}=0\), so
every weighted covariance vanishes identically.

We now choose \(D_p\) to select the two blocks.  For the strong
prediction-energy weight,
\[
D_{\mathrm{str},p}
=
\begin{pmatrix}
I_k&0\\
0&0
\end{pmatrix},
\]
the corresponding trace is
\begin{equation}
\label{eq:master-Ls}
L_{\mathrm{str},p}(\tau)
:=
\operatorname{tr}
\left[
D_{\mathrm{str},p}
(\Sigma_p+\tau I_p)^{-2}
\right]
=
\frac{k}{p}\frac{1}{(1+\tau)^2},
\end{equation}

For the weak prediction-energy weight,
$
D_{\mathrm{wk},p}
=
\begin{pmatrix}
0&0\\
0&\delta_p I_{p-k}
\end{pmatrix},
$
we obtain
\begin{equation}
\label{eq:master-Lw}
L_{\mathrm{wk},p}(\tau)
:=
\operatorname{tr}
\left[
D_{\mathrm{wk},p}
(\Sigma_p+\tau I_p)^{-2}
\right]
=
\frac{p-k}{p}
\frac{\delta_p}{(\delta_p+\tau)^2}.
\end{equation}

Because
$
D_{\mathrm{str},p}+D_{\mathrm{wk},p}=\Sigma_p,
$
the denominator in \eqref{eq:master-weighted-cov-DE} is exactly
$
L_{\mathrm{str},p}(\tau)+L_{\mathrm{wk},p}(\tau).
$
Define the weak fraction of the Stage-I fluctuation energy by
\begin{equation}
\label{eq:master-Theta-def}
\Theta_{T,p}
:=
\frac{
L_{\mathrm{wk},p}(\tau_{T,p})
}{
L_{\mathrm{str},p}(\tau_{T,p})
+
L_{\mathrm{wk},p}(\tau_{T,p})
}.
\end{equation}

Substituting
\eqref{eq:master-Ls}--\eqref{eq:master-Lw} gives
\begin{align}
\Theta_{T,p}
=&
\frac{
\dfrac{p-k}{p}
\dfrac{\delta_p}
{(\delta_p+\tau_{T,p})^2}
}{
\dfrac{k}{p}
\dfrac1{(1+\tau_{T,p})^2}
+
\dfrac{p-k}{p}
\dfrac{\delta_p}
{(\delta_p+\tau_{T,p})^2}
}
=
\left[
1+
\frac{
k(\delta_p+\tau_{T,p})^2
}{
(p-k)\delta_p
(1+\tau_{T,p})^2
}
\right]^{-1}.
\label{eq:master-Theta-expanded}
\end{align}

Hence
\begin{equation}
\label{eq:master-Theta}
\Theta_{T,p}
=
\left[
1+
\frac{
k(\delta_p+\tau_{T,p})^2
}{
(p-k)\delta_p
(1+\tau_{T,p})^2
}
\right]^{-1}.
\end{equation}

The weak mean vanishes exactly.  Hence
\begin{align*}
\mathcal F_{D_{\mathrm{wk}},p}
=&
\mathbb E\!\left[
\widehat\beta_{T,\mathrm{wk}}^\top
(\delta_pI_{p-k})
\widehat\beta_{T,\mathrm{wk}}
\right]
=
\delta_p\mathbb E
\|\widehat\beta_{T,\mathrm{wk}}\|_2^2
=
V_{T,p}.
\end{align*}
Therefore \eqref{eq:master-weighted-cov-DE} and
\eqref{eq:master-Theta-def} give
\begin{equation}
\label{eq:master-V-DE}
V_{T,p}
\overset{\mathrm{DE}}{\simeq}
V_{T,p}^{\mathrm{DE}}
:=
c_{T,p}\Theta_{T,p}.
\end{equation}

For the strong weight, the same formula gives
$
\mathcal F_{D_{\mathrm{str}},p}
\overset{\mathrm{DE}}{\simeq}
c_{T,p}(1-\Theta_{T,p}).
$
The strong block contains both this fluctuation energy and the
deterministic mean energy
$
a_{T,p}^2,
$
and therefore
\begin{align}
S_{T,p}
&\overset{\mathrm{DE}}{\simeq}
a_{T,p}^2
+
c_{T,p}(1-\Theta_{T,p})
=
a_{T,p}^2
+
(a_{T,p}-a_{T,p}^2)
-
c_{T,p}\Theta_{T,p}
=
a_{T,p}
-
c_{T,p}\Theta_{T,p}.
\end{align}

Therefore
\begin{equation}
\label{eq:master-S-DE}
S_{T,p}
\overset{\mathrm{DE}}{\simeq}
S_{T,p}^{\mathrm{DE}}
:=
a_{T,p}
-
c_{T,p}\Theta_{T,p}.
\end{equation}

Combining
\eqref{eq:master-V-DE} and
\eqref{eq:master-S-DE},
\begin{equation}
\label{eq:master-SV-DE}
\begin{aligned}
V_{T,p}^{\mathrm{DE}}
=&
c_{T,p}\Theta_{T,p},
\\
S_{T,p}^{\mathrm{DE}}
=&
a_{T,p}
-
c_{T,p}\Theta_{T,p}.
\end{aligned}
\end{equation}

As a consistency check,
$
S_{T,p}^{\mathrm{DE}}
+
V_{T,p}^{\mathrm{DE}}
=
a_{T,p},
$
so
$
1-
S_{T,p}^{\mathrm{DE}}
-
V_{T,p}^{\mathrm{DE}}
=
1-a_{T,p}
=
\frac{\tau_{T,p}}{1+\tau_{T,p}}
=
\overline{\mathcal R}_{T,p}^{\mathrm{DE}},
$
exactly as required by the projection identity.

\end{proof}
\subsection{Stage II: linear-random-feature derivation conditional on Stage I}
\label{sec:master-stage2}
% ============================================================

We now derive Stage II by specializing the linear-random-feature
calculation of \cite[Sec.~IV]{atanasov2026scaling}. Conditional on
Stage I, set
\[
b:=\widehat\beta_T
\]
and regard \(b\in\mathbb R^p\) as deterministic. Throughout Stage II,
\(\mathbb E_S\) denotes expectation over the Stage-II randomness
\((U_S,\widetilde X)\), conditional on the Stage-I teacher \(b\).

Define
$
F_S:=\frac1{\sqrt{N_S}}U_S^\top\in\mathbb R^{p\times N_S},
\widetilde\Sigma
:=\frac1m\widetilde X^\top\widetilde X.
$
Because the pseudo-labels are
$
\widetilde y=\widetilde Xb,
$
the induced input-space coefficient of the minimum-norm ridgeless
student is
\begin{equation}
\label{eq:master-stage2-beta-exact}
\widehat\beta_S
=
F_S(\widetilde XF_S)^\dagger\widetilde Xb.
\end{equation}
Introduce a ridge \(\lambda>0\) in the normalized feature
coordinates.  Define
\begin{align}
\widehat v_{S,\lambda}
:=&
\left(
F_S^\top\widetilde\Sigma F_S+\lambda I_{N_S}
\right)^{-1}
F_S^\top\widetilde\Sigma b,
\nonumber\\
\widehat\beta_{S,\lambda}
:=&
F_S\widehat v_{S,\lambda}
=
H_{S,\lambda}b,
\label{eq:master-stage2-H-def}
\end{align}
where
\[
H_{S,\lambda}
:=
F_S
\left(
F_S^\top\widetilde\Sigma F_S+\lambda I_{N_S}
\right)^{-1}
F_S^\top\widetilde\Sigma.
\]
Let
$
G_S:=F_SF_S^\top.
$
The push-through identity
$
A(BA+\lambda I)^{-1}
=
(AB+\lambda I)^{-1}A
$
gives
\begin{align}
H_{S,\lambda}
=&
G_S\widetilde\Sigma
(G_S\widetilde\Sigma+\lambda I_p)^{-1}
=
I_p-\lambda
(G_S\widetilde\Sigma+\lambda I_p)^{-1}.
\label{eq:master-stage2-H-push}
\end{align}
Therefore the regularized displacement from the fixed teacher is
\begin{equation}
\label{eq:master-stage2-d-lambda}
d_{S,\lambda}
:=
\widehat\beta_{S,\lambda}-b
=
-Q_{S,\lambda}b,
\qquad
Q_{S,\lambda}
:=
\lambda
(G_S\widetilde\Sigma+\lambda I_p)^{-1}.
\end{equation}
Thus Stage II requires two different random-matrix objects:
\begin{align}
\mathbb E_S[Q_{S,\lambda}]
&\quad\text{for the cross term},
\label{eq:master-stage2-one-point-object}
\\
\mathbb E_S
\left[
Q_{S,\lambda}^\top\Sigma_pQ_{S,\lambda}
\right]
&\quad\text{for the mimic term}.
\label{eq:master-stage2-two-point-object}
\end{align}
The first is a one-point resolvent problem; the second is a two-point
or source-differentiated resolvent problem.

For the ridgeless displacement
$
d_S:=\widehat\beta_S-b,
$
we have
\[
\widehat\beta_S-\beta^\star
=
(b-\beta^\star)+d_S.
\]
Consequently,
\begin{align}
\mathcal R(\widehat h_S)
=&
(b-\beta^\star)^\top\Sigma_p(b-\beta^\star)
+d_S^\top\Sigma_pd_S
+2(b-\beta^\star)^\top\Sigma_pd_S.
\label{eq:master-stage2-risk-expand}
\end{align}
Define
$
\mathcal R_S^{\mathrm{mim}}(b)
:=
\mathbb E_S
\left[
d_S^\top\Sigma_pd_S
\right],
\mathcal C_S^{\mathrm{cross}}(b)
:=
2\mathbb E_S
\left[
(b-\beta^\star)^\top\Sigma_pd_S
\right].
$
Then the conditional identity is
\begin{equation}
\label{eq:master-stage2-risk-conditional}
\mathbb E_S[\mathcal R(\widehat h_S)]
=
\mathcal R(h_b)
+\mathcal R_S^{\mathrm{mim}}(b)
+\mathcal C_S^{\mathrm{cross}}(b).
\end{equation}

Define the feature-dressed population covariance
$
\Sigma_{F,S}
:=
\Sigma_p^{1/2}G_S\Sigma_p^{1/2}.
$
For a positive semidefinite \(p\times p\) matrix \(A\), write
\begin{align}
\operatorname{df}_{A,p}^{(1)}(\kappa)
:=&
\operatorname{tr}
\left[
A(A+\kappa I_p)^{-1}
\right],
\nonumber\\
\operatorname{df}_{A,p}^{(2)}(\kappa)
:=&
\operatorname{tr}
\left[
A^2(A+\kappa I_p)^{-2}
\right].
\end{align}
At fixed \(F_S\), the sample average in
\cite[Eq.~(35)]{atanasov2026scaling} introduces the first effective
ridge \(\widehat\kappa_{1,\lambda}(F_S)>0\), determined by
\begin{equation}
\label{eq:master-stage2-kappa1}
\widehat\kappa_{1,\lambda}(F_S)
\left[
1-
\frac pm
\operatorname{df}_{\Sigma_{F,S},p}^{(1)}
\!\left(\widehat\kappa_{1,\lambda}(F_S)\right)
\right]
=
\lambda.
\end{equation}
The corresponding two-resolvent stability parameter is
\begin{equation}
\label{eq:master-stage2-gamma1-fixedF}
\widehat\gamma_{1,\lambda}(F_S)
:=
\frac pm
\operatorname{df}_{\Sigma_{F,S},p}^{(2)}
\!\left(\widehat\kappa_{1,\lambda}(F_S)\right).
\end{equation}

The positive-ridge mimic quadratic is exactly
\begin{align}
d_{S,\lambda}^\top\Sigma_p d_{S,\lambda}
=&
\lambda^2 b^\top
(\widetilde\Sigma G_S+\lambda I_p)^{-1}
\Sigma_p
(G_S\widetilde\Sigma+\lambda I_p)^{-1}b.
\label{eq:master-stage2-mimic-random}
\end{align}
To evaluate this product, introduce
\begin{equation}
\label{eq:master-stage2-tf-tilde}
\widetilde T_{b,F}(\kappa_1)
:=
b^\top\Sigma_p^{1/2}
(\Sigma_{F,S}+\kappa_1I_p)^{-1}
\Sigma_p^{1/2}b.
\end{equation}
The source-differentiation calculation of
\cite[Eqs.~(34)--(36)]{atanasov2026scaling} gives
\begin{equation}
\label{eq:master-stage2-data-averaged-mimic}
\mathbb E_{\widetilde X}
\!\left[
d_{S,\lambda}^\top\Sigma_p d_{S,\lambda}
\,\middle|\,F_S,b
\right]
\overset{\mathrm{DE}}{\simeq}
-
\frac{\widehat\kappa_{1,\lambda}(F_S)^2}
{1-\widehat\gamma_{1,\lambda}(F_S)}
\frac{\partial}{\partial\kappa_1}
\widetilde T_{b,F}(\kappa_1)
\bigg|_{\kappa_1=\widehat\kappa_{1,\lambda}(F_S)}.
\end{equation}
The derivative generates the two copies of the population resolvent,
while
\((1-\widehat\gamma_{1,\lambda}(F_S))^{-1}\)
is the corresponding two-point stability factor.

We now average over \(F_S\). Let
$
d_{1,\lambda}
:=
\operatorname{df}_p(\kappa_{2,\lambda}),
d_{2,\lambda}
:=
\operatorname{df}_p^{(2)}(\kappa_{2,\lambda}),
$
where \(\operatorname{df}_p\) and
\(\operatorname{df}_p^{(2)}\) are the Stage-I degrees-of-freedom
functions.
The feature average in
\cite[Eqs.~(37)--(38)]{atanasov2026scaling} introduces a second
effective ridge
\begin{equation}
\label{eq:master-stage2-kappa2-general}
\kappa_{2,\lambda}
=
\bar\kappa_{1,\lambda}
S_{G_S}(-d_{1,\lambda}),
\end{equation}
where \(S_{G_S}\) is the limiting \(S\)-transform of
\(G_S=F_SF_S^\top\).
For
$
F_S=\frac1{\sqrt{N_S}}U_S^\top,
$
the matrix \(G_S\) is a white Wishart Gram matrix with aspect ratio
\(p/N_S\).  Its \(S\)-transform, evaluated at
\(-d_{1,\lambda}\), is
\begin{equation}
\label{eq:master-stage2-white-S-transform}
S_{G_S}(-d_{1,\lambda})
=
\frac1{1-(p/N_S)d_{1,\lambda}}.
\end{equation}
Consequently,
\begin{equation}
\label{eq:master-stage2-kappa2-white}
\kappa_{2,\lambda}
=
\frac{\bar\kappa_{1,\lambda}}
{1-(p/N_S)d_{1,\lambda}}.
\end{equation}
After averaging over \(F_S\), introduce the deterministic first
effective ridge
\begin{equation}
\label{eq:master-stage2-kappa1-white}
\bar\kappa_{1,\lambda}
:=
\frac{\lambda}
{1-(p/m)d_{1,\lambda}}.
\end{equation}

Eliminating \(\bar\kappa_{1,\lambda}\) between
\eqref{eq:master-stage2-kappa2-white} and
\eqref{eq:master-stage2-kappa1-white} gives the central nested-ridge
equation
\begin{equation}
\label{eq:master-stage2-nested-ridge}
\lambda
=
\kappa_{2,\lambda}
\left[
1-\frac pm d_{1,\lambda}
\right]
\left[
1-\frac p{N_S}d_{1,\lambda}
\right],
\qquad
d_{1,\lambda}
=
\operatorname{df}_p(\kappa_{2,\lambda}).
\end{equation}
The first bracket is the sample renormalization and the second is the
random-feature renormalization. To state the feature-averaged mimic formula, define
$
\alpha_{S,\lambda}
:=
\frac{d\log\kappa_{2,\lambda}}
{d\log\bar\kappa_{1,\lambda}}.
$
The feature-averaged version of the sample stability parameter is
\begin{equation}
\label{eq:master-stage2-gamma1-feature-averaged}
\bar\gamma_{1,\lambda}
:=
\frac pm d_{1,\lambda}
\left[
1-
\frac{d_{1,\lambda}-d_{2,\lambda}}
{d_{1,\lambda}}
\alpha_{S,\lambda}
\right].
\end{equation}
For a deterministic teacher vector \(b\), define
$
T_b(\kappa)
:=
b^\top\Sigma_p(\Sigma_p+\kappa I_p)^{-1}b,
$
so that
$
-T_b'(\kappa)
=
b^\top\Sigma_p(\Sigma_p+\kappa I_p)^{-2}b.
$
Define the positive-ridge conditional mimic risk by
\begin{equation}
\label{eq:master-stage2-mimic-positive-def}
\mathcal R_{S,\lambda}^{\mathrm{mim}}(b)
:=
\mathbb E_S
\left[
d_{S,\lambda}^\top\Sigma_p d_{S,\lambda}
\right].
\end{equation}

Specializing \cite[Eq.~(40)]{atanasov2026scaling} to zero label noise
gives the deterministic equivalent
\begin{align}
\mathcal R_{S,\lambda}^{\mathrm{mim}}(b)
\overset{\mathrm{DE}}{\simeq}
{}&
-
\frac{
\kappa_{2,\lambda}^2
T_b'(\kappa_{2,\lambda})
}{
1-\bar\gamma_{1,\lambda}
}
\alpha_{S,\lambda}
+
\frac{
\kappa_{2,\lambda}
T_b(\kappa_{2,\lambda})
}{
1-\bar\gamma_{1,\lambda}
}
(1-\alpha_{S,\lambda}).
\label{eq:master-stage2-positive-ridge-mimic}
\end{align}
This formula is the explicit result of first averaging over the
Stage-II samples and then averaging over the random features. The same two subordinations applied to the one-point quantity in
\eqref{eq:master-stage2-one-point-object} give
\begin{equation}
\label{eq:master-stage2-Q-one-point-DE}
\mathbb E_S[Q_{S,\lambda}]
\overset{\mathrm{DE}}{\simeq}
\kappa_{2,\lambda}
(\Sigma_p+\kappa_{2,\lambda}I_p)^{-1}.
\end{equation}
Define the deterministic Stage-II ridgeless effective ridge by
$
\tau_{S,p}
:=
\lim_{\lambda\downarrow0}\kappa_{2,\lambda}.
$
Equation \eqref{eq:master-stage2-nested-ridge} shows algebraically that
this limit can be reached in three distinct branches.

\medskip\noindent\emph{No bottleneck.}
If
$
p<\min\{m,N_S\},
$
both brackets in \eqref{eq:master-stage2-nested-ridge} remain positive
at \(\kappa_{2,\lambda}=0\).  Hence
$
\tau_{S,p}=0.
$
In fact, the conclusion is exact: \(F_S\) has full row rank and
\(\widetilde X\) has full column rank almost surely, so
$
\widehat\beta_S=b.
$

\medskip\noindent\emph{Feature bottleneck.}
If
$
N_S<\min\{m,p\},
$
the feature bracket vanishes while the sample bracket remains
positive.  Thus \(\tau_{S,p}>0\) is the unique solution of
$
\operatorname{df}_p(\tau_{S,p})
=
\frac{N_S}{p}.
$
In this branch,
\[
\bar\kappa_{1,\lambda}\downarrow0,
\qquad
\kappa_{2,\lambda}\longrightarrow\tau_{S,p}>0,
\]
so
\[
\alpha_{S,\lambda}\longrightarrow0,
\qquad
\bar\gamma_{1,\lambda}\longrightarrow\frac{N_S}{m}.
\]
Taking the ridgeless limit in
\eqref{eq:master-stage2-positive-ridge-mimic} yields
\begin{equation}
\label{eq:master-stage2-mimic-feature}
\mathcal R_S^{\mathrm{mim}}(b)
\overset{\mathrm{DE}}{\simeq}
\frac{
\tau_{S,p}T_b(\tau_{S,p})
}{1-N_S/m}.
\end{equation}
This is exactly the zero-noise specialization of
\cite[Eq.~(42)]{atanasov2026scaling}.

\medskip\noindent\emph{Sample bottleneck.}
If
$
m<\min\{p,N_S\},
$
the sample bracket vanishes while the feature bracket remains positive.
Therefore \(\tau_{S,p}>0\) is the unique solution of
$
\operatorname{df}_p(\tau_{S,p})
=
\frac mp.
$
Let
$
\gamma_{S,p}^{\mathrm{samp}}
:=
\frac pm
\operatorname{df}_p^{(2)}(\tau_{S,p}).
$
The overparameterized ridgeless formula
\cite[Eq.~(44)]{atanasov2026scaling}, with zero label noise, gives
\begin{align}
\mathcal R_S^{\mathrm{mim}}(b)
\overset{\mathrm{DE}}{\simeq}
{}&
\frac{-\tau_{S,p}^2T_b'(\tau_{S,p})}
{1-\gamma_{S,p}^{\mathrm{samp}}}
+
\tau_{S,p}T_b(\tau_{S,p})
\frac{d\log S_{G_S}(-d)}{d\log d}
\bigg|_{d=m/p}.
\end{align}
For the white random features,
$
S_{G_S}(-d)
=
\frac1{1-(p/N_S)d},
$
and therefore
\begin{align}
\left.
\frac{d\log S_{G_S}(-d)}{d\log d}
\right|_{d=m/p}
=&
\left.
\frac{(p/N_S)d}{1-(p/N_S)d}
\right|_{d=m/p}
=
\frac{m/N_S}{1-m/N_S}.
\end{align}
Substitution gives
\begin{equation}
\label{eq:master-stage2-mimic-sample}
\begin{aligned}
\mathcal R_S^{\mathrm{mim}}(b)
\overset{\mathrm{DE}}{\simeq}
{}&
\frac{
-\tau_{S,p}^2T_b'(\tau_{S,p})
}{
1-(p/m)\operatorname{df}_p^{(2)}(\tau_{S,p})
}
+
\frac{
\tau_{S,p}T_b(\tau_{S,p})(m/N_S)
}{1-m/N_S}.
\end{aligned}
\end{equation}
The first term is the sample-projection contribution.  Its stability
denominator comes from the two-point/source differentiation in the
data average.  The second term is the finite-random-feature
contribution and comes from the logarithmic derivative of the white
Wishart \(S\)-transform.  In particular, the denominator
$
1-\frac pm\operatorname{df}_p^{(2)}(\tau_{S,p})
$
does not follow merely by differentiating the scalar function
\(\operatorname{df}_p(\tau)\); it is the two-point stability factor in
\cite[Eq.~(44)]{atanasov2026scaling}.

Define the positive-ridge conditional cross term by
\begin{align}
\mathcal C_{S,\lambda}^{\mathrm{cross}}(b)
:=&
2\mathbb E_S
\left[
(b-\beta^\star)^\top
\Sigma_p d_{S,\lambda}
\right]
=
-2(b-\beta^\star)^\top
\Sigma_p
\mathbb E_S[Q_{S,\lambda}]b,
\label{eq:master-stage2-cross-positive-def}
\end{align}
Using \eqref{eq:master-stage2-Q-one-point-DE} and then sending
\(\lambda\downarrow0\),
\begin{equation}
\label{eq:master-stage2-cross-conditional-DE}
\mathcal C_S^{\mathrm{cross}}(b)
\overset{\mathrm{DE}}{\simeq}
-2\tau_{S,p}
(b-\beta^\star)^\top
\Sigma_p
(\Sigma_p+\tau_{S,p}I_p)^{-1}b.
\end{equation}

Set \(b=\widehat\beta_T\). Throughout this subsection,
\(\mathbb E_T\) denotes expectation over the Stage-I randomness
\(U_T\). Define the exact Stage-I-averaged Stage-II contributions by

\begin{align}
\mathcal M_{S,p}
:=&
\mathbb E_T
\left[
\mathcal R_S^{\mathrm{mim}}(\widehat\beta_T)
\right],
\nonumber\\
\mathcal C_{S,p}^{\mathrm{cross}}
:=&
\mathbb E_T
\left[
\mathcal C_S^{\mathrm{cross}}(\widehat\beta_T)
\right].
\end{align}
Recall
$
S_{T,p}
=
\mathbb E\|\theta_{T,\mathrm{str}}\|_2^2,
V_{T,p}
=
\mathbb E\|\theta_{T,\mathrm{wk}}\|_2^2,
$
where
$
\theta_T=\Sigma_p^{1/2}\widehat\beta_T.
$
Because
$
\widehat\beta_{T,\mathrm{str}}
=
\theta_{T,\mathrm{str}},
\delta_p
\mathbb E\|\widehat\beta_{T,\mathrm{wk}}\|_2^2
=
V_{T,p},
$
the conditional target functions average exactly as
\begin{align}
\mathbb E_T[T_{\widehat\beta_T}(\tau)]
=&
\frac{S_{T,p}}{1+\tau}
+
\frac{V_{T,p}}{\delta_p+\tau}
=:
\mathcal T_{T,p}^{(1)}(\tau),
\label{eq:master-stage2-T1}
\\
\mathbb E_T[-T_{\widehat\beta_T}'(\tau)]
=&
\frac{S_{T,p}}{(1+\tau)^2}
+
\frac{V_{T,p}}{(\delta_p+\tau)^2}
=:
\mathcal T_{T,p}^{(2)}(\tau).
\label{eq:master-stage2-T2}
\end{align}
The cross target has a stronger exact simplification. Define
\[
\mathcal B_{T,p}(\tau)
:=
\mathbb E_T
\left[
(\widehat\beta_T-\beta^\star)^\top
\Sigma_p(\Sigma_p+\tau I_p)^{-1}
\widehat\beta_T
\right].
\]
The exact Stage-I projection identity gives
$
(\theta_{T,\mathrm{str}}-\beta^\star_{\mathrm{str}})^\top
\theta_{T,\mathrm{str}}
=
-\|\theta_{T,\mathrm{wk}}\|_2^2.
$
Hence the strong part of \(\mathcal B_{T,p}(\tau)\) is
$
-\frac{V_{T,p}}{1+\tau},
$
whereas the weak part is
$
\frac{V_{T,p}}{\delta_p+\tau}.
$
Therefore
\begin{equation}
\label{eq:master-stage2-BT-exact}
\mathcal B_{T,p}(\tau)
=
V_{T,p}\Delta_p(\tau),
\qquad
\Delta_p(\tau)
:=
\frac1{\delta_p+\tau}
-
\frac1{1+\tau}.
\end{equation}
Since \(0<\delta_p<1\), one has \(\Delta_p(\tau)>0\).
Let
$
\mathcal T_{T,p}^{(r),\mathrm{DE}}(\tau),
r\in\{1,2\},
$
denote \eqref{eq:master-stage2-T1}--\eqref{eq:master-stage2-T2}
with \(S_{T,p},V_{T,p}\) replaced by their Stage-I deterministic
equivalents.
\begin{equation}
\label{eq:master-stage2-M-DE}
\mathcal M_{S,p}^{\mathrm{DE}}
=
\begin{cases}
0,
&p<\min\{m,N_S\},
\\[3mm]
\displaystyle
\frac{
\tau_{S,p}
\mathcal T_{T,p}^{(1),\mathrm{DE}}(\tau_{S,p})
}{1-N_S/m},
&N_S<\min\{m,p\},
\\[5mm]
\displaystyle
\frac{
\tau_{S,p}^2
\mathcal T_{T,p}^{(2),\mathrm{DE}}(\tau_{S,p})
}{
1-(p/m)\operatorname{df}_p^{(2)}(\tau_{S,p})
}
+
\frac{
\tau_{S,p}
\mathcal T_{T,p}^{(1),\mathrm{DE}}(\tau_{S,p})
(m/N_S)
}{1-m/N_S},
&m<\min\{p,N_S\}.
\end{cases}
\end{equation}
Likewise, \eqref{eq:master-stage2-cross-conditional-DE} and
\eqref{eq:master-stage2-BT-exact} give
\begin{equation}
\label{eq:master-stage2-cross-averaged-DE}
\mathcal C_{S,p}^{\mathrm{cross}}
\overset{\mathrm{DE}}{\simeq}
\mathcal C_{S,p}^{\mathrm{cross},\mathrm{DE}}
:=
-2\tau_{S,p}
V_{T,p}^{\mathrm{DE}}
\Delta_p(\tau_{S,p}).
\end{equation}

% ============================================================
\begin{theorem}[Stage-II risk and weak-to-strong gain]
\label{thm:master-stage2-risk}
% ============================================================
Under the proportional-regime assumptions of the Stage-I theorem and
the linear-random-feature deterministic-equivalent assumptions of
\cite[Sec.~IV]{atanasov2026scaling},
\begin{equation}
\label{eq:master-stage2-risk-DE}
\overline{\mathcal R}_{S,p}
\overset{\mathrm{DE}}{\simeq}
\overline{\mathcal R}_{S,p}^{\mathrm{DE}}
:=
\overline{\mathcal R}_{T,p}^{\mathrm{DE}}
+
\mathcal M_{S,p}^{\mathrm{DE}}
+
\mathcal C_{S,p}^{\mathrm{cross},\mathrm{DE}}.
\end{equation}
Consequently,
\begin{equation}
\label{eq:master-stage2-gain-DE}
\mathsf{W2SG}_p
\overset{\mathrm{DE}}{\simeq}
\mathsf{W2SG}_p^{\mathrm{DE}}
:=
\overline{\mathcal R}_{T,p}^{\mathrm{DE}}
-
\overline{\mathcal R}_{S,p}^{\mathrm{DE}}.
\end{equation}
\end{theorem}

\begin{proof}
Averaging \eqref{eq:master-stage2-risk-conditional} over Stage I gives
the exact identity
$
\overline{\mathcal R}_{S,p}
=
\overline{\mathcal R}_{T,p}
+
\mathcal M_{S,p}
+
\mathcal C_{S,p}^{\mathrm{cross}}.
$
Substitute the Stage-I risk deterministic equivalent,
\eqref{eq:master-stage2-M-DE}, and
\eqref{eq:master-stage2-cross-averaged-DE} to obtain
\eqref{eq:master-stage2-risk-DE}.  Finally, from
$
\mathsf{W2SG}_p
=
\overline{\mathcal R}_{T,p}
-
\overline{\mathcal R}_{S,p},
$
replacing the teacher and student risks by their deterministic
equivalents yields
\[
\mathsf{W2SG}_p
\overset{\mathrm{DE}}{\simeq}
\mathsf{W2SG}_p^{\mathrm{DE}}
=
\overline{\mathcal R}_{T,p}^{\mathrm{DE}}
-
\overline{\mathcal R}_{S,p}^{\mathrm{DE}},
\]
which is \eqref{eq:master-stage2-gain-DE}.
\end{proof}
\subsection{Positive-signal-fraction regime: assumptions and phase diagram}
\label{subsec:positive-signal-fraction}

\begin{assumption}[Proportional asymptotic scaling]
\label{assB:positive-fraction}
Let $d_p:=p-k$ and assume
\begin{equation}
\label{eqB:ass-positive-rates}
\frac{k}{p}\to\vartheta,\qquad
d_p\delta_p\to\eta,\qquad
\frac{N_T}{p}\to\rho_T,\qquad
\frac{m}{p}\to\rho_m,\qquad
\frac{N_S}{p}\to\rho_S,
\end{equation}
where $\vartheta,\rho_T\in(0,1)$, $\eta\in(0,\infty)$, and
$\rho_m,\rho_S\in(0,\infty)$. We consider the regular fixed-ratio regime
in which $\rho_T\neq\vartheta$, $\rho_m\neq\rho_S$,
$\rho_m,\rho_S\neq\vartheta$, and $\rho_m,\rho_S\neq1$. We exclude
these critical boundaries from the fixed-ratio analysis; the relevant
finite-dimensional and critical-scaling phenomena are discussed separately
in Section~\ref{sec:case-II-critical-boundaries}.
\end{assumption}

\begin{theorem}[Positive-signal-fraction phase diagram within the master DE]
\label{thm:positive-fixed-ratio-DE}
Under Assumption~\ref{assB:positive-fraction} and the conditions of the
Stage-I and Stage-II deterministic equivalents, let
$\mathsf{W2SG}_p^{\mathrm{DE}}$ denote the master deterministic equivalent
of the teacher-minus-student prediction risk. Define
\begin{equation}
\label{eq:positive-LambdaT-theorem}
\Lambda_T:=
\begin{cases}
\vartheta-\rho_T, & 0<\rho_T<\vartheta,\\[1mm]
\dfrac{(1-\vartheta)(\rho_T-\vartheta)}{1-\rho_T},
& \vartheta<\rho_T<1.
\end{cases}
\end{equation}

\noindent\textnormal{(a) Student bottleneck below the signal dimension.}
If $\min\{\rho_m,\rho_S\}<\vartheta$, then
$\mathsf{W2SG}_p^{\mathrm{DE}}<0$ for all sufficiently large $p$.

\medskip
\noindent\textnormal{(b) Sample bottleneck above the signal dimension.}
Suppose $\vartheta<\rho_m<\min\{1,\rho_S\}$. If the left-hand side of
\begin{equation}
\label{eq:positive-theorem-SB-sign}
(\rho_m-2\vartheta)\rho_S
-\rho_m(2\rho_m-3\vartheta+\Lambda_T)>0
\end{equation}
is nonzero, then $\mathsf{W2SG}_p^{\mathrm{DE}}>0$ for all sufficiently
large $p$ if and only if \eqref{eq:positive-theorem-SB-sign} holds.
Equivalently, $\mathsf{W2SG}_p^{\mathrm{DE}}>0$ if and only if
\begin{equation}
\label{eq:positive-theorem-SB-quadratic}
2\rho_m^2-(\rho_S+3\vartheta-\Lambda_T)\rho_m+2\vartheta\rho_S<0.
\end{equation}
Define
$D_{\mathrm{SB}}:=(\rho_S+3\vartheta-\Lambda_T)^2-16\vartheta\rho_S$.
If $D_{\mathrm{SB}}>0$, let
\[
\rho_{m,\pm}:=
\frac{\rho_S+3\vartheta-\Lambda_T\pm\sqrt{D_{\mathrm{SB}}}}{4}.
\]
Then the open positive phase on the sample-bottleneck branch is
$\rho_m\in(\rho_{m,-},\rho_{m,+})\cap
(\vartheta,\min\{1,\rho_S\})$. If $D_{\mathrm{SB}}\le0$, or if this
intersection is empty, there is no open positive phase on this branch.

\medskip
\noindent\textnormal{(c) Feature bottleneck above the signal dimension.}
Suppose $\vartheta<\rho_S<\min\{1,\rho_m\}$. If the left-hand side of
\begin{equation}
\label{eq:positive-theorem-FB-sign}
(\rho_m-2\rho_S)(\rho_S-\vartheta)-\rho_m\Lambda_T>0
\end{equation}
is nonzero, then $\mathsf{W2SG}_p^{\mathrm{DE}}>0$ for all sufficiently
large $p$ if and only if \eqref{eq:positive-theorem-FB-sign} holds.
Equivalently, if $\rho_S-\vartheta-\Lambda_T>0$, the open positive phase is
\begin{equation}
\label{eq:positive-theorem-FB-phase}
\rho_m>
\frac{2\rho_S(\rho_S-\vartheta)}
{\rho_S-\vartheta-\Lambda_T}.
\end{equation}
If $\rho_S-\vartheta-\Lambda_T\le0$, then
$\mathsf{W2SG}_p^{\mathrm{DE}}<0$ for all sufficiently large $p$
throughout the feature-bottleneck branch.

\medskip
\noindent\textnormal{(d) No student bottleneck.}
If $1<\min\{\rho_m,\rho_S\}$, then, under the almost-sure full-rank
conditions of the two-stage model,
\[
\widehat\beta_S=\widehat\beta_T,\qquad \mathsf{W2SG}_p=0
\]
almost surely for all sufficiently large $p$.
\end{theorem}

\begin{proof}
Under Assumption~\ref{assB:positive-fraction},
$\delta_p\sim\eta/((1-\vartheta)p)\to0$. By the Stage-II master
deterministic equivalent,
\begin{equation}
\label{eq:positive-master-gain}
\mathsf{W2SG}_p^{\mathrm{DE}}
=
2\tau_{S,p}V_{T,p}^{\mathrm{DE}}\Delta_p(\tau_{S,p})
-\mathcal M_{S,p}^{\mathrm{DE}}.
\end{equation}

\noindent\textbf{Stage-I asymptotics.}
The equation $\operatorname{df}_p(\tau_{T,p})=N_T/p$ is equivalent to
\[
\frac{N_T}{p}\tau_{T,p}^2
+\left[
\frac{N_T}{p}-\frac{k}{p}
+\delta_p\left(\frac{N_T}{p}+\frac{k}{p}-1\right)
\right]\tau_{T,p}
-\left(1-\frac{N_T}{p}\right)\delta_p=0.
\]
If $0<\rho_T<\vartheta$, its positive root satisfies
$\tau_{T,p}\to(\vartheta-\rho_T)/\rho_T$, whereas if
$\vartheta<\rho_T<1$, writing $\tau_{T,p}=\delta_p t_{T,p}$ gives
$\tau_{T,p}/\delta_p\to(1-\rho_T)/(\rho_T-\vartheta)$. Substitution into
the Stage-I deterministic equivalents yields
\[
\left(S_{T,p}^{\mathrm{DE}},
\frac{V_{T,p}^{\mathrm{DE}}}{\delta_p}\right)
\longrightarrow
\begin{cases}
\left(\dfrac{\rho_T}{\vartheta},
\dfrac{(1-\vartheta)\rho_T}{\vartheta(\vartheta-\rho_T)}\right),
&0<\rho_T<\vartheta,\\[3mm]
\left(1,\dfrac{1-\rho_T}{\rho_T-\vartheta}\right),
&\vartheta<\rho_T<1.
\end{cases}
\]
Hence, with $\Lambda_T$ as in \eqref{eq:positive-LambdaT-theorem},
\begin{equation}
\label{eq:positive-teacher-summary}
\frac{\delta_pS_{T,p}^{\mathrm{DE}}}{V_{T,p}^{\mathrm{DE}}}
\longrightarrow\frac{\Lambda_T}{1-\vartheta},\qquad
V_{T,p}^{\mathrm{DE}}\longrightarrow0,\qquad
\frac{\delta_p^2S_{T,p}^{\mathrm{DE}}}{V_{T,p}^{\mathrm{DE}}}
\longrightarrow0.
\end{equation}

\noindent\textbf{Bottleneck below the signal dimension.}
Write $s_T:=\lim_{p\to\infty}S_{T,p}^{\mathrm{DE}}>0$. First suppose
$\rho_m<\min\{\vartheta,\rho_S\}$. Then
$\tau_{S,p}\to(\vartheta-\rho_m)/\rho_m$ and
$\operatorname{df}_p^{(2)}(\tau_{S,p})\to\rho_m^2/\vartheta$.
The beneficial term in \eqref{eq:positive-master-gain} tends to zero,
while substitution into the Stage-II mimicry term gives
\[
\mathcal M_{S,p}^{\mathrm{DE}}
\longrightarrow
s_T\frac{\vartheta-\rho_m}{\vartheta}
\frac{\rho_S}{\rho_S-\rho_m},
\]
and therefore
\[
\mathsf{W2SG}_p^{\mathrm{DE}}
\longrightarrow
-s_T\frac{\vartheta-\rho_m}{\vartheta}
\frac{\rho_S}{\rho_S-\rho_m}<0.
\]
If instead $\rho_S<\min\{\vartheta,\rho_m\}$, then
$\tau_{S,p}\to(\vartheta-\rho_S)/\rho_S$, and the same substitution on
the feature-bottleneck branch gives
\[
\mathsf{W2SG}_p^{\mathrm{DE}}
\longrightarrow
-s_T\frac{\vartheta-\rho_S}{\vartheta}
\frac{\rho_m}{\rho_m-\rho_S}<0.
\]
This proves part~\textnormal{(a)}.

\noindent\textbf{Bottleneck above the signal dimension.}
For an active Stage-II bottleneck ratio $\rho\in(\vartheta,1)$, the
fixed-point equation gives
\[
\frac{\tau_{S,p}}{\delta_p}\longrightarrow
\frac{1-\rho}{\rho-\vartheta},
\qquad
\tau_{S,p}\Delta_p(\tau_{S,p})
\longrightarrow
A(\rho):=\frac{1-\rho}{1-\vartheta}.
\]
Using \eqref{eq:positive-teacher-summary},
\begin{equation}
\label{eq:positive-above-signal-targets}
\frac{\tau_{S,p}\mathcal T_{T,p}^{(1),\mathrm{DE}}(\tau_{S,p})}
{V_{T,p}^{\mathrm{DE}}}
\longrightarrow
A(\rho)\left(1+\frac{\Lambda_T}{\rho-\vartheta}\right),
\qquad
\frac{\tau_{S,p}^2\mathcal T_{T,p}^{(2),\mathrm{DE}}(\tau_{S,p})}
{V_{T,p}^{\mathrm{DE}}}
\longrightarrow A(\rho)^2.
\end{equation}

On the sample-bottleneck branch
$\vartheta<\rho_m<\min\{1,\rho_S\}$, take $\rho=\rho_m$. Moreover,
\[
\operatorname{df}_p^{(2)}(\tau_{S,p})
\longrightarrow
\vartheta+\frac{(\rho_m-\vartheta)^2}{1-\vartheta},
\qquad
1-\frac{p}{m}\operatorname{df}_p^{(2)}(\tau_{S,p})
\longrightarrow
\frac{(\rho_m-\vartheta)(1-\rho_m)}
{\rho_m(1-\vartheta)}.
\]
Together with $m/(N_S-m)\to\rho_m/(\rho_S-\rho_m)$, substitution into
\eqref{eq:positive-master-gain} gives
\begin{align}
\frac{\mathsf{W2SG}_p^{\mathrm{DE}}}
{V_{T,p}^{\mathrm{DE}}A(\rho_m)}
&\longrightarrow
2-\frac{\rho_m}{\rho_m-\vartheta}
-\frac{\rho_m}{\rho_S-\rho_m}
\left(1+\frac{\Lambda_T}{\rho_m-\vartheta}\right)
=
\frac{(\rho_m-2\vartheta)\rho_S
-\rho_m(2\rho_m-3\vartheta+\Lambda_T)}
{(\rho_m-\vartheta)(\rho_S-\rho_m)}.
\end{align}
Since $V_{T,p}^{\mathrm{DE}}>0$, $A(\rho_m)>0$, and
$(\rho_m-\vartheta)(\rho_S-\rho_m)>0$, the eventual sign of
$\mathsf{W2SG}_p^{\mathrm{DE}}$ is the sign of the numerator. This gives
\eqref{eq:positive-theorem-SB-sign}, equivalently
\eqref{eq:positive-theorem-SB-quadratic}; the discriminant and root
characterization in part~\textnormal{(b)} follow from the quadratic formula.

On the feature-bottleneck branch
$\vartheta<\rho_S<\min\{1,\rho_m\}$, take $\rho=\rho_S$ in
\eqref{eq:positive-above-signal-targets}. Since
$m/(m-N_S)\to\rho_m/(\rho_m-\rho_S)$, substitution into
\eqref{eq:positive-master-gain} yields
\begin{align}
\frac{\mathsf{W2SG}_p^{\mathrm{DE}}}
{V_{T,p}^{\mathrm{DE}}A(\rho_S)}
&\longrightarrow
2-\frac{\rho_m}{\rho_m-\rho_S}
\left(1+\frac{\Lambda_T}{\rho_S-\vartheta}\right)
=
\frac{(\rho_m-2\rho_S)(\rho_S-\vartheta)-\rho_m\Lambda_T}
{(\rho_S-\vartheta)(\rho_m-\rho_S)}.
\end{align}
Since $V_{T,p}^{\mathrm{DE}}>0$, $A(\rho_S)>0$, and
$(\rho_S-\vartheta)(\rho_m-\rho_S)>0$, the eventual sign of
$\mathsf{W2SG}_p^{\mathrm{DE}}$ is again the sign of the numerator. This
gives \eqref{eq:positive-theorem-FB-sign}. Rearranging yields
\eqref{eq:positive-theorem-FB-phase} when
$\rho_S-\vartheta-\Lambda_T>0$; if this quantity is nonpositive, the
limiting numerator is strictly negative throughout the feature-bottleneck
branch. This proves parts~\textnormal{(b)} and~\textnormal{(c)}.

\noindent\textbf{No student bottleneck.}
If $1<\min\{\rho_m,\rho_S\}$, then eventually $m>p$ and $N_S>p$. Under
the stated full-rank conditions, $\widetilde X$ has full column rank and
$F_S$ has full row rank. Hence $F_S$ is surjective, so there exists $v_0$
with $F_Sv_0=\widehat\beta_T$, and the minimum Stage-II residual is zero.
Any least-squares minimizer $\widehat v_S$ therefore satisfies
$\widetilde X(F_S\widehat v_S-\widehat\beta_T)=0$. Since
$\ker(\widetilde X)=\{0\}$, it follows that
$F_S\widehat v_S=\widehat\beta_T$. Hence
\[
\widehat\beta_S=\widehat\beta_T,\qquad
\mathsf{W2SG}_p=0
\]
almost surely for all sufficiently large $p$. This proves
part~\textnormal{(d)}.
\end{proof}

%%%%%%%%%%%%%%%%%%%%%%%%%%%%%%%%%%%%%%%%%%%%%%%%%%%%%%%%%%%%%%%%%%%%%%%%%%%%%%%%%%%%%%%%%%%%%%%%%%%%%%%%%%%%%%%%%%%%%%%%%%%%%%%%%%%%%%%%%%%%%%%%%%%%%%%%%%%%%%%%%%%%%%%%%%%%%%%%%%%%%%%%%%%%%%%%%%%%%%%%%%%%%%%%%%%%%%%%%%%%%%%%%%%%%%%%%%%%%%%%%%%%%%%%%%%%%%%%%%%%%%%%%%%%%%%%%%%%%%%%%%%%%%%%%%%%%%%%%%%%%%%%%%%%%%%%%%%%%%%%%%%%%%%%%%%%%%%%%%%%%%%%%%%%%%%%%%%%%%%%%%%%%%%%%%%%%%%%%%%%%%%%%%%%%%%%%%%%%%%%%%%%%%%%%%%%%%%%%%%%%%%%%%%%%%%%%%%%%%%%%%%%%%%%%%%%%%%%%%%%%%%%%%%%%%%%%%%%%%%%%%%%%%%%%%%%%%%%%%%%%%%%%%%%%%%%%%%%%%%%%%%%%%%%%%%%%%%%%%%%%%%%%%%%%%%%%%%%%%%%%%%%%%%%%%%%%%%%%%%%%%%%%%%%%%%%%%

\subsection{Critical Boundaries and Finite-Dimensional Effects}
\label{sec:case-II-critical-boundaries}

The regular fixed-ratio phase diagram excludes the interpolation
boundary, the signal-rank threshold, and the ambient no-bottleneck
boundary. We record here the two boundary phenomena most relevant to
the W2SG phase diagram.

\paragraph{Exact interpolation singularity.}
Let
$
r_p:=\min\{m,N_S\},
g_p:=|m-N_S|,
$
and assume that the Stage-I coefficient
$b=\widehat\beta_T$ is nonzero almost surely and has finite population
risk.

\begin{theorem}[Exact Stage-II interpolation-boundary classification]
\label{thm:case-II-exact-boundary-classification}
The following finite-dimensional statements hold.

\begin{enumerate}
\item If $r_p\ge p$, then both Stage-II maps have sufficient rank to
identify the teacher coefficient:
$\widehat\beta_S=\widehat\beta_T$
almost surely. Consequently,
$
\overline{\mathcal R}_{S,p}
=
\overline{\mathcal R}_{T,p}$,
and
$\mathsf{W2SG}_p=0$.

\item If $r_p<p$ and $g_p\le1$, then the relevant inverse-Wishart first moment diverges. Under the nondegeneracy conditions of the two-stage model,
$
\overline{\mathcal R}_{S,p}=+\infty$ and 
$\mathsf{W2SG}_p=-\infty$.

\item If $r_p<p$ and $g_p\ge2$, the Stage-II expected risk is
finite. Its excess contribution contains the inverse-Wishart
amplification factor
\[
\frac{1}{g_p-1}.
\]
\end{enumerate}
\end{theorem}

\begin{proof}
If $m\ge p$ and $N_S\ge p$, then
$\widetilde X$ has full column rank and the student feature map has
full row rank almost surely. The Stage-II normal equations therefore
identify the teacher coefficient uniquely, giving
$\widehat\beta_S=\widehat\beta_T$.

If $N_S<m$ and $N_S<p$, conditioning on the student feature map
reduces Stage II to an $N_S$-dimensional Gaussian least-squares
problem. Its variance contains the factor
\[
\frac{N_S}{m-N_S-1}.
\]
If $m<N_S$ and $m<p$, conditioning instead on the Stage-II sample
gives the dual factor
\[
\frac{1}{N_S-m-1}.
\]
These expectations are infinite when $|m-N_S|\le1$ and finite when
$|m-N_S|\ge2$, proving the result.
\end{proof}

In particular,
$
m=N_S<p$ gives
$
\overline{\mathcal R}_{S,p}=+\infty$,
whereas the exceptional point
$
m=N_S=p
$
belongs to the exact no-bottleneck regime and satisfies
\[
\widehat\beta_S=\widehat\beta_T,
\qquad
\mathsf{W2SG}_p=0.
\]
Thus, the internal interpolation singularity terminates when both
Stage-II resources reach ambient dimension.

\paragraph{The signal threshold has a $\sqrt p$ critical window.}
Assume
$
\vartheta_p:=\frac{k}{p}\to\vartheta\in(0,1),
(p-k)\delta_p\to\eta\in(0,\infty),$
and
$\varepsilon_p:=\sqrt{\delta_p}.$ Define
\[
A(z)
:=
\frac{\sqrt{z^2+4\vartheta(1-\vartheta)}-z}
     {2\vartheta},
\qquad
L(z)
:=
\sqrt{z^2+4\vartheta(1-\vartheta)}.
\]
Let $\rho_p$ denote a dimension-to-ambient ratio, such as
$m/p$, $N_S/p$, or $N_T/p$. If
\[
\frac{\rho_p-\vartheta_p}{\varepsilon_p}\to z
\]
and the corresponding effective resolution satisfies
$
\operatorname{df}_p(\tau_p)=\rho_p,
$
then
$
\dfrac{\tau_p}{\varepsilon_p}\to A(z).
$
Because
\[
p\varepsilon_p
\sim
\sqrt{\frac{\eta p}{1-\vartheta}},
\]
the transition around the signal dimension occurs in a dimension
window of order $\sqrt p$. We next consider the jointly critical regime in which both the teacher
width and the active Stage-II resource approach the signal dimension
within their respective $\sqrt p$ windows. Suppos
\[
\frac{N_T-k}{p\varepsilon_p}\to z_T,
\qquad
\frac{\min\{m,N_S\}-k}{p\varepsilon_p}\to z_S,
\]
and define
\[
A_S:=A(z_S),
\qquad
L_S:=L(z_S),
\qquad
v_T^{\mathrm{crit}}
:=
\frac{1-\vartheta}{L(z_T)}.
\]

Within the master deterministic equivalent, the two Stage-II
bottleneck branches then behave as follows.

If $m$ is the active bottleneck and
$N_S/p\to\rho_S>\vartheta$, then
\[
\mathsf{W2SG}_p^{\mathrm{DE}}
\longrightarrow
-\frac{\vartheta v_T^{\mathrm{crit}}}{L_S}
<0.
\]

If $N_S$ is the active bottleneck and
$m/p\to\rho_m>\vartheta$, then
\[
\frac{\mathsf{W2SG}_p^{\mathrm{DE}}}{\varepsilon_p}
\longrightarrow
2v_T^{\mathrm{crit}}
-
\frac{\rho_m}{\rho_m-\vartheta}
\left(A_S+v_T^{\mathrm{crit}}\right).
\]
Consequently, the leading critical feature-bottleneck gain is positive
exactly when
\[
(\rho_m-2\vartheta)v_T^{\mathrm{crit}}
>
\rho_m A_S.
\]

Thus joint signal criticality is asymmetric: the sample-bottleneck
branch remains unfavorable, while a nonempty positive W2SG region may survive on the feature-bottleneck branch. These statements concern the master deterministic equivalent; equality in the displayed sign condition requires a higher-order expansion.

\end{document}